\pdfoutput=1
\documentclass{article}
\usepackage[preprint,nonatbib]{neurips_2026}

\usepackage[round]{natbib}

\usepackage{amsmath}
\usepackage{amssymb}
\usepackage{amsfonts}
\usepackage{amsthm}
\usepackage{mathtools}
\usepackage{mathrsfs}
\usepackage{graphicx}
\graphicspath{{arxiv/}{./}}
\usepackage{booktabs}
\usepackage{array}
\usepackage{xcolor}
\usepackage{tikz}
\usepackage{tcolorbox}
\tcbuselibrary{breakable}
\usetikzlibrary{positioning,fit,calc,arrows.meta}
\definecolor{ink}{HTML}{0C0C0C}
\definecolor{textbody}{HTML}{3F5661}
\definecolor{rule}{HTML}{758D99}
\definecolor{latent}{HTML}{9A607F}
\definecolor{mechfill}{HTML}{F1F0E9}
\definecolor{probfill}{HTML}{EAF2F9}
\definecolor{probtitle}{HTML}{D6E6F2}
\definecolor{probframe}{HTML}{9FC0DA}
\newtcolorbox{problembox}{breakable,colback=probfill,colframe=probframe,colbacktitle=probtitle,coltitle=ink,fonttitle=\bfseries,boxrule=0.5pt,titlerule=0.3pt,arc=2pt,left=6pt,right=6pt,top=5pt,bottom=5pt,title=Problem statement}

\newtheorem{theorem}{Theorem}
\newtheorem{assumption}{Assumption}

\newtheorem{definition}{Definition}
\newtheorem{proposition}{Proposition}
\newtheorem{lemma}{Lemma}
\newtheorem{corollary}{Corollary}
\newtheorem{observation}{Observation}
\newenvironment{restated}[3]{\par\medskip\noindent\textbf{#1~\ref{#2}} (#3)\textbf{.}\itshape\ }{\upshape\par\medskip}

\newcommand{\Ucal}{\mathcal{U}}
\newcommand{\TV}{\operatorname{TV}}

\usepackage[colorlinks=true,linkcolor=linkblue,citecolor=linkblue,urlcolor=linkblue]{hyperref}
\definecolor{linkblue}{RGB}{0,51,153}
\hypersetup{
  pdftitle={Exploration and Experimentation with a Misspecified Simulator},
  pdfauthor={Harsh Parikh, Gabriel Levin-Konigsberg, Sohrab Andaz, Akhil Bagaria, Dominique Perrault-Joncas, Alexander Volfovsky},
  pdfsubject={Causal reinforcement learning; experimentation versus exploration with a misspecified simulator},
  pdfkeywords={simulator, misspecification, experimentation, exploration, adaptation, identification, transportability}
}

\makeatletter
\renewcommand{\eqref}[1]{Eq.~\textup{\tagform@{\ref{#1}}}}
\makeatother

\usepackage{etoolbox}
\makeatletter
\patchcmd{\@maketitle}{\LARGE\bf \@title}{\fontsize{14.5}{17}\selectfont\bfseries \@title}{}{\errmessage{title size patch failed}}
\makeatother

\begin{document}

\title{Exploration \& Experimentation with a Misspecified Simulator}

\author{Harsh Parikh \\ Amazon SCOT and Yale University \And
                Gabriel Levin-Konigsberg \\ Amazon SCOT \And
                Sohrab Andaz \\ Amazon SCOT \AND
                Akhil Bagaria \\ Amazon SCOT \And
                Dominique Perrault-Joncas \\ Amazon SCOT \And
                Alexander Volfovsky \\ Amazon SCOT and Duke University}

\maketitle

\begin{abstract}
Randomized experiments are the standard for measuring causal effects that guide decisions, but are costly. Simulators calibrated to historical data offer a cheaper alternative, but their policies can become suboptimal under drift, a change in dynamics after calibration. We study if, when, and how to collect data for recalibrating the simulator.
We consider a system with multiple units running in parallel and managed by a single operator. The real-world system is a partially observed process with a latent variable, while the simulator models it as a fully observed Markovian process, and thus is misspecified. Two data collection strategies exploit the current policy: (i) never recalibrating and (ii) recalibrating from on-policy data (passive adaptation). Two others collect off-policy data by trying actions outside the current policy: (iii) continuous exploration across all units and (iv) randomized experimentation on a sample of units over a predefined window. We investigate these strategies theoretically and in a vending-machine simulation. Experiments can improve on passive adaptation when drift is invisible in on-policy data or recalibration lags behind continuing drift. Experimentation is preferable to exploration when trying a worse action is expensive and the dynamics may not have changed.
\end{abstract}

\section{Introduction}
\label{sec:intro}

Randomized controlled trials guide decisions in many domains \citep{rosenberger2012randomization,kohavi2020trustworthy}, but they are expensive. A simulator of the system, calibrated to historical data \citep{tao2019digital,hafner2023mastering}, offers a cheaper alternative. An operator can use it to learn the \emph{simulator-optimal policy}, the policy that is optimal in the simulator, and transport this policy to the real system, that is, deploy it there \citep{zhao2020sim}. A retailer, for example, can fit a demand model to the historical sales of its stores, the \emph{units} of the decision problem, and compute each store's orders from it.

The policy is only as good as the simulator's description of the present system, and the present system can drift away from the dynamics of the historical data. When the distribution of an unobserved variable that drives outcomes, such as demand, shifts, the \emph{value} of the simulator-optimal policy, that is, its expected total reward, can fall. A common response is to keep recalibrating the simulator with data collected regularly from the system \citep{chebotar2019closing,kapteyn2021probabilistic}. These data, however, contain outcomes only for the actions taken by the policy in use, the \emph{deployed policy}. Recalibrating from them can therefore miss a drift in the outcomes of the actions that the policy does not assign. If a store keeps its old orders and sells out, the retailer learns that all the units it stocked sold, but not whether a larger order would have sold as well. The retailer could try a larger order, but trying other actions can lower the value as well. Because the system runs only once, with no second run under another policy, the operator cannot make up for this loss. The operator therefore faces a decision about data collection, and we ask: \emph{in the presence of a historically calibrated simulator, if, when, and how must the operator collect additional data to aid decision-making?}

This decision involves a tradeoff, because trying an action that the simulator does not recommend is the only way to learn about that action, and it is also where the operator loses value. Our key observation is that the strategies an operator can take differ in how much of this loss they admit. We therefore propose a framework that limits the loss from trying actions that the \emph{current policy}, the policy that the recalibrated simulator recommends at the time, does not recommend. The framework computes this loss under the recalibrated simulator as it stands at the time, so the limit is a budget on a simulated loss. With a zero budget, one can (i) follow the simulator-optimal policy without recalibrating the simulator, or (ii) adapt it passively. Passive adaptation recalibrates the simulator only from \emph{on-policy data}, the outcomes of the actions that the deployed policy assigns \citep{sutton2018reinforcement}. The other two strategies collect \emph{off-policy data}, the outcomes of actions that the current policy does not recommend. With an unlimited budget, one can (iii) explore, trying such actions on all units at periods chosen from the data \citep{auer2002finite,garivier2011upper}. With a finite budget, one can (iv) experiment, trying them at random on a sample of units over a window declared in advance \citep{hernan2020causal,chaloner1995bayesian}. The four strategies are thus the settings of one budget, at zero, finite, and unlimited, and the two off-policy strategies differ as well in what they do to the real system. Exploration changes the policy that all units follow. Its departures from the current policy, the periods in which a unit takes an action that this policy does not recommend, continue for as long as the system runs. An experiment sets aside a sample of units for a declared window, assigns their actions at random by a rule fixed before the window opens, and leaves the other units on the current policy.

Within this framework we answer the three parts of the question, beginning with whether additional data add value at all. When the system has not changed since the calibration, no strategy improves on the simulator-optimal policy (Theorem~\ref{thm:nochange}). Under a small drift, moreover, this policy stays close in value to the best fixed policy, the fixed policy with the highest value over the \emph{horizon}, the periods for which the system runs. We refer to a strategy's loss against the best fixed policy over the horizon as its \emph{regret}. Under a small drift, the regret of the simulator-optimal policy is at most a quantity that is proportional to the per-period drift and grows with the horizon (Corollary~\ref{cor:sopregret}, Appendix~\ref{app:moved}). Over a long enough horizon, the regret is larger when the drift occurs only at an action that the policy does not assign. The on-policy data then contain no outcome of that action, so neither following the policy nor recalibrating from its data detects the drift. The regret of either strategy then grows with the square of the horizon (Theorem~\ref{thm:hiddendrift}). More generally, a drift is \emph{hidden} when it makes an action best at a state where the policy does not assign it. A hidden drift is \emph{visible} when the on-policy data determine the new best action, and \emph{invisible} otherwise. Passive adaptation then suffices when a visible drift has ended, since the on-policy data lead it to the new best action. Additional data add value in two cases (Theorem~\ref{thm:testvalue}). When an invisible drift has ended, adaptation deploys a worse action. When a drift continues, adaptation switches to the new best action with a delay.

The second and third parts ask when and how to collect the data. Detecting a hidden drift requires a \emph{departure}, a period in which a unit takes an action that the current policy does not recommend. Where nothing has changed and the untried action is worse, every departure costs value, so detection cannot be free (Theorem~\ref{thm:detectprice}). The budget controls this price. In an unchanged system the budget bounds the loss of the budgeted procedure that we define in Section~\ref{sec:now} (Proposition~\ref{prop:purse}). Exploration has no such bound, and its losses can end in \emph{bankruptcy}, an unrecoverable state. Experimentation is therefore preferable to exploration when trying a worse action is expensive and the system may not have changed.

A benchmark simulation of a fleet of vending machines (Section~\ref{sec:sim}) is consistent with the theory. When a stocking mistake, an order that does not sell, is cheap, no strategy improves on the simulator-optimal policy by more than a few percent. When it is expensive and demand shifts at few machines, adaptation has the highest value where the shift is hidden, and its gain shrinks as more of the shift becomes visible. Once the shift reaches many machines, the experimenting strategies, of which the study compares two variants, have the highest value in most of the simulated scenarios, or \emph{worlds}, with one exploring strategy, Thompson sampling, ahead at the slowest shifts.

The rest of the paper proceeds as follows. Section~\ref{sec:related} reviews prior work on experimental design, bandits, sim-to-real transfer, and nonstationary learning, and states what this paper adds. Section~\ref{sec:setup} introduces the setting and the assumptions, and Section~\ref{sec:problem} states the data-collection problem. Section~\ref{sec:exploit} gives the identification results, Section~\ref{sec:now} defines the budgeted procedure and bounds its loss, and Section~\ref{sec:sim} reports the simulation study. Section~\ref{sec:conclusion} summarizes our findings, their practical implications, and the directions they open. Appendix~\ref{app:moved} gives the formal statements of the assumptions and of the results that we state in words, and the proofs are in Appendices~\ref{app:theory}, \ref{app:proofs}, and~\ref{app:fourtheorems}. Figure~\ref{fig:theory-map} in Appendix~\ref{app:theory} maps the results and the lemmas that enter each proof.

\section{Related Work}
\label{sec:related}

Prior work related to this paper addresses four problems: how to value and select experiments, how to act while learning, how to transfer a policy or evidence from a model to a real system, and how to respond when the system changes. In this section we review each of them and then state what this paper adds. Each of these lines of work makes an assumption that does not hold in our setting. Sim-to-real transfer assumes that a bound on the simulator's error is known \citep{niu2022trust,as2025spidr}, and Bayesian adaptive control assumes a correct model class \citep{feldbaum1960dual,duff2002optimal}. Exploration methods assume that the loss from trying actions on all units is acceptable \citep{kearns2002near,cheung2020nonstationary,garcia2015comprehensive}, and experimental design assumes a fixed study population \citep{chaloner1995bayesian}.

Our use of on-policy and off-policy data follows the distinction between the behavior policy, which generates the data, and the target policy, which is evaluated and improved \citep[Chapter~5]{sutton2018reinforcement}. Here the target policy is the policy that is optimal under the simulator as currently fitted, the plug-in policy defined in Section~\ref{sec:problem}. An experiment assigns treatments to a study sample by a known randomization \citep{hernan2020causal}, and an adaptive experiment lets that assignment depend on the data collected so far \citep{hadad2021confidence}. Adaptive experiments that target a choice of policy at their end \citep{kasy2021adaptive} share the objective of the pure-exploration bandit problems, in which only the action recommended at the end is evaluated \citep[Chapter~33]{lattimore2020bandit}. We value the data that a strategy collects by the effect of those data on the strategy's value, as in the value of information \citep{howard1966information}, rather than by an expected information gain \citep{lindley1956measure}. We measure each strategy against the simulator-optimal policy, a reference that need not be best after a shift, whereas a regret is measured against the best policy in a competitor class \citep[Section~1.1]{lattimore2020bandit}.

Classical design procedures rank candidate observations by expected information gain or by decision improvement \citep{lindley1956measure,chaloner1995bayesian,howard1966information}, and recent methods amortize this computation \citep{foster2021deep,rainforth2024modern,claxton2002bayesian}; ranking-and-selection and multi-fidelity methods allocate effort inside the simulation stack \citep{kim2006fully,frazier2008knowledge,peherstorfer2018survey}. All of these design methods treat the \emph{scoring model}, the model under which candidate observations are valued, as the model that generates the data, or they designate the highest-fidelity source as the ground truth. Our scoring model is the simulator, which, by contrast, may already be misspecified after calibration, and an updating strategy can lose value by running an experiment.

Researchers have estimated demand from stockout-censored sales and have coupled those estimates with lost-sales ordering \citep{conrad1976sales,anupindi1998estimation,huh2009adaptive}. An analyst using a structural demand model can identify responses at order quantities that the deployed policy never places. By contrast, the \emph{lock-in} that we study in Section~\ref{sec:exploit} is a property of the fitted model rather than of the data. For an action that no unit tries, a model fitted separately at each state-action pair keeps its estimate of the response at the simulator's value.

Researchers in dual control model an action as both regulating outcomes and producing information \citep{feldbaum1960dual,mesbah2018stochastic}, and Bayes-adaptive planning represents the trade-off in a belief-state Markov decision process (MDP) \citep{duff2002optimal}. Free evidence has nonnegative expected value for a coherent Bayesian \citep{blackwell1953equivalent,good1967principle}, though not under misspecified beliefs \citep{heidhues2018unrealistic}. Dual control and Bayes-adaptive planning combine probing and control in one optimal policy under an assumed prior model.

Bandit algorithms embed exploration in each assignment \citep{thompson1933likelihood,lai1985asymptotically}; adaptive designs target a terminal policy choice \citep{kasy2021adaptive}, test-and-roll designs optimize expected outcomes \citep{feit2019test}, and surveillance designs trade off testing volume \citep{warren2025bandits}. Batched, restless, and latent-MDP formalisms constrain when a policy may change or carry an evolving latent state \citep{perchet2016batched,whittle1988restless,kwon2021rl}. We instead treat a bounded experiment as a discrete costly overlay that displaces the deployed policy's assignments. The budget that we introduce in Section~\ref{sec:problem} resembles the cost constraint of a constrained Markov decision process \citep{altman1999constrained}, enforced period by period as in constrained exploration \citep{efroni2020exploration,ding2021provably}. The cost in our constraint is computed under the fitted model rather than observed, as in conservative bandits \citep{wu2016conservative,kazerouni2017conservative}.

Industrial A/B testing, always-valid inference, interference-aware marketplace designs, and clinical group-sequential and platform designs determine how to randomize, monitor, and infer once an analyst has chosen a test \citep{kohavi2009controlled,johari2022experimental,wager2021experimenting,jennison1999group,berry2015platform}. These methods condition on a chosen test, whereas in our setting a simulator-based procedure constructs the candidate real-system tests from a calibrated system model.

Sim-to-real methods hedge the reality gap by domain randomization or recalibration \citep{tobin2017domain,chebotar2019closing,zhao2020sim}, and more recent methods shape real-world data collection from simulation \citep{wagenmaker2024sim2real,memmel2024asid}; hybrid methods mix simulated data with online interaction and are benchmarked with deliberately biased simulators \citep{song2023hybrid,niu2022trust,matsushima2021deployment,fu2024benchmarks}. Sim-to-real and hybrid methods transfer a policy or an exploration strategy optimized for continuing real interaction. We instead transport a bounded \emph{design}, selecting at infrequent reviews a real test for a fleet, where the on-policy data may carry no information about the relevant change.

Model-based agents plan against learned dynamics \citep{sutton1991dyna}, the simulation lemma bounds how model error propagates to policy value \citep{kearns2002near,kakade2002approximately}, offline reinforcement learning controls extrapolation by pessimism \citep{levine2020offline,fujimoto2019off}, and off-policy evaluation extends to unobserved confounding \citep{uehara2022review,kallus2020confounding,namkoong2020off}. Our fixed-policy bounds use the same model-comparison structure as the simulation lemma \citep{li2006towards}, but they apply to the response and transition tables of a partially observed system, which depend on the policy that has been deployed.

Nonstationary learning tracks drift with restarts, windows, or variation budgets \citep{besbes2019optimal,cheung2020nonstationary}, and change-detection, concept-drift, and cost-aware-retraining methods monitor observation streams \citep{page1954continuous,gama2014survey,mahadevan2024cost}. Performative-prediction models represent reactive distributions \citep{perdomo2020performative}, robust dynamic programming hedges a fixed description of ignorance \citep{iyengar2005robust}, and calibration and digital-twin procedures continuously recalibrate a virtual asset \citep{kennedy2001calibration,kapteyn2021probabilistic}. Our results on changes at actions that the deployed policy does not take (Sections~\ref{sec:exploit} and~\ref{sec:now}) address the case that these methods leave out. There, without additional structure, a decision rule based on the collected data cannot attain both size and power, that is, control of the false-alarm rate and detection of a true change.

Researchers use transportability theory to map causal conclusions across domains \citep{bareinboim2016causal}, trial-generalizability methods reweight experimental samples toward target populations \citep{dahabreh2019generalizing,degtiar2023review}, and fusion methods combine experimental and observational data \citep{kallus2018removing,parikh2023double}, characterize underrepresented populations \citep{parikh2024who}, and regularize extrapolation beyond support \citep{parikh2025extrapolation}, on the identification foundations of \citet{robins1986new}. Our setup in Section~\ref{sec:setup} differs from this classical transport in three respects. First, the source, the simulator in our setting, is a model in which the latent covariate that modifies the response ($U$ in Section~\ref{sec:setup}) has been averaged out, whereas in classical transport the source is a sample in which that covariate is recorded. Second, in the target, the system over the decision horizon, the latent covariate is not recorded. Third, the distribution of the latent covariate changes over periods, whereas classical transport takes it to be fixed. Without structural extrapolation these methods do not identify outcomes for actions that the deployed policy never assigns, whereas our simulator-guided procedure randomizes those actions, and Section~\ref{sec:exploit} shows that identification then extends only to the state-action pairs that some unit reaches with positive probability.

Relative to these lines of work, we compare the four strategies introduced in Section~\ref{sec:intro} under one objective, the value in the real system. We bound the value loss from response and transition error (Proposition~\ref{prop:simgap}, Appendix~\ref{app:moved}). We then show that no strategy improves on the simulator-optimal policy while the simulator is correctly specified (Theorem~\ref{thm:nochange}) and that keeping a fitted model's value at an untried action incurs a quadratic loss when that action's response drifts (Theorem~\ref{thm:hiddendrift}). We also show that a change at state-action pairs that no unit reaches is not detectable from the collected data (Theorem~\ref{thm:undetectable}, Appendix~\ref{app:theory}). Finally, we characterize identification under prespecified and history-dependent collection (Theorem~\ref{thm:ident}), the identifiability of the latent distribution from the visited pairs (Theorem~\ref{thm:regimeid}), and the tradeoff between detection and the loss where nothing has changed (Theorem~\ref{thm:detectprice}).

\section{Setting}
\label{sec:setup}

The setting has three layers, and this section introduces them in the order in which a reader needs them. The first layer is the units with their observed and latent states, the second the policies and strategies that generate the record, and the third the simulator together with the working model that a strategy refits from the record. Figure~\ref{fig:setting} summarizes these objects. The section then states the nine assumptions on which our results rest and discusses what each one requires.

\subsection{Units, Record, and Potential Outcomes}
\label{sec:setup-units}
We observe units $i\in\mathcal V=\{1,\dots,n\}$ over a historical period ($t<0$) and a decision horizon $t\in\{0,\dots,T\}$. We call the data collected over the decision horizon the \emph{record} and the data from the historical period the \emph{historical record}. In period $t$, unit $i$ has an observed state $X_{i,t}\in\mathcal X$, an action $A_{i,t}\in\mathcal A$, an outcome $Y_{i,t}$, and an unobserved latent covariate $U_{i,t}\in\Ucal$. The outcome is bounded by a known constant, $|Y_{i,t}|\le B_Y$. The sets $\mathcal X$ and $\mathcal A$ are finite and $\Ucal$ is standard Borel. A bar marks a sequence over periods, so that $\bar a_{i,t}=(a_{i,0},\dots,a_{i,t})$ is an action path and $\bar A_{i,t}$ the random path of the actions taken. Under a path, $Y_{i,t}(\bar a_{i,t})$ is the potential outcome, and $X_{i,t}(\bar a_{i,t-1})$ and $U_{i,t}(\bar a_{i,t-1})$ are the potential observed state and latent covariate; the state and the covariate of period $t$ depend on the actions up to period $t-1$ only. $Y_{i,t}(\bar a_{i,t-1},a)$ is the potential outcome at the path $\bar a_{i,t-1}$ followed by $a$.

\subsection{Policies, Strategies, and Target Quantities}
\label{sec:setup-policies}
A \emph{policy} for period $t$ assigns to each unit $i$ and observed state $x$ a distribution $\pi_{i,t}(\cdot\mid x)$ over actions, $\pi_t=(\pi_{1,t},\dots,\pi_{n,t})$. A \emph{policy sequence} is $\bar\pi=(\pi_0,\dots,\pi_T)$, and $\bar\Pi$ is the set of all policy sequences. The \emph{predecision history} $\mathcal F_t$ consists of the historical record, the states $X_{1,t},\dots,X_{n,t}$, and the transitions completed through period $t-1$. A \emph{strategy} $\sigma$ chooses the period-$t$ policy $\pi^\sigma_t$ as a function of $\mathcal F_t$ and then draws each $A_{i,t}$ from $\pi^\sigma_{i,t}(\cdot\mid X_{i,t})$, independently across units. If this choice does not use the data, the strategy is \emph{prespecified} and is a single policy sequence $\bar\pi\in\bar\Pi$; otherwise it is \emph{history-dependent}. For one unit under a fixed strategy we suppress the unit index $i$ and the superscript $\sigma$. $\mathbb E$ and $\Pr$ then refer to that strategy's record. Under it, the \emph{latent distribution} $\nu_t(\cdot\mid x)$ is the distribution of $U_t$ given $X_t=x$; its change over periods is the drift that we study. A hat marks an estimate from the record and the superscript $\mathrm{sim}$ the simulator's quantity.

\begin{figure}[tp]
\centering
\begin{tikzpicture}[x=1cm,y=1cm,
  >={Stealth[length=2mm]},
  every node/.style={font=\normalsize},
  q/.style={inner sep=2pt},
  rec/.style={q, draw=ink, rounded corners=1.5pt, minimum width=10.5mm, minimum height=6.4mm},
  lat/.style={q, draw=latent, dashed, text=latent, rounded corners=1.5pt, minimum width=10.5mm, minimum height=6.4mm},
  mech/.style={q, draw=ink, fill=mechfill, rounded corners=1.5pt, inner xsep=5pt},
  arr/.style={->, semithick, draw=ink},
  larr/.style={->, thick, draw=latent},
  darr/.style={->, semithick, draw=rule, dashed},
  tag/.style={font=\small, text=textbody, inner sep=1pt},
  ltag/.style={font=\small, text=latent, inner sep=1pt},
  lbl/.style={font=\small, inner sep=1pt, fill=white}
]
\node[anchor=west, font=\normalsize\bfseries] at (-0.3,2.75) {(a)};
\node[tag, anchor=west] at (0.35,2.75) {one unit over one period $t$};

\node[rec] (A) at (3.4,1.1) {$A_{t}$};
\node[rec] (X) at (0.6,0) {$X_{t}$};
\node[lat] (U) at (0.6,-1.3) {$U_{t}$};
\node[mech, minimum height=27mm, minimum width=14mm] (M) at (6.3,-0.1) {$\theta,\ f$};
\node[tag, anchor=south] at (M.north) {mechanism};
\node[rec] (Y) at (6.3,-2.75) {$Y_{t}$};
\node[rec] (Xn) at (9.6,0) {$X_{t+1}$};
\node[lat] (Un) at (9.6,-1.3) {$U_{t+1}$};
\node[ltag, anchor=north] at (9.6,-1.72) {latent, never recorded};

\draw[arr, rounded corners=3pt] (X.north) -- (0.6,1.1) -- node[lbl, draw=rule, rounded corners=1.5pt, inner sep=2.5pt] (P) {$\pi_t(\cdot\mid x)$} (A.west);
\draw[arr] (A.east) -- (M.west |- A);
\draw[arr] (X.east) -- (M.west |- X);
\draw[larr] (U.east) -- (M.west |- U);
\draw[arr] (M.east |- Xn) -- node[lbl, above=0pt] {$f$} (Xn.west);
\draw[larr] (M.east |- Un) -- (Un.west);
\draw[arr] (M.south) -- node[lbl, right=0pt] {$\theta$} (Y.north);
\draw[rule, semithick] (-0.3,-3.45) -- (11.4,-3.45);
\node[tag, anchor=north west] at (-0.3,-3.55) {record: $X_s, A_s, Y_s$ of every unit, $s\le t$};
\fill[rule] (11.1,-3.45) circle (1.3pt);
\draw[darr, rounded corners=3pt] (11.1,-3.45) -- (11.1,1.95) -- node[lbl, above=0pt, pos=0.5] {strategy $\sigma$: from the record to the policy $\pi_t$} (P.north |- 0,1.95) -- (P.north);

\node[anchor=west, font=\normalsize\bfseries] at (-0.3,-4.75) {(b)};
\node[tag, anchor=west] at (0.35,-4.75) {source and target};
\node[tag, anchor=south] at (2.6,-5.45) {source: historical record, $t<0$};
\node[tag, anchor=south] at (8.8,-5.45) {target: decision horizon, $0\le t\le T$};
\node[mech, minimum width=96mm, minimum height=6.4mm] (Mb) at (5.7,-6.05) {mechanism $\theta,\ f$, common to all units and periods (A2)};
\node[lat, minimum width=30mm] (nus) at (2.6,-7.55) {$U\sim\nu^{\mathrm{sim}}(\cdot\mid x)$};
\node[lat, minimum width=30mm] (nut) at (8.8,-7.55) {$U\sim\nu_t(\cdot\mid x)$};
\node[ltag, anchor=west] at (nus.east) {\ fixed (A3)};
\node[ltag, anchor=west] at (nut.east) {\ drifts (A5)};
\draw[arr] (Mb.south -| nus) -- (nus.north);
\draw[arr] (Mb.south -| nut) -- (nut.north);
\node[rec, minimum width=30mm] (sim) at (2.6,-9.35) {$\theta^{\mathrm{sim}},\ f^{\mathrm{sim}}$};
\node[tag, anchor=west] at (sim.east) {\ the simulator's tables};
\node[rec, minimum width=30mm] (tgt) at (8.8,-9.35) {$m_t,\ g_t$};
\node[tag, anchor=west] at (tgt.east) {\ target quantities, period $t$};
\draw[larr] (nus.south) -- node[lbl, right=1pt] {average} (sim.north);
\draw[larr] (nut.south) -- node[lbl, right=1pt] {average} (tgt.north);
\end{tikzpicture}
\caption{The setting. (a) One unit over one period $t$. Solid boxes are recorded quantities and dashed purple boxes the latent covariate, which no record contains. The policy $\pi_t$, framed in gray, draws the action $A_t$ from the observed state $X_t$ alone, and a strategy $\sigma$ chooses that policy (dashed gray) from the predecision history, the record of the earlier periods together with the current states. The observed state, the latent covariate, and the action enter one mechanism $(\theta,f)$, common to all units and periods (A2): the outcome $Y_t$ has mean $\theta(x,a,u)$ and the next full state $(X_{t+1},U_{t+1})$ is drawn from $f(\cdot\mid x,a,u)$. The recorded quantities of every unit form the record. (b) The source and the target share the mechanism and differ in the distribution of the latent covariate. In the historical record, the source, it follows $\nu^{\mathrm{sim}}(\cdot\mid x)$, which does not change (A3); over the decision horizon, the target, it follows $\nu_t(\cdot\mid x)$, which drifts by at most $\lambda$ per period (A5). Averaging $\theta$ and the observed-state marginal of $f$ over $\nu^{\mathrm{sim}}$ gives the simulator's tables $\theta^{\mathrm{sim}}$ and $f^{\mathrm{sim}}$, and averaging them over $\nu_t$ gives the target quantities $m_t$ and $g_t$ of period $t$.}
\label{fig:setting}
\end{figure}

In the language of transportability, the data-generating process over the horizon is the \emph{target} and the simulator the \emph{source}. We call a state-action pair $(x,a)$ a \emph{cell}. The \emph{target quantities} of period $t$ at a cell are the mean response and the transition probabilities,
\begin{align}
    m_t(x,a) &:= \mathbb{E}\big[Y_t(\bar{A}_{t-1},a) \mid X_t = x\big], \label{eq:mdef}\\
    g_t(x' \mid x,a) &:= \Pr\big[X_{t+1}(\bar{A}_{t-1},a) = x' \mid X_t = x\big], \label{eq:gdef}
\end{align}
under the actions taken so far followed by $a$.

\subsection{The Simulator and the Working Model}
\label{sec:setup-sim}
\label{sec:simulator}
A \emph{simulator} consists of two time-invariant tables and a noise model: a response table $\theta^{\mathrm{sim}}(x,a)$, an observed-state transition kernel $f^{\mathrm{sim}}(\cdot\mid x,a)$, and a known mean-zero noise distribution $Q(\cdot\mid x,a)$. Because its tables are time-invariant, it cannot represent a drifting system. A strategy computes its decisions from a \emph{working model}, a response table $\hat m_t(x,a)$ and a transition table $\hat g_t(\cdot\mid x,a)$. We take them to be the posterior means, given $\mathcal F_t$, under a prior $\Xi_0$ over candidate tables $(m^{\mathrm{wk}},g^{\mathrm{wk}})$, the superscript marking a candidate working-model table, centred at the simulator's tables. Recalibrating the simulator means updating this working model from the record; the simulator's own tables stay fixed as the centre of the prior.

\subsection{Assumptions}
\label{sec:assumptions}
We now state the nine assumptions on which the results below rest and discuss them afterwards. Appendices~\ref{app:moved} and~\ref{app:theory} restate them formally, together with the measurability conditions that the proofs need.

\begin{assumption}[Consistency and no interference]\label{asm:consistency}
Each unit's outcome, observed state, and latent covariate are its potential ones under its own action path. The actions of the other units do not affect them.
\end{assumption}

\begin{assumption}[Markov full state and common mechanism]\label{asm:markov}
We call the pair $(x,u)$ of observed state and latent covariate the \emph{full state}. Given a unit's past, its full state, and its action $a$, its next full state is drawn from a transition kernel $f(\cdot\mid x,a,u)$, with observed-state marginal $f_X$, and its outcome has mean $\theta(x,a,u)$. The mechanism $(\theta,f)$ is common to all units and periods.
\end{assumption}

\begin{assumption}[Exact calibration]\label{asm:calib}
Over the historical record the conditional distribution of the latent covariate given the observed state $x$ is a fixed $\nu^{\mathrm{sim}}(\cdot\mid x)$. The simulator's tables are the averages of $\theta$ and $f_X$ over $\nu^{\mathrm{sim}}(\cdot\mid x)$.
\end{assumption}

\begin{assumption}[Smoothness]\label{asm:smooth}
For fixed metrics on states, actions, and latent covariates, the response $\theta$ and the marginal kernel $f_X$ are Lipschitz in the full state and the action. The mean response $m_t$ is Lipschitz across cells, with constant $L_m$.
\end{assumption}

\begin{assumption}[Bounded drift]\label{asm:drift}
In the Wasserstein-$1$ distance on the latent covariates, the latent distribution $\nu_t(\cdot\mid x)$ changes by at most $\lambda$ from one period to the next. The distribution $\nu_0(\cdot\mid x)$ at the start of the horizon lies within $\lambda$ of the calibration distribution $\nu^{\mathrm{sim}}(\cdot\mid x)$.
\end{assumption}

\begin{assumption}[Cell-wise prior]\label{asm:cellwise}
Under the prior $\Xi_0$ the entries of the candidate tables $(m^{\mathrm{wk}},g^{\mathrm{wk}})$ are mutually independent across cells.
\end{assumption}

\begin{assumption}[Working sampling model]\label{asm:likelihood}
Under the working model, each recorded transition $(x,a,y,x')$ is an independent draw with outcome mean $m^{\mathrm{wk}}(x,a)$, noise from $Q(\cdot\mid x,a)$, and next state from $g^{\mathrm{wk}}(\cdot\mid x,a)$.
\end{assumption}

\begin{assumption}[Correct specification]\label{asm:correctspec}
Given the predecision history and the period's actions, the units' outcomes and next states are independent across units and are drawn from the simulator at their cells.
\end{assumption}
We use (\ref{asm:correctspec}) only in Theorem~\ref{thm:nochange}, Proposition~\ref{prop:purse}, and Theorem~\ref{thm:detectprice}.

\begin{assumption}[Cross-unit independence]\label{asm:crossunit}
The potential outcomes, states, and latent covariates of distinct units are mutually independent and independent of the randomization with which a strategy draws its actions, and (\ref{asm:markov}) holds along every potential action path of each unit.
\end{assumption}
We use (\ref{asm:crossunit}) only in Theorem~\ref{thm:ident}(ii)--(iii) and Appendices~\ref{app:theory} and~\ref{app:fourtheorems}.

We now discuss each assumption in turn, that is, what it requires of the system, when it fails, and what it affords the analysis.

Assumption~\ref{asm:consistency} is the consistency and no-interference condition of the potential-outcome framework \citep{rubin1980randomization}. It states that the outcome and the states we record at a unit are its potential ones under the actions it actually received, and that the actions assigned to other units do not change them. It fails when units interact through the outcome, for example when two stores serve the same customers, so that a stock-out at one raises the sales of the other. It is not testable from the record. Under it, each quantity recorded at a unit is a function of that unit's own potential process, which is the object that the identification results in Section~\ref{sec:exploit} concern.

Assumption~\ref{asm:markov} gives the system the Markov structure of a sequential treatment problem with a latent state \citep{murphy2003optimal}. Given the full state and the action, the next full state and the mean outcome do not depend on the earlier past, and one mechanism $(\theta,f)$ applies to all units in all periods. It fails when the mechanism itself changes over time or differs between units beyond what the latent covariate carries, for example when a new product changes how sales respond to stock. Under it, all heterogeneity between units and all change over time pass through the latent covariate, so that a drift is a drift in its distribution $\nu_t$ alone, and observations at a few units are informative about the rest.

Assumption~\ref{asm:calib} states that the historical record was generated under a fixed latent distribution $\nu^{\mathrm{sim}}$ and that the simulator's tables are exactly the mechanism averaged over it. It fails when the historical period itself contained a drift or when the calibration carries estimation error. Under it, the simulator is exact for the historical period and differs from the target over the horizon only through the latent distribution, so that the misspecification we study is the shift in $\nu_t$ rather than an error in $\theta$ or $f$. Proposition~\ref{prop:simgap} in Appendix~\ref{app:moved} bounds the value loss that errors in the tables themselves cause, and that bound is the term through which an inexact calibration would enter.

Assumptions~\ref{asm:smooth} and~\ref{asm:drift} are quantitative. Under (\ref{asm:smooth}) the mean response at a cell differs from that at a nearby cell by at most $L_m$ times their distance, so that an observation at one cell is informative about its neighbours; it fails when the response jumps between adjacent actions. Under (\ref{asm:drift}) the latent distribution moves in one period by at most $\lambda$, and the horizon starts within $\lambda$ of the calibration; it fails under a sudden shock larger than $\lambda$. Both constants enter the error terms of the results below, $L_m$ through the smoothness terms and $\lambda$ through the lag term of Theorem~\ref{thm:testvalue} and the staleness term of Lemma~\ref{lem:window} in Appendix~\ref{app:fourtheorems}.

Assumptions~\ref{asm:cellwise} and~\ref{asm:likelihood} are different in kind. They constrain the working model, that is, the prior $\Xi_0$ and the likelihood by which the record updates it, and they place no condition on the mechanism that generates the outcomes. Under them each cell is updated from its own observations alone, and we take $\Xi_0$ and $Q$ to be Gaussian, so that the working mean at a cell is a conjugate posterior mean. Because each cell is updated only from its own observations, the working model keeps the simulator's response at a cell that no unit visits, even after the true response there has changed (Proposition~\ref{prop:censoring}, Appendix~\ref{app:moved}). We refer to this property as \emph{lock-in}. Theorem~\ref{thm:hiddendrift} shows that the regret it can cause grows with the square of the horizon, and Theorem~\ref{thm:testvalue} gives the conditions under which a designed test, which covers cells that no unit visits, adds value.

Assumption~\ref{asm:correctspec} is the boundary case in which the system has not changed since the calibration. It enters as a condition of specific results rather than as a standing assumption: it is the case in which the simulator-optimal policy cannot be improved on (Theorem~\ref{thm:nochange}) and against which the price of a departure is measured (Proposition~\ref{prop:purse} and Theorem~\ref{thm:detectprice}).

Assumption~\ref{asm:crossunit} strengthens the no-interference clause of (\ref{asm:consistency}), under which the actions of other units do not affect a unit, to the independence of the units' processes from one another and from the randomization of the strategy. We need it only when the assignment at one unit depends on the records of the others, as it does under a history-dependent strategy, because the other units' outcomes could otherwise carry information about a unit's latent covariate.

\label{sec:source-target}
The above assumptions imply the following identities between the target quantities, the mechanism, and the simulator's tables:
\begin{align}
m_t(x,a) &= \int\theta(x,a,u)\,\nu_t(du\mid x) &&\text{(by \ref{asm:consistency}, \ref{asm:markov})} \label{eq:key-m}\\
g_t(x'\mid x,a) &= \int f_X(x'\mid x,a,u)\,\nu_t(du\mid x) &&\text{(by \ref{asm:consistency}, \ref{asm:markov})} \label{eq:key-g}\\
\theta^{\mathrm{sim}}(x,a) &= \int\theta(x,a,u)\,\nu^{\mathrm{sim}}(du\mid x) &&\text{(by \ref{asm:calib})} \label{eq:key-sim}\\
f^{\mathrm{sim}}(x'\mid x,a) &= \int f_X(x'\mid x,a,u)\,\nu^{\mathrm{sim}}(du\mid x). &&\text{(by \ref{asm:calib})} \label{eq:key-fsim}
\end{align}
Eqs.~(\ref{eq:key-m}) and~(\ref{eq:key-g}) say that the target quantities of a period are the mechanism averaged over the latent distribution of that period, and Eqs.~(\ref{eq:key-sim}) and~(\ref{eq:key-fsim}) that the simulator's tables are the same averages over the calibration distribution. The source and the target thus share the mechanism $(\theta,f)$ and differ only in the distribution of the latent covariate, unobserved in the target. The simulator is therefore \emph{misspecified} for the target whenever $\nu_t$ differs from $\nu^{\mathrm{sim}}$, and $\nu_t$ must be inferred from the outcomes collected under some policy.

\section{The Problem}
\label{sec:problem}

We now state the data-collection problem. We first define the value by which we compare strategies, reported as a margin over the simulator-optimal policy, and state the problem, which is to choose the strategy of highest value. We then define the simulated price of a departure from the plug-in policy and a budget on that price. The budget is the device through which we operationalize the choice among strategies, and the four strategies are its settings. The rest of the section supplies what the strategies share, namely the designs they may run, the three classes into which they fall, the perturbation that an exploring strategy applies before the selector, and the transport update that pools observations across units.

\subsection{Value and Margin}
\label{sec:problem-value}
The \emph{value} of a strategy $\sigma$ is its expected total reward over the horizon, $V_T(\sigma):=\mathbb E_\sigma[\sum_{t=0}^{T}\sum_{i\in\mathcal V} Y_{i,t}]$, under the record collected under $\sigma$. For a prespecified strategy, which is one policy sequence $\bar\pi$, we write $V_T(\bar\pi)$. A gain or a loss is a difference in value from a named reference. A \emph{regret} is a loss relative to the best fixed policy. For a pair $(\tilde m,\tilde g)$ of a response table and a transition table, $V_{t:T}(\bar\pi;\tilde m,\tilde g)$ is the expected total reward of $\bar\pi$ over periods $t,\dots,T$ when outcomes and transitions follow these tables. The noise model is $Q$, and the periods start from the states of period $t$. Applied at a period $t$, the \emph{selector} $S$ maps tables $(\tilde m,\tilde g)$ to a policy sequence $S(\tilde m,\tilde g)\in\bar\Pi$ that maximizes $V_{t:T}(\cdot\,;\tilde m,\tilde g)$, with ties broken deterministically. A \emph{review} is a scheduled period $t_0<t_1<\cdots$ at which the working model is refit from the record, and the periods from one review to the next form an \emph{epoch}, indexed by $k$. The \emph{plug-in policy sequence} of review $k$ is $\hat{\bar\pi}:=S(\hat m_{t_k},\hat g_{t_k})$. Its period-$t$ policy $\hat\pi_t$, for $t$ in epoch $k$, is the \emph{plug-in policy}. The \emph{simulator-optimal policy sequence} is $\bar\pi^{\mathrm{sim}}:=S(\theta^{\mathrm{sim}},f^{\mathrm{sim}})$ at period $0$. The strategy $\sigma^{\mathrm{sop}}$ (SOP, for simulator-optimal policy) deploys $\bar\pi^{\mathrm{sim}}$ in each period and does not refit. For a review cadence $\kappa\ge1$, the strategy $\sigma^{\mathrm{asop}}_\kappa$ (A-SOP, for adaptive SOP) reviews at $t\in\{0,\kappa,2\kappa,\dots\}$ and deploys $S(\hat m_t,\hat g_t)$ until the next review (Definition~\ref{def:sop}, Appendix~\ref{app:moved}). SOP's value is positive in all the worlds we study. We measure a strategy by its \emph{margin} over SOP, its gain in value in percent of SOP's value,
\begin{equation}
\operatorname{margin}(\sigma):=100\,\frac{V_T(\sigma)-V_T(\sigma^{\mathrm{sop}})}{V_T(\sigma^{\mathrm{sop}})}.
\label{eq:margin}
\end{equation}

\subsection{Departures, Their Simulated Price, and the Budget}
\label{sec:problem-controls}
\label{sec:problem-budget}
The data-collection decision is the choice of which units receive which action distributions in which periods. We state the problem as follows.
\begin{problembox}
Given a calibrated simulator and a horizon over which the distribution of the latent variable may have shifted, choose a strategy $\sigma$, that is, which units receive which action distributions in which periods, to maximize the expected total reward in the real system,
\begin{equation*}
\max_{\sigma}\; V_T(\sigma).
\end{equation*}
The latent variable is not recorded, so a strategy learns of the shift only through the outcomes and transitions at the cells it visits. Whether, when, and at which units and actions the strategy should depart from the plug-in policy is the decision we study.
\end{problembox}
We call a unit-period at which the strategy's action distribution differs from the plug-in policy a \emph{departure}. Fix a review $k$ at period $t_k$, and write $(\hat m,\hat g):=(\hat m_{t_k},\hat g_{t_k})$ for its working model and $\hat{\bar\pi}$ for its plug-in policy sequence. We call the model in which outcomes and transitions follow $(\hat m,\hat g)$ from the states of period $t_k$ the \emph{plug-in model} at the review. Let $\bar\pi^{(k)}$ equal $\hat{\bar\pi}$ except at the departures of epoch $k$, the epoch that begins at this review. The \emph{simulated price} of these departures is $\operatorname{price}_k(\sigma):=V_{t_k:T}(\hat{\bar\pi};\hat m,\hat g)-V_{t_k:T}(\bar\pi^{(k)};\hat m,\hat g)$, the value lost under the plug-in model by the departures of epoch $k$, both terms being evaluated with the working model of review $k$. In words, we evaluate the plug-in policy sequence under the working model of the review, once as it is and once with the departures of epoch $k$ in place of its assignments, and the price is the difference in simulated value. The price is a simulated quantity, and Proposition~\ref{prop:purse} bounds the real loss by it up to the error of the working model. We operationalize the choice among strategies through a budget on this price. For a \emph{budget} $C\ge0$ we admit only the strategies whose departures at each review have simulated price at most $C$, so that the problem becomes
\begin{equation}
\max_{\sigma}\; V_T(\sigma)\quad\text{s.t.}\quad \operatorname{price}_k(\sigma)\le C \ \text{ for all }k.
\label{eq:problem}
\end{equation}
The four strategies are settings of the budget: following the simulator-optimal policy is $C=0$ with no reviews, adaptation is $C=0$ with reviews, experimentation is $0<C<\infty$, and exploration is $C=\infty$. The budget thus places the four strategies on one axis, so that we compare them under the one objective stated above. The quantity that it limits is simulated, since we price each departure under the working model before it runs, and no value is spent when we admit a candidate. The expense arises when the departures run in the real system. It is the real loss, which the operator bears and cannot compute in advance, and Proposition~\ref{prop:purse} bounds its expectation by the budget plus a term linear in the working model's errors at the covered cells. The simulator is in this sense a wind tunnel for departures, in which we compute the price of a departure before the operator bears the real loss of running it. When the working model is correct, that is, when its tables are the true tables, the extra term is zero and the expected real loss is at most the budget. The budget is thus the loss that the operator is prepared to absorb where nothing has changed, exactly when the working model is correct and up to its error otherwise.

\subsection{Designs, Classes of Strategies, and the Transport Update}
\label{sec:problem-designs}
A \emph{design} $d=(\mathcal C,\mathcal T_d,\mathcal A_d)$ consists of a cohort $\mathcal C\subseteq\mathcal V$ of units, a window $\mathcal T_d$ of periods, and a nonempty set $\mathcal A_d\subseteq\mathcal A$ of tested actions. Its \emph{slice} is the set of unit-periods $(i,t)$ with $i\in\mathcal C$ and $t\in\mathcal T_d$. On the slice the units receive the uniform distribution on $\mathcal A_d$, and all other unit-periods follow the plug-in policy. An implemented design is an \emph{experiment}, and a single-period design is a \emph{test}. The \emph{current policy} is the plug-in policy $\hat\pi_t$. The \emph{deployed policy} $\pi_t$ is the policy that a strategy assigns. The two coincide when the strategy exploits and differ where it departs. We group the four strategies into three classes \citep[Section~8.3]{sutton2018reinforcement}. An \emph{exploiting strategy} deploys $\pi_t=\hat\pi_t$ in each period, and an \emph{adaptation}, or passive adaptation, is an exploiting strategy that refits its working model at each review, only from on-policy data. SOP is an exploiting strategy, and A-SOP is an adaptation, as is S-SOP, a structural variant that we define in Section~\ref{sec:ident-structure}. An \emph{exploring strategy} may depart from the plug-in policy at any unit in any period, and it fixes neither the units nor the periods in advance. An \emph{experimenting strategy} departs only on the slices of designs that it declares at a review, before the first period of each design's window, and keeps the simulated price of these designs within the budget $C$.

In an exploring strategy we perturb the working mean of each cell before applying the selector. The perturbation is either a fixed multiple of a scale, added upward, or a random draw with that scale as its standard deviation. We take the scale to be the \emph{tunnel width} $s_t(x,a)=s_0\,n_0/\{n_0+N_t(x,a)\}$, a band around the working mean that narrows with the data. Here $s_0>0$ is common to all cells, $n_0$ is the \emph{prior weight}, the number of observations that the simulator's table counts for, and $N_t(x,a)$ is the number of visits to the cell before period $t$. Whatever its class, each strategy also shares observations across units through a \emph{transport update}. We partition the units into \emph{zones}, each sharing one working model and each with a set of neighbouring zones. $R_t^{j}(x,a)$ and $N_t^{j}(x,a)$ are the sum of the outcomes and the number of visits at the cell among the units of zone $j$ before period $t$. We refit the working model of zone $j$ from the pooled sums
\begin{equation}
\begin{aligned}
\tilde R_t^{j}(x,a)&=R_t^{j}(x,a)+\varkappa\sum_{j'\sim j}R_t^{j'}(x,a),\\
\tilde N_t^{j}(x,a)&=N_t^{j}(x,a)+\varkappa\sum_{j'\sim j}N_t^{j'}(x,a),
\end{aligned}
\label{eq:transport}
\end{equation}
where $j'\sim j$ ranges over the neighbours of $j$ and $\varkappa\in[0,1]$ is the weight at which a neighbour's observations enter. The structural fit introduced in Section~\ref{sec:ident-structure} pools at the same weight. Lastly, a strategy \emph{covers} a cell $(x,a)$ at unit $i$ and period $t$ if $\Pr[X_{i,t}=x,A_{i,t}=a]>0$ \citep{hernan2020causal}.

In the real system the two classes that collect off-policy data, exploring and experimenting strategies, differ in four respects. First, they differ in what they change, in that an exploring strategy changes the deployed policy itself, whereas an experimenting strategy changes the action assignment of a cohort of units over a window and leaves the deployed policy in place elsewhere. In each period the exploring strategy may assign to any unit an action that the plug-in policy does not recommend, it chooses these departures from the record as the record accumulates, and it has no stopping rule. An experimenting strategy leaves the deployed policy in place at the units outside the cohort and departs only on the slice of a design declared at a review. Each departure therefore has a start, an end, and a set of units, all fixed before it begins, and we do not consider stopping a design before its window ends. A review at which no design is declared leaves the policy unchanged, so whether and when to depart are decisions taken at reviews. An exploring strategy instead departs wherever the perturbed working model, or a random draw, assigns an action other than the plug-in one, with no declared design, no fixed window, and no budget. An experiment thus chooses both which units take part and what they receive, whereas exploration chooses only what each unit receives, since all units take part.

Second, they differ in how the departures are assigned, and hence in what the collected data support. On a slice the tested actions are drawn from a distribution fixed before the window opens, so the assignment probability is constant on the cohort and uncorrelated with the response. The average outcome of the cohort's units at a tested action then estimates their mean response at that action without reweighting, as we show in Section~\ref{sec:ident-cells}. An exploring strategy assigns an action with a probability that depends on the record, so the raw average is biased unless this probability is uncorrelated with the response. The correction then requires assignment probabilities that are positive at each history, as in adaptive experiments \citep{hadad2021confidence}. Deterministic assignment given the record does not meet this requirement.

Third, they differ in the loss accepted and in who bears it. An exploring strategy has no bound on the loss that its departures incur, since its budget is unbounded, and this loss can fall on any unit, since all units follow the changed policy. Where the simulator is correctly specified and a change could occur only at cells that the simulator-optimal policy does not visit, the loss is at least a per-departure penalty times the expected number of departures into those cells, as we show in Section~\ref{sec:proc-detect}. An experimenting strategy keeps the simulated price of its designs within the budget, so the operator knows before a window opens the most that the designs can cost under the plug-in model. Its departures fall on the cohort's units, and this cost is the loss of value under the plug-in model from their tested assignments. The real loss exceeds this amount by at most the number of tested units times a constant that is linear in the working model's errors and is zero when the working model is correct, as we show in Section~\ref{sec:proc-loss}.

Fourth, they differ in how the information reaches the units that do not depart. An experiment informs the units outside the cohort only through the structural fit within the cohort's zone and through the transport update \eqref{eq:transport} in the neighbouring zones. Its gain to the units of its own zone depends on the structure that they share with the cohort, as we show in Section~\ref{sec:testvalue}. Under an exploring strategy the units of each zone depart themselves, so the information about a zone arises at that zone and the loss from obtaining it falls there. In the operator's terms, an experiment is a randomized trial on a sample of units over a declared window at a capped simulated price, run while the other units follow the current policy, and exploration is a change to the policy that all units follow. The operator can therefore price an experiment before it runs and knows in advance which units its departures fall on, whereas the loss of exploration can fall on any unit at periods that are not fixed in advance.

\section{Identification}
\label{sec:exploit}

This section answers two questions: which quantities the record identifies, that is, determines as functions of its own distribution, and when a departure from the simulator-optimal policy can raise the total reward. Theorems~\ref{thm:nochange} and~\ref{thm:hiddendrift} answer the second at its two extremes. While the simulator is correctly specified, no history-dependent strategy improves on the simulator-optimal policy. After a drift at an action that this policy does not take, a departure can raise the total reward by an amount that grows with the square of the horizon, a regret that the simulator-optimal policy and its recalibration incur however often they refit. Theorems~\ref{thm:ident} and~\ref{thm:regimeid} answer the first, for the mean responses of the cells and for the regime distribution that adaptation estimates. Theorem~\ref{thm:testvalue} connects the two questions: it compares a designed test with adaptation and gives the two cases in which the test adds value.

\subsection{Limits of the Record}
\label{sec:record-alone}
\label{subsec:twoways}\label{subsec:censoring}
If the record before period $t$ has no observation at a cell $(x,a)$, then under Assumptions~\ref{asm:cellwise} and~\ref{asm:likelihood} the working model retains its prior there, $\hat m_t(x,a)=\theta^{\mathrm{sim}}(x,a)$ (Proposition~\ref{prop:censoring}, Appendix~\ref{app:moved}). Two theorems bound the value attainable without departures.

\begin{theorem}[No-change boundary]\label{thm:nochange}
Under Assumption~\ref{asm:correctspec} (correct specification), for each history-dependent strategy $\sigma$, $V_T(\sigma)\le V_T(\sigma^{\mathrm{sop}})=\max_{\bar\pi\in\bar\Pi}V_T(\bar\pi)$.
\end{theorem}

The proof is in Appendix~\ref{app:theory}. Thus, while the simulator is correctly specified, no departure from the simulator-optimal policy can raise the value, and a departure can raise it only after a drift. A \emph{drift} is a change in the latent distribution $\nu_t$ after calibration. We call a drift \emph{hidden} when some action has become best at a state where the deployed policy does not assign it. As in the introduction, a hidden drift is visible or invisible according to whether the on-policy data determine the new best action. Let $\bar\pi^*\in\arg\max_{\bar\pi\in\bar\Pi}V_T(\bar\pi)$ be the best fixed policy sequence. Our \emph{hidden-drift instance} has one unit, one observed state, two actions $a_0$ and $a_1$, and degenerate reward noise, that is, a noise model $Q$ that is a point mass at zero. Its latent covariate $u$ is zero throughout the historical record, advances by one in each period in which $a_0$ is assigned, and is unchanged under $a_1$. The response at $a_0$ does not depend on $u$, and the response at $a_1$ rises by $D>0$ with each unit of $u$. Under the simulator's tables, calibrated at $u=0$, the response at $a_0$ exceeds that at $a_1$ by $\Delta_0>0$. The simulator-optimal policy therefore assigns $a_0$ in each period, so that $u$ advances and the two responses first coincide at the crossing time $T_0:=\Delta_0/D$.

\begin{theorem}[Quadratic regret under hidden drift]\label{thm:hiddendrift}
Fix $D>0$, $\Delta_0>0$, and a horizon $T\geq T_0$. The hidden-drift instance satisfies Assumptions~\ref{asm:consistency}--\ref{asm:likelihood}, with the constants listed in Appendix~\ref{app:theory}. In it, $\sigma^{\mathrm{sop}}$ and $\sigma^{\mathrm{asop}}_\kappa$, for each cadence $\kappa$, assign the same action in each period, and $V_T(\bar\pi^*)-V_T(\sigma^{\mathrm{sop}})\geq D\big(\tfrac{T+1-T_0}{2}\big)^{2}-\tfrac{D}{4}$, with the same lower bound against $\sigma^{\mathrm{asop}}_\kappa$.
\end{theorem}

The proof is in Appendix~\ref{app:theory}. The drift is invisible to both strategies because the response at the action they take does not change. The response at the action they never take rises by $D$ in every period in which they act. A fixed policy that switched to that action at a period $t_s$ after the crossing would keep, in each later period, the advantage $D(t_s-T_0)$ that the response has gained by then. The best such switch, halfway between the crossing and the horizon, is therefore ahead by an amount that grows with the square of the horizon. The regret of SOP and A-SOP can thus grow with the square of the horizon however often they refit, so that bounding it requires a departure from the plug-in policy.

\subsection{Identification of Cell Means}
\label{sec:ident-cells}\label{sec:value}\label{subsec:pricing}
A period-$t$ quantity is \emph{identified} \citep{hernan2020causal} if it is a function of the distribution of the record through $t$ under the strategy in use. Fix a unit, and let $P:=\pi_t(a\mid x)$ be the probability with which the strategy assigns $a$ to it at the predecision history $\mathcal F_t$. Let $M:=\mathbb E[Y_t(\bar A_{t-1},a)\mid X_t=x,\mathcal F_t]$ be its history-conditional response at the cell, so that $m_t(x,a)=\mathbb E[M\mid X_t=x]$.

\begin{theorem}[Identification under prespecified and history-dependent collection]\label{thm:ident}
Under Assumptions~\ref{asm:consistency} (consistency) and~\ref{asm:markov} (Markov full state), fix a unit, a period $t$, and a cell $(x,a)$.
\begin{enumerate}
\item[(i)] Under a prespecified strategy, if $\Pr[X_t=x,A_t=a]>0$, then $m_t(x,a)=\mathbb E[Y_t\mid X_t=x,A_t=a]$. If instead $\Pr[X_{j,s}=x,A_{j,s}=a]=0$ for all $j$ and $s\le T$, then $m_t(x,a)$ is not identified.
\item[(ii)] Under a history-dependent strategy and (\ref{asm:crossunit}), if $P>0$ almost surely on $\{X_t=x\}$, then
\begin{equation}\label{eq:m-ipw}
m_t(x,a)=\mathbb E\!\left[\frac{Y_t\,\mathbf 1\{A_t=a\}}{P}\,\middle|\,X_t=x\right].
\end{equation}
\item[(iii)] Under a history-dependent strategy and (\ref{asm:crossunit}), if $\Pr[X_t=x,A_t=a]>0$, then $\mathbb E[Y_t\mid X_t=x,A_t=a]-m_t(x,a)=\operatorname{Cov}(P,M\mid X_t=x)/\mathbb E[P\mid X_t=x]$, which is zero if and only if $P$ and $M$ are uncorrelated given $\{X_t=x\}$.
\end{enumerate}
\end{theorem}

The proof is in Appendix~\ref{app:fourtheorems}. Thus the record identifies a cell's mean response only where some unit visits the cell. Under history-dependent assignment the raw cell mean is biased unless the assignment probability is uncorrelated with the response. A history-dependent strategy assigns an action with a probability that depends on the record, for example more often where the recorded outcomes were good. The raw mean therefore weights each unit by how likely the strategy was to assign the action to it, and this weight can be correlated, through the history, with the unit's response. Dividing each outcome by its assignment probability, as in \eqref{eq:m-ipw}, removes this bias.

\subsection{Structural Identification}
\label{sec:ident-structure}
Suppose that the latent covariate takes finitely many values $u_1,\dots,u_K$, referred to as \emph{regimes}. Suppose further that the \emph{regime distribution} $\nu\in\Delta_K$, the vector of regime shares, does not depend on the observed state. Each cell mean is linear in $\nu$: for a cell $z=(x,a)$ with regime responses $\theta(z,\cdot)\in\mathbb R^K$, $m(z)=\theta(z,\cdot)^\top\nu$. If only $\nu$ has shifted while the regime-conditional response is unchanged, the visited-cell means are therefore informative about $\nu$ without a test. Here the regime responses $\theta(z,\cdot)$ are known, the visited cells and their means are observed, and the regime distribution $\nu$ is the unknown. S-SOP (structural SOP) estimates $\nu$ from the set $V$ of visited cells and sets the table to the implied cell means. So far \emph{adaptation} has named any recalibration from on-policy data, A-SOP included. S-SOP is the only adaptation that the simulation study compares, so from here on we write \emph{adaptation} for S-SOP; A-SOP appears again only in the appendix. Let $\Theta_V$ be the $|V|\times K$ matrix of the regime responses of the visited cells, and let $G:=[\Theta_V;\mathbf 1^\top]$ be $\Theta_V$ with a row of ones stacked beneath it, which encodes that $\nu$ sums to one. For any set $W$ of cells we write $\Theta_W$ and $G_W:=[\Theta_W;\mathbf 1^\top]$ in the same way. The \emph{identified set} of $\nu$ is $\mathcal I(\nu):=\{\mu\in\Delta_K:\Theta_V\mu=\Theta_V\nu\}$. S-SOP fits $\hat\nu_t$ by least squares on the visited-cell means, weighted by the visit counts, with the penalty $n_0\lVert\mu-\nu^{\mathrm{sim}}\rVert_2^2$ toward the simulator's regime distribution $\nu^{\mathrm{sim}}\in\Delta_K$. The limit of this fit as $n_0\to0$ is the \emph{ridge limit}, that is, the least-squares fit nearest $\nu^{\mathrm{sim}}$ when the fit is not unique.

\begin{theorem}[Identifiability of the regime distribution from visited cells]\label{thm:regimeid}
Let Assumptions~\ref{asm:consistency} and~\ref{asm:markov} hold, let the latent covariate take $K$ regimes with known regime responses $\theta(z,\cdot)$, and fix a unit whose regime distribution $\nu\in\Delta_K$ does not depend on its state. Let $V$ be its set of visited cells, with $\Theta_V$ and $G$ as defined above.
\begin{enumerate}
\item[(i)] $\mathcal I(\nu)=\{\nu\}$ for each $\nu\in\Delta_K$ if and only if $\operatorname{rank}G=K$, and then each cell mean $m(z)=\theta(z,\cdot)^\top\nu$ is identified.
\item[(ii)] If $\operatorname{rank}G<K$, then $\mathcal I(\nu)=(\nu+\ker G)\cap\Delta_K$, and for a cell $z$, $m(z)$ is identified for each $\nu$ if and only if $\theta(z,\cdot)$ lies in the row space of $G$.
\end{enumerate}
\end{theorem}

The proof is in Appendix~\ref{app:regimeid}. Intuitively, the visited cells determine the regime distribution exactly when their regime responses, together with the constraint that the shares sum to one, span $K$ independent directions. Otherwise several regime distributions can be consistent with the visited-cell means, and the plug-in policies of these distributions may or may not agree. When the rank condition holds, even a hidden drift is visible, since the visited-cell means identify each cell mean. When it fails, the ridge limit is the point of the identified set nearest $\nu^{\mathrm{sim}}$. It agrees with $\nu$ along the row space of $G$, and it agrees with $\nu^{\mathrm{sim}}$ along $\ker G$ whenever the point $\tilde\mu_0$ of $\nu+\ker G$ nearest $\nu^{\mathrm{sim}}$ is itself a distribution (Appendix~\ref{app:fourtheorems}). An action that has become best can therefore stay untried.

\subsection{Value of a Designed Test}
\label{sec:testvalue}
We now compare a designed test with S-SOP in one zone of $n$ units under the conditions of Theorem~\ref{thm:regimeid}, without the transport update. In the \emph{population limit} the record carries expected rather than realized visits and rewards. Past observations enter the working model at equal weight. Each deployed policy satisfies a \emph{minorization condition}: the state distribution after $\varrho$ periods is, with probability at least $\epsilon$, independent of the start. The \emph{drift condition} is the counterpart of Assumption~\ref{asm:drift} for this zone, $\lVert\nu_{t+1}-\nu_t\rVert_1\le\lambda$ in the $\ell_1$ norm. For a regime distribution $\mu$ held fixed, $g_\mu(\pi)$ is the average reward per unit and period of the Markov chain that a policy $\pi$ induces on the observed state. We write $\pi[\mu]$ for the plug-in policy computed from the table $\sum_u\mu(u)\theta(\cdot,\cdot,u)$, so that S-SOP deploys $\pi[\hat\nu_t]$, and we write $\pi^*:=\pi[\nu]$ for the plug-in policy of the true table. In this comparison the experimenting strategy, the \emph{tested S-SOP}, adds to S-SOP one test at $t_0$. The test's observations enter at weight $w\ge1$ and cover cells $Z$ disjoint from the cells $V$ visited before $t_0$. Write $\mathrm{price}$ and $\mathrm{Gain}$ for the test's \emph{real price} and \emph{real gain}, the value lost by running it and the value added by refitting on its data, both under the real system. They differ from the simulated price of Section~\ref{sec:problem-controls}, which is computed under the working model. The \emph{switch condition} holds if both strategies deploy one policy $\pi_0$ until they change to a second policy $\pi_1$, the experimenting strategy at $t^E$ and S-SOP at $t^A$. We then write $\delta$ for the smallest per-period advantage $g_{\nu_t}(\pi_1)-g_{\nu_t}(\pi_0)$ over $t^E<t\le t^A$, and $V_t$ for the cells visited by period $t$. The constants $\kappa_0,\kappa_1,\kappa_2$ below do not depend on $T$, $n$, $t_0$, or $w$.

\begin{theorem}[Value of a designed test over structural adaptation]\label{thm:testvalue}
Let a zone of $n$ units in the population limit satisfy Assumptions~\ref{asm:consistency} (consistency) and~\ref{asm:markov} (Markov full state), the conditions of Theorem~\ref{thm:regimeid}, and the minorization condition.
\begin{enumerate}
\item[(i)] The value of the experimenting strategy exceeds the value of S-SOP by $\mathrm{Gain}-\mathrm{price}$.
\item[(ii)] If $\nu_t=\nu$ for all $t$, both strategies use the ridge limit, and $\pi[\mu]=\pi[\nu]$ for each $\mu$ in the identified set $\mathcal I(\nu)$ computed from the cells $V$, then $\mathrm{Gain}=0$.
\item[(iii)] If $\nu_t=\nu$ for all $t$, S-SOP deploys from $t_0$ on one policy $\tilde\pi\ne\pi^*$ visiting only cells in $V$, and $\operatorname{rank}G_{V\cup Z}=K$, then in the ridge limit the experimenting strategy deploys $\pi^*$ and $\mathrm{Gain}\ge n(T-t_0)\{g_\nu(\pi^*)-g_\nu(\tilde\pi)\}-\kappa_2 n$.
\item[(iv)] Under the switch and drift conditions, if $\operatorname{rank}G_{V_t}=K$ at each period considered and $t^A\ge t^E$, then $\mathrm{Gain}\ge n(\delta-\kappa_1\lambda)(t^A-t^E)-\kappa_0n$.
\end{enumerate}
\end{theorem}

The proof is in Appendix~\ref{app:testvalue}. A test thus adds value when the visited cells do not identify the regime distribution (clause~(iii)) or when S-SOP lags behind a continuing drift (clause~(iv)). When the structure already determines the policy (clause~(ii)), the test only costs its price.

\section{The Budgeted Procedure}
\label{sec:now}

Here, we present one procedure whose five parameters, among them the budget of Section~\ref{sec:problem}, determine when and how it departs from the plug-in policy, so that the four strategies become settings of a single algorithm. After the parameters we describe one review of the procedure, in which it scores a menu of candidate designs under the working model and admits tests under two conditions within the budget. We then list the settings at which it coincides with five of the six strategies of the simulation study. Two results close the section: Proposition~\ref{prop:purse} bounds the loss that the admitted tests can cause, and Theorem~\ref{thm:detectprice} shows that detecting a hidden drift requires departures whose budget bounds the power of any detector.

\subsection{Parameters}
\label{sec:proc-params}
Throughout, we hold fixed the review schedule $t_k=t_0+kh$, with a first review $t_0$ and an interval $h$, the prior weight $n_0$, the neighbour weight $\varkappa$ in \eqref{eq:transport}, and the selector $S$. The procedure then has five parameters. The \emph{centre} is the table toward which the working model is shrunk, either the simulator's table $\theta^{\mathrm{sim}}$ or the \emph{structural table} $\sum_u\hat\nu_t(u)\theta(\cdot,\cdot,u)$ that S-SOP fits. The \emph{centre weight} $\tau\in(0,\infty]$ is the number of observations the centre counts for. At each cell the working mean is the average of the centre and the observed mean, weighted by $\tau$ and by the observation count, so that the working model is shrunk toward the centre. The \emph{memory rate} $\omega\in[0,1)$ discounts past data: the sums in \eqref{eq:transport} are multiplied by $1-\omega$ each period. The \emph{bonus} $b$ is a multiple of the tunnel width, and the \emph{budget} $C$ in \eqref{eq:problem} is the largest simulated price admitted at a review.

\subsection{Reviews and Candidate Tests}
\label{sec:proc-tests}
We call $\hat m+bs$, the working mean plus the bonus times the tunnel width at each cell, the \emph{optimistic table}. At $C=\infty$ we deploy $S$ on it at each review. The exploring strategies of the study review in every period, so their policy is recomputed daily. At a finite budget the procedure departs from the plug-in policy only through tests, which it chooses at each review as follows. A \emph{candidate test} $d$, a design under consideration at a review, is a one-period randomized design on part of one zone, specified by its zone, its set of tested actions, and the fraction of the zone's units tested. Its \emph{simulated gain} $\operatorname{gain}(d)$ is the increase in the zone's total reward that the test produces when we simulate the zone forward with the test and without it. The simulation runs under the optimistic table over a \emph{planning window}, a fixed number of periods that starts at the review. In this simulation the test's observations equal the optimistic table's entries at the tested cells, and the plug-in policy is recomputed from the working model as these observations enter it. Its simulated price $\operatorname{price}(d)$ is the drop in the zone's total reward over the same window, under $\hat m$, from assigning the tested actions on the test's units for one period. On the day of a test, its observations enter the sums in \eqref{eq:transport} with a \emph{test weight} $w\ge1$. We admit candidates one at a time. At each step we score every remaining candidate by its simulated gain times the number of its tested units that no admitted candidate already covers, and we admit the one with the highest score. We admit it under one of two \emph{admission conditions}, \emph{positive gain}, $\operatorname{gain}(d)>0$, or \emph{net gain}, $\operatorname{gain}(d)>\operatorname{price}(d)$, until the total simulated price of the admitted candidates reaches $C$. A run of the procedure uses one of the two conditions, fixed in advance; Section~\ref{sec:sim-strategies} states which strategy uses which.

\subsection{The Strategies as Settings}
\label{sec:proc-corners}
The four strategies are settings of this procedure, and so are five of the six strategies that the simulation study compares (Section~\ref{sec:sim-strategies}), all but Thompson sampling (Table~\ref{tab:settings}, Appendix~\ref{app:theory}). SOP and S-SOP exploit at a zero budget, UCB (for \emph{upper confidence bound}) explores at an unbounded budget, and SEP (for \emph{simulator-aided experimentation}) and SEP-when, a variant that also chooses when to test, experiment at a finite budget.

\subsection{Bounded Loss}
\label{sec:proc-loss}
In Proposition~\ref{prop:purse} (Appendix~\ref{app:theory}) we bound the loss from the tests admitted at a review relative to the plug-in policy that the procedure would deploy without them, under Assumptions~\ref{asm:consistency} and~\ref{asm:markov}. Measured under the plug-in model at the review, this loss equals the total admitted price and is therefore at most $C$. The true loss exceeds it by at most the \emph{departure mass}, the number of tested units at the review, times a constant that is linear in the working model's errors at the covered cells. This constant is zero when the working model has no error there. The departure mass differs from the departure count of Section~\ref{sec:proc-detect}, which counts unit-periods inside a hidden change set. In the same proposition we also follow a running balance from which a fee is deducted each period, the accounting of the bankruptcy criterion that Section~\ref{sec:sim-results} applies. We give a condition under which the balance stays nonnegative, so that the operation never goes bankrupt.

\subsection{Detection and Its Price}
\label{sec:proc-detect}
Let $P_0$ and $P_1$ be two data-generating processes that satisfy Assumptions~\ref{asm:consistency} (consistency) and~\ref{asm:markov} (Markov full state). They agree on the initial distribution of the full state and on the historical record. They differ only on a nonempty \emph{hidden change set} $Z^\ast$ of cells. Outside $Z^\ast$ the conditional distributions of the outcome, the next state, and the next latent covariate coincide. Under $P_0$ the simulator is correctly specified (Assumption~\ref{asm:correctspec}), so that $\hat{\bar\pi}:=S(\theta^{\mathrm{sim}},f^{\mathrm{sim}})$ is optimal by Theorem~\ref{thm:nochange}. We assume \emph{hidden support}, that is, $\hat{\bar\pi}$ does not visit $Z^\ast$ under $P_0$. We also assume \emph{strict suboptimality}, that is, under the calibrated tables, taking an action of $Z^\ast$ at any one period, and the optimal actions afterward, lowers a unit's remaining-horizon value by at least $p_{\min}>0$. For a strategy $\sigma$, its \emph{departure count} $N_\sigma$ is the number of unit-periods whose state and action form a cell of $Z^\ast$. A \emph{detector} is an event $E$ determined by the record and the strategy's randomization. It has \emph{size} $\alpha$ if $P_0(E)\le\alpha$, and it has \emph{power} $1-\vartheta$ if $P_1(E)\ge1-\vartheta$. The \emph{no-change loss} of $\sigma$ is $\operatorname{loss}(\sigma):=V_T(\hat{\bar\pi})-V_T(\sigma)$ under $P_0$.

\begin{theorem}[The detection--price tradeoff]
\label{thm:detectprice}
Let $P_0$ and $P_1$ be the two processes and $Z^\ast$ the hidden change set, and suppose that hidden support and strict suboptimality hold. For each strategy $\sigma$:
\begin{enumerate}
\item[(i)] each detector $E$ satisfies the inequalities $P_1(E)-P_0(E)\le P_0(N_\sigma\ge1)\le\min\{1,\mathbb E_0[N_\sigma]\}$;
\item[(ii)] $\operatorname{loss}(\sigma)\ge p_{\min}\,\mathbb E_0[N_\sigma]$;
\item[(iii)] if $\sigma$ admits a detector of size $\alpha$ and power $1-\vartheta$, then $\operatorname{loss}(\sigma)\ge p_{\min}(1-\alpha-\vartheta)$;
\item[(iv)] if $\sigma$ is the budgeted procedure with budget $C<\infty$ over $\bar k$ reviews and a working model equal to the simulator's tables, then $\mathbb E_0[N_\sigma]\le\bar kC/p_{\min}$, so that each of its detectors satisfies $P_1(E)\le P_0(E)+\bar kC/p_{\min}$.
\end{enumerate}
\end{theorem}

The proof is in Appendix~\ref{app:detectprice}. The two processes agree outside $Z^\ast$, so the records that they generate coincide until a departure enters $Z^\ast$. A detector can therefore separate them only through such departures, and clause~(i) bounds its advantage by the chance of one. Detecting a hidden drift thus requires departures, each costing at least $p_{\min}$ where nothing has changed. A simulated-price budget therefore caps any detector's power.

\section{Simulation Study}
\label{sec:sim}

We use a simulation study to answer three questions. They are the three parts of the question posed in the introduction, made concrete for a fleet of vending machines in which a stocking mistake is either cheap or expensive.
\begin{enumerate}
\item[(Q1)] \emph{Whether additional data add value.} In which settings does a strategy that collects additional data improve on the simulator-optimal policy, and by how much? The strategy may recalibrate from on-policy data, explore, or conduct an experiment. For a single zone, Theorem~\ref{thm:testvalue} distinguishes the cases in which a test adds value from the case in which it only costs its price. A test only costs its price where the structure already determines the policy (clause~(ii)). It can add value, however, where the adaptation fit lags behind a continuing drift (clause~(iv)). We ask whether a fleet of nine zones with the transport update shows the same two cases.
\item[(Q2)] \emph{When to run a test.} What is the loss from a test where nothing has changed, and how close is that loss to the test's simulated price? By Proposition~\ref{prop:purse} the true loss exceeds the simulated price by at most a term linear in the working model's errors. Is a test still worth its price when the shift is later reversed?
\item[(Q3)] \emph{How to collect the data.} When a shift has occurred, does exploration match an experiment, and what is the loss from each where nothing has changed? Within experimentation, what does the simulator contribute to the design of a test: the decision whether to test, the price of each test, or the choice among the candidate tests?
\end{enumerate}
To answer them we run the six strategies in worlds, scenarios that differ in how the target departs from the simulator. We group the worlds into four families by the breadth of the shift and the cost of a mistake. We score each strategy by its margin over the simulator-optimal policy. Section~\ref{sec:sim-benchmark} describes the fleet and the worlds and Section~\ref{sec:sim-strategies} the criterion and the strategies' settings; Sections~\ref{sec:sim-whether} to~\ref{sec:sim-results} then answer the three questions, one in each. Figure~\ref{fig:breadth-map} gives the best strategy in each world. Figure~\ref{fig:speed-stress} gives the margins against the speed of the shift in the stress family, where the shift is broad and a mistake expensive; its three panels differ in how much of the shift the on-policy data show.

\subsection{Benchmark and Worlds}
\label{sec:sim-benchmark}\label{sec:sim-worlds}
We simulate a fleet of vending machines on a three-by-three grid of nine zones; a period is a day and the action at a machine is one of six order quantities. The latent covariate is the demand regime, low, mid, or high, and the simulator is calibrated to a regime distribution $\nu^{\mathrm{sim}}$ common to all zones (Appendix~\ref{app:simulation}). A mistake is an order that does not sell; it is cheap when an unsold unit stays saleable for four further days and expensive when the unit stays saleable for one day at higher disposal and warehouse costs. In each \emph{world} the regime of each machine in two corner zones is redrawn daily with probability $\beta$, the \emph{speed} of the shift, from a new zone distribution. In the \emph{bust} zone demand falls, and the fall appears in the on-policy data. In the \emph{boom} zone demand rises above the quantity that the simulator-optimal policy stocks, so that the record contains no observation of a larger order's profit. The \emph{visible share} $\phi$, the fraction of the shift in the bust zone, takes five values from $0$ to $1$. The worlds with $\phi=0$ form the \emph{hidden row}. In them the record contains no outcome of the orders whose profit the shift raises. A shift there is therefore a hidden drift, as defined in Section~\ref{sec:record-alone}, whenever a larger order has become best. Whether the on-policy data then determine the new best order depends on the working model (Section~\ref{sec:sim-whether}). The speed takes five values from $\beta=0.003$, where half of the shift takes about $230$ days, to $\beta=0.3$, where it takes about two days. Four \emph{families} cross the breadth of the shift with the cost of a mistake. A \emph{sparse} shift is confined to the two corner zones, and in a \emph{broad} one the seven other zones also shift toward the boom distribution. The broad and expensive family is the \emph{stress family} and the sparse and cheap family the \emph{baseline family}. Each family has $25$ worlds and, at each speed, a \emph{no-change world} in which nothing shifts.

\subsection{Criterion and Strategies}
\label{sec:sim-strategies}
We compare the strategies by their margin over SOP, \eqref{eq:margin}, over $T=1024$ days, in the population limit introduced in Section~\ref{sec:testvalue}, which replaces realized visits and rewards by their expectations, for $999$ machines, $111$ per zone. In sampled fleets of $999$ machines the lead of the experiment over adaptation, which decides the stress squares of the map, lies within $2$ percent of this limit (Appendix~\ref{app:sampled}). For the map we group the six strategies into four \emph{modes}, SOP, adaptation, exploration (UCB and Thompson sampling), and experimentation (SEP and SEP-when). We therefore treat a lead of at most $2$ percent over the best strategy of another mode as a tie, the \emph{tie band}. Adaptation (S-SOP) sets the centre to the structural table, the \emph{structural centre}, at a zero budget. It refits each zone's regime distribution daily from the on-policy data, with neighbouring zones entering at weight $\varkappa=1/2$ in \eqref{eq:transport}. UCB, the discounted upper-confidence-bound strategy \citep{garivier2011upper}, deploys $S$ on the optimistic table with an unbounded budget and a positive memory rate, and it reviews daily, as does Thompson sampling. Thompson sampling \citep{thompson1933likelihood} applies $S$ to a table drawn at random around the working model. UCB and Thompson sampling run at one setting each in all families, with UCB's bonus and memory rate chosen in the baseline family. SEP uses the structural centre and admits candidates under the positive gain condition. It reviews every $28$ days from day $10$ with the test weight $w=8$ and a budget $c_0$ at each review, and it scores candidates over a planning window of $128$ days (Appendix~\ref{app:simulation}). Here $c_0$ is the simulated price, at each family's cost of a mistake, of a one-day test of all six actions in three zones that together neighbour all zones. SEP-when admits a candidate only under net gain, and it reviews every $7$ days with the budget $c_0/4$. We chose SEP's remaining settings on nine of the $25$ worlds of each family and evaluated it on all $25$ (Appendix~\ref{app:simulation}).

\subsection{Whether Additional Data Add Value (Q1)}
\label{sec:sim-whether}
\begin{figure*}[t]
\centering
\includegraphics[width=\textwidth]{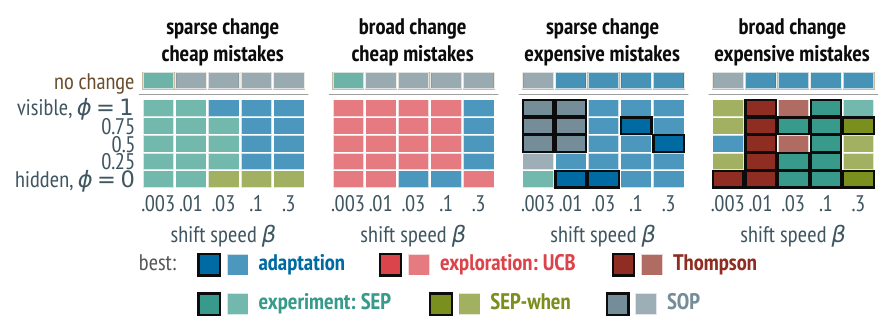}
\caption{Each panel is one family of worlds at $999$ machines, named by the breadth of the shift and the cost of a mistake. Within a panel the columns are the five shift speeds $\beta$ and the rows the five visible shares $\phi$, from the hidden row $\phi=0$ at the bottom to $\phi=1$ at the top; the separate row above each panel holds the family's no-change worlds, one per speed. Each square is one world, coloured by the strategy with the highest margin over SOP; a bordered, darker square marks a lead of more than $2$ percent over the best strategy of another mode, and a lighter square a tie within that band.}
\label{fig:breadth-map}
\end{figure*}

\begin{figure}[t]
\centering
\includegraphics[width=\textwidth]{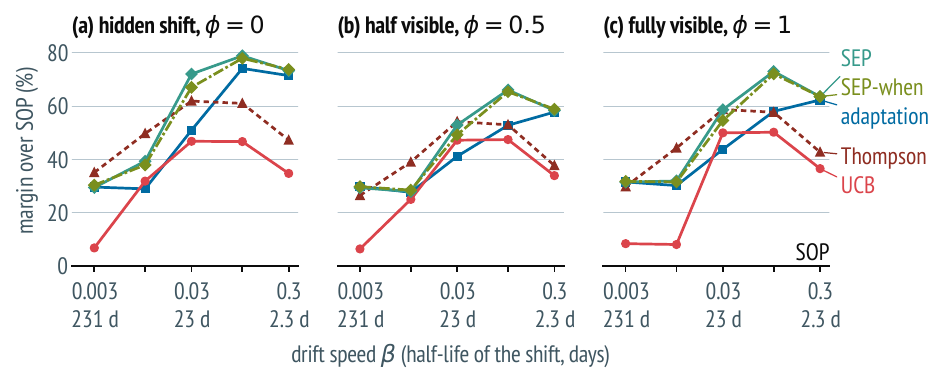}
\caption{Margin over SOP, \eqref{eq:margin}, against the drift speed in the stress family at $999$ machines, for three visible shares. In (a), the hidden row, the corner shift is entirely a rise in demand that the on-policy data do not show; in (b) half of it is a fall in demand in the bust zone; in (c) the corner shift is entirely that fall, and the only hidden change is the shift of the seven other zones. The margin is relative to the SOP value of each world, which differs across the panels, so we compare the strategies within a panel. The labelled speeds are $\beta=0.003$, $0.03$, and $0.3$, with the half-life of the shift, the number of days until half of the shift has occurred, beneath; the unlabelled ticks are $\beta=0.01$ and $0.1$. The line labels in (c) apply to all panels; Thompson sampling is the median of five seeds.}
\label{fig:speed-stress}
\end{figure}

Additional data add value only in the stress family. When a mistake is cheap, the simulator-optimal policy already stocks near capacity, so a shift leaves little value to gain. No strategy improves on SOP by more than $1$ percent in the baseline family, and in the broad and cheap family adaptation, UCB, SEP, and SEP-when improve on SOP by about $3$ percent. When a mistake is expensive and the shift is sparse, all distributions in the identified set yield the same plug-in policy (clause~(ii) of Theorem~\ref{thm:testvalue}). Adaptation has the highest hidden-row margin, $12$ to $18$ percent over SOP at all but the slowest speed. UCB and Thompson sampling are below SOP. In the stress family the adaptation fit follows the drifting regime distribution with a lag, the case of clause~(iv). Figure~\ref{fig:speed-stress}(a) shows the hidden row: the median over the five speeds of SEP's per-world margin is $72$ percent, compared with $51$ for adaptation, $50$ for Thompson sampling, and $35$ for UCB. Its lead over adaptation is $10$ percent at $\beta=0.01$, $21$ percent at $\beta=0.03$, and at most $5$ percent elsewhere. Panels (b) and (c) show the worlds in which half and all of the corner shift is a fall in demand in the bust zone, which the on-policy data show. In (c) the only hidden change is the shift of the seven other zones toward the boom distribution. As the visible share rises from (a) to (c), the experiment's lead over adaptation moves to faster speeds. It is $12$ and $13$ percent at $\beta=0.03$ and $0.1$ when $\phi=0.5$, $15$ percent at both speeds when $\phi=1$, and within $2$ percent at the other speeds.

The map of Figure~\ref{fig:breadth-map} counts the same family world by world. Over its $25$ worlds, SEP or SEP-when has the highest value in $16$ and Thompson sampling in $8$, with $10$ and $6$ of these leads beyond the tie band (Table~\ref{tab:cast}).

\subsection{When to Run a Test (Q2)}
\label{sec:sim-when}
In the expensive no-change worlds the budget bounds the loss (Proposition~\ref{prop:purse}). At $999$ machines SEP's loss relative to adaptation is about $2$ percent of SOP's value except at the slowest speed, close to the tests' simulated price, $1.4$ to $1.8$ percent. Where the shift reverses every $256$ days, SEP's lead over adaptation on $27$ machines falls to $2$ percent (Figure~\ref{fig:recurrent}, Appendix~\ref{app:results}). A test is therefore worth its price where a shift has occurred and is not soon reversed, and where nothing has changed it costs about its simulated price.

\subsection{How to Collect the Data (Q3)}
\label{sec:sim-results}
Exploration can match an experiment: on the stress hidden row Thompson sampling has the highest value at the two slowest speeds (Figure~\ref{fig:speed-stress}(a)). It also has the highest value at $\beta=0.01$ in panels (b) and (c), where it ties the experiment at $\beta=0.03$. Where the system has not changed, however, UCB's value is $16$ to $30$ percent below SOP's and Thompson sampling's $87$ to $96$ percent below. We also apply the bankruptcy criterion of \citet{backlund2025vending}: a strategy starts with a balance of thirty days of the plug-in policy's mean profit, pays a daily fee, and is bankrupt after ten consecutive days with a negative balance. Under this rule the exploring strategies go bankrupt most often and adaptation least often when few machines change (Table~\ref{tab:ruin-family}). Experimentation is therefore preferable to exploration when a mistake is expensive and the system may not have changed. We compare SEP with two experimenting strategies that share its settings but design without the simulator, a uniform test of all six actions and a random draw from SEP's candidate menu, each on random whole zones. On the stress hidden row the random design matches SEP's median margin within the tie band, at $999$ machines and on sampled fleets of $27$, so ranking by simulated gain does not improve on random placement. Where nothing has changed, both designs spend and lose more than SEP at $999$ machines but as much at $27$, where no candidate fails admission (Appendix~\ref{app:simulation}). In the stress family the simulator therefore contributes to the experiment through the decision whether to test and through the price of each test, while the choice among the candidate tests adds nothing beyond random placement.

\section{Conclusion}
\label{sec:conclusion}

A policy computed from a simulator calibrated to historical data is optimal for the system to which the simulator was calibrated. We asked what to do when the present system may differ from that one, and in particular whether, when, and how an operator should collect additional data. In this paper we introduce a framework in which the four available strategies, continuing to run the simulator-optimal policy, recalibrating from on-policy data, exploring, and experimenting, are settings of a single budget. The budget bounds the loss, computed under the current simulator, from \emph{departures}, that is, from actions that the current policy does not recommend. This framing let us compare the four strategies under one objective, the value in the real system, and study one procedure whose parameters move it from one strategy to another.

The theory addresses the first part of the question, whether to collect data at all, before turning to when. While the simulator is correctly specified, no strategy improves on the simulator-optimal policy (Theorem~\ref{thm:nochange}), and under a small drift that policy stays close in value to the best fixed policy, within a gap proportional to the drift (Corollary~\ref{cor:sopregret}, Appendix~\ref{app:moved}). The mean response at a state-action pair is identified only where some unit reaches that pair (Theorem~\ref{thm:ident}). Under history-dependent collection the raw mean of the outcomes there is biased unless the assignment probability is uncorrelated with the response, and reweighting by that probability removes the bias. The distribution of the latent variable is identified from the visited pairs exactly when their responses, one per value of the latent variable, together with the requirement that the shares of those values sum to one, span all the values (Theorem~\ref{thm:regimeid}). A designed test therefore adds value in two cases, when the visited pairs leave the latent distribution undetermined in a way that changes the best action, and when recalibration follows a continuing drift with a lag. When the structure already determines the policy, the test only costs its price (Theorem~\ref{thm:testvalue}). As to when, a drift at actions that the policy does not take can be detected only through departures, and the price of a departure is paid whether or not the system has changed. The same budget that bounds the loss where nothing has changed therefore limits the power of any test for such a drift (Theorem~\ref{thm:detectprice}, Proposition~\ref{prop:purse}).

The simulation study agrees with this account and adds the cost of a mistake as the organizing axis. When a stocking mistake is cheap, the simulator-optimal policy already stocks near capacity and no strategy improves on it by more than a few percent, so that additional data buy little. When a mistake is expensive and the shift reaches few machines, recalibration that fits the distribution of the latent variable recovers the new best action from the on-policy data alone, and the exploring strategies fall below the simulator-optimal policy. When a mistake is expensive and the shift reaches many machines, the experimenting strategies have the highest value in most worlds. Their lead over recalibration concentrates at moderate speeds, where the recalibrated fit lags behind the drift, and Thompson sampling leads at the slowest shifts. Where nothing has changed, the experimenting strategies lose about the simulated price of their tests, whereas the exploring strategies lose a large share of the value and, when few machines change, most often reach the bankruptcy state defined in Section~\ref{sec:sim-results}.

These findings have a practical reading. The budget limits the simulated price of departures, which we compute in the simulator before they run, while the operator bears the real loss only when they run. Where nothing has changed, that expected real loss is at most the budget plus a term linear in the working model's error (Proposition~\ref{prop:purse}), so the budget is the loss that the operator is prepared to absorb for learning whether anything has changed. Three cases then organize the choice of strategy. When a wrong action is cheap, recalibrating from on-policy data loses little and a test is not worth its price. When a wrong action is expensive and a drift would reach few units, a fit of the latent distribution from the visited pairs recovers the new best action without a test, provided that all distributions consistent with the visited pairs yield the same best action. Clause~(ii) of Theorem~\ref{thm:testvalue} states this condition, which can hold even when the visited pairs do not identify the distribution, as in the sparse and expensive family that we simulate. Like the rank condition of Theorem~\ref{thm:regimeid}, it can be checked from the responses at the visited pairs and the fitted distribution before any test is run. When a wrong action is expensive and a drift may reach many units, a test admitted under the budget is worth running. Two further implications concern the experiment itself. The gain from a test grows with the number of units that share the tested structure (Theorem~\ref{thm:testvalue}), so that the gain from a test on a sample of units accrues through the units that never take part. Where the shift is broad and a mistake expensive, moreover, ranking candidate tests by their simulated gain did not improve on a random choice from the same menu. What the simulator contributes to an experiment is therefore the decision whether to test and the price of each test rather than the choice of which test to run. In practice the two strategies that collect off-policy data differ in what they change. An experiment fixes its sample, its window, and the randomization of its actions before it runs, so its simulated price is known in advance and the cohort's averages estimate the tested actions' responses at the cohort without reweighting. Exploration changes the policy that all units follow, chooses its departures from the record, and has no stopping rule. Experimentation is preferable to exploration when a mistake is expensive and the system may not have changed, since no budget bounds the loss of exploration there. Where a mistake is cheap, or where the shift is slow and broad, exploration can match an experiment.

These findings depend on a finite latent structure and on one simulated system, so the boundaries between the three cases may lie elsewhere in other systems. Theorem~\ref{thm:testvalue} holds in the population limit and for a shift that is not later reversed; in sampled fleets of few machines the lead of the designed experiment over recalibration narrows (Appendix~\ref{app:sampled}), and where the shift reverses often the two tie. The exploring strategies run at one setting each, and the review schedule, the budget, and the test weight are fixed in advance rather than learned from the record. Each of these limitations marks a direction for further work. Natural extensions include a version of Theorem~\ref{thm:testvalue} for finite populations and for recurrent shifts, where the study already shows how the lead narrows. A further extension is a version of the procedure that learns the schedule, the budget, and the test weight from the record as it runs. The framework itself requires only a simulator calibrated to historical data, a latent variable that the record omits, and a simulated price for departures, and these are available wherever a calibrated simulator supports operating decisions. Ultimately, the framework lets an operator supplement a calibrated simulator with data from the system rather than choose between the two. Our results show when the budget for those data is worth spending.

\subsection*{Code and data availability}

The simulation study uses no real data. Appendix~\ref{app:simulation} specifies the benchmark, the worlds, and every strategy's settings so that the study can be re-implemented from the paper; the code is not released with this version.

\subsection*{AI use statement}

In this work, we used generative AI tools to help with writing and polishing the paper, with the implementation of the empirical studies, and with the verification of the theory. We have reviewed all AI-assisted work and take responsibility for the final content of this work, including text, claims, and artifacts produced with the aid of generative AI.

\begin{ack}
This work was funded by Amazon Supply Chain Optimization Technologies (SCOT).
\end{ack}

\bibliographystyle{apalike}
\bibliography{references}

\clearpage
\appendix

\begin{center}
{\Large\bfseries Exploration and Experimentation with a Misspecified Simulator:\\[2pt]
Appendices\par}
\end{center}
\vspace{1ex}

\section{Notation}
\label{app:notation}

As stated in Section~\ref{sec:setup}, the main text suppresses two indices, which the appendix makes explicit. First, the unit $i$ and the strategy $q$ (Section~\ref{sec:setup-units}): in the appendix we write $\nu_t$, $m_t(x,a)$, and $g_t(x'\mid x,a)$ as $\nu^q_{i,t}$, $m^q_{i,t}(x,a)$, and $g^q_{i,t}(x'\mid x,a)$. Second, the discount: the main text does not discount future outcomes, whereas the appendix uses a discount $\gamma\in(0,1]$ carried as a superscript on the value, so the main text's undiscounted $V_T$ is the case where the discount equals one, written $V^1_T$. We write each notation table with its cell $(x,a)$ as the first arguments. A hat marks a working estimate, the posterior mean of a candidate table, and the superscript $\mathrm{sim}$ marks a quantity of the simulator. Tables~\ref{tab:notation}--\ref{tab:notation-c} list the main-text symbols in order of appearance, and Table~\ref{tab:notation-app} lists the symbols that appear only in the appendix. Throughout, the letter $b$ denotes the optimism bonus, both in the main text and in Appendix~\ref{app:results}, except in Corollary~\ref{cor:switch}, where $a$, $b$, and $e$ name a hyperplane's normal, its offset, and the unit drift direction.

\begin{table}[p]
\centering
\caption{Notation: the symbols of the main text in order of appearance (Part~1 of~3).}
\label{tab:notation}
\footnotesize
\renewcommand{\arraystretch}{0.85}
\setlength{\tabcolsep}{4pt}
\begin{tabular}{@{}>{\raggedright\arraybackslash}p{0.30\textwidth}>{\raggedright\arraybackslash}p{0.51\textwidth}>{\raggedright\arraybackslash}p{0.14\textwidth}@{}}
\toprule
symbol & meaning & defined in \\
\midrule
$i\in\mathcal V=\{1,\dots,n\}$; $t\in\{0,\dots,T\}$ & units and their number; decision periods, with $t<0$ the historical period & Section~\ref{sec:setup-units} \\
$X_{i,t}\in\mathcal X$, $A_{i,t}\in\mathcal A$, $Y_{i,t}\in[-B_Y,B_Y]$, $U_{i,t}\in\Ucal$ & observed state, action, outcome, and latent covariate of unit $i$ in period $t$ & Section~\ref{sec:setup-units} \\
$f(\cdot\mid x,a,u)$, $f_X$; $\theta(x,a,u)$ & full-state Markov kernel and its $\mathcal X$-marginal; latent mean response at a full state & Section~\ref{sec:setup-units} \\
$\mathbf X_t$, $\mathbf A_t$ & the units' states and actions in period $t$ & Section~\ref{sec:setup-units} \\
$(x,a)$; $\bar a_{i,t}=(a_{i,0},\dots,a_{i,t})$ & a cell (state-action pair); an action path, a bar marking a sequence over periods & Section~\ref{sec:setup-units} \\
$\mathbb E$, $\Pr$; the superscript $\mathrm{sim}$; a hat & expectation and probability under the fixed strategy; a quantity of the simulator; an estimate from the record & Section~\ref{sec:setup-units} \\
$\nu_t(\cdot\mid x)$ & conditional distribution of $U_t$ given $X_t=x$ under the fixed strategy & Section~\ref{sec:setup-units} \\
$m_t(x,a)$, $g_t(x'\mid x,a)$ & target quantities at a cell: the mean response and transition probabilities of the target & Section~\ref{sec:setup-units} \\
$\theta^{\mathrm{sim}}(x,a)$, $f^{\mathrm{sim}}(\cdot\mid x,a)$, $Q(\cdot\mid x,a)$, $\eta_t$ & the simulator's response table, transition table, outcome-noise distribution, and a noise draw (Definition~\ref{def:simulator}) & Section~\ref{sec:setup-sim} \\
$\Xi_0$; $(m^{\mathrm{wk}},g^{\mathrm{wk}})$; $\mathcal H_{-1}$ & the working prior; the candidate tables it ranges over; the historical record & Section~\ref{sec:setup-sim} \\
$\mathcal F_t$ & the predecision history, all that is available before the action at period $t$ & Section~\ref{sec:setup-sim} \\
$(\hat m_t,\hat g_t)$ & the working model's response and transition tables, the posterior means of the candidate tables & Section~\ref{sec:setup-sim} \\
$n_0$; $N_t(x,a)$, $R_t(x,a)$ & the prior weight; the visit count and reward sum of a cell before period $t$ & Section~\ref{sec:setup-sim} \\
$s_t(x,a)$; $s_0$ & the tunnel width at a cell; its initial value, shared by all cells & Section~\ref{sec:setup-sim} \\
$\nu^{\mathrm{sim}}(\cdot\mid x)$ & the conditional latent distribution over the historical record (Assumption~\ref{asm:calib}) & Section~\ref{sec:assumptions} \\
$d_X$, $d_A$, $d_U$, $d$; $z=(x,a,u)$; $W_1$; $\nu_{-1}$ & metrics on states, actions, latent values, and triples; a full-state triple; the Wasserstein-$1$ distance; the calibration distribution $\nu^{\mathrm{sim}}$ & Section~\ref{sec:assumptions} \\
$L_\theta$, $L_f$, $L_m$; $\lambda$ & Lipschitz constants of $\theta$, $f_X$, and $m_t$ (Assumption~\ref{asm:smooth}); the one-period drift bound on $\nu_t$ (Assumption~\ref{asm:drift}) & Section~\ref{sec:assumptions} \\
$B_\theta$; $\mu_X$ & the magnitude bound $\max_z\lVert\theta(z,\cdot)\rVert_\infty$ on the responses (Appendix~\ref{app:testvalue}); the $\mathcal X$-marginal of the initial distribution of the full state (Proposition~\ref{prop:simgap}) & Appendix~\ref{app:moved} \\
$\bar H_{i,t}$ & a unit's observed history before its period-$t$ action (Lemma~\ref{lem:nonmarkov}) & Section~\ref{sec:source-target} \\
$\pi_{i,t}:\mathcal X\to\Delta(\mathcal A)$; $\pi_t$; $\bar\pi\in\bar\Pi$ & a unit's period-$t$ policy; all units' period-$t$ policies; a policy sequence in the memoryless class & Section~\ref{sec:problem-controls} \\
$\mathsf P[i,t,x,a]$ & the policy tensor $\pi_{i,t}(a\mid x)$, the data-collection decision & Section~\ref{sec:problem-controls} \\
$S$; $\hat\pi_t$; $\pi_t$ & the selector, a map from a working model to a value-maximizing policy; the plug-in policy at the last review; the deployed policy & Section~\ref{sec:problem-controls} \\
$V_T(\bar\pi)$, $V^{\mathrm{sim}}_T$ & undiscounted real and simulated value of a policy sequence or a strategy & Section~\ref{sec:problem-value} \\
$\bar\pi^{\mathrm{sim}}$; SOP & the simulator-optimal policy, the maximizer of $V^{\mathrm{sim}}_T$; the strategy that deploys it and never reviews & Section~\ref{sec:problem-value} \\
\bottomrule
\end{tabular}
\end{table}

\begin{table}[p]
\centering
\caption{Notation: the symbols of the main text (Part~2 of~3).}
\label{tab:notation-b}
\footnotesize
\renewcommand{\arraystretch}{0.85}
\setlength{\tabcolsep}{4pt}
\begin{tabular}{@{}>{\raggedright\arraybackslash}p{0.30\textwidth}>{\raggedright\arraybackslash}p{0.51\textwidth}>{\raggedright\arraybackslash}p{0.14\textwidth}@{}}
\toprule
symbol & meaning & defined in \\
\midrule
$\operatorname{margin}(\sigma)$; $\sigma$ & the margin of a strategy over SOP, $100\,(V_T(\sigma)-V_T(\mathrm{SOP}))/V_T(\mathrm{SOP})$; a generic strategy & Section~\ref{sec:problem-value} \\
$\sigma^{\mathrm{sop}}$, $\sigma^{\mathrm{asop}}_\kappa$; $\kappa$ & the strategies SOP and A-SOP; the review cadence of A-SOP, a number of periods, distinct from the neighbour weight $\varkappa$ & Section~\ref{sec:problem-value} \\
$\operatorname{price}_k(\sigma)$, $C$ & the simulated price of the departures of epoch $k$; the budget on that price & Section~\ref{sec:problem-budget} \\
$d=(\mathcal C,\mathcal T_d,\mathcal A_d)$ & a design: cohort, window, and tested actions (Definition~\ref{def:design}) & Section~\ref{sec:problem-designs} \\
$\epsilon_m$, $\epsilon_g$; $\bar\theta$ & the largest response and transition discrepancies over the cells a policy covers; $\max_{x,a}|\theta^{\mathrm{sim}}(x,a)|$ & Section~\ref{sec:record-alone} \\
$a^\dagger$, $b_0$; $\bar\pi^*$ & a fixed action at a state and its simulated response $\theta^{\mathrm{sim}}(x,a^\dagger)$ (Corollary~\ref{cor:lockin}); a maximizer of $V_T$ over $\bar\Pi$ & Section~\ref{sec:record-alone} \\
$a_0$, $a_1$; $D$, $\Delta_0$, $T_0$ & the two actions of the hidden-drift instance; its drift rate, calibrated gap, and crossing time $\Delta_0/D$ (Theorem~\ref{thm:hiddendrift}) & Section~\ref{sec:record-alone} \\
$P$; $M$ & a strategy's $\mathcal F_t$-measurable assignment probability; the history-conditional response at a cell & Section~\ref{sec:ident-cells} \\
$u_1,\dots,u_K$; $K$; $\Delta_K$ & the regimes, the finite values of the latent covariate; the number of regimes; the set of regime distributions & Section~\ref{sec:ident-structure} \\
$\nu$; $\hat\nu$ & the regime distribution, a probability vector over the regimes; the regime distribution estimated by S-SOP & Section~\ref{sec:ident-structure} \\
$V$; $\Theta_V$; $G=[\Theta_V;\mathbf 1^\top]$ & the visited cells; the matrix of their regime responses; $\Theta_V$ with a row of ones appended & Section~\ref{sec:ident-structure} \\
$\mu$; $\mathcal I(\nu)$ & a candidate regime distribution; the identified set of a regime distribution & Section~\ref{sec:ident-structure} \\
$\sigma$ (Appendix~\ref{app:regimeid}); $P_{\ker G}$, $P_{\mathrm{row}(G)}$ & the restricted modulus, the smallest singular value of $\Theta_V$ on the sum-zero subspace (not the strategy symbol); the projections onto the kernel and row space of $G$ & Section~\ref{sec:ident-structure} \\
$J_t(\mu)$; $\bar m_t(z)$ & the penalized least-squares objective of the fit (S-SOP); the record's mean at a cell $z$ & Section~\ref{sec:ident-structure} \\
$e$ (Theorem~\ref{thm:regimeid}); $\widehat m_V$ & the additive error on the visited means; the observed visited means $\Theta_V\nu+e$ & Section~\ref{sec:ident-structure} \\
$\tilde\mu_0$; $\theta^{\parallel}$, $\theta^{\perp}$ & the point of $\nu+\ker G$ nearest $\nu^{\mathrm{sim}}$; the row-space and kernel parts of a cell's regime response & Section~\ref{sec:ident-structure} \\
$Z$; $G_{V\cup Z}$ & the cells a design tests; the matrix $G$ built from $V\cup Z$ & Section~\ref{sec:ident-structure} \\
$N_r$, $\bar Y_r$ (Lemma~\ref{lem:window}); $\varepsilon$ & the number of a cohort's units at a cell in period $r$ and their mean outcome; a tolerance & Section~\ref{sec:ident-structure} \\
$d_r(z)$, $q_r(z)$, $w_r(z)$, $w$ & occupancy of a cell, its weighted observation mass $n\,d_r(z)\,w_r(z)$, its per-observation weight ($w$ for a test observation, $1$ otherwise), and the test weight & Section~\ref{sec:testvalue} \\
$\omega$; $V_t$; $\omega_{t,r}(z)$; $\bar a_t(z)$ & the memory rate; the cells visited before $t$; the age weights; the mean age of the record at a cell & Section~\ref{sec:testvalue} \\
$\varrho$, $\epsilon$, $\mathcal N$, $\phi$ (Appendix~\ref{app:testvalue}); $P_{\pi,\mu}$ & the constants and set of the minorization condition and its common distribution (not the visible share); the frozen transition kernel & Section~\ref{sec:testvalue} \\
$g_\mu(\pi)$, $h_\mu(\pi)$, $r_{\pi,\mu}$; $L_h$, $\bar h$ & the average reward, bias, and reward of the chain frozen at $\mu$; the Lipschitz constant and supremum of the bias & Section~\ref{sec:testvalue} \\
$\pi_0$, $\pi_1$; $t_0$, $t^E$, $t^A$ & the two deployed policies of the switch condition; the test period and the switch periods of the experiment and of adaptation & Section~\ref{sec:testvalue} \\
$\pi[\mu]$, $\pi^*$, $\tilde\pi$; $D(\pi)$ & the plug-in policy of a regime distribution, the true-distribution policy, and an alternative policy; the decision region of a policy & Section~\ref{sec:testvalue} \\
$\mathrm{price}$, $\mathrm{Gain}$; $\kappa_2$, $\delta$, $\kappa_0$, $\kappa_1$, $\delta'$ & the price and gain of a designed test; the horizon-free constants and average-reward gaps of Theorem~\ref{thm:testvalue} & Section~\ref{sec:testvalue} \\
$G_V$, $G_Z$; $a$, $b$, $e$, $t^\ast$ (Corollary~\ref{cor:switch}) & the matrices $G$ built from $V$ and from $Z$; the hyperplane normal, offset, and unit drift direction in Corollary~\ref{cor:switch}, and the crossing period & Section~\ref{sec:testvalue} \\
$t_k=t_0+kh$, $h$ & the review schedule and the review interval of the procedure & Section~\ref{sec:proc-params} \\
\bottomrule
\end{tabular}
\end{table}

\begin{table}[p]
\centering
\caption{Notation: the symbols of the main text (Part~3 of~3).}
\label{tab:notation-c}
\footnotesize
\renewcommand{\arraystretch}{0.85}
\setlength{\tabcolsep}{4pt}
\begin{tabular}{@{}>{\raggedright\arraybackslash}p{0.30\textwidth}>{\raggedright\arraybackslash}p{0.51\textwidth}>{\raggedright\arraybackslash}p{0.14\textwidth}@{}}
\toprule
symbol & meaning & defined in \\
\midrule
$\mu_0$, $\tau$, $b$ & the centre and centre weight toward which the working model is shrunk; the bonus, a multiple of the tunnel width & Section~\ref{sec:proc-params} \\
$\hat m$, $\hat g$; $R$, $N$, $G$ (Appendix~\ref{app:simulation}) & the procedure's working response and transition tables; the discounted reward, visit, and next-state masses at a cell (this $G$ is distinct from the regime-response matrix $G$) & Section~\ref{sec:proc-params} \\
$\varkappa$ & the weight at which a zone's neighbours enter its sums & Section~\ref{sec:proc-params} \\
$\operatorname{gain}(d)$, $\operatorname{price}(d)$ & the simulated gain and price of a candidate test & Section~\ref{sec:proc-tests} \\
$\mathcal F_{t_k}$; $\hat{\bar\pi}$; $\mathcal P_k$ & the predecision history at review $k$; the plug-in policy sequence for the remaining periods; the admitted package of designs & Section~\ref{sec:proc-loss} \\
$\mathcal D_k$, $D_k$ & the tested units and the departure mass, the number of unit-periods on which the procedure departs & Section~\ref{sec:proc-loss} \\
$\rho$; $\Lambda$; $M$, $R$ & a policy tensor; a reference in $\{M,R\}$; the mean world and the target distribution at the review & Section~\ref{sec:proc-loss} \\
$V^\Lambda(\rho)$; $\operatorname{loss}^{M}$, $\operatorname{loss}^{R}$ & the value of a tensor over the remaining periods; the loss under the mean world and the true loss & Section~\ref{sec:proc-loss} \\
$w^{M}_{i,u,t'}(x')$, $\bar V$ & the mean-world continuation value from a state; the value span, its largest range over states & Section~\ref{sec:proc-loss} \\
$w^{R}_{i,a,t'}(x')$, $\bar G$; $\hat\theta$ & the target's continuation value; its largest gap from the mean world's; $\max_{x,a}|\hat m(x,a)|$ & Section~\ref{sec:proc-loss} \\
$F$, $B_0$; $r^\pi_u$, $\hat r_u$; $\underline s_t$, $B^\pi_t$ & the per-period fee and the initial balance; the procedure's and the $C=0$ rewards; the $C=0$ balance and the procedure's balance & Section~\ref{sec:proc-loss} \\
$\Sigma^{(k)}$, $\ell_k(t)$ & the strategy that departs only through epoch $k$; the reward that the departures of epoch $k$ cost & Section~\ref{sec:proc-loss} \\
$P_0$, $P_1$ & the unchanged and changed environments of the detection trade-off & Section~\ref{sec:proc-detect} \\
$Q^\ast_t$, $V^\ast_t$ & the remaining-horizon action value and value of a unit under the calibrated tables & Section~\ref{sec:proc-detect} \\
$Z^\ast$; $p_{\min}$ & the hidden change set, the cells at which $P_1$ differs from $P_0$; the smallest per-period suboptimality on it & Section~\ref{sec:proc-detect} \\
$\xi$; $N_\sigma$ & a strategy's internal randomization; its departure count into $Z^\ast$ & Section~\ref{sec:proc-detect} \\
$E$; $\alpha$, $1-\vartheta$ & a detector, an event on the record and $\xi$; its size and power & Section~\ref{sec:proc-detect} \\
$\operatorname{loss}(\sigma)$; $\bar k$ & the no-change loss $V_T(\hat{\bar\pi})-V_T(\sigma)$ under $P_0$; the number of reviews of the procedure & Section~\ref{sec:proc-detect} \\
$r_j$; $k^\ast$, $t_{k^\ast}$ & the period of the $j$-th test; the first review index at which the plug-in policy switches, and the switch period & Section~\ref{sec:proc-detect} \\
$u_{r_j}$; $\bar\pi^{(s)}$ & the latent state at the $j$-th test; the fixed policy that assigns $a_0$ before period $s$ and $a_1$ from $s$ on & Section~\ref{sec:proc-detect} \\
$V^{\mathrm{proc}}$, $V^{\mathrm{sop}}$, $V^{\mathrm{fix}}$ & the values of the procedure, of SOP, and of the best fixed switch policy in the hidden-drift instance & Section~\ref{sec:proc-detect} \\
$n_{\mathrm{zone}}$ & the number of machines in a zone of the benchmark & Section~\ref{sec:sim-benchmark} \\
$\beta$; $\nu_\infty$ & the speed of the shift; the shifted regime distribution of a zone & Section~\ref{sec:sim-worlds} \\
$\phi$ & the visible share, the fraction of the shift placed in the bust zone & Section~\ref{sec:sim-worlds} \\
$c_0$ & the simulated no-change price of a one-day test, the per-review budget of SEP & Section~\ref{sec:sim-strategies} \\
\bottomrule
\end{tabular}
\end{table}

\begin{table}[p]
\centering
\caption{Notation used only in the appendix.}
\label{tab:notation-app}
\footnotesize
\renewcommand{\arraystretch}{0.85}
\setlength{\tabcolsep}{4pt}
\begin{tabular}{@{}>{\raggedright\arraybackslash}p{0.30\textwidth}>{\raggedright\arraybackslash}p{0.51\textwidth}>{\raggedright\arraybackslash}p{0.14\textwidth}@{}}
\toprule
symbol & meaning & defined in \\
\midrule
$q$; $\gamma$; $\Gamma_T(\gamma)$, $\Lambda_T(\gamma)$ & a fixed strategy; the discount; $\sum_{t\le T}\gamma^t$ and $\sum_{t\le T}t\gamma^t$ \\
$\delta_t(x,a)$; $\TV(\nu,\nu')$ & the simulator's error at a cell, $m_t(x,a)-\theta^{\mathrm{sim}}(x,a)$; the total variation distance between two distributions (Observation~\ref{prop:delta}, Appendix~\ref{app:moved}) \\
$\Xi_t$; $\mathcal L_{<t}$ & the working posterior before period $t$; selected likelihood factors, Eqs.~(\ref{eq:filtration})--(\ref{eq:preposterior}) \\
$\mathcal M_{L_m}$ & the class of $L_m$-Lipschitz tables \\
$\mathcal Z$; $\mathcal R$; $\mathcal K_d$ & the set of cells; the targeted cells; the covered cell-period pairs \\
$p^{\bar\pi}_{i,t}(x)$ & distribution of the observed state under a fixed policy sequence (Lemma~\ref{lem:flow}) \\
$\bar{\mathcal H}_{i,t}$; $\xi_{i,t}$ & the $\sigma$-field of a unit's realized quantities through $(X_{i,t},U_{i,t},A_{i,t})$; per-unit assignment randomization (Lemma~\ref{lem:pathlaw}) \\
$\ell$, $\alpha$ & bandwidth and assignment parameters of the earlier sweeps (Appendix~\ref{app:results}) \\
\bottomrule
\end{tabular}
\end{table}
\section{Additional Results}
\label{app:results}
In this appendix, we report each number quoted in Section~\ref{sec:sim}. Throughout, a margin is the margin over SOP (the fixed plug-in policy) in percent of SOP's value. We first present the six strategies on the full world sets, then the identification and capacity checks, next the sparse-and-expensive and no-change numbers together with the bankruptcy study, and finally the sampled-fleet check of the tie band.

We report the six strategies on the full world sets introduced in Section~\ref{sec:sim}. We call the full world set of a family its \emph{cast}. Table~\ref{tab:cast} gives, for each fleet and family, which strategies lead outright and how many worlds tie (the map's verdicts), together with three summaries per strategy (hidden-row median, grid median, no-change margin). An \emph{outright lead} is a margin more than two percent above the best strategy of another \emph{mode}, where the four modes are SOP, adaptation, exploration, and experimentation; the exploiting class thus splits into SOP and adaptation. Figure~\ref{fig:map27} draws the map at $27$ machines under the same rule as Figure~\ref{fig:breadth-map}. In the stress family at $999$ machines, SEP or SEP-when has the highest value in $16$ of the $25$ worlds and Thompson sampling in $8$, with $10$ and $6$ of these leads beyond the tie band; at $27$ machines the experimenting strategies have the highest value in $22$ worlds. At $27$ machines SEP leads outright more often in the stress family ($15$ outright leads against $8$ at $999$). This is because adaptation (S-SOP) gains more than SEP as machines per zone are added: in the sparse-and-expensive family adaptation has $11$ outright leads at $27$ machines and $4$ at $999$. Pooling also explains why SOP has the highest value in the visible rows ($\phi\ge1/2$) of the sparse-and-expensive family at the two slowest speeds. In those worlds, at both fleet sizes, pooled adaptation sits $1$ to $3$ percent below SOP, while the same refit with the neighbour weight set to zero sits $1$ to $6$ percent above it. This difference arises because the pooled refit borrows part of the bust zone's demand drop into its neighbours' fits. SEP-when runs at a weekly cadence on the same monthly budget as SEP under the admission condition \emph{net gain} (it tests only when the simulated gain exceeds the simulated price). Its hidden-row margins stay within $8$ percent of SEP's in each family, and its no-change losses match SEP's at $999$ machines but grow at $27$ ($-3.5$ to $-5.9$ against $-1.9$ to $-2.6$ percent in the stress family).

\begin{table*}[p]
\centering
\caption{The six strategies on the full world sets. For each fleet and family: the strategies with outright leads (a lead of more than $2$ percent over the best strategy of another mode, SOP counted) and the number of tied worlds; then, per strategy, the hidden-row median, the median over the $25$ worlds, and the margin in the no-change world at $\beta=0.01$, all as margins over SOP in percent of SOP's value at $T=1024$ (Thompson: seed medians over $5$ seeds).}
\label{tab:cast}
\small
\setlength{\tabcolsep}{4pt}
\begin{tabular}{@{}>{\raggedright\arraybackslash}p{19em}lrrr@{}}
\toprule
fleet, family; verdicts & strategy & hidden row & grid & no change \\
\midrule
27 machines, baseline (sparse, cheap); outright: none; ties 25 & S-SOP & $-0.1$ & $+0.1$ & $-0.5$ \\
 & UCB & $+1.3$ & $+0.9$ & $+0.7$ \\
 & Thompson & $-11.9$ & $-12.2$ & $-12.5$ \\
 & SEP & $+0.3$ & $+0.3$ & $-0.1$ \\
 & SEP-when & $-0.4$ & $-0.3$ & $-0.8$ \\
\midrule
27 machines, broad-and-cheap; outright: none; ties 25 & S-SOP & $+2.9$ & $+2.8$ & $-0.5$ \\
 & UCB & $+3.1$ & $+3.1$ & $+0.7$ \\
 & Thompson & $-11.5$ & $-11.3$ & $-12.5$ \\
 & SEP & $+2.7$ & $+2.7$ & $-0.1$ \\
 & SEP-when & $+2.8$ & $+2.6$ & $-0.8$ \\
\midrule
27 machines, sparse-and-expensive; outright: SOP 7, S-SOP 11; ties 7 & S-SOP & $+16.1$ & $+2.2$ & $+0.7$ \\
 & UCB & $-9.0$ & $-22.2$ & $-37.3$ \\
 & Thompson & $-60.8$ & $-85.5$ & $-111.6$ \\
 & SEP & $+14.0$ & $-0.6$ & $-2.6$ \\
 & SEP-when & $+12.8$ & $-2.5$ & $-5.2$ \\
\midrule
27 machines, stress (broad, expensive); outright: UCB 1, SEP 15, SEP-when 1; ties 8 & S-SOP & $+48.6$ & $+42.9$ & $+0.7$ \\
 & UCB & $+32.8$ & $+34.4$ & $-37.3$ \\
 & Thompson & $+31.5$ & $+27.3$ & $-111.6$ \\
 & SEP & $+73.4$ & $+57.7$ & $-2.6$ \\
 & SEP-when & $+65.4$ & $+53.0$ & $-5.2$ \\
\midrule
999 machines, baseline (sparse, cheap); outright: none; ties 25 & S-SOP & $-0.1$ & $+0.1$ & $-0.3$ \\
 & UCB & $-1.7$ & $-1.5$ & $-3.2$ \\
 & Thompson & $-8.1$ & $-8.4$ & $-9.2$ \\
 & SEP & $-0.0$ & $+0.2$ & $-0.1$ \\
 & SEP-when & $+0.2$ & $+0.1$ & $-0.3$ \\
\midrule
999 machines, broad-and-cheap; outright: none; ties 25 & S-SOP & $+2.9$ & $+2.6$ & $-0.3$ \\
 & UCB & $+3.1$ & $+3.1$ & $-3.2$ \\
 & Thompson & $-6.3$ & $-6.5$ & $-9.2$ \\
 & SEP & $+2.7$ & $+2.6$ & $-0.1$ \\
 & SEP-when & $+2.7$ & $+2.5$ & $-0.3$ \\
\midrule
999 machines, sparse-and-expensive; outright: SOP 6, S-SOP 4; ties 15 & S-SOP & $+16.0$ & $+2.1$ & $+0.7$ \\
 & UCB & $-5.4$ & $-13.9$ & $-26.4$ \\
 & Thompson & $-42.4$ & $-72.7$ & $-94.7$ \\
 & SEP & $+11.1$ & $+0.3$ & $-1.1$ \\
 & SEP-when & $+13.6$ & $+0.6$ & $-1.6$ \\
\midrule
999 machines, stress (broad, expensive); outright: Thompson 6, SEP 8, SEP-when 2; ties 9 & S-SOP & $+50.9$ & $+43.0$ & $+0.7$ \\
 & UCB & $+34.8$ & $+34.5$ & $-26.4$ \\
 & Thompson & $+49.7$ & $+44.3$ & $-94.7$ \\
 & SEP & $+72.0$ & $+55.9$ & $-1.1$ \\
 & SEP-when & $+67.0$ & $+54.6$ & $-1.6$ \\
\bottomrule
\end{tabular}
\end{table*}

\begin{figure*}[t]
\centering
\includegraphics[width=\textwidth]{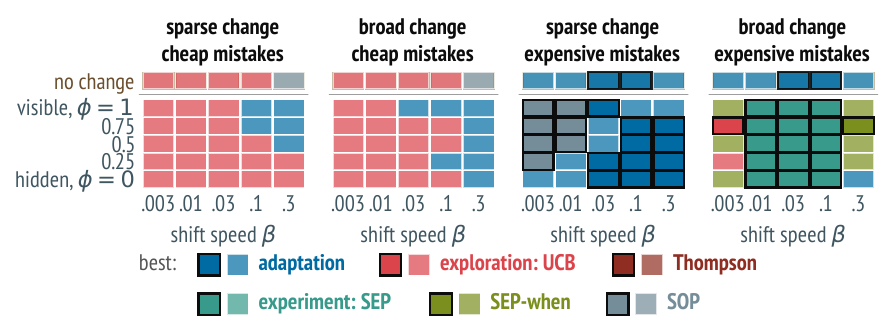}
\caption{The map of Figure~\ref{fig:breadth-map} at $27$ machines (three per zone), population limit, same tie band and colours.}
\label{fig:map27}
\end{figure*}

We next report the identification check, the test weight, and the budget behind Section~\ref{sec:sim}. Here each margin is the margin over SOP of \eqref{eq:margin}, in percent of SOP's value in the same world. Under the cheap economics, the largest margin that any of the five strategies attains in any of the $50$ worlds at $999$ machines is $+0.6$ percent in the baseline family and $+4.0$ percent in the broad-and-cheap family. Over the same worlds, the largest lead over the better of SOP and adaptation is $+1.4$ percent. Doubling the machines' capacity from $3$ to $6$ slots raises the margin over SOP of a policy re-solved from the true regime distribution every four days only from $0.8$ to $2.0$ percent in the baseline boom worlds at $\beta=0.03$ and $0.1$ over $512$ days.

Under the lean economics the deployed policy visits three or four cells of the boom zone. The regime-response matrix $G=[\Theta_V;\mathbf 1^\top]$ of those cells, formed by stacking their regime responses over a row of ones, has rank $2$ with singular values $(4.22,1.22,0)$, and the direction it leaves unidentified (the unit vector spanning the matrix's kernel) is $(-0.23,+0.79,-0.57)$ over the regimes low, mid, and high. For the stress boom distribution $(0.05,0.25,0.70)$, the point of the identified set nearest $\nu^{\mathrm{sim}}$ is $(0.029,0.324,0.647)$. The plug-in policy computed from that point agrees with the policy computed from the true distribution at all $20$ states, both in the boom corner and in the background zones.

In the sparse-and-expensive family the cast (full-world-set) adaptation arm, pooled with neighbours at half weight, attains $+2.7$, $+12.3$, $+17.2$, $+17.8$, and $+16.0$ percent on the hidden row at $999$ machines across the five speeds; without pooling we obtain $+11.1$, $+20.0$, $+32.9$, $+35.3$, and $+33.8$ for the same refit, and SEP attains $+4.0$, $+9.8$, $+13.2$, $+12.4$, and $+11.1$. SEP's margin is lower in this family because two of its three tests lie in zones that the change did not reach.

In the stress no-change worlds at $999$ machines, the loss of SEP relative to adaptation is $1.9$, $1.7$, $2.0$, and $2.0$ percent of SOP's value at $\beta=0.01$, $0.03$, $0.1$, and $0.3$, and SEP gains $0.5$ percent at $\beta=0.003$. In the same worlds the simulated price of the tests SEP ran, the total simulated price it spent over the horizon, was $1.57$, $1.66$, $1.81$, $1.81$, and $1.43$ percent of SOP's value. At four of the five speeds the loss therefore stays close to the simulated test price; in Proposition~\ref{prop:purse} we bound the real loss by that price plus a working-model error term, which is positive here because the stock-regime dependence described in Appendix~\ref{app:simulation} breaks exact calibration. In the bankruptcy study a strategy starts with a balance of thirty days of the plug-in policy's mean profit, pays a daily fee equal to the stated fraction of that profit, and goes bankrupt once its balance is negative on ten consecutive days. Table~\ref{tab:ruin-family} splits the bankruptcy study by family at the fee of $0.9$.

\begin{table*}[p]
\centering
\caption{Simulator-aided design against designs chosen without the simulator. For each fleet and family we report SEP and the two design controls defined in Section~\ref{sec:sim-strategies}, together with the control that also draws the fractions. The three summaries per strategy are the hidden-row median, the median over the $25$ worlds, and the margin in the no-change world at $\beta=0.01$, all as margins over SOP in percent of SOP's value at $T=1024$ in the population limit (controls: seed medians over $5$ seeds). The last column counts the worlds, of $25$, in which SEP leads the control by more than the $2$ percent tie band / trails it by more than the band. The last block repeats the stress family on sampled fleets of $27$ machines ($10$ seeds, seed medians; the counts use leads paired by seed and count a lead when it exceeds the larger of $2$ percent and two paired standard errors).}
\label{tab:design-controls}
\small
\setlength{\tabcolsep}{4pt}
\begin{tabular}{@{}>{\raggedright\arraybackslash}p{9.5em}lrrrc@{}}
\toprule
fleet, family & strategy & hidden row & grid & no change & SEP ahead / behind \\
\midrule
999 machines, baseline (sparse, cheap) & SEP & $0.0$ & $+0.2$ & $-0.1$ &  \\
 & uniform test & $+0.1$ & $+0.2$ & $-0.1$ & $0$ / $0$ \\
 & random design & $0.0$ & $+0.1$ & $-0.2$ & $0$ / $0$ \\
 & random, fractions drawn & $0.0$ & $+0.2$ & $-0.2$ & $0$ / $0$ \\
\midrule
999 machines, broad-and-cheap & SEP & $+2.7$ & $+2.6$ & $-0.1$ &  \\
 & uniform test & $+2.8$ & $+2.5$ & $-0.1$ & $0$ / $0$ \\
 & random design & $+2.7$ & $+2.4$ & $-0.2$ & $0$ / $0$ \\
 & random, fractions drawn & $+2.8$ & $+2.5$ & $-0.2$ & $0$ / $0$ \\
\midrule
999 machines, sparse-and-expensive & SEP & $+11.1$ & $+0.3$ & $-1.1$ &  \\
 & uniform test & $+12.8$ & $-0.3$ & $-3.6$ & $3$ / $1$ \\
 & random design & $+12.6$ & $-0.8$ & $-3.0$ & $2$ / $1$ \\
 & random, fractions drawn & $+13.8$ & $+0.3$ & $-2.7$ & $0$ / $4$ \\
\midrule
999 machines, stress (broad, expensive) & SEP & $+72.0$ & $+55.9$ & $-1.1$ &  \\
 & uniform test & $+68.2$ & $+55.6$ & $-3.6$ & $3$ / $1$ \\
 & random design & $+72.8$ & $+55.3$ & $-3.0$ & $0$ / $6$ \\
 & random, fractions drawn & $+63.0$ & $+49.3$ & $-2.7$ & $10$ / $0$ \\
\midrule
27 machines, baseline (sparse, cheap) & SEP & $+0.3$ & $+0.3$ & $-0.1$ &  \\
 & uniform test & $0.0$ & $+0.2$ & $-0.1$ & $0$ / $0$ \\
 & random design & $-0.1$ & $+0.1$ & $-0.5$ & $0$ / $0$ \\
 & random, fractions drawn & $0.0$ & $+0.1$ & $-0.3$ & $0$ / $0$ \\
\midrule
27 machines, broad-and-cheap & SEP & $+2.7$ & $+2.7$ & $-0.1$ &  \\
 & uniform test & $+2.7$ & $+2.7$ & $-0.1$ & $0$ / $0$ \\
 & random design & $+2.7$ & $+2.7$ & $-0.5$ & $0$ / $0$ \\
 & random, fractions drawn & $+2.8$ & $+2.7$ & $-0.3$ & $0$ / $0$ \\
\midrule
27 machines, sparse-and-expensive & SEP & $+14.0$ & $-0.6$ & $-2.6$ &  \\
 & uniform test & $+12.9$ & $-0.4$ & $-3.0$ & $0$ / $2$ \\
 & random design & $+11.9$ & $-0.5$ & $-3.3$ & $1$ / $3$ \\
 & random, fractions drawn & $+14.2$ & $+0.6$ & $-2.3$ & $0$ / $5$ \\
\midrule
27 machines, stress (broad, expensive) & SEP & $+73.4$ & $+57.7$ & $-2.6$ &  \\
 & uniform test & $+66.4$ & $+53.3$ & $-3.0$ & $13$ / $0$ \\
 & random design & $+71.4$ & $+55.6$ & $-3.3$ & $10$ / $0$ \\
 & random, fractions drawn & $+60.7$ & $+46.8$ & $-2.3$ & $14$ / $0$ \\
\midrule
27 machines, stress, sampled fleets & SEP & $+70.2$ & $+60.4$ & $-1.5$ &  \\
 & uniform test & $+69.7$ & $+58.3$ & $-0.7$ & $5$ / $0$ \\
 & random design & $+70.0$ & $+61.3$ & $-3.0$ & $0$ / $0$ \\
\bottomrule
\end{tabular}
\end{table*}

\begin{table}[t]
\centering
\caption{Bankruptcy probability at a daily fee of $0.9$ of the plug-in policy's mean profit, by family ($27$ machines; $26$ worlds per family, the $25$-world grid and the family's no-change world at $\beta=0.01$; $10$ seeds, $5$ for SEP-when), with one standard error.}
\label{tab:ruin-family}
\small
\setlength{\tabcolsep}{5pt}
\begin{tabular}{@{}lcc@{}}
\toprule
strategy & sparse-and-expensive & stress \\
\midrule
SOP & $0.138\pm0.021$ & $0.238\pm0.026$ \\
adaptation (S-SOP) & $0.065\pm0.015$ & $0.204\pm0.025$ \\
UCB & $0.862\pm0.021$ & $0.254\pm0.027$ \\
Thompson & $0.412\pm0.031$ & $0.262\pm0.027$ \\
SEP & $0.115\pm0.020$ & $0.215\pm0.025$ \\
SEP-when & $0.138\pm0.030$ & $0.223\pm0.037$ \\
\bottomrule
\end{tabular}
\end{table}
\begin{figure}[t]
\centering
\includegraphics[width=3.25in]{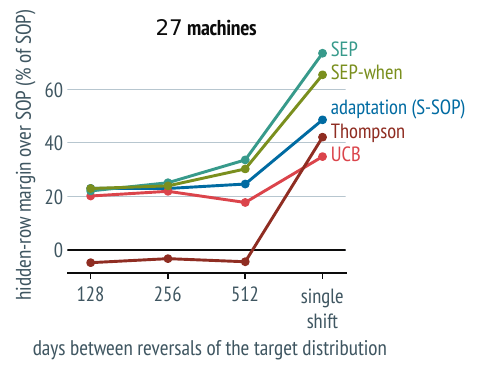}
\caption{Recurrent shifts, six strategies, in the stress family (a broad shift with expensive mistakes) at $27$ machines in the population limit, $T=1024$: the hidden-row margin over SOP in percent of SOP's value against the number of days between reversals of the target distribution, with SOP as the zero line; the right-most position is the single-shift stress family. The target distribution alternates between the shifted and the calibration distribution at a fixed period, and the chain mixes toward it at rate $\beta$. Each point is the median over $\beta\in\{0.01,0.03,0.1\}$ on the hidden row ($\phi=0$); the shift toward the boom is hidden, the reversal is a visible demand fall, and the no-change worlds are unaffected. The lead of the experiment over adaptation is $25$ percent for the single shift, $9$ percent at a reversal period of $512$ days, $2$ percent at $256$ days, and absent at $128$ days; Thompson sampling stays below SOP at all finite reversal periods.}
\label{fig:recurrent}
\end{figure}
We now report the sampled-fleet check behind the tie band set in Section~\ref{sec:sim-strategies}.\label{app:sampled} We ran it on five procedures of an earlier sweep: cell-wise adaptation (A-SOP), structural adaptation without pooling across zones, discounted UCB, a pooled randomized test, and the budgeted procedure at the setting of SEP without pooling, which we label the designed experiment in Figure~\ref{fig:sampled}. A sampled fleet draws each machine's demand and regime path. We simulated fleets of $27$, $999$, and $9{,}999$ machines in the nine worlds of the stress family with $\beta\in\{0.003,0.03,0.1\}$ and $\phi\in\{0,0.5,1\}$, using $10$ seeds at $27$ machines and $6$ at the two larger fleets. We paired each procedure with SOP of the same seed by common random numbers. At $999$ and $9{,}999$ machines, the medians over the nine worlds of structural adaptation, discounted UCB, and the designed experiment lie within $3.6$ percent of SOP's value of the population limit. The paired difference between the designed experiment and structural adaptation, which decides the stress-family squares of the map, agrees with the population limit within $2.0$ percent in each world at $999$ machines and within $1.1$ percent at $9{,}999$ machines. At $27$ machines, however, structural adaptation is trapped in a suboptimal policy (lock-in) in the population limit but escapes the lock-in in sampled fleets. Its sampled median then exceeds the population limit by $17.7$ percent, and the designed experiment's lead over it falls to between $0.8$ and $4.5$ percent. The check therefore supports the tie band of $2$ percent at $999$ machines for these five procedures.

\begin{figure}[t]
\centering
\includegraphics[width=\textwidth]{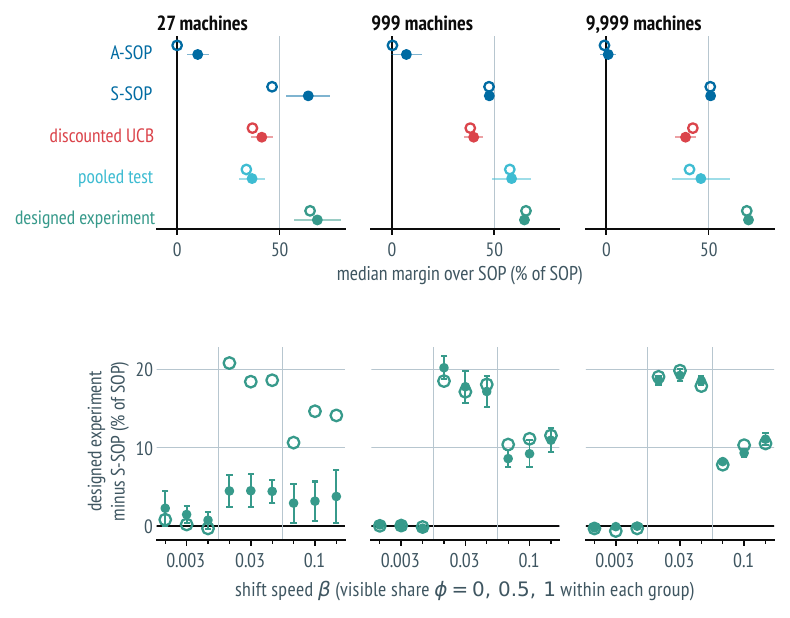}
\caption{Sampled fleets against the population limit for five procedures of an earlier sweep, stress family, $T=1024$, at $27$, $999$, and $9{,}999$ machines ($3$, $111$, and $1{,}111$ per zone). The rows are cell-wise adaptation (A-SOP), structural adaptation without pooling (S-SOP), discounted UCB, the pooled randomized test with budget $c_0$, and the designed experiment. A sampled fleet draws each machine's demand and regime path; its margin is paired with sampled SOP of the same seed by common random numbers, over $10$ seeds at $27$ machines and $6$ at the two larger fleets. Top: the median over the nine focus worlds of the margin over SOP, in percent of SOP's value. Filled markers are sampled fleets, with bars of the largest twice-standard-error among the nine worlds, and hollow markers, drawn just above, are the population limit. Bottom: the paired difference between the designed experiment and structural adaptation in each world, with twice its standard error; within each shift speed $\beta$ the three points are the visible shares $\phi=0$, $0.5$, and $1$, from left to right.}
\label{fig:sampled}
\end{figure}
\section{Formal Setup and Procedure Definitions}
\label{app:moved}

In this appendix, we present the formal definitions behind the main text, namely the latent process, the policies, the identification targets, and the updating of the working model, and we then define the procedures.

\paragraph{The simulator.}
We first recall from Section~\ref{sec:setup} the object on which the analysis rests.
\begin{definition}[Simulator]
\label{def:simulator}
A \emph{simulator} consists of two time-invariant tables and a noise model. The two tables are a response table $\theta^{\mathrm{sim}}(x,a)$ and an observed-state transition kernel $f^{\mathrm{sim}}(\cdot\mid x,a)$, and the noise model is a known mean-zero distribution $Q(\cdot\mid x,a)$ at each cell. For each unit the simulator is the generative process
\begin{equation}
    X_{t+1} \sim f^{\mathrm{sim}}(\cdot \mid X_t,A_t),\qquad
    Y_t=\theta^{\mathrm{sim}}(X_t,A_t)+\eta_t,\quad \eta_t\sim Q(\cdot\mid X_t,A_t),
    \label{eq:sim}
\end{equation}
in which the noises $\eta_t$ are independent across units and periods given the simulated paths.
\end{definition}

\paragraph{The three classes of strategies.}
We next define the three classes of strategies used in Section~\ref{sec:problem}, with $\hat\pi_t$ the plug-in policy, $S$ the selector, and designs $d=(\mathcal C,\mathcal T_d,\mathcal A_d)$ as in the main text.
\begin{definition}[Exploitation]
\label{def:modes}
A strategy is an \emph{exploiting strategy} if $\pi_t=\hat\pi_t$ at each $t$.
\end{definition}
\begin{definition}[Exploration]
\label{def:exploremode}
A strategy is an \emph{exploring strategy} if, for each unit, it has no prespecified periods and no stopping rule, and the policy $\pi_t$ is one of the two policies that follow. The first is the selector $S$, applied at each review to the working model after that model is perturbed by a multiple of the tunnel width, either optimistically or by sampling. The second is the plug-in policy mixed with uniform assignment.
\end{definition}
\begin{definition}[Experimentation]
\label{def:experimentmode}
A strategy is an \emph{experimenting strategy} if $\pi_t$ equals the plug-in policy except on the design slices, each of which satisfies four conditions. (i) Each design is declared before the first period of its window. (ii) The cohort is assigned by a distribution fixed in advance, uniform over the tested actions. (iii) An ex-ante budget bounds the simulated price, either per review or cumulatively. (iv) The working model at units outside the cohort is updated from the cohort's observations through a \emph{transport update}, a map from those observations to the working-model tables of the non-cohort units.
\end{definition}

\paragraph{Predecision timing and history-dependent values.}
Throughout, we carry the discount as a superscript, with the undiscounted case recovering the main text's value, write unit-stacked states and actions in boldface, and take the selected likelihood factors to be those from the fully observed tuples. We now define the predecision timing and the values used throughout Section~\ref{sec:problem}. Let $\mathcal H_{-1}$ be the historical record and define
\begin{equation}
\mathcal F_t:=\sigma\!\left(\mathcal H_{-1},\mathbf X_t,
 \{(X_{i,s},A_{i,s},Y_{i,s},X_{i,s+1}):i\in\mathcal V,\ 0\leq s<t\}\right).
\label{eq:filtration}
\end{equation}
A strategy $\sigma$ selects an $\mathcal F_t$-measurable period-$t$ policy $\pi_t^\sigma=(\pi_{i,t}^\sigma)_i$ and draws $\mathbf A_t$ conditionally from $\prod_i\pi_{i,t}^\sigma(\cdot\mid X_{i,t})$. Its real and simulated values are
\begin{equation}
V_T^\gamma(\sigma):=\mathbb E_\sigma\!\left[\sum_{t=0}^T\sum_i\gamma^tY_{i,t}\right],\qquad
V_T^{\gamma,\mathrm{sim}}(\sigma;\tilde m,\tilde g):=\mathbb E^{\mathrm{sim}}_{\sigma;\tilde m,\tilde g}\!\left[\sum_{t=0}^T\sum_i\gamma^tY_{i,t}\right].
\label{eq:strategyvalue}
\end{equation}
For strategy $\sigma$, let $\mathcal L^\sigma_{<t}$ be the selected likelihood factors from the fully observed tuples. Its predecision working posterior and posterior means are
\begin{equation}
\Xi_t^\sigma:=\Xi_0(\,\cdot\mid\mathcal L^\sigma_{<t}),\qquad (\hat m_t,\hat g_t):=\mathbb E_{\Xi_t^\sigma}(m^{\mathrm{wk}},g^{\mathrm{wk}}).
\label{eq:preposterior}
\end{equation}

\paragraph{The policy tensor.}
Recall from Section~\ref{sec:problem} the fixed memoryless, unit-separable class used in the main text,
\begin{equation}
    \bar{\Pi} := \big\{ \bar{\pi} = (\pi_0,\dots,\pi_T) \;:\; \pi_t = (\pi_{1,t},\dots,\pi_{n,t}),\ \ \pi_{i,t} : \mathcal{X} \to \Delta(\mathcal{A}) \big\},
    \label{eq:policyclass}
\end{equation}
Because $\mathcal{X}$ and $\mathcal{A}$ are finite, each $\bar{\pi} \in \bar{\Pi}$ is represented exactly by a tensor $\mathsf P \in [0,1]^{n \times (T+1) \times |\mathcal{X}| \times |\mathcal{A}|}$ with
\begin{equation}
    \mathsf P[i,t,x,a] = \pi_{i,t}(a \mid x), \qquad \sum_{a \in \mathcal{A}} \mathsf P[i,t,x,a] = 1 \ \ \text{for every } (i,t,x).
    \label{eq:tensor}
\end{equation}
The undiscounted real value, the main text's $V_T$, is
\begin{equation}
    V_T(\bar{\pi}) := \mathbb{E}\Big[\, \textstyle\sum_{t=0}^{T} \sum_{i \in \mathcal{V}} Y_{i,t} \,\Big],
    \label{eq:value}
\end{equation}
with $V^\gamma_T$ its discounted version and $V^{\mathrm{sim}}_T$ the same functional under the simulator. A design in the sense of Definition~\ref{def:design} fixes $\mathsf P[i,t,x,a] = 1/|\mathcal{A}_d|$ for $a\in\mathcal A_d$, and $0$ for the other actions, on its cohort-window slice. Within this class, the real-distribution fixed-policy maximizer satisfies
\begin{equation}
    \bar{\pi}^* \in \arg \max_{\bar{\pi} \in \bar{\Pi}} V^\gamma_T(\bar{\pi}),
    \label{eq:objective}
\end{equation}
which at $\gamma=1$ is the maximizer $\bar\pi^*$ of $V_T$ in the main text.
\noindent We first restate the five setup assumptions, then derive their immediate consequences, namely the additive reward form, the accumulated latent-drift bound, the simulator-error displacement, the simulation gap and the regret of SOP, and the identification and marginal-flow facts, which the proofs throughout this appendix invoke.

\begin{restated}{Assumption}{asm:consistency}{Consistency and no interference, formal statement}
Fix a unit $i \in \mathcal{V}$ and a period $t \in \{0,\dots,T\}$. We assume that (i) the realized quantities agree with the potential ones along the path actually taken, $Y_{i,t} = Y_{i,t}(\bar{A}_{i,t})$ and $(X_{i,t}, U_{i,t}) = \big(X_{i,t}(\bar{A}_{i,t-1}), U_{i,t}(\bar{A}_{i,t-1})\big)$; and (ii) the potential quantities of unit $i$ depend on its own full assignment path $\bar{a}_i$ alone, so that reassigning the path of any $j \neq i$ leaves them unchanged \citep{rubin1980randomization, murphy2003optimal}.
\end{restated}

\begin{restated}{Assumption}{asm:markov}{Markov full state and common mechanism, formal statement}
Suppose there exist a transition kernel $f$ on $\mathcal{X} \times \Ucal$ and, for each $a \in \mathcal{A}$, a response surface $\theta(\cdot,a,\cdot) : \mathcal{X} \times \Ucal \to \mathbb{R}$, each invariant across units and across time, such that for every path $\bar{a}_{i,t}$,
\begin{align}
    \big(X_{i,t+1}(\bar{a}_{i,t}),\, U_{i,t+1}(\bar{a}_{i,t})\big) &\sim f\big(\,\cdot \mid X_{i,t}(\bar{a}_{i,t-1}),\, a_{i,t},\, U_{i,t}(\bar{a}_{i,t-1})\big), \label{eq:transition}\\
    \mathbb{E}\big[\,Y_{i,t}(\bar{a}_{i,t}) \mid X_{i,t}(\bar{a}_{i,t-1}) = x,\ U_{i,t}(\bar{a}_{i,t-1}) = u\,\big] &= \theta(x,a_{i,t},u). \label{eq:response}
\end{align}
Both displays condition on the potential history along the path. Write $\mathcal G_{i,t}(\bar a)$ for the $\sigma$-field generated by $(X_{i,s}(\bar a_{i,s-1}),U_{i,s}(\bar a_{i,s-1}),Y_{i,s}(\bar a_{i,s}))_{s<t}$ together with $(X_{i,t}(\bar a_{i,t-1}),U_{i,t}(\bar a_{i,t-1}))$. Conditionally on $\mathcal G_{i,t}(\bar a)$, the pair $(X_{i,t+1}(\bar a_{i,t}),U_{i,t+1}(\bar a_{i,t}))$ has distribution $f(\cdot\mid X_{i,t}(\bar a_{i,t-1}),a_{i,t},U_{i,t}(\bar a_{i,t-1}))$, and given $\mathcal G_{i,t}(\bar a)$ the outcome $Y_{i,t}(\bar a_{i,t})$ has conditional mean $\theta(X_{i,t}(\bar a_{i,t-1}),a_{i,t},U_{i,t}(\bar a_{i,t-1}))$. We fix a version of $\theta$ with values in $[-B_Y,B_Y]$; such a version exists because $\theta$ is a conditional mean of an outcome bounded by $B_Y$.
\end{restated}

\begin{restated}{Assumption}{asm:calib}{Exact calibration, formal statement}
Suppose that over the pre-deployment historical record $t<0$ the conditional distribution of $U_{i,t}$ given $X_{i,t}=x$ equals a fixed $\nu^{\mathrm{sim}}(\cdot \mid x)$, and that calibration is exact at that distribution: for each $(x,a)$,
$\theta^{\mathrm{sim}}(x,a) = \int \theta(x,a,u)\, \nu^{\mathrm{sim}}(du \mid x)$ and $f^{\mathrm{sim}}(\cdot\mid x,a) = \int f_X(\cdot\mid x,a,u)\, \nu^{\mathrm{sim}}(du \mid x)$.
\end{restated}

\begin{restated}{Assumption}{asm:smooth}{Smoothness, formal statement}
Equip $\Ucal$, $\mathcal{X}$, and $\mathcal{A}$ with fixed metrics $d_U$, $d_X$, and $d_A$, and write $d((x,a,u),(x',a',u')) := d_X(x,x') + d_A(a,a') + d_U(u,u')$ and $d_{\mathcal{Z}}((x,a),(x',a')):=d_X(x,x')+d_A(a,a')$. Suppose there are finite constants $L_\theta$, $L_f$, and $L_m$ such that
\begin{align}
    \big| \theta(x,a,u) - \theta(x',a',u') \big| &\;\leq\; L_\theta \, d\big((x,a,u),(x',a',u')\big), \label{eq:liptheta}\\
    \big\| f_X(\cdot\mid x,a,u) - f_X(\cdot\mid x',a',u') \big\|_1 &\;\leq\; L_f \, d\big((x,a,u),(x',a',u')\big), \label{eq:lipf}
\end{align}
for all $(a,a',x,x',u,u')$, and such that, uniformly over units $i$, periods $t\geq-1$, and fixed induced distributions $q$, the induced outcome-regression surfaces (the mean-response surfaces with a fixed unit and strategy restored) satisfy
\begin{equation}
    \big|m^{q}_{i,t}(x,a)-m^{q}_{i,t}(x',a')\big| \;\leq\; L_m\,d_{\mathcal{Z}}\big((x,a),(x',a')\big). \label{eq:lipm}
\end{equation}
Each occurrence of $\nu^q_{i,t}(\cdot\mid x)$ and of $m^{q}_{i,t}$ uses the fixed regular-conditional versions chosen in Section~\ref{sec:setup}, including at null states.
\end{restated}

\begin{restated}{Assumption}{asm:drift}{Bounded latent-distribution drift, formal statement}
Let $W_1$ denote the Wasserstein-$1$ distance induced by $d_U$, and suppose the conditional latent distributions $\nu^{q}_{i,t}(\cdot\mid x)$ below have finite first moments under $d_U$. Suppose further that there is a finite constant $\lambda$ such that, uniformly over units $i$, periods $t\geq-1$, and fixed induced distributions $q$,
\begin{equation}
    W_1\big(\nu^{q}_{i,t+1}(\cdot \mid x),\nu^{q}_{i,t}(\cdot \mid x)\big) \;\leq\; \lambda, \label{eq:drift}
\end{equation}
where $\nu^{q}_{i,-1}$ denotes the conditional latent distribution over the historical record.
\end{restated}

\begin{restated}{Assumption}{asm:crossunit}{Cross-unit independence, formal statement}
The units' potential-outcome processes $\{(Y_{i,\cdot}(\cdot),X_{i,\cdot}(\cdot),U_{i,\cdot}(\cdot))\}_{i\in\mathcal V}$ are mutually independent across units, and jointly independent of the assignment uniforms $\xi=(\xi_{i,t})_{i,t}$ of convention (R) in Appendix~\ref{app:proofs} and of any internal randomization of the strategy. Each unit's own potential process obeys Assumption~\ref{asm:markov}: along each path, conditionally on unit $i$'s potential history through $t$, $(X_{i,t+1},U_{i,t+1})\sim f(\cdot\mid X_{i,t},A_{i,t},U_{i,t})$ and $\mathbb E[Y_{i,t}\mid\text{unit }i\text{'s potential history}]=\theta(X_{i,t},A_{i,t},U_{i,t})$, with the full-state kernel $f$ and response $\theta$ of Assumption~\ref{asm:markov}.
\end{restated}

\paragraph{The additive reward form (Section~\ref{sec:setup}).} Setting $\epsilon_{i,t} := Y_{i,t} - \theta(X_{i,t},A_{i,t},U_{i,t})$ yields the additive form
\begin{equation}
    Y_{i,t} = \sum_{a \in \mathcal{A}} \mathbf{1}[A_{i,t} = a]\, \theta(X_{i,t},a,U_{i,t}) + \epsilon_{i,t},
    \qquad \mathbb{E}\big[\epsilon_{i,t} \mid X_{i,t}, U_{i,t}, A_{i,t}\big] = 0,
    \label{eq:reward}
\end{equation}
where the mean-zero property holds by construction. To see this, fix $\bar{A}_{i,t-1} = \bar{a}_{-1}$ and $A_{i,t} = a$; consistency implies $Y_{i,t} = Y_{i,t}(\bar{a}_{-1}, a)$, whose conditional mean given the full state is $\theta(x,a,u)$ by \eqref{eq:response}, free of $\bar{a}_{-1}$. Averaging over $\bar{a}_{-1}$ therefore leaves this conditional mean unchanged.

\paragraph{Consequences of Assumptions~\ref{asm:smooth} and~\ref{asm:drift} (Section~\ref{sec:setup}).} Under the fixed induced distribution whose indices we suppress, iterating \eqref{eq:drift} gives
\begin{equation}
    W_1\big(\nu_t(\cdot \mid x),\, \nu_{\tau}(\cdot \mid x)\big) \;\leq\; \lambda\,|t-\tau|
    \qquad \text{for all } t, \tau \geq -1,
    \label{eq:driftacc}
\end{equation}
the temporal displacement bound the proofs below use.

\begin{observation}[Simulator error as displacement of the latent distribution]
\label{prop:delta}
Let Assumptions~\ref{asm:markov} and~\ref{asm:calib} hold, fix a unit and a fixed induced policy distribution, and suppress those indices. Write $\TV(\nu,\nu') := \sup_A |\nu(A) - \nu'(A)|$ and $\|v\|_1 := \sum_{x'} |v(x')|$. Then for each $(x,a)$ and each $t \geq 0$,
\begin{equation}
    \delta_t(x,a) = \int \theta(x,a,u)\,\big[\nu_t - \nu^{\mathrm{sim}}\big](du \mid x),
    \qquad
    \big|\delta_t(x,a)\big| \;\leq\; 2B_Y \cdot \TV\big(\nu_t(\cdot \mid x),\, \nu^{\mathrm{sim}}(\cdot \mid x)\big),
    \label{eq:delta}
\end{equation}
and the transition error satisfies $\big\| g_t(\cdot\mid x,a) - f^{\mathrm{sim}}(\cdot\mid x,a) \big\|_1 \leq 2\,\TV\big(\nu_t(\cdot \mid x),\, \nu^{\mathrm{sim}}(\cdot \mid x)\big)$. If in addition Assumptions~\ref{asm:smooth} and~\ref{asm:drift} hold, then
\begin{equation}
\begin{gathered}
    \big|\delta_t(x,a)\big| \leq L_\theta \, W_1\big(\nu_t(\cdot \mid x), \nu^{\mathrm{sim}}(\cdot \mid x)\big) \leq L_\theta \lambda (t+1),
    \\
    \big\| g_t(\cdot\mid x,a) - f^{\mathrm{sim}}(\cdot\mid x,a) \big\|_1 \leq L_f \lambda (t+1).
\end{gathered}
    \label{eq:deltasmooth}
\end{equation}
\end{observation}

\begin{proposition}[Simulation gap]
\label{prop:simgap}
Let Assumptions~\ref{asm:consistency} and~\ref{asm:markov} hold, and let the simulator's initial state distribution be $\mu_X$, the $\mathcal{X}$-marginal of the initial distribution of the full state $(X_0,U_0)$, which the simulator and the target share. Then (i) for each $\bar{\pi} \in \bar{\Pi}$,
\begin{equation}
    \big| V^\gamma_T(\bar{\pi}) - V^{\gamma,\mathrm{sim}}_T(\bar{\pi}; \theta^{\mathrm{sim}}, f^{\mathrm{sim}}) \big|
    \;\leq\; n \big[\, \epsilon_m\, \Gamma_T(\gamma) \;+\; \bar{\theta}\, \epsilon_g\, \Lambda_T(\gamma) \,\big].
    \label{eq:simgap}
\end{equation}
where the bracket collects the uniform reward error scaled by its discount sum and the uniform transition error scaled by the reward bound and its discount sum, and the leading factor sums over the units.
(ii) If Assumptions~\ref{asm:calib}, \ref{asm:smooth}, and~\ref{asm:drift} also hold, then $\epsilon_m \leq L_\theta \lambda (T+1)$ and $\epsilon_g \leq L_f \lambda (T+1)$, so that the right side of (i) is at most
\begin{equation}
    n\,\lambda\,(T+1)\,\big[\, L_\theta\, \Gamma_T(\gamma) \;+\; \bar{\theta}\, L_f\, \Lambda_T(\gamma) \,\big].
    \label{eq:simgapsmooth}
\end{equation}
\end{proposition}

\begin{corollary}[Regret of SOP against the best fixed policy]
\label{cor:sopregret}
Under the conditions of Proposition~\ref{prop:simgap}, $V^\gamma_T(\bar{\pi}^*) - V^\gamma_T(\bar{\pi}^{\mathrm{sim}}) \leq 2n[\epsilon_m \Gamma_T(\gamma) + \bar{\theta}\epsilon_g \Lambda_T(\gamma)]$.
\end{corollary}

\begin{proposition}[Identification under a fixed prespecified policy]
\label{prop:ident}
Let Assumptions~\ref{asm:consistency} and~\ref{asm:markov} hold, and let $\bar{\pi}\in\bar{\Pi}$ be fixed before the horizon outcomes are observed and then deployed. Then for each $(i,t,x,a)$ with $\Pr_{\bar\pi}[X_{i,t}=x]>0$ and $\pi_{i,t}(a\mid x)>0$, the projections induced by this fixed distribution satisfy
\begin{equation}
\begin{gathered}
    \mathbb{E}_{\bar\pi}\big[Y_{i,t} \mid X_{i,t} = x, A_{i,t} = a\big] = m^{\bar\pi}_{i,t}(x,a),
    \\ \Pr_{\bar\pi}\big[X_{i,t+1} = x' \mid X_{i,t} = x, A_{i,t} = a\big] = g^{\bar\pi}_{i,t}(x'\mid x,a).
\end{gathered}
    \label{eq:ident}
\end{equation}
The same conclusion holds for a fixed, prespecified randomized design represented by such a policy tensor.
\end{proposition}

\begin{lemma}[Marginal flow]
\label{lem:flow}
Let Assumptions~\ref{asm:consistency} and~\ref{asm:markov} hold, and let the prespecified $\bar{\pi} \in \bar{\Pi}$ be deployed. Write $p_{i,t}^{\bar\pi}(x):=\Pr_{\bar\pi}[X_{i,t}=x]$ and suppress $(i,\bar\pi)$ on $m$ and $g$ as in the main text. Then
\begin{align}
    p_{i,t+1}^{\bar\pi}(x') &= \sum_{x \in \mathcal{X}} p_{i,t}^{\bar\pi}(x) \sum_{a \in \mathcal{A}} \pi_{i,t}(a \mid x)\, g_t(x'\mid x,a), \notag\\
    V^\gamma_T(\bar{\pi}) &= \sum_{i \in \mathcal{V}} \sum_{t=0}^{T} \gamma^t \sum_{x \in \mathcal{X}} p_{i,t}^{\bar\pi}(x) \sum_{a \in \mathcal{A}} \pi_{i,t}(a \mid x)\, m_t(x,a).
    \label{eq:flow}
\end{align}
\end{lemma}
We first formalize the cell-wise prior and the working sampling model, then show that an unobserved working cell keeps its simulator mean, and we use this to define SOP and A-SOP and to derive the lock-in behind the hidden-drift regret of Theorem~\ref{thm:hiddendrift}. Finally, we define the bounded cohort--window designs and show that a fixed policy need not leave the latent dynamics Markov.

\begin{restated}{Assumption}{asm:cellwise}{Cell-wise prior, formal statement}

Under the working prior $\Xi_0$, the working-cell quantities $\{m^{\mathrm{wk}}(x,a)\}_{(x,a)}$ and $\{g^{\mathrm{wk}}(\cdot\mid x,a)\}_{(x,a)}$ are mutually independent.
\end{restated}
\begin{restated}{Assumption}{asm:likelihood}{Working sampling model, formal statement}

Under the working model, a record $(x,a,y,x')$ is a draw of $Y=m^{\mathrm{wk}}(x,a)+\eta$ with $\eta\sim Q(\cdot\mid x,a)$ and of $X'\sim g^{\mathrm{wk}}(\cdot\mid x,a)$, independently across records. In the factorized pseudo-likelihood of A-SOP, the selected completed tuples (the fully observed records retained at a review) factorize into, for each tuple, the reward factor of the location family $Y=m^{\mathrm{wk}}(x,a)+\eta$ and a next-state multinomial factor with probabilities $g^{\mathrm{wk}}(\cdot\mid x,a)$. The reward and transition factors, and distinct cells, are conditionally independent.
\end{restated}
In posterior calculations we may suppress the superscript $\mathrm{wk}$, and $\hat m$ and $\hat g$ always denote posterior means.

\begin{proposition}[Unobserved working cells retain their prior marginals]
\label{prop:censoring}
Let $\sigma^{\mathrm{asop}}_\kappa$ run under Assumptions~\ref{asm:cellwise} (cell-wise prior) and~\ref{asm:likelihood} (working sampling model). Fix a cell $(x,a)$ and suppose that the operating likelihood $\mathcal L^{\mathrm{op}}_{<t}$ (the selected factors accumulated strictly beforehand) contains no selected reward or transition likelihood factor for it. Then, under the predecision posterior $\Xi_t=\Xi_0(\cdot\mid\mathcal L^{\mathrm{op}}_{<t})$, the cell's joint working marginal equals its $\Xi_0$ marginal; in particular, its response mean is $\theta^{\mathrm{sim}}(x,a)$.
\end{proposition}

\begin{definition}[SOP and A-SOP]
\label{def:sop}
\label{def:asop}
Fix deterministic selectors for any nonunique policy maximizations, and let $\bar{\pi}^{\mathrm{sim}}$ be the selected maximizer of $V^{\gamma,\mathrm{sim}}_T(\cdot;\theta^{\mathrm{sim}},f^{\mathrm{sim}})$, the discounted simulated value. The strategy $\sigma^{\mathrm{sop}}$ deploys the simulator-optimal policy $\bar{\pi}^{\mathrm{sim}}$ (SOP) at each $t$ and never updates. Fix a cadence $\kappa$, with $\kappa=1$ the default. The history-dependent strategy $\sigma^{\mathrm{asop}}_\kappa$ (A-SOP) reviews at the periods $t\in\{0,\kappa,2\kappa,\dots\}$. At each review it retains only the completed operating likelihood contributions, that is, those from fully observed tuples. It then uses the fixed selector to choose a policy that maximizes the remaining-horizon expected simulated value under the posterior-mean tables $(\hat m_t,\hat g_t)$, and it deploys that policy until the next review, so that its period-$t$ selection is $\mathcal F_t$-measurable.
\end{definition}
Under Assumption~\ref{asm:cellwise}, a working coordinate for which no likelihood factor has been selected retains its $\Xi_0$ marginal, by Proposition~\ref{prop:censoring}.

\begin{corollary}[Persistent nonassignment under pessimistic ranking]
\label{cor:lockin}
Work in the setting of Proposition~\ref{prop:censoring} and fix a state $x$ and an action $a^\dagger$. Suppose that the following hold. (i) At each review A-SOP applies the myopic policy $\arg\max_{a\in\mathcal A}\hat m_t(x,a)$ with the fixed selector. (ii) At an observed cell, the posterior mean lies in the convex hull of its $\Xi_0$ mean and the period-specific response values used in the actual A-SOP update. (iii) For some $a'\neq a^\dagger$ and constants $\alpha>b_0:=\theta^{\mathrm{sim}}(x,a^\dagger)$, both $\theta^{\mathrm{sim}}(x,a')\geq\alpha$ and every period-specific value entering the actual update for $(x,a')$ is at least $\alpha$. Then $a^\dagger$ is never assigned at $x$.
\end{corollary}

\begin{definition}[Bounded cohort--window experimental design]
\label{def:design}
A bounded cohort--window experimental design is given by the design vector $d = (\mathcal{C}, \mathcal{T}_d, \mathcal{A}_d)$, where $\mathcal{C} \subseteq \mathcal{V}$ is the cohort, $\mathcal{T}_d \subseteq \{0,\dots,T\}$ is the window, and the nonempty set $\mathcal{A}_d \subseteq \mathcal{A}$ is the set of tested actions. For each $i \in \mathcal{C}$, $t \in \mathcal{T}_d$, and $x \in \mathcal{X}$, the action is drawn uniformly at random from $\mathcal{A}_d$, while all other assignments follow the plug-in policy; equivalently, $\mathsf P[i,t,x,a] = 1/|\mathcal{A}_d|$ on the design's slice of the policy tensor in \eqref{eq:tensor}. Concurrent designs are permitted provided their cohort-window slices are pairwise disjoint.
\end{definition}

Before using these designs we note one structural fact from Section~\ref{sec:setup}: projecting the latent dynamics through a fixed policy need not leave a Markov process.
\begin{lemma}[The observed process need not be Markov]
\label{lem:nonmarkov}
Let Assumptions~\ref{asm:consistency} and~\ref{asm:markov} hold and let the fixed, prespecified $\bar{\pi} \in \bar{\Pi}$ be deployed. Write $\bar H_{i,t}:=(X_{i,0},A_{i,0},Y_{i,0},\dots,A_{i,t-1},Y_{i,t-1},X_{i,t})$ for the unit's observed history before its period-$t$ action. For any history $h$ and action $a$ with $\Pr(\bar H_{i,t}=h,A_{i,t}=a)>0$,
\begin{equation}
    \Pr\big[X_{i,t+1} = x' \mid \bar{H}_{i,t},\, A_{i,t} = a\big]
    = \int f_X(x'\mid X_{i,t},a,u)\, \Pr\big[U_{i,t} \in du \mid \bar{H}_{i,t}\big].
    \label{eq:histlaw}
\end{equation}
Then the following hold. (i) If the conditional distribution of $U_{i,t}$ given $\bar{H}_{i,t}$ depends on the history only through $X_{i,t}$, then \eqref{eq:histlaw} equals $g_t(x'\mid X_{i,t},a)$ and the observed process is Markov. (ii) There are instances satisfying Assumptions~\ref{asm:consistency} and~\ref{asm:markov} in which the observed process is Markov while the conditional distribution of $U_{i,t}$ given $\bar H_{i,t}$ depends on the history beyond $X_{i,t}$. (iii) There are instances satisfying Assumptions~\ref{asm:consistency} and~\ref{asm:markov} with $|\mathcal{X}| \geq 2$ in which \eqref{eq:histlaw} differs from $g_t(x'\mid X_{i,t},a)$ for some history of positive probability, so the observed process is not Markov.
\end{lemma}
\section{Bounded Loss and Boundary Proofs}
\label{app:theory}

Figure~\ref{fig:theory-map} maps the theory: the results that the main text states or cites, the appendix in which each is proved, and the supporting results and lemmas that enter each proof. We suggest reading it before the proofs.

\definecolor{srcA}{HTML}{006BA2}\definecolor{srcB}{HTML}{3EBCD2}\definecolor{srcC}{HTML}{379A8B}\definecolor{srcD}{HTML}{CE8E00}\definecolor{srcE}{HTML}{9A607F}\definecolor{srcF}{HTML}{DB444B}
\definecolor{srcA}{HTML}{006BA2}\definecolor{srcB}{HTML}{3EBCD2}\definecolor{srcC}{HTML}{379A8B}\definecolor{srcD}{HTML}{CE8E00}\definecolor{srcE}{HTML}{DB444B}\definecolor{srcF}{HTML}{9A607F}
\definecolor{srcA}{HTML}{006BA2}\definecolor{srcB}{HTML}{3EBCD2}\definecolor{srcC}{HTML}{7A8F1E}\definecolor{srcD}{HTML}{CE8E00}\definecolor{srcE}{HTML}{DB444B}\definecolor{srcF}{HTML}{9A607F}
\definecolor{srcA}{HTML}{006BA2}\definecolor{srcB}{HTML}{3EBCD2}\definecolor{srcC}{HTML}{7A8F1E}\definecolor{srcD}{HTML}{CE8E00}\definecolor{srcE}{HTML}{DB444B}\definecolor{srcF}{HTML}{9A607F}
\begin{figure}[p]
\centering
\begin{tikzpicture}[x=1cm,y=1cm,
  >={Stealth[length=2mm]},
  every node/.style={font=\small},
  main/.style={draw=probframe, fill=probfill, rounded corners=1.5pt, text width=58mm, align=left, inner sep=4pt},
  aux/.style={draw=rule, fill=white, rounded corners=1.5pt, text width=36mm, align=left, inner sep=3.5pt},
  tool/.style={draw=rule, fill=mechfill, rounded corners=1.5pt, align=left, inner sep=4pt},
  use/.style={->, semithick, draw=rule, rounded corners=2pt},
  special/.style={->, semithick, draw=rule, dashed},
  hdr/.style={font=\small\bfseries, text=textbody, inner sep=1pt}
]
\node[hdr, anchor=west] at (1.1,14.4) {supporting results, Appendix~\ref{app:proofs}};
\node[hdr, anchor=west] at (6.9,14.4) {whether additional data add value};
\draw[rule, thin] (6.9,3.45) -- (12.95,3.45);
\node[hdr, anchor=west] at (6.9,3.15) {when to depart, and at what price};

\node[main] (T1) at (9.95,13.5) {\textbf{Theorem~\ref{thm:nochange}}\\ no-change boundary\\ \textcolor{textbody}{Appendix~\ref{app:theory}}};
\node[main] (C1) at (9.95,11.95) {\textbf{Corollary~\ref{cor:sopregret}}\\ regret of SOP against the best fixed policy\\ \textcolor{textbody}{Appendix~\ref{app:proofs}}};
\node[main] (T2) at (9.95,10.4) {\textbf{Theorem~\ref{thm:hiddendrift}}\\ quadratic regret under hidden drift\\ \textcolor{textbody}{Appendix~\ref{app:theory}}};
\node[main] (T3) at (9.95,8.85) {\textbf{Theorem~\ref{thm:ident}}\\ identification of cell means\\ \textcolor{textbody}{Lemmas~\ref{lem:ident-L1}--\ref{lem:ident-L4}; Appendix~\ref{app:ident}}};
\node[main] (L12) at (9.95,7.3) {\textbf{Lemma~\ref{lem:window}}\\ estimation across units and staleness\\ \textcolor{textbody}{Appendix~\ref{app:window}}};
\node[main] (T4) at (9.95,5.75) {\textbf{Theorem~\ref{thm:regimeid}}\\ identifiability of the regime distribution\\ \textcolor{textbody}{Lemmas~\ref{lem:regimeid-L1}--\ref{lem:regimeid-L4}; Appendix~\ref{app:regimeid}}};
\node[main] (T5) at (9.95,4.2) {\textbf{Theorem~\ref{thm:testvalue}}\\ value of a designed test\\ \textcolor{textbody}{Lemmas~\ref{lem:tv-A}--\ref{lem:tv-I}, Corollary~\ref{cor:switch}; Appendix~\ref{app:testvalue}}};
\node[main] (P4) at (9.95,2.25) {\textbf{Proposition~\ref{prop:purse}}\\ bounded loss and no bankruptcy\\ \textcolor{textbody}{Appendix~\ref{app:theory}}};
\node[main] (T6) at (9.95,0.7) {\textbf{Theorem~\ref{thm:detectprice}}\\ detection--price tradeoff\\ \textcolor{textbody}{Lemmas~\ref{lem:detectprice-L1}--\ref{lem:detectprice-L4}; Appendix~\ref{app:detectprice}}};
\node[main] (T7) at (9.95,-0.85) {\textbf{Theorem~\ref{thm:undetectable}}\\ size--power bound with zero departures\\ \textcolor{textbody}{Appendix~\ref{app:theory}}};

\node[aux, draw=srcA, thick] (P3) at (3.0,13.5) {\textbf{Proposition~\ref{prop:censoring}}\\ prior kept at unvisited cells};
\node[aux, draw=srcE, thick] (L7) at (3.0,12.45) {\textbf{Lemma~\ref{lem:attain}}\\ the maximum is attained};
\node[aux, draw=srcD, thick] (P1) at (3.0,11.4) {\textbf{Proposition~\ref{prop:simgap}}\\ simulation gap};
\node[aux, draw=srcB, thick] (C2) at (3.0,10.22) {\textbf{Corollary~\ref{cor:lockin}}\\ persistent nonassignment};
\node[aux, draw=srcC, thick] (P2) at (3.0,8.85) {\textbf{Proposition~\ref{prop:ident}}\\ identification, fixed policy};
\node[aux, draw=srcF, thick] (L1) at (3.0,5.57) {\textbf{Lemma~\ref{lem:flow}}\\ marginal flow};
\node[tool, text width=118mm] (TB) at (7.0,-2.3) {\textbf{tools used in the proofs} (Appendices~\ref{app:toolbox} and~\ref{app:proofs}): Lemma~\ref{lem:hoeffding}, Hoeffding's inequality; Lemma~\ref{lem:tvfacts}, total-variation facts; Lemma~\ref{lem:pathlaw}, the path distribution under a fixed policy; Lemma~\ref{lem:decomp}, the per-unit value decomposition.};

\draw[use, draw=srcA] (P3.west) -- ++(-0.6,0) |- (C2.west);
\draw[use, draw=srcC] (P2.south) -- (L1.north);
\draw[use, draw=srcF] (L1.west) -- ++(-0.4,0) |- (P1.west);

\draw[use, draw=srcA] (P3.east) -- (T1.west);
\draw[use, draw=srcA] (P3.east) -- ++(0.375,0) |- ([yshift=5pt]T2.west);
\draw[use, draw=srcA] (P3.east) -- ++(0.375,0) |- ([yshift=8pt]T6.west);
\draw[use, draw=srcE] (L7.east) -- ++(0.625,0) |- ([yshift=5pt]C1.west);
\draw[use, draw=srcD] (P1.east) -- ++(0.875,0) |- ([yshift=-5pt]C1.west);
\draw[use, draw=srcD] (P1.east) -- ++(0.875,0) |- ([yshift=5pt]P4.west);
\draw[use, draw=srcB] (C2.east) -- ([yshift=-5pt]T2.west);
\draw[use, draw=srcB] (C2.east) -- ++(1.125,0) |- ([yshift=5pt]T4.west);
\draw[use, draw=srcC] (P2.east) -- (T3.west);
\draw[use, draw=srcC] (P2.east) -- ++(1.375,0) |- ([yshift=-8pt]T6.west);
\draw[use, draw=srcF] (L1.east) -- ([yshift=-5pt]T4.west);
\draw[use, draw=srcF] (L1.east) -- ++(1.625,0) |- ([yshift=-5pt]P4.west);

\draw[use, draw=ink] (T3.south) -- (L12.north);
\draw[use, draw=ink] (P4.south) -- (T6.north);
\draw[use, draw=ink, {Circle[length=4pt]}-{Stealth[length=2mm]}] ([yshift=8pt]T1.east) -- ++(0.35,0) |- ([yshift=8pt]T6.east);
\draw[use, draw=ink, {Circle[length=4pt]}-{Stealth[length=2mm]}] (T7.east) -- ++(0.8,0) |- ([yshift=-10pt]T3.east);
\draw[special, draw=ink] ([xshift=-14pt]T6.south) -- ([xshift=-14pt]T7.north);
\end{tikzpicture}
\caption{A map of the theory. Pastel boxes are the results that the main text states or cites, grouped by the question they answer; white boxes are the supporting results of Appendix~\ref{app:proofs}, and the gray strip lists the general lemmas, two probability lemmas and two structural identities, that enter many proofs and are drawn without arrows. A solid arrow runs from a result to a proof that uses it; it carries the colour of the supporting result it leaves, and is black, with a dot at its tail, between two main results. The dashed arrow marks Theorem~\ref{thm:undetectable} as the zero-departure case of Theorem~\ref{thm:detectprice}. The lemmas named inside a box are proved beside it, in the appendix given there.}
\label{fig:theory-map}
\end{figure}

\paragraph{Proofs for Section~\ref{sec:now} (the budgeted procedure).}
We use two conventions of the budgeted procedure (Section~\ref{sec:now}) below. By (G), we form the working transition table $\hat g$ through the shrinkage rule, which combines the next-state masses with the centre's transition table; the two tables of the structural centre (the fitted regime mixture) are $\sum_u\hat\nu(u)\theta(\cdot,\cdot,u)$ and $\sum_u\hat\nu(u)f_X(\cdot\mid\cdot,\cdot,u)$, as in the implementation. By (C0), no candidate is considered at $C=0$. Throughout, we use the selector convention (S) defined in Appendix~\ref{app:proofs}.

\begin{table}[t]
\centering
\caption{Settings of the budgeted procedure at which it coincides with five of the six strategies of the simulation study; $h'<h$ is SEP-when's interval. The first column is the centre toward which the working model is shrunk. A fixed value gives the setting, a bare parameter name marks a free parameter, and a dash marks a parameter with no effect.}
\label{tab:settings}
\setlength{\tabcolsep}{4pt}
\begin{tabular}{@{}lccccccl@{}}
\toprule
strategy & $\mu_0$ & $\tau$ & $\omega$ & $b$ & $C$ & interval & admission rule \\
\midrule
SOP & $\theta^{\mathrm{sim}}$ & $\infty$ & -- & -- & $0$ & -- & -- \\
S-SOP & structural & $\infty$ & $0$ & -- & $0$ & -- & -- \\
UCB & $\theta^{\mathrm{sim}}$ & $n_0$ & $\omega$ & $b$ & $\infty$ & -- & -- \\
SEP & structural & $\infty$ & $0$ & $b$ & $C$ & $h$ & gain \\
SEP-when & structural & $\infty$ & $0$ & $b$ & $Ch'/h$ & $h'<h$ & net \\
\bottomrule
\end{tabular}
\end{table}

\paragraph{The strategies as settings of the procedure.}\label{app:settings}
Table~\ref{tab:settings} lists the settings at which the procedure coincides with five of the six strategies of the simulation study: at each setting the procedure and the named strategy follow the same policy sequence $\pi_{i,t}$ along each trajectory. We read the table row by row. With $\tau=\infty$ and $\mu_0=\theta^{\mathrm{sim}}$, the shrinkage rule returns $(\theta^{\mathrm{sim}},f^{\mathrm{sim}})$ in each period. This leaves the memory rate and the borrowing weights without effect, and by (C0) no test is admitted. The policy in force is then $S(\theta^{\mathrm{sim}},f^{\mathrm{sim}})$, whose period-$t$ component coincides, by (S), with the same period's component of the simulator-optimal policy $\bar\pi^{\mathrm{sim}}$ (Definition~\ref{def:sop}). This is SOP. With the structural centre in place of $\theta^{\mathrm{sim}}$, $\tau=\infty$, and $\omega=0$, the working model is the pair of structural tables in (G); its regime distribution $\hat\nu$ is refit in each period from the previous period's visited-cell means pooled with the neighbouring zones, the masses undiscounted. By (C0) there is again no departure, and the policy in force is thus the plug-in policy of the structural tables, which is S-SOP (Section~\ref{sec:exploit}). With $\mu_0=\theta^{\mathrm{sim}}$ and $\tau=n_0$, the shrinkage rule returns $\hat m=(n_0\theta^{\mathrm{sim}}+R)/(n_0+N)$, the masses discounted by $1-\omega$ per period and pooled across neighbours, and (G) gives $\hat g$ likewise; with an unbounded budget the procedure deploys $S(\hat m+bs,\hat g)$ at each review, where $s=s_0n_0/(n_0+N)$ is the tunnel width, and the review interval has no further effect; the study's UCB reviews in every period. This is the discounted upper-confidence-bound strategy; see Section~\ref{sec:sim}. The SEP and SEP-when rows are the procedure at the stated settings, with nothing to verify. We note only that the two settings draw on the budget at the same per-period rate. SEP admits a simulated price of at most $C$ at a review held every $h$ periods, and SEP-when at most $Ch'/h$ at a review held every $h'$ periods; the rate is $C/h$ per period in both cases. Under the net rule (the SEP-when admission rule) a candidate is admitted only when $\operatorname{gain}(d)>\operatorname{price}(d)$.

At any finite budget the procedure is an experimenting strategy in the sense of Definition~\ref{def:experimentmode}. It departs from the plug-in policy only on the slices of the admitted tests. Each test is chosen at the review before the period's assignments are drawn, assigns actions uniformly over a declared set on a prespecified fraction of one zone's units, and is admitted subject to $\sum_d\operatorname{price}(d)\le C$. Its observations enter $(R,N)$ with weight $w$ and update the tables of the zone's untested units and of the neighbouring zones. This is the transport update that condition (iv) of the definition requires.

Thompson sampling is the one strategy of the study that is not a setting of the procedure. It draws one response table from the working posterior at each review and runs that table's plug-in policy on all units of a zone for the epoch, and therefore makes the policy in force itself random and holds it fixed across the epoch's periods. The procedure, by contrast, follows a policy that is a deterministic function of the history and randomizes assignments only inside admitted tests, independently across units and across periods. We conclude that no choice of the five parameters produces a randomly drawn policy that persists across periods.

\paragraph{Objects for the formal statement of Proposition~\ref{prop:purse}.}
We first define the review and the working model, then the two worlds and the per-unit value functions, and finally the population-limit objects used only in part (c). Fix a review $k$ at period $t_k$ and condition on its predecision history $\mathcal F_{t_k}$. Let $(\hat m,\hat g)$ be the working model at the review, let $\hat{\bar\pi}:=S(\hat m,\hat g)$ be the plug-in policy for periods $t_k,\dots,T$, and let $\mathcal P_k$ be the admitted package, one-period randomized designs at period $t_k$ on disjoint cohorts with $\sum_{d\in\mathcal P_k}\operatorname{price}(d)\le C$. Let $\mathcal D_k\subseteq\mathcal V$ be the tested units and $D_k:=|\mathcal D_k|$. Let $\pi$ be the policy tensor that equals $\hat{\bar\pi}$ except at $(i,t_k)$ with $i\in\mathcal D_k$, where it is the design's uniform assignment, and let $\pi_d$ be the tensor that departs on the slice of $d$ alone. The \emph{mean world} $M$ is the distribution of \eqref{eq:sim} with $(\hat m,\hat g)$ in place of $(\theta^{\mathrm{sim}},f^{\mathrm{sim}})$, started from the current states, and $R$ is the real distribution; both are taken conditionally on $\mathcal F_{t_k}$, with assignments drawn under (R). For a tensor $\rho$ and $\Lambda\in\{M,R\}$ let $V^\Lambda(\rho):=\mathbb E^\Lambda_\rho[\sum_{u=t_k}^T\sum_iY_{i,u}\mid\mathcal F_{t_k}]$, let $\operatorname{loss}^\Lambda:=V^\Lambda(\hat{\bar\pi})-V^\Lambda(\pi)$, and let $\operatorname{price}(d):=V^M(\hat{\bar\pi})-V^M(\pi_d)$. For a unit $i$, periods $t_k<u\le t'\le T$, and a state $x'$, let $w^M_{i,u,t'}(x')$ be the mean-world value of following $\hat{\bar\pi}$ from state $x'$ at period $u$ through $t'$, and let $\bar V:=\max_{i,u,t'}[\max_{x'}w^M_{i,u,t'}(x')-\min_{x'}w^M_{i,u,t'}(x')]$. For $i\in\mathcal D_k$, an action $a$, and a state $x'$, let $w^R_{i,a,t'}(x')$ be the real value of following $\hat{\bar\pi}$ over periods $t_k+1,\dots,t'$ conditionally on $\mathcal F_{t_k}$, $A_{i,t_k}=a$, and $X_{i,t_k+1}=x'$, and let $\bar G:=\max_{i,a,x',t'}|w^R_{i,a,t'}(x')-w^M_{i,t_k+1,t'}(x')|$. Let $\epsilon_m$ and $\epsilon_g$ bound, over $i\in\mathcal D_k$, the tensors $\hat{\bar\pi}$ and $\pi$, periods $u\ge t_k$, and cells $(x,a)$ of positive conditional probability, the discrepancies $|m^\rho_{i,u}(x,a)-\hat m(x,a)|$ and $\|g^\rho_{i,u}(\cdot\mid x,a)-\hat g(\cdot\mid x,a)\|_1$ between the working model and the real conditional projections. Here the projections are $m^\rho_{i,u}(x,a):=\mathbb E^R_\rho[Y_{i,u}\mid\mathcal F_{t_k},X_{i,u}=x,A_{i,u}=a]$ and $g^\rho_{i,u}(x'\mid x,a):=\Pr^R_\rho[X_{i,u+1}=x'\mid\mathcal F_{t_k},X_{i,u}=x,A_{i,u}=a]$. The same two constants also bound these discrepancies for the projections conditional in addition on $A_{i,t_k}$ and $X_{i,t_k+1}$. Write $\hat\theta:=\max_{x,a}|\hat m(x,a)|$. For part (c), consider the population limit, in which each period's reward is its expectation under the propagated distribution and the working model is updated from expected masses (Appendix~\ref{app:simulation}). Let a fee $F$ be paid each period from an initial balance $B_0$. In this case let $r^\pi_u$ and $\hat r_u$ be the fleet rewards at period $u$ under the procedure and under the strategy at $C=0$, let $\underline s_t:=\sum_{u\le t}(\hat r_u-F)$, and let $B^\pi_t:=B_0+\sum_{u\le t}(r^\pi_u-F)$. Let $K(t)$ be the number of reviews at or before $t$. For each review $k$ let $\Sigma^{(k)}$ be the strategy that follows the procedure through epoch $k$ and the strategy at $C=0$ thereafter, and let $F^{(k)}$ and $\hat F^{(k)}$ be the strategies that follow the procedure through epoch $k-1$ and then use, respectively, $\pi^{(k)}$ and $\hat{\bar\pi}^{(k)}$ as fixed tensors through $T$. For a strategy $\sigma$ write $R_t(\sigma):=\sum_{u\le t}r^\sigma_u$, and set $\ell_k(t):=R_t(\Sigma^{(k-1)})-R_t(\Sigma^{(k)})$, $\lambda_k(t):=R_t(\hat F^{(k)})-R_t(F^{(k)})$, and $\iota_k(t):=\ell_k(t)-\lambda_k(t)$. Here $\lambda_k(t)$ is the loss through $t$ of epoch $k$'s departures with both continuations held at $\hat{\bar\pi}^{(k)}$, and $\iota_k(t)$ compares two continuations: the value the later refits add when the epoch's tests are omitted, minus the value they add when they are kept.

\begin{proposition}[Bounded loss and no bankruptcy]\label{prop:purse}
Let Assumptions~\ref{asm:consistency} (consistency and no interference) and~\ref{asm:markov} (Markov full state and common mechanism) hold, and let the objects be as above. Then:
(a) $\operatorname{loss}^M=\sum_{d\in\mathcal P_k}\operatorname{price}(d)\le C$.
(b) $\operatorname{loss}^R\le\operatorname{loss}^M+D_k(2\epsilon_m+\bar V\epsilon_g+2\bar G)\le C+D_k(2\epsilon_m+\bar V\epsilon_g+2\bar G)$, and $\bar G\le(T-t_k)\epsilon_m+\hat\theta\epsilon_g(T-t_k-1)(T-t_k)/2$.
(c) In the population limit, $B^\pi_t=B_0+\underline s_t-\sum_{k\le K(t)}\ell_k(t)$ for each $t$. If in addition Assumption~\ref{asm:correctspec} (correct specification) holds with the centre at the true tables, then $\ell_k(t)\le C+D_k\bar V$ for all $k\le K(t)$, and $B^\pi_t\ge0$ whenever $B_0+\underline s_t\ge\sum_{k\le K(t)}(C+D_k\bar V)$. In general, $\lambda_k(t)\le C+D_k(\bar V+2\epsilon_m+\bar V\epsilon_g+2\bar G)$ where $\epsilon_m$, $\epsilon_g$, $\bar V$, $\bar G$ are the constants defined for review $k$, yielding $B^\pi_t\ge0$ if $B_0+\underline s_t\ge\sum_{k\le K(t)}[C+D_k(\bar V+2\epsilon_m+\bar V\epsilon_g+2\bar G)+\iota_k(t)]$; in particular, the balance stays nonnegative when these bounds are covered and each $\iota_k(t)$ is nonpositive.
\end{proposition}

\begin{proof}
\emph{(a)} The mean world has the product form of \eqref{eq:sim}, and we obtain from Lemma~\ref{lem:decomp} (simulated version) that $V^M(\rho)=\sum_iv^M_i(\rho_i)$, where $v^M_i$ depends on the tensor only through unit $i$'s component. The tensors $\hat{\bar\pi}$, $\pi$, and $\pi_d$ agree at each unit outside $\mathcal D_k$, and $\pi_d$ agrees with $\pi$ on the cohort of $d$ and with $\hat{\bar\pi}$ elsewhere. Since the cohorts $\mathcal C_d$ are disjoint, $\operatorname{loss}^M=\sum_{i\in\mathcal D_k}[v^M_i(\hat\pi_i)-v^M_i(\pi_i)]=\sum_{d\in\mathcal P_k}\sum_{i\in\mathcal C_d}[v^M_i(\hat\pi_i)-v^M_i(\pi_{d,i})]=\sum_{d\in\mathcal P_k}\operatorname{price}(d)$, and by the packing rule $\sum_d\operatorname{price}(d)\le C$.

\emph{(b)} We next prove (b). Under $R$, Lemma~\ref{lem:decomp} applies conditionally on $\mathcal F_{t_k}$ by Assumption~\ref{asm:consistency}, so we obtain likewise $\operatorname{loss}^R=\sum_{i\in\mathcal D_k}[v^R_i(\hat\pi_i)-v^R_i(\pi_i)]$ and
\begin{equation*}
\operatorname{loss}^R-\operatorname{loss}^M=\sum_{i\in\mathcal D_k}\Big\{\big[v^R_i(\hat\pi_i)-v^R_i(\pi_i)\big]-\big[v^M_i(\hat\pi_i)-v^M_i(\pi_i)\big]\Big\}.
\end{equation*}
Fix $i\in\mathcal D_k$, write $x:=X_{i,t_k}$, and abbreviate $\Delta(a):=\hat\pi_{i,t_k}(a\mid x)-\pi_{i,t_k}(a\mid x)$, so that $\sum_a|\Delta(a)|\le2$. Both tensors follow $\hat{\bar\pi}$ after period $t_k$. Under (R) the period-$t_k$ action is drawn from the tensor's distribution independently of everything else; conditioning on it and on the next state then gives, for $\rho\in\{\hat{\bar\pi},\pi\}$,
\begin{align*}
v^R_i(\rho_i)&=\sum_a\rho_{i,t_k}(a\mid x)\Big[m_{i,t_k}(x,a)+\sum_{x'}g_{i,t_k}(x'\mid x,a)\,w^R_{i,a,T}(x')\Big],\\
v^M_i(\rho_i)&=\sum_a\rho_{i,t_k}(a\mid x)\Big[\hat m(x,a)+\sum_{x'}\hat g(x'\mid x,a)\,w^M_{i,t_k+1,T}(x')\Big],
\end{align*}
where $m_{i,t_k}(x,a)=\mathbb E[Y_{i,t_k}(\bar A_{i,t_k-1},a)\mid\mathcal F_{t_k}]$ and $g_{i,t_k}(\cdot\mid x,a)$ do not depend on $\rho$, because the outcome and next state given the action are the same potential quantities under both tensors, and $w^R_{i,a,T}$ does not depend on $\rho$, because both tensors continue with $\hat{\bar\pi}$. Subtracting,
\begin{align*}
&\big[v^R_i(\hat\pi_i)-v^R_i(\pi_i)\big]-\big[v^M_i(\hat\pi_i)-v^M_i(\pi_i)\big]\\
&\quad=\sum_a\Delta(a)\big[m_{i,t_k}-\hat m\big](x,a)\\
&\qquad+\sum_a\Delta(a)\sum_{x'}g_{i,t_k}(x'\mid x,a)\big[w^R_{i,a,T}-w^M_{i,t_k+1,T}\big](x')\\
&\qquad+\sum_a\Delta(a)\sum_{x'}\big[g_{i,t_k}-\hat g\big](x'\mid x,a)\,w^M_{i,t_k+1,T}(x').
\end{align*}
Each $a$ with $\Delta(a)\ne0$ is assigned with positive probability at $x$ under one of the two tensors, which supports the cell $(x,a)$ and makes the bounds $\epsilon_m,\epsilon_g$ apply. The first term is at most $\sum_a|\Delta(a)|\epsilon_m\le2\epsilon_m$ in absolute value, and the second at most $\sum_a|\Delta(a)|\sum_{x'}g_{i,t_k}(x'\mid x,a)\bar G\le2\bar G$. For the third term, since $\sum_{x'}[g_{i,t_k}-\hat g](x'\mid x,a)=0$ for each $a$, subtracting the constant $c:=\tfrac12[\max_{x'}w^M+\min_{x'}w^M]$ from $w^M_{i,t_k+1,T}$ leaves the sum unchanged. This yields $|\sum_{x'}[g_{i,t_k}-\hat g](x'\mid x,a)(w^M(x')-c)|\le\|g_{i,t_k}(\cdot\mid x,a)-\hat g(\cdot\mid x,a)\|_1\bar V/2\le\epsilon_g\bar V/2$; summing against $|\Delta(a)|$, we obtain at most $\epsilon_g\bar V$. Adding the three bounds and summing over the $D_k$ tested units establishes the first inequality of (b), and (a) gives the second. It remains to bound $\bar G$. Fix $i$, $a$, $x'$, $t'$ and condition on $\mathcal F_{t_k}$, $A_{i,t_k}=a$, $X_{i,t_k+1}=x'$. Under the fixed tensor $\hat{\bar\pi}$ from period $t_k+1$, the marginal flow identities hold with the conditional projections, by the tower rule and (R), precisely as in Lemma~\ref{lem:flow}. The argument of Proposition~\ref{prop:simgap} for one unit applies with $\hat m$ and $\hat g$ in the roles of $\theta^{\mathrm{sim}}$ and $f^{\mathrm{sim}}$, with $\hat\theta$ in the role of $\bar\theta$, and with the two state marginals starting at the same point mass $\delta_{x'}$, whence $e_0=0$. The difference between the two values over the $t'-t_k\le T-t_k$ periods is thus at most $\sum_{s=0}^{T-t_k-1}(\epsilon_m+\hat\theta s\epsilon_g)=(T-t_k)\epsilon_m+\hat\theta\epsilon_g(T-t_k-1)(T-t_k)/2$.

\emph{(c)} We now prove (c). By definition, $B^\pi_t-(B_0+\underline s_t)=\sum_{u\le t}(r^\pi_u-\hat r_u)=R_t(\Sigma^{(K(t))})-R_t(\Sigma^{(0)})$, since $\Sigma^{(0)}$ is the strategy at $C=0$ and $\Sigma^{(K(t))}$ follows the procedure through epoch $K(t)$, which contains period $t$. Telescoping, $R_t(\Sigma^{(K(t))})-R_t(\Sigma^{(0)})=\sum_{k=1}^{K(t)}[R_t(\Sigma^{(k)})-R_t(\Sigma^{(k-1)})]=-\sum_{k\le K(t)}\ell_k(t)$, which is the identity asserted in (c). Suppose now that Assumption~\ref{asm:correctspec} also holds. Then the real distribution is the simulated distribution of \eqref{eq:sim} with $(\theta^{\mathrm{sim}},f^{\mathrm{sim}})$, and in the population limit each observed mass equals its expectation, whence at each cell $R=N\theta^{\mathrm{sim}}(x,a)$ and the next-state masses are $Nf^{\mathrm{sim}}(\cdot\mid x,a)$, under any strategy, tests included. The shrinkage rule with centre $\theta^{\mathrm{sim}}$ then returns $(\theta^{\mathrm{sim}},f^{\mathrm{sim}})$ at each period, and so does the structural centre, since the regime fit minimizes a sum of squared residuals over visited cells plus a ridge term toward $\nu^{\mathrm{sim}}$. Both terms vanish at $\nu=\nu^{\mathrm{sim}}$, and the second vanishes nowhere else. We conclude that $\hat\nu=\nu^{\mathrm{sim}}$, and the structural tables are $(\theta^{\mathrm{sim}},f^{\mathrm{sim}})$ by Assumption~\ref{asm:calib}. Each review's plug-in policy is therefore $\hat{\bar\pi}=S(\theta^{\mathrm{sim}},f^{\mathrm{sim}})$ by (S), the mean world of each review is the real distribution, and $\Sigma^{(k)}$ and $\Sigma^{(k-1)}$ are the fixed tensors $\pi^{(k)}$ and $\hat{\bar\pi}^{(k)}$ of (a) continued by $\hat{\bar\pi}$, differing only on epoch $k$'s departure slices. By Lemma~\ref{lem:decomp} and the decomposition of (b) with the window $\{t_k,\dots,t\}$ in place of $\{t_k,\dots,T\}$, $\ell_k(t)=\sum_{i\in\mathcal D_k}\sum_a\Delta_i(a)[\hat m(x_i,a)+\sum_{x'}\hat g(x'\mid x_i,a)w^M_{i,t_k+1,t}(x')]$, while $\operatorname{price}(d)$ is the same expression with $w^M_{i,t_k+1,T}$ in place of $w^M_{i,t_k+1,t}$. The mean-world value from $t_k+1$ through $T$ is the value through $t$ plus the value from $t+1$ through $T$ of the continuation, giving $w^M_{i,t_k+1,T}(x')=w^M_{i,t_k+1,t}(x')+\sum_{x''}q_{i,t}(x''\mid x')w^M_{i,t+1,T}(x'')$ with $q_{i,t}(\cdot\mid x')$ the distribution of the state at $t+1$ under $\hat{\bar\pi}$ from $x'$. Therefore
\begin{equation*}
\ell_k(t)-\sum_{d\in\mathcal P_k}\operatorname{price}(d)=-\sum_{i\in\mathcal D_k}\sum_a\Delta_i(a)\sum_{x'}\hat g(x'\mid x_i,a)\sum_{x''}q_{i,t}(x''\mid x')\,w^M_{i,t+1,T}(x''),
\end{equation*}
and for each $i$ the inner double sum is a difference of two averages of $w^M_{i,t+1,T}$ (the weights $\sum_{x'}\hat g(x'\mid x_i,a)q_{i,t}(x''\mid x')$ sum to one for each $a$ and $\sum_a\Delta_i(a)=0$), hence at most $\bar V$ in absolute value; at $t=T$ the continuation is empty and the right side is zero. Thus $\ell_k(t)\le\sum_d\operatorname{price}(d)+D_k\bar V\le C+D_k\bar V$, and the identity implies $B^\pi_t\ge B_0+\underline s_t-\sum_{k\le K(t)}(C+D_k\bar V)$. With $F^{(k)}$, $\hat F^{(k)}$, $\lambda_k(t)$, and $\iota_k(t)$ as defined before the statement, $\iota_k(t)=[R_t(\Sigma^{(k-1)})-R_t(\hat F^{(k)})]-[R_t(\Sigma^{(k)})-R_t(F^{(k)})]$. The quantity $\lambda_k(t)$ is the real-distribution loss of (b) over the window $\{t_k,\dots,t\}$. By the argument of (b) with $T$ replaced by $t$, this quantity is at most the mean-world loss through $t$ plus $D_k(2\epsilon_m+\bar V\epsilon_g+2\bar G)$, where the maxima defining $\bar V$ and $\bar G$ cover the shorter window. The displacement computation above then bounds the mean-world loss through $t$ by $\sum_d\operatorname{price}(d)+D_k\bar V\le C+D_k\bar V$. Substituting into the identity gives the last display of (c).
\end{proof}
\begin{restated}{Assumption}{asm:correctspec}{Correct specification with cross-unit independence, formal statement}

Fix a period $t$ and an assignment vector $\mathbf a$, and condition on the predecision history $\mathcal F_t$ and on $\mathbf A_t=\mathbf a$. Suppose that the tuples $(Y_{i,t},X_{i,t+1})$, $i\in\mathcal V$, are mutually independent across units, and that for each $i$ the conditional distribution of that tuple depends only on $(X_{i,t},a_i)=(x,a)$ and equals the calibrated simulator's kernel, that is, $Y=\theta^{\mathrm{sim}}(x,a)+\eta$ with $\eta\sim Q(\cdot\mid x,a)$, $X'\sim f^{\mathrm{sim}}(\cdot\mid x,a)$, and $Y$ independent of $X'$ given $(x,a)$. Suppose finally that the initial observed states have the distribution $\mu_X$, the $\mathcal X$-marginal of the initial full-state distribution, and that the simulator also starts from $\mu_X$. These conditions are assumed to hold for each period and each assignment vector.
\end{restated}

Assumption~\ref{asm:correctspec} is equivalent to requiring that the real distribution equal the simulated distribution of \eqref{eq:sim} under each strategy.

\begin{restated}{Theorem}{thm:nochange}{No-change boundary, general discount}

Let $\sigma$ be a history-dependent strategy and suppose that Assumption~\ref{asm:correctspec} (correct specification) holds. Then $V^\gamma_T(\sigma)\leq V^\gamma_T(\sigma^{\mathrm{sop}})=\max_{\bar\pi\in\bar\Pi}V^\gamma_T(\bar\pi)$. Suppose in addition that Assumptions~\ref{asm:cellwise} and~\ref{asm:likelihood} hold, that the reward noise kernels $Q(\cdot\mid x,a)$ are point masses at zero, and that the calibrated transitions $f^{\mathrm{sim}}(\cdot\mid x,a)$ are point masses. Then $\sigma^{\mathrm{asop}}_\kappa$ deploys $\bar\pi^{\mathrm{sim}}$ at each review, and $V^\gamma_T(\sigma^{\mathrm{asop}}_\kappa)=V^\gamma_T(\sigma^{\mathrm{sop}})$.
\end{restated}

\begin{proof}
We first bound the value inside the simulated distribution and identify its maximizer there, then carry the bound to the real system and settle the degenerate case, which together give the no-change boundary.

\emph{Step 1.} We work inside the simulated distribution. Fix a history-dependent $\sigma$ with internal randomization $\xi$ from a fixed environment-independent distribution; conditional on $\xi$, each slice is $\mathcal F_t$-measurable. Under Assumption~\ref{asm:correctspec}, the conditional distribution of $(\mathbf Y_t,\mathbf X_{t+1})$ given $(\mathcal F_t,\mathbf A_t)$ equals the simulator's product kernel, and induction over $t$ then equates the trajectory distributions of \eqref{eq:sim}. We conclude that $V^\gamma_T(\sigma)=V^{\gamma,\mathrm{sim}}_T(\sigma;\theta^{\mathrm{sim}},f^{\mathrm{sim}})$, and it suffices to argue inside the simulated distribution, a finite-horizon Markov decision process with state $\mathbf X_t\in\mathcal X^n$, product kernel, and additive mean reward $R(\mathbf x,\mathbf a)=\sum_i\theta^{\mathrm{sim}}(x_i,a_i)$.

\emph{Step 2.} We bound the value by dynamic programming. Set $V_{T+1}\equiv0$ and, downward for $t=T,\dots,0$,
\begin{equation*}
\begin{gathered}
Q_t(\mathbf x,\mathbf a):=\sum_i\theta^{\mathrm{sim}}(x_i,a_i)+\gamma\,\mathbb E\big[V_{t+1}(\mathbf X')\big],\qquad \mathbf X'\sim\textstyle\prod_i f^{\mathrm{sim}}(\cdot\mid x_i,a_i),\\
V_t(\mathbf x):=\max_{\mathbf a\in\mathcal A^n}Q_t(\mathbf x,\mathbf a).
\end{gathered}
\end{equation*}
Downward induction establishes $\mathbb E_\sigma[\sum_{s=t}^T\gamma^{s-t}\sum_iY_{i,s}\mid\mathcal F_t]\leq V_t(\mathbf X_t)$ almost surely: at $t$, conditioning on $\mathbf A_t$, the summed reward across units has mean $R(\mathbf X_t,\mathbf A_t)$, the continuation is at most $V_{t+1}(\mathbf X_{t+1})$ by the downward induction hypothesis, and $\mathbb E[V_{t+1}(\mathbf X_{t+1})\mid\mathcal F_t,\mathbf A_t]$ depends only on current cells by Assumption~\ref{asm:correctspec}, so that the total is at most $Q_t(\mathbf X_t,\mathbf A_t)\leq V_t(\mathbf X_t)$. Averaging over $\mathbf A_t$ at $t=0$ gives $V^\gamma_T(\sigma)\leq\mathbb E[V_0(\mathbf X_0)]$.

\emph{Step 3.} We identify the maximizer. Since the kernel is a product and the reward additive, $Q_t(\mathbf x,\mathbf a)=\sum_iQ_{i,t}(x_i,a_i)$ and $V_t(\mathbf x)=\sum_iV_{i,t}(x_i)$, with per-unit recursions $V_{i,t}(x)=\max_a\{\theta^{\mathrm{sim}}(x,a)+\gamma\sum_{x'}f^{\mathrm{sim}}(x'\mid x,a)V_{i,t+1}(x')\}$. The fixed selector then defines a memoryless $\bar\pi^\star\in\bar\Pi$ with $V^\gamma_T(\bar\pi^\star)=\mathbb E[V_0(\mathbf X_0)]$. This yields $\sup_\sigma V^\gamma_T(\sigma)=\max_{\bar\pi\in\bar\Pi}V^\gamma_T(\bar\pi)$.

\emph{Step 4.} We transfer this to the real distribution. By Step 1, $V^\gamma_T\equiv V^{\gamma,\mathrm{sim}}_T$ on $\bar\Pi$, and $\bar\pi^{\mathrm{sim}}$ is thus a real-distribution maximizer over $\bar\Pi$; combining this with Step 2 we obtain the stated inequality, and in particular $V^\gamma_T(\sigma^{\mathrm{asop}}_\kappa)\leq V^\gamma_T(\sigma^{\mathrm{sop}})$.

\emph{Step 5.} It remains to treat the degenerate case. With $Q(\cdot\mid x,a)$ a point mass at zero, each realized reward at a visited cell equals $\theta^{\mathrm{sim}}(x,a)$, the deterministic calibrated transitions are at no point contradicted, and unvisited cells retain their prior ($\Xi_0$) marginals by Proposition~\ref{prop:censoring}. Because each posterior mean under the working model is a precision-weighted average of the prior centre and the observations, observations equal to the centre leave the visited cells' posterior means at $(\theta^{\mathrm{sim}},f^{\mathrm{sim}})$ as well. Each A-SOP simulated-objective maximization then has inputs $(\theta^{\mathrm{sim}},f^{\mathrm{sim}})$, and the shared selector consequently selects $\bar\pi^{\mathrm{sim}}$ at each review, yielding $V^\gamma_T(\sigma^{\mathrm{asop}}_\kappa)=V^\gamma_T(\sigma^{\mathrm{sop}})$.
\end{proof}

\begin{restated}{Theorem}{thm:hiddendrift}{Quadratic regret under hidden drift, formal statement}

Fix $D>0$, $\Delta_0>0$, and a horizon $T\geq T_0:=\Delta_0/D$. Then there is an instance satisfying Assumptions~\ref{asm:consistency}--\ref{asm:calib}, \ref{asm:smooth} (with $L_\theta=1$, $L_f=0$), \ref{asm:drift} (with $\lambda=D$), \ref{asm:cellwise}, and~\ref{asm:likelihood}, for which, at $\gamma=1$: (i) Both $\sigma^{\mathrm{sop}}$ and $\sigma^{\mathrm{asop}}_\kappa$, for any $\kappa$, deterministically assign the same action at each period. (ii) The summed conditional per-period regret of that path is at least $D(T-T_0)^2/2$. (iii) $V^1_T(\bar\pi^*)-V^1_T(\sigma^{\mathrm{sop}})\geq D\big(\tfrac{T+1-T_0}{2}\big)^2-\tfrac D4$, and the same lower bound holds against $\sigma^{\mathrm{asop}}_\kappa$.
\end{restated}

\begin{proof}
We first build the instance and verify the assumptions, then establish the lock-in, sum the regret along the realized path, and compare it with the best switch policy, which yields the quadratic-regret bound under hidden drift.

We construct the instance first. Take one unit (independent copies for general $n$), one observed state $x$ with $X_{t+1}=x$ always, and actions $\mathcal A=\{a_0,a_1\}$. Let the latent set be $\Ucal=\{0,1,2,\dots\}$ with $d_U(u,u')=D|u-u'|$, and let the homogeneous kernel advance $u\mapsto u+1$ under $a_0$ and leave $u$ unchanged under $a_1$. Since the operating policy assigned $a_1$ throughout the historical record, $\nu^{\mathrm{sim}}=\delta_0$ and Assumption~\ref{asm:calib} holds with $\theta^{\mathrm{sim}}(x,a_0)=\mu_0$ and $\theta^{\mathrm{sim}}(x,a_1)=\mu_0-\Delta_0$. Set the responses $\theta(x,a_0,u)=\mu_0$ and $\theta(x,a_1,u)=\mu_0-\Delta_0+D\min(u,T{+}1)$, capped to keep outcomes bounded, with reward noise degenerate, $Q(\cdot\mid x,a)$ a point mass at zero. We check the bookkeeping. The response bound $|\theta(x,a_1,u)-\theta(x,a_1,u')|\leq D|u-u'|=d_U(u,u')$ yields \eqref{eq:liptheta} with $L_\theta=1$, provided $d_A(a_0,a_1)\geq\Delta_0+D(T{+}1)$; the geometry is horizon-indexed, and the invariant object is $L_\theta\lambda=D$. Since the transition is constant, \eqref{eq:lipf} holds with $L_f=0$. One increment per period gives $W_1(\nu_{t+1},\nu_t)\leq D=:\lambda$, that is, \eqref{eq:drift}. Finally, \eqref{eq:lipm} holds with $L_m=1$ under the same $d_A$, since the cross-cell increment satisfies $|m_t(x,a_0)-m_t(x,a_1)|\leq\Delta_0+D(T{+}1)\leq d_A(a_0,a_1)$. The drift is hidden, since at the visited cell $(x,a_0)$ neither the response nor the transition depends on $u$.

We prove (i), the lock-in. With one observed state, SOP's objective decomposes periodwise, and since $\mu_0>\mu_0-\Delta_0$ under the calibrated comparison, $a_0$ is the per-period maximizer, so that $\bar\pi^{\mathrm{sim}}\equiv a_0$. Under deployment, each observed reward at $(x,a_0)$ equals $\mu_0=\theta^{\mathrm{sim}}(x,a_0)$, so A-SOP's pooled update holds that working mean at $\mu_0$, while the unvisited cell $(x,a_1)$ retains its calibration mean $\mu_0-\Delta_0$ by Proposition~\ref{prop:censoring}. At each review the myopic rule $\arg\max_a\hat m_t(x,a)$ hence selects $a_0$ again, by Corollary~\ref{cor:lockin}, applied with the visited-cell and unvisited-cell means $\alpha=\mu_0>b_0=\mu_0-\Delta_0$, for each cadence $\kappa$.

We prove (ii). Along the realized path $u_t=t$, we find $\mu_t(a_0)=\mu_0$ and $\mu_t(a_1)=\mu_0-\Delta_0+Dt$ for $t\leq T$, and assigning $a_0$ at $t$ has conditional per-period regret $D\,(t-T_0)_+$. We use the elementary inequality
\begin{equation}
\sum_{t=0}^{n}(t-T_0)_+\;\geq\;\frac{\big((n-T_0)_+\big)^2}{2}\qquad\text{for all real }T_0\geq0,
\label{eq:regretsum}
\end{equation}
which is proved by induction on $n$; the step is the pointwise inequality $x_+^2/2\leq((x-1)_+)^2/2+x_+$ at $x=n+1-T_0$, checked on $x\leq0$, $0<x\leq1$, and $x>1$. Applying \eqref{eq:regretsum} at $n=T$ gives summed regret at least $D(T-T_0)^2/2$.

We prove (iii), comparing with the best fixed policy. Under the switch-at-$s$ policy $\bar\pi^{(s)}\in\bar\Pi$, which assigns $a_0$ for $t<s$ and $a_1$ thereafter, $u$ increments to $s$ under $a_0$ and remains $s$ after the switch to $a_1$. The value of this policy is therefore $V^1_T(\bar\pi^{(s)})=s\mu_0+(T{+}1-s)(\mu_0-\Delta_0+Ds)=(T{+}1)\mu_0+D(T{+}1-s)(s-T_0)$. The parabola $s\mapsto(T{+}1-s)(s-T_0)$ peaks at $s^\star=\tfrac{T+1+T_0}{2}$ with value $\big(\tfrac{T+1-T_0}{2}\big)^2$, and rounding the maximizer to the nearest integer, a displacement of at most $\tfrac12$, loses at most $\tfrac12\cdot(\tfrac12)^2$ times the parabola's (constant) second difference $2D$, that is, at most $D/4$. Since lock-in yields $V^1_T(\sigma^{\mathrm{sop}})=V^1_T(\sigma^{\mathrm{asop}}_\kappa)=(T{+}1)\mu_0$ and $\bar\pi^*$ weakly improves on the best switch policy, the display follows.
\end{proof}
We sharpen the detection--price tradeoff into an exact undetectability result: when two environments differ only at the cells the operating strategy never visits, no stopping rule distinguishes them.

\begin{theorem}[Identical operating-history distributions and a size--power limit]
\label{thm:undetectable}
Let $\sigma$ be a history-dependent strategy with internal randomization $\xi$ under (R); we call $\sigma$ the \emph{operating strategy} and the predecision history $\mathcal F_t$ that it generates the \emph{operating history}. Let $P_0$ and $P_1$ be two environments that satisfy Assumptions~\ref{asm:consistency} and~\ref{asm:markov} and that agree on the initial distribution of the full state and on the distribution of the historical record. Call a cell $(x,a)$ \emph{positive-support} if $P_0[X_{i,t}=x,A_{i,t}=a]>0$ for some unit $i$ and some $t\le T$ under $\sigma$, and \emph{zero-support} otherwise. Suppose that at each positive-support cell the two environments share the response kernel, the conditional distribution of the outcome given the full state and the action, with mean $\theta(x,a,\cdot)$, and the transition kernel $f(\cdot\mid x,a,\cdot)$ on $\Ucal$, and that they differ at some zero-support cell. Suppose also that in both environments the innovations of distinct units are conditionally independent given the current full states and assignments, a condition that Assumption~\ref{asm:crossunit} (cross-unit independence) implies. Let $\tau$ be a stopping time with respect to the filtration generated by $\xi$, by the operating history, and by the internal randomization of the stopping rule, the latter drawn once from a fixed environment-independent distribution. Then (i) $P_0(\tau\leq t)=P_1(\tau\leq t)$ for all $t\leq T$; and (ii) no such rule has size $P_0(\tau\le T)\leq\alpha$ under $P_0$ and power $P_1(\tau\le T)\geq1-\vartheta$ under $P_1$ unless $\alpha+\vartheta\geq1$.
\end{theorem}

\begin{proof}
We show that the two environments produce the same joint distribution of the randomizations and the full-state operating history, from which both conclusions follow. Write $\xi'$ for the stopping rule's internal randomization and $\mathcal G_t$ for the $\sigma$-field generated by $\xi$, $\xi'$, the historical record, the full states and actions $(X_{i,s},U_{i,s},A_{i,s})_{i\in\mathcal V,s<t}$, the outcomes $(Y_{i,s})_{i\in\mathcal V,s<t}$, and the period-$t$ full states $(X_{i,t},U_{i,t})_{i\in\mathcal V}$; the operating history $\mathcal F_t$ is a sub-$\sigma$-field of $\mathcal G_t$. We argue by induction over periods that the joint distribution of the generators of $\mathcal G_t$ agrees under $P_0$ and $P_1$. The base case holds because $\xi$ and $\xi'$ are drawn from environment-independent distributions and the historical record and the initial full states share one distribution. For the inductive step, the period-$t$ assignment is a measurable function of $(\xi,\mathcal F_t)$, hence of the generators of $\mathcal G_t$, and before $\tau$ the realized path is the operating one, so no experimental data enter the assignment. By the induction hypothesis, $P_1[X_{i,t}=x,A_{i,t}=a]=P_0[X_{i,t}=x,A_{i,t}=a]$ for every cell, so a zero-support cell is reached with probability zero under both environments at period $t$. Conditionally on $(X_{i,t},U_{i,t},A_{i,t})=(x,u,a)$ at a positive-support cell, the innovation $(Y_{i,t},X_{i,t+1},U_{i,t+1})$ is drawn from the response kernel at $(x,a,u)$ and from $f(\cdot\mid x,a,u)$, on which the environments agree by assumption, and the innovations of distinct units are conditionally independent given the full states and assignments by assumption. Hence the generators of $\mathcal G_{t+1}$ share one distribution, which completes the induction. Equivalently, the likelihood ratio of the full-state operating history equals $1$ almost surely, so the two environments induce the same distribution of $(\xi,\xi',\mathcal F_t)$ for every $t$; this is (i), because $\{\tau\leq t\}\in\sigma(\xi,\xi',\mathcal F_t)$. For (ii), the distribution function of $\tau$ is therefore the same under both environments, so $P_1(\tau\leq T)=P_0(\tau\leq T)\leq\alpha$, which forces $1-\vartheta\leq\alpha$.
\end{proof}

Agreement ``on the kernels at a cell'' is agreement of the full-state objects $\theta(x,a,\cdot)$ and $f(\cdot\mid x,a,\cdot)$ on $\Ucal$; projection agreement alone would not equate the history distributions (Lemma~\ref{lem:nonmarkov}). The positive-support cells are the same under $P_0$ and $P_1$, since the induction shows period by period that the two environments give each cell the same visit probability.

\subsection{Probability toolbox}
\label{app:toolbox}

We state two standard facts, cited to their classical sources, that later proofs in this appendix use: Hoeffding's inequality in a concentration step and the total-variation facts in a coupling bound.

\begin{lemma}[Hoeffding's lemma and inequality]
\label{lem:hoeffding}
(i) Let $X\in[a,b]$ almost surely with $\mathbb EX=0$. Then $\mathbb Ee^{sX}\leq e^{s^2(b-a)^2/8}$ for all $s\in\mathbb R$. (ii) Let $X_1,\dots,X_K$ be independent with $X_k\in[a,b]$ almost surely, set $\bar X:=K^{-1}\sum_kX_k$, and $\mu:=\mathbb E\bar X$. Then $\Pr(|\bar X-\mu|\ge\varepsilon)\leq2\exp(-2K\varepsilon^2/(b-a)^2)$ for all $\varepsilon>0$.
\end{lemma}

Both are standard \citep{hoeffding1963probability}; see also \citet[Section~2.6]{boucheron2013concentration}.

\begin{lemma}[Total-variation facts]
\label{lem:tvfacts}
For probability measures on a common measurable space, with $\TV(\mu,\nu)=\sup_A|\mu(A)-\nu(A)|$: (i) $|\int f\,d\mu-\int f\,d\nu|\leq2\|f\|_\infty\TV(\mu,\nu)$ for each bounded measurable $f$; (ii) there is a coupling $(X,Y)$ with marginals $\mu,\nu$ and $\Pr(X\neq Y)=\TV(\mu,\nu)$; and (iii) $\TV(\bigotimes_{i=1}^n\mu_i,\bigotimes_{i=1}^n\nu_i)\leq\sum_{i=1}^n\TV(\mu_i,\nu_i)$.
\end{lemma}

\begin{proof}
Parts (i) and (ii) are the standard total-variation identities, the Jordan-decomposition bound and the maximal coupling \citep[Chapter~4]{levin2009markov}. For (iii), couple each pair $(X_i,Y_i)$ maximally as in (ii) and independently across $i$; then $(X_1,\dots,X_n)$ and $(Y_1,\dots,Y_n)$ have the two product distributions, and the coupling inequality gives $\TV(\bigotimes_i\mu_i,\bigotimes_i\nu_i)\le\Pr(\exists i:X_i\neq Y_i)\le\sum_i\Pr(X_i\neq Y_i)=\sum_i\TV(\mu_i,\nu_i)$.
\end{proof}\section{Simulation Specification}
\label{app:simulation}

In this appendix, we describe the environment, the calibration, the procedure implementations, and the evaluation protocol behind the reported comparisons. Throughout we distinguish the idealized method definitions from the implemented approximations we actually run (Appendix~\ref{app:moved}), and we note where hiddenness follows from the cell-wise updater rather than from the raw censored-demand data.

The environment is a $3\times3$ grid of zones, each zone holding $n_{\mathrm{zone}}$ machines, with the $L^1$ metric between zones; one period is one day. The study runs the benchmark at $n_{\mathrm{zone}}\in\{3,111,1111\}$, that is, fleets of $27$, $999$, and $9{,}999$ machines; the main text reports the population limit at $999$ machines, the recurrent-shift check in the population limit at $27$ machines, and the bankruptcy and design-control checks on sampled fleets of $27$ machines. The observed state is $(m_1,\dots,m_4,p_1,p_2)$. Here $m_j$ counts the units that arrived $j$ days ago, with one such age slot for each day of shelf life after arrival, $p_1$ counts those arriving this morning, and $p_2$ those arriving tomorrow, under the capacity $\sum_jm_j+p_1+p_2\leq k=3$. The shelf life $q=4$ (a unit is sellable on its arrival day and four further days) leaves $\binom{9}{6}=84$ inventory states, the placements of the capacity-bounded units across the six state slots. We admit six actions, with quantities clipped to the remaining capacity: hold; a local order $j\in\{1,2\}$ (same-day, joining the age-0 pool); and a warehouse order $j\in\{1,2,3\}$ (two-day lead, entering $p_2$). Each day $p_1$ merges into the age-0 pool, the selected action is applied, demand is drawn and served oldest-first, unsold units age one day, and units ending the day at age $4$ are discarded. The per-period reward is $2.5\times\text{sales}-c\times\text{order}-0.2\times\text{expired}$, with $c=1.5$ for a local order and $c=0.8$ for a warehouse order; we do not clip the environment rewards, since the clip to $[-4,13]$ applies only to the synthetic response tables.

We draw daily demand as $\mathrm{Binomial}(5,\,\bar d/5)$, where $\bar d=\text{regime mean}\times\text{zone multiplier}$: the regime means are $(0.6,1.8,3.6)$ for low, mid, and high, and the multiplier is $0.8+0.4\,(i+j)/4$ at grid coordinates $(i,j)$, with this index distinct from the age index above. The never-observed per-machine regime $U_{i,t}\in\{\text{low},\text{mid},\text{high}\}$ evolves by the daily kernel $K=(1-\beta)I+\beta\,\mathbf 1\nu_\infty^{\!\top}$, with $\nu_\infty$ the zone's stationary distribution, yielding, from $U_{-1}\sim\nu^{\mathrm{sim}}$ at the close of the calibration era, $\nu_t=(1-\beta)^{t+1}\nu^{\mathrm{sim}}+(1-(1-\beta)^{t+1})\nu_\infty$; this is the shift schedule defined in Section~\ref{sec:sim}, at speed $\beta$. The calibration distribution is $\nu^{\mathrm{sim}}=(0.5,0.4,0.1)$ in each zone. We set three change profiles. The ``boom'' corner $(0,0)$ moves toward $(0.05,0.25,0.70)$, a hidden change, since exploitation requires orders that the deployed policy never places; the ``bust'' corner $(2,2)$ moves toward $(0.92,0.06,0.02)$, a visible one, because the operator observes the fall in sales in the operating data; and the other zones stay at $\nu^{\mathrm{sim}}$. The displacement $\mu$ and the visible share $\phi$ scale how far a corner moves toward its target distribution, through the boom weight $\mu(1-\phi)$ and the bust weight $\mu\phi$; we use $\mu=1$ throughout the simulation study. Machines start empty, with $U_0\sim\nu^{\mathrm{sim}}K$, one kernel step past the calibration era.

The stress-test family keeps the grid, the price, and the demand model. In this case it replaces the calibration distribution by $\nu^{\mathrm{sim}}=(0.85,0.12,0.03)$, the shelf life by $q=1$ (a unit sellable on its arrival day and one further day, for $20$ inventory states), the disposal cost by $1.0$, and the warehouse cost by $1.2$. Each non-corner zone moves toward the grown regime with weight $0.8$. This makes the background displacement boom-type and hidden under the lean deployed policy. The corner zones follow the boom and bust profiles with weights $\mu(1-\phi)$ and $\mu\phi$, as in the baseline family. In the no-change worlds of each family we set each displacement weight to zero, one such world at each $\beta$. The no-change column of Table~\ref{tab:cast} uses the one at $\beta=0.01$.

The baseline and stress families differ in two properties at once: the breadth of the change (corner zones only, or every zone) and the economics of a mistake (the calibration distribution with disposal $0.2$, or the lean distribution with $q=1$, disposal $1.0$, and warehouse cost $1.2$). Two further families complete the $2\times2$ on the same grid. The \emph{sparse-and-expensive} family keeps the baseline displacement (corner zones only, background weight $0$) under the stress economics, and the \emph{broad-and-cheap} family applies the stress family's background weight $0.8$ under the baseline economics. Because the no-change worlds depend only on the economics, the two families of each economics share the placebo worlds at $\beta\in\{0.003,0.01,0.03,0.1,0.3\}$.

We calibrate the simulator analytically to the marginal regime distribution, setting $\theta^{\mathrm{sim}}(x,a)=\sum_u\theta(x,a,u)\,\nu^{\mathrm{sim}}(u)$ and $f^{\mathrm{sim}}(\cdot\mid x,a)=\sum_u\nu^{\mathrm{sim}}(u)\,f_X(\cdot\mid x,a,u)$, where $\theta$ and $f_X$ are the exact tables implied by the dynamics. Assumption~\ref{asm:calib} concerns the record's conditional distribution of the regime given the state. Under regime persistence ($\beta<1$) a machine's stock is informative about its regime, so the two distributions differ, and the assumption holds only up to that dependence, which weakens as $\beta$ grows. Because of this dependence, strategies whose tables condition on the stock have a positive no-change margin, their value over SOP in percent. At $999$ machines, in the stress no-change worlds, the margin of adaptation is $1.7$ percent below SOP at $\beta=0.003$ and $+0.7$, $+4.0$, $+2.9$, and $+1.4$ percent at $\beta=0.01$, $0.03$, $0.1$, and $0.3$. In the stress no-change world at $\beta=0.01$ the daily re-solve alternates between two orders of equal working value in three zones. Each strategy uses response and transition estimates shrunk toward the simulator tables with prior weight $n_0=25$, namely $\hat m=(n_0\theta^{\mathrm{sim}}+R)/(n_0+N)$ and $\hat g=(n_0f^{\mathrm{sim}}+G)/(n_0+N)$, where $R$, $N$, and $G$ are the per-cell reward, visit, and next-state masses (expected in the population limit, realized in finite fleets). The forgetting variants multiply $(R,N,G)$ by $0.99$ daily, and we take $s_0=\operatorname{sd}(\theta^{\mathrm{sim}})$ over all cells ($1.04$ at the defaults).

For each strategy, the policy computations use batched value iteration with a receding $24$-day horizon at discount factor $\gamma=1$; because the effective horizon is much shorter than $24$, each evaluation horizon is a prefix sum of one $T=1024$ run. SOP computes one policy under the calibrated tables and does not update it, and ties in the policy argmax resolve to the lowest action index in the order hold, local $1$, local $2$, warehouse $1$, $2$, $3$. The five other strategies are the procedure at the settings of Table~\ref{tab:settings}, implemented in one runner. S-SOP refits each zone's regime mix daily by nonnegative least squares on the simplex, using the visited-cell means of the zone and its neighbours (neighbours at half weight) and a ridge toward $\nu^{\mathrm{sim}}$ of weight $n_0$, and it follows the plug-in policy of the fitted table. Under the UCB rule we solve the table $\hat m+0.5\,s$ with the borrowing weights and a memory rate of $0.01$ (the sufficient statistics multiplied by $0.99$ daily). Thompson sampling refits each $28$ days from a draw $\hat m+s\,Z$, with $Z$ standard normal per cell and clipped to $[-4,13]$, on the cell-wise table. SEP reviews on days $10,38,66,\dots$ with budget $c_0$ per review, the simulated no-change price of a one-day test in the radius-one covering net $\{(0,0),(0,2),(2,1)\}$ computed under each family's economics. It scores candidate tests (zone, dose set, fraction) by gain times marginal reach at bonus $b=2$, where the marginal reach of a candidate is the number of its units that no already admitted candidate covers. It weights test observations by $w=8$, so that a test observation counts as $w$ visits under the deployed policy when the working model is next refit. The dose sets of the menu are the following. They are the six actions together; local $2$ with warehouse $1$; warehouse $1$ with $2$; warehouse $2$ with $3$; warehouse $3$ alone; local $2$ alone; and the design that assigns, at each state, the plug-in order and the order one unit larger with equal probability. The fractions are $1$, $1/3$, $1/10$, and $1/30$ of the zone's machines. The two design controls of Section~\ref{sec:sim-strategies} use the same runner with the ranking replaced by a seeded random order and with each candidate admissible, so that the working model prices the designs without choosing them. The uniform test's only candidate per zone is the six-action design on the whole zone, and the random design's candidates are the seven dose designs on the whole zone; each review admits candidates in random order until the budget $c_0$ is spent. A third control also draws the fractions at random and spends under half the budget, because small designs are cheap; Table~\ref{tab:design-controls} reports it. At $27$ machines the population limit shows SEP ahead of the random design in $10$ of the $25$ stress worlds and of the uniform test in $13$; we validate this fleet's accuracy against sampled fleets only to within several points (Appendix~\ref{app:sampled}). On sampled fleets of $27$ machines that lead does not survive. We use $10$ seeds, paired by seed with the stored SEP runs, and count a lead when it exceeds the larger of $2$ percent and two paired standard errors. SEP then ties the random design in each world (paired median $+0.6$) and leads the uniform test in $5$ worlds, all at speeds $\beta\le0.03$ (paired median $+1.6$); the ranking thus does not beat random placement at either fleet. The admission condition, by contrast, depends on the fleet. Where nothing changed, SEP at $999$ machines spends $1.4$ to $1.8$ percent of SOP's value in simulated price against the random design's $2.5$ to $3.3$, declining about half of the spend. At $27$ machines it spends $2.8$ to $3.4$ against $2.4$ to $3.3$ and declines nothing, and there SEP and the controls lose alike against adaptation. The visit count in the tunnel width is the discounted count $N$. SEP-when reviews weekly with budget $c_0/4$ under the admission condition \emph{net gain}, which admits a candidate only when its simulated gain exceeds its simulated price, whereas SEP uses the admission condition \emph{positive gain}, admitting a candidate when its simulated gain is positive. Both score each candidate by simulating its zone forward over a planning window of $128$ days under the optimistic table, with the test on day $10$ of the window and the zone's plug-in policy refit every $7$ days from the simulated observations. Thompson sampling draws from one generator per run seeded by the run seed, and the other five strategies are deterministic in the population limit.

We evaluate the strategies in the population limit, in which runs propagate each zone's exact joint distribution over the (inventory, regime) states, $84\times3$ of them in the cheap families and $20\times3$ in the expensive ones. The expected reward integrates the true response against the current conditional latent distribution, and each population-limit difference is therefore exact. Values are margins over SOP, in percent of SOP's value \eqref{eq:margin} at $T=1024$, computed on the $25$ configurations $(\beta,\phi)\in\{0.003,0.01,0.03,0.1,0.3\}\times\{0,0.25,0.5,0.75,1\}$ of each family and its no-change worlds. We count a tie when the leader's margin over the runner-up (SOP included) is at most $2$ percent of SOP's value, with the accuracy of the population limit validated against sampled fleets at $111$ machines per zone. Sampled fleets draw each machine's demand and regime path; the procedures share the environment draws of a seed (common random numbers) and use a separate stream for assignment randomization. This makes paired differences precise. Thompson sampling draws from one generator per run seeded by the run seed and consumed zone by zone at each review. The bankruptcy study follows \citet{backlund2025vending}: a daily fee of $90$ percent of the plug-in profit is paid from a balance of $30$ days of plug-in profit, and $10$ consecutive negative days end the operation.
\section{Supporting Proofs and Examples}
\label{app:proofs}

In this appendix, we prove the supporting propositions and lemmas of Appendix~\ref{app:moved}, together with the examples that establish the sign and horizon phenomena. The headline boundary results are proved with their statements in Appendix~\ref{app:theory}. We begin with two lemmas used throughout the proofs that follow: the first factors the distribution of a unit's realized path under a fixed memoryless policy, and the second reduces the fixed-policy value to a sum over units.

\paragraph{Conventions used in the proofs.}
We use three conventions throughout.
\emph{(R) Assignment randomization.} Each strategy draws its assignments from independent uniform variables $\xi_{i,t}$, $i\in\mathcal V$, $t\ge0$, independent of the historical record and of the potential-outcome processes. Given the policy $\pi_{i,t}$ and the state $X_{i,t}=x$, unit $i$ receives action $a$ when $\xi_{i,t}$ falls in the interval $I_{i,t}(x,a)\subseteq[0,1)$ of length $\pi_{i,t}(a\mid x)$; for a fixed state these intervals are disjoint across actions and partition the unit interval. This realizes the conditional assignment distribution $\prod_i\pi_{i,t}(\cdot\mid X_{i,t})$ used in \eqref{eq:strategyvalue}. A strategy's internal randomization is likewise independent of the environment.
\emph{(M) Markov reading.} We read Eqs.~(\ref{eq:transition})--(\ref{eq:response}) conditionally on the potential history along the path, as in the formal restatement of Assumption~\ref{asm:markov}.
\emph{(S) Selector.} The selector $S$ of Section~\ref{sec:problem} is a fixed deterministic map from a pair of tables and a period $t$ to a period-$t$ policy. We compute it by a value recursion under those tables with a fixed tie-break at each state (the backward recursion in Step~3 of the proof of Theorem~\ref{thm:nochange}, or the study's receding-horizon recursion). Its period-$t$ output therefore depends on the tables alone and not on the review at which it is invoked. Definition~\ref{def:sop} (SOP and A-SOP) uses this property.

\begin{lemma}[Path distribution under a fixed memoryless policy]
\label{lem:pathlaw}
Let Assumption~\ref{asm:consistency} hold, and let a fixed $\bar\pi\in\bar\Pi$ be deployed under (R). Fix a unit and suppress its index. Then, for each $t$, each observed-state path $\bar x_t=(x_0,\dots,x_t)$, each action path $\bar a_t$, and all Borel sets $B_0,\dots,B_t\subseteq\Ucal$,
\begin{equation}
\begin{split}
&\Pr\big[\bar X_t=\bar x_t,\ U_s\in B_s\ (s\le t),\ \bar A_t=\bar a_t\big]\\
&\qquad=\Big[\prod_{s=0}^t\pi_s(a_s\mid x_s)\Big]\Pr\big[X_s(\bar a_{s-1})=x_s,\ U_s(\bar a_{s-1})\in B_s\ (s\le t)\big],
\end{split}
\label{eq:pathlaw}
\end{equation}
where the second factor is the distribution of the potential path along the fixed action sequence $\bar a_t$ and does not depend on $\bar\pi$. The identity continues to hold when the events $\{Y_s\in C_s\}$, $s\le t$, are adjoined on the left and the events $\{Y_s(\bar a_s)\in C_s\}$ on the right. If Assumption~\ref{asm:markov} (Markov) also holds, then, writing $\bar{\mathcal H}_t$ for the $\sigma$-field generated by the realized quantities $(X_s,U_s,A_s,Y_s)_{s<t}$ and $(X_t,U_t,A_t)$,
\begin{equation}
\mathbb E[Y_t\mid\bar{\mathcal H}_t]=\theta(X_t,A_t,U_t),\qquad
\Pr[X_{t+1}=x',\,U_{t+1}\in B\mid\bar{\mathcal H}_t]=f\big(\{x'\}\times B\mid X_t,A_t,U_t\big).
\label{eq:realizedmarkov}
\end{equation}
\end{lemma}

\begin{proof}
We first prove the path identity. Let $E:=\{X_s(\bar a_{s-1})=x_s,\ U_s(\bar a_{s-1})\in B_s\ (s\le t)\}$, and we claim that
\begin{equation*}
\{\bar X_t=\bar x_t,\ U_s\in B_s\ (s\le t),\ \bar A_t=\bar a_t\}=E\cap\bigcap_{s\le t}\{\xi_s\in I_s(x_s,a_s)\}.
\end{equation*}
We argue by induction on $s$. On $\{\bar A_{s-1}=\bar a_{s-1}\}$, Assumption~\ref{asm:consistency}(i) implies $X_s=X_s(\bar a_{s-1})$ and $U_s=U_s(\bar a_{s-1})$. On $\{X_s=x_s\}$ the assignment $A_s$ equals $a_s$ exactly when $\xi_s\in I_s(x_s,a_s)$, by (R). The events $E$ and $\bigcap_s\{\xi_s\in I_s(x_s,a_s)\}$ are independent under (R), and the second has probability $\prod_s\pi_s(a_s\mid x_s)$, yielding \eqref{eq:pathlaw}. On $\{\bar A_t=\bar a_t\}$ the realized $Y_s$ equals $Y_s(\bar a_s)$ by Assumption~\ref{asm:consistency}(i), and adjoining the outcome events therefore leaves the identity unchanged.

We now prove \eqref{eq:realizedmarkov}. Fix $\bar a_t$ and a bounded measurable function $\phi$ of $(X_s,U_s,Y_s)_{s<t}$ and $(X_t,U_t)$, and work on $\{\bar A_t=\bar a_t\}$. There the realized quantities coincide with the potential ones along $\bar a_t$, and $\phi\mathbf 1\{\bar A_t=\bar a_t\}=\tilde\phi\,\mathbf 1\{\bar A_t=\bar a_t\}$ for a $\mathcal G_t(\bar a)$-measurable $\tilde\phi$, the same function evaluated at the potential quantities, where $\mathcal G_t(\bar a)$ (the index-suppressed potential history of the restated Assumption~\ref{asm:markov} (Markov)) is that potential history, and $\{\bar A_t=\bar a_t\}=\bigcap_{s\le t}\{\xi_s\in I_s(X_s(\bar a_{s-1}),a_s)\}\in\mathcal G_t(\bar a)\vee\sigma(\xi_{\le t})$. Because $\xi$ is independent of the potential-outcome process, we have $\mathbb E[Y_t(\bar a_t)\mid\mathcal G_t(\bar a)\vee\sigma(\xi)]=\mathbb E[Y_t(\bar a_t)\mid\mathcal G_t(\bar a)]=\theta(X_t(\bar a_{t-1}),a_t,U_t(\bar a_{t-1}))$ by (M). Consequently,
\begin{equation*}
\mathbb E\big[Y_t\,\phi\,\mathbf 1\{\bar A_t=\bar a_t\}\big]=\mathbb E\big[Y_t(\bar a_t)\,\tilde\phi\,\mathbf 1\{\bar A_t=\bar a_t\}\big]=\mathbb E\big[\theta(X_t,A_t,U_t)\,\phi\,\mathbf 1\{\bar A_t=\bar a_t\}\big].
\end{equation*}
Each bounded $\bar{\mathcal H}_t$-measurable function is a finite sum $\sum_{\bar a_t}\phi_{\bar a_t}\mathbf 1\{\bar A_t=\bar a_t\}$ of such functions, which establishes $\mathbb E[Y_t\psi]=\mathbb E[\theta(X_t,A_t,U_t)\psi]$ for all such $\psi$, the first identity. The second follows in the same way from the conditional distribution of $(X_{t+1}(\bar a_t),U_{t+1}(\bar a_t))$ given $\mathcal G_t(\bar a)$ in (M).
\end{proof}

\begin{lemma}[Per-unit decomposition]
\label{lem:decomp}
Suppose Assumption~\ref{asm:consistency} (consistency) holds. Then, for each $\bar{\pi} \in \bar{\Pi}$,
\begin{equation*}
    V^\gamma_T(\bar{\pi}) \;=\; \sum_{i \in \mathcal{V}} v_i(\pi_{i,\cdot}),
    \qquad
    v_i(\pi_{i,\cdot}) := \mathbb{E}\Big[\, \sum_{t=0}^{T} \gamma^t Y_{i,t} \,\Big],
\end{equation*}
where $v_i$ depends on $\bar{\pi}$ only through $\pi_{i,0},\dots,\pi_{i,T}$. Consequently, a policy that maximizes each $v_i$ separately maximizes $V^\gamma_T$ over $\bar{\Pi}$. The same statements hold for $V^{\gamma,\mathrm{sim}}_T(\cdot\,; \tilde m, \tilde g)$ under any stationary surrogate tables $(\tilde m,\tilde g)$.
\end{lemma}

\begin{proof}
Because the outcomes are bounded by $B_Y$, the sum over the finitely many units may be exchanged with the expectation, and by linearity $V^\gamma_T(\bar\pi)=\sum_i\mathbb E[\sum_t\gamma^tY_{i,t}]=\sum_iv_i(\pi_{i,\cdot})$. We next show that each per-unit value depends on the strategy only through that unit's own policies. Fix $i$ and two sequences $\bar\pi,\bar\pi'\in\bar\Pi$ with $\pi_{i,t}=\pi'_{i,t}$ for each $t$. Under (R), unit $i$'s realized path is, by induction on $t$, a measurable function of unit $i$'s potential process, of $(\xi_{i,t})_{t\le T}$, and of $\pi_{i,\cdot}$ alone. The base action $A_{i,0}$ is determined by $\pi_{i,0},X_{i,0},\xi_{i,0}$. Given $\bar A_{i,t}$, Assumption~\ref{asm:consistency} gives $(X_{i,t+1},U_{i,t+1},Y_{i,t})=(X_{i,t+1}(\bar A_{i,t}),U_{i,t+1}(\bar A_{i,t}),Y_{i,t}(\bar A_{i,t}))$, quantities of unit $i$'s own process by clause (ii), and $A_{i,t+1}$ is determined by $\pi_{i,t+1},X_{i,t+1},\xi_{i,t+1}$. Because the function is the same under $\bar\pi$ and $\bar\pi'$, and the joint distribution of unit $i$'s potential process and of $(\xi_{i,t})_t$ is policy-free, $v_i$ takes the same value under both. It follows that any $\bar\pi^\dagger\in\bar\Pi$ with $v_i(\pi^\dagger_{i,\cdot})\ge v_i(\pi_{i,\cdot})$ for each $i$ and each $\bar\pi\in\bar\Pi$ satisfies $V^\gamma_T(\bar\pi)=\sum_iv_i(\pi_{i,\cdot})\le\sum_iv_i(\pi^\dagger_{i,\cdot})=V^\gamma_T(\bar\pi^\dagger)$. Because $\bar\Pi$ is a product over units, the components may be chosen separately. Under \eqref{eq:sim} with $(\tilde m,\tilde g)$ taken as the response and transition tables, unit $i$'s path is a function of its initial state, its own transition and noise draws, $(\xi_{i,t})_t$ and $\pi_{i,\cdot}$, all with a joint distribution that does not depend on the policy, and the same argument applies.
\end{proof}
\begin{lemma}[The maximum in \eqref{eq:objective} is attained]
\label{lem:attain}
Let Assumption~\ref{asm:consistency} hold. Then the following hold. (i) $V^\gamma_T(\bar{\pi})$ is a polynomial in the entries of the tensor $\mathsf P$ of \eqref{eq:tensor}, and $\bar{\Pi}$ corresponds to a compact subset of $\mathbb{R}^{n(T+1)|\mathcal{X}||\mathcal{A}|}$. (ii) The supremum in \eqref{eq:objective} is attained. (iii) There are instances with more than one maximizer. (iv) Statements (i) and (ii) hold for $V^{\gamma,\mathrm{sim}}_T(\cdot\,; \tilde m, \tilde g)$ under any stationary surrogate tables $(\tilde m,\tilde g)$.
\end{lemma}

\begin{proof}
By Lemma~\ref{lem:decomp} it suffices to show that each $v_i$ is a polynomial in the entries $\mathsf P[i,\cdot,\cdot,\cdot]$; we fix a unit $i$ and suppress its index throughout. By Assumption~\ref{asm:consistency}(i), $Y_t=\sum_{\bar a_t}Y_t(\bar a_t)\mathbf 1\{\bar A_t=\bar a_t\}$, and on $\{\bar A_t=\bar a_t\}$ the realized states are the potential ones along $\bar a_t$. With the notation of Lemma~\ref{lem:pathlaw}, we obtain
\begin{equation*}
\begin{split}
\mathbb E[Y_t]&=\sum_{\bar a_t\in\mathcal A^{t+1}}\ \sum_{\bar x_t\in\mathcal X^{t+1}}\mathbb E\Big[Y_t(\bar a_t)\,\mathbf 1\{X_s(\bar a_{s-1})=x_s,\ s\le t\}\prod_{s\le t}\mathbf 1\{\xi_s\in I_s(x_s,a_s)\}\Big]\\
&=\sum_{\bar a_t,\bar x_t}w_t(\bar x_t,\bar a_t)\prod_{s=0}^t\pi_s(a_s\mid x_s),
\end{split}
\end{equation*}
where $w_t(\bar x_t,\bar a_t):=\mathbb E\big[Y_t(\bar a_t)\mathbf 1\{X_s(\bar a_{s-1})=x_s,\ s\le t\}\big]$ is finite, because $|Y|\le B_Y$, and free of the policy; the factorization uses the independence of $(\xi_s)$ from the potential process, which holds under the randomization condition~(R), as in \eqref{eq:pathlaw}. Consequently, $v(\pi)=\sum_t\gamma^t\mathbb E[Y_t]$ is a polynomial of degree at most $T+1$ in the entries $\pi_s(a\mid x)$, and $V^\gamma_T$ is a polynomial in the entries of $\mathsf P$. The admissible set in \eqref{eq:tensor} is the product $\Delta(\mathcal A)^{n(T+1)|\mathcal X|}$ of finitely many simplices, and is thus compact. Since a polynomial is continuous, the maximum in \eqref{eq:objective} is attained. We next prove (iv). For the simulated value, with $\mu_X$ the initial state distribution (the $\mathcal X$-marginal of the initial distribution of the full state), the path distribution under \eqref{eq:sim} with tables $(\tilde m,\tilde g)$ is $\mu_X(x_0)\prod_{s<t}\tilde g(x_{s+1}\mid x_s,a_s)\prod_{s\le t}\pi_s(a_s\mid x_s)$ and the period-$t$ mean outcome is $\tilde m(x_t,a_t)$, yielding the same conclusion with $w_t(\bar x_t,\bar a_t)=\mu_X(x_0)\prod_{s<t}\tilde g(x_{s+1}\mid x_s,a_s)\,\tilde m(x_t,a_t)$. It remains to prove (iii), that uniqueness fails in general. If the outcome map $\theta(x,a,u)$ and $f(\cdot\mid x,a,u)$ do not depend on the action $a$, then by the Markov condition~(M) the weights $w_t(\bar x_t,\bar a_t)=\mathbb E[\theta(x_t,a_t,U_t(\bar a_{t-1}))\mathbf 1\{X_s(\bar a_{s-1})=x_s,\ s\le t\}]$ do not depend on $\bar a_t$. The policy factors sum to one over $\bar a_t$, which makes $V^\gamma_T$ constant on $\bar\Pi$, and each element of $\bar\Pi$ is a maximizer.
\end{proof}

\begin{proof}[Proof of Proposition~\ref{prop:ident}]
We first establish the two identities, then the non-identification claim. Fix $(i,\bar\pi)$ and suppress the indices, and write $E_x:=\{X_t=x\}$ and $E_{x,a}:=\{X_t=x,A_t=a\}$. Under (R), $A_t=a$ exactly when $\xi_t\in I_t(X_t,a)$, and $\xi_t$ is independent of $(X_t,U_t)$. It follows that for each Borel $B\subseteq\Ucal$,
\begin{equation*}
\Pr[U_t\in B,\ A_t=a,\ X_t=x]=\pi_t(a\mid x)\Pr[U_t\in B,\ X_t=x],
\end{equation*}
and dividing by $\Pr[E_{x,a}]=\pi_t(a\mid x)\Pr[E_x]>0$ establishes that the conditional distribution of $U_t$ given $E_{x,a}$ equals its conditional distribution given $E_x$, namely $\nu_t(\cdot\mid x)$:
\begin{equation}
    A_t \perp U_t \mid X_t.
    \label{eq:seqrand}
\end{equation}
By \eqref{eq:realizedmarkov}, $\mathbb E[Y_t\mid\bar{\mathcal H}_t]=\theta(X_t,A_t,U_t)$, and $E_{x,a}\in\bar{\mathcal H}_t$, so the tower rule gives
\begin{equation*}
\mathbb E[Y_t\mid E_{x,a}]=\mathbb E[\theta(x,a,U_t)\mid E_{x,a}]=\int\theta(x,a,u)\,\nu_t(du\mid x)=m_t(x,a),
\end{equation*}
the middle equality by \eqref{eq:seqrand} and the last by \eqref{eq:mdef}. By the same argument, $\Pr[X_{t+1}=x'\mid\bar{\mathcal H}_t]=f_X(x'\mid X_t,A_t,U_t)$ by \eqref{eq:realizedmarkov}, and we obtain $\Pr[X_{t+1}=x'\mid E_{x,a}]=\int f_X(x'\mid x,a,u)\,\nu_t(du\mid x)=g_t(x'\mid x,a)$ by \eqref{eq:seqrand} and \eqref{eq:gdef}. The identities above also hold for a prespecified randomized design, that is, a policy tensor fixed before the horizon outcomes, since the same argument applies. We now show that an unsupported cell is not identified. Suppose $\Pr[X_{i,t}=x]\,\pi_{i,t}(a\mid x)=0$ for each $(i,t)$, and let $\mathcal E'$ differ from $\mathcal E$, a data-generating process over the horizon, only in the conditional distribution of $(Y_t,X_{t+1},U_{t+1})$ given the full state and action at the cell $(x,a)$, all else the same. By \eqref{eq:pathlaw} with outcome events adjoined, the distribution of the observed record is determined by the policy factors and by the distributions of the potential paths along fixed action sequences. Under (M), the distribution of a potential path along $\bar a$ up to period $s$ is determined by the initial distribution and by the conditional distributions at the cells $(x_r,a_r)$, $r<s$, that it visits. The two environments thus agree on each finite-dimensional event of a potential path up to its first visit of $(x,a)$, and on each path that does not visit $(x,a)$. A path visiting $(x,a)$ first at period $s$ has, in either environment, probability at most $\prod_{r\le s}\pi_r(a_r\mid x_r)\Pr[X_r(\bar a_{r-1})=x_r,\ r\le s]=\Pr[\bar X_s=\bar x_s,\bar A_s=\bar a_s]\le\Pr[X_s=x,A_s=a]=0$. The two probabilities agree across environments because they involve only the prefix. We conclude that the observed distributions under $\mathcal E$ and $\mathcal E'$ coincide, and so the cell is not identified.
\end{proof}

\begin{proof}[Proof of Lemma~\ref{lem:flow}]
Fix $i$ and suppress it. By (R), $\Pr[X_t=x,A_t=a]=p_t(x)\pi_t(a\mid x)$. Summing the transition identity of Proposition~\ref{prop:ident} over the cells with positive probability,
\begin{equation*}
\begin{split}
p_{t+1}(x')&=\sum_{x,a:\ p_t(x)\pi_t(a\mid x)>0}p_t(x)\pi_t(a\mid x)\Pr[X_{t+1}=x'\mid X_t=x,A_t=a]\\
&=\sum_xp_t(x)\sum_a\pi_t(a\mid x)g_t(x'\mid x,a),
\end{split}
\end{equation*}
where cells with $p_t(x)\pi_t(a\mid x)=0$ contribute zero to both sides, whatever value $g_t$ takes there. By the response identity, $\mathbb E[Y_t]=\sum_{x,a}p_t(x)\pi_t(a\mid x)\mathbb E[Y_t\mid X_t=x,A_t=a]=\sum_xp_t(x)\sum_a\pi_t(a\mid x)m_t(x,a)$. Multiplying by $\gamma^t$, summing over $t$, and summing over units by Lemma~\ref{lem:decomp} gives \eqref{eq:flow}.
\end{proof}

\begin{proof}[Proof of Lemma~\ref{lem:nonmarkov}]
Fix $i$ and suppress it. The $\sigma$-field generated by $(\bar H_t,U_t,A_t)$ is contained in $\bar{\mathcal H}_t$ of Lemma~\ref{lem:pathlaw}, so by \eqref{eq:realizedmarkov} and the tower rule, $\Pr[X_{t+1}=x'\mid\bar H_t,U_t,A_t]=f_X(x'\mid X_t,A_t,U_t)$, and it follows that
\begin{equation*}
\Pr[X_{t+1}=x'\mid\bar H_t=h,A_t=a]=\int f_X(x'\mid x_t,a,u)\,\Pr[U_t\in du\mid\bar H_t=h,A_t=a],
\end{equation*}
with $x_t$ the last state of $h$. Under (R), $A_t=a$ exactly when $\xi_t\in I_t(x_t,a)$, and $\xi_t$ is independent of $(\bar H_t,U_t)$. Hence $\Pr[U_t\in B,A_t=a\mid\bar H_t=h]=\pi_t(a\mid x_t)\Pr[U_t\in B\mid\bar H_t=h]$. Dividing by $\pi_t(a\mid x_t)>0$, the conditional distribution of $U_t$ given $(\bar H_t=h,A_t=a)$ equals its conditional distribution given $\bar H_t=h$, which yields \eqref{eq:histlaw}. Suppose first that $\Pr[U_t\in\cdot\mid\bar H_t]$ depends on $\bar H_t$ only through $X_t$. Then by the tower rule it equals $\nu_t(\cdot\mid X_t)$, and \eqref{eq:histlaw} becomes $\int f_X(x'\mid X_t,a,u)\,\nu_t(du\mid X_t)=g_t(x'\mid X_t,a)$ by \eqref{eq:gdef}; the one-step distribution of the observed state then depends on the history only through $(X_t,A_t)$, and $A_t$ only through $X_t$ and fresh randomization, and consequently the observed process is Markov. Suppose instead that $f_X(\cdot\mid x,a,u)$ does not depend on $u$. Then \eqref{eq:histlaw} equals $f_X(x'\mid X_t,a)=g_t(x'\mid X_t,a)$ for each history and the observed process is Markov, while the conditional distribution of $U_t$ given $\bar H_t$ may still depend on past outcomes, and we conclude that the converse fails. For an instance where the observed process is not Markov, take $\mathcal X=\{0,1\}$, one action $a$, outcomes identically zero, and a static latent bit $U\in\{0,1\}$ with $\Pr[U=1]=\tfrac12$: the kernel keeps $u$ and draws $X_{t+1}\sim\mathrm{Bernoulli}(p_u)$ with $p_0=\tfrac14$ and $p_1=\tfrac34$, and $X_0\sim\mathrm{Bernoulli}(p_U)$. Assumptions~\ref{asm:consistency} and~\ref{asm:markov} hold. By Bayes' rule, $\Pr[U=1\mid X_0=0,X_1=0]=\tfrac1{10}$ and $\Pr[U=1\mid X_0=1,X_1=0]=\tfrac12$, and \eqref{eq:histlaw} at $x'=1$ then equals $\tfrac1{10}\cdot\tfrac34+\tfrac9{10}\cdot\tfrac14=\tfrac3{10}$ after the history $(0,0)$ and $\tfrac12\cdot\tfrac34+\tfrac12\cdot\tfrac14=\tfrac12$ after $(1,0)$, two positive-probability histories ending at the same state $X_1=0$; the state-only kernel is $g_1(1\mid0,a)=\tfrac14\cdot\tfrac34+\tfrac34\cdot\tfrac14=\tfrac38$. The two history-conditional distributions differ from each other and from $g_1$, and therefore the observed process is not Markov.
\end{proof}

\begin{proof}[Proof of Observation~\ref{prop:delta}]
Fix the unit and the strategy and suppress them; fix $(x,a)$ and $t\ge0$, and write $\zeta:=\nu_t(\cdot\mid x)-\nu^{\mathrm{sim}}(\cdot\mid x)$, a finite signed measure with $\zeta(\Ucal)=0$ and Jordan decomposition $\zeta=\zeta^+-\zeta^-$, for which $\zeta^+(\Ucal)=\zeta^-(\Ucal)=\TV(\nu_t(\cdot\mid x),\nu^{\mathrm{sim}}(\cdot\mid x))=:\upsilon$. By \eqref{eq:mdef} and Assumption~\ref{asm:calib}, $\delta_t(x,a)=m_t(x,a)-\theta^{\mathrm{sim}}(x,a)=\int\theta(x,a,u)\,\zeta(du)$. Since $|\theta|\le B_Y$, $|\int\theta\,d\zeta|\le\int|\theta|\,d\zeta^++\int|\theta|\,d\zeta^-\le B_Y\big(\zeta^+(\Ucal)+\zeta^-(\Ucal)\big)=2B_Y\upsilon$. An identical computation gives $g_t(\cdot\mid x,a)-f^{\mathrm{sim}}(\cdot\mid x,a)=\int f_X(\cdot\mid x,a,u)\,\zeta(du)$ by \eqref{eq:gdef} and Assumption~\ref{asm:calib}, and exchanging the finite sum over $x'$ with the integrals,
\begin{equation*}
\begin{split}
\big\|g_t(\cdot\mid x,a)-f^{\mathrm{sim}}(\cdot\mid x,a)\big\|_1&\le\sum_{x'}\Big[\int f_X(x'\mid x,a,u)\,\zeta^+(du)+\int f_X(x'\mid x,a,u)\,\zeta^-(du)\Big]\\
&=\zeta^+(\Ucal)+\zeta^-(\Ucal)=2\upsilon,
\end{split}
\end{equation*}
because $\sum_{x'}f_X(x'\mid x,a,u)=1$ for each $u$. It remains to prove the Wasserstein bounds. Let $\varpi$ be any coupling of $\nu_t(\cdot\mid x)$ and $\nu^{\mathrm{sim}}(\cdot\mid x)$. Then $\int\theta\,d\zeta=\iint[\theta(x,a,u)-\theta(x,a,u')]\,\varpi(du,du')$, and \eqref{eq:liptheta} at equal $(x,a)$ shows $|\int\theta\,d\zeta|\le L_\theta\iint d_U(u,u')\,\varpi(du,du')$; taking the infimum over couplings, $|\delta_t(x,a)|\le L_\theta W_1(\nu_t(\cdot\mid x),\nu^{\mathrm{sim}}(\cdot\mid x))$. In the same way, $\|g_t-f^{\mathrm{sim}}\|_1=\sum_{x'}|\iint[f_X(x'\mid x,a,u)-f_X(x'\mid x,a,u')]\,\varpi|\le\iint\|f_X(\cdot\mid x,a,u)-f_X(\cdot\mid x,a,u')\|_1\,\varpi\le L_f\iint d_U\,\varpi$ by \eqref{eq:lipf}, whence $\|g_t-f^{\mathrm{sim}}\|_1\le L_fW_1(\nu_t,\nu^{\mathrm{sim}})$. Finally, Assumption~\ref{asm:calib} identifies $\nu_{-1}(\cdot\mid x)$ with $\nu^{\mathrm{sim}}(\cdot\mid x)$, and \eqref{eq:driftacc} at $\tau=-1$, for the fixed $x$, implies $W_1(\nu_t(\cdot\mid x),\nu^{\mathrm{sim}}(\cdot\mid x))\le\lambda(t+1)$, which yields \eqref{eq:deltasmooth}.
\end{proof}

\begin{proof}[Proof of Proposition~\ref{prop:simgap}]
Both values decompose over units, by Lemma~\ref{lem:decomp} and the same argument applied to \eqref{eq:sim}. Fix a unit and $\bar\pi$, drop those indices on the induced projections, and let $p_t$ and $p^{\mathrm{sim}}_t$ be the real and simulated observed-state marginals. Define
\begin{equation*}
    r_t(q) := \sum_x q(x) \sum_a \pi_t(a \mid x) m_t(x,a),
    \qquad
    r^{\mathrm{sim}}_t(q) := \sum_x q(x) \sum_a \pi_t(a \mid x) \theta^{\mathrm{sim}}(x,a).
\end{equation*}
By Lemma~\ref{lem:flow} the unit's real value is $\sum_t \gamma^t r_t(p_t)$, and by \eqref{eq:sim} its simulated value is $\sum_t \gamma^t r^{\mathrm{sim}}_t(p^{\mathrm{sim}}_t)$. We split each period,
\begin{equation*}
    r_t(p_t) - r^{\mathrm{sim}}_t(p^{\mathrm{sim}}_t)
    = \underbrace{\big[r_t(p_t) - r^{\mathrm{sim}}_t(p_t)\big]}_{\text{response error}}
    + \underbrace{\big[r^{\mathrm{sim}}_t(p_t) - r^{\mathrm{sim}}_t(p^{\mathrm{sim}}_t)\big]}_{\text{state-displacement error}} .
\end{equation*}
The response error satisfies $|r_t(p_t) - r^{\mathrm{sim}}_t(p_t)| = |\sum_x p_t(x)\sum_a \pi_t(a\mid x)\delta_t(x,a)| \leq \epsilon_m$, since $p_t$ and $\pi_t$ are probability weights. For the state-displacement error, $c_t(x) := \sum_a \pi_t(a \mid x)\theta^{\mathrm{sim}}(x,a)$ satisfies $|c_t| \leq \bar{\theta}$ and $r^{\mathrm{sim}}_t(q)=\sum_xq(x)c_t(x)$ is linear in $q$, and therefore the error is bounded by $\bar{\theta}\, e_t$ with $e_t := \|p_t - p^{\mathrm{sim}}_t\|_1$. It remains to bound $e_t$. Let $P_t$ be the linear map $q \mapsto \sum_x q(x)\sum_a \pi_t(a\mid x)g_t(\cdot\mid x,a)$ and $P^{\mathrm{sim}}_t$ its $f^{\mathrm{sim}}$ analogue, and we obtain $p_{t+1} = P_t p_t$ and $p^{\mathrm{sim}}_{t+1} = P^{\mathrm{sim}}_t p^{\mathrm{sim}}_t$ by Lemma~\ref{lem:flow}. Then
\begin{equation*}
    e_{t+1} \leq \big\| (P_t - P^{\mathrm{sim}}_t) p_t \big\|_1 + \big\| P^{\mathrm{sim}}_t (p_t - p^{\mathrm{sim}}_t) \big\|_1
    \leq \epsilon_g + e_t,
\end{equation*}
the second term because $P^{\mathrm{sim}}_t$ is an average of stochastic kernels, hence non-expansive in $\ell_1$, and the first by averaging $\sum_{x'}|g_t(x'\mid x,a)-f^{\mathrm{sim}}(x'\mid x,a)|\leq\epsilon_g$ over $p_t$ and $\pi_t$. The initial-distribution condition gives $e_0 = 0$, and $e_t \leq t\,\epsilon_g$ follows by induction; collecting, the difference between the unit's real and simulated values is at most $\sum_{t=0}^{T}\gamma^t(\epsilon_m + \bar{\theta}\,t\,\epsilon_g) = \epsilon_m \Gamma_T(\gamma) + \bar{\theta}\epsilon_g \Lambda_T(\gamma)$, and multiplying by $n$ gives \eqref{eq:simgap}.

We now prove the smooth claim. By \eqref{eq:deltasmooth}, $|\delta_t(x,a)| \leq L_\theta\lambda(t+1)$ and $\|g_t(\cdot\mid x,a) - f^{\mathrm{sim}}(\cdot\mid x,a)\|_1 \leq L_f\lambda(t+1)$, giving $\epsilon_m \leq L_\theta\lambda(T+1)$ and $\epsilon_g \leq L_f\lambda(T+1)$, and substitution yields \eqref{eq:simgapsmooth}.
\end{proof}

\begin{proof}[Proof of Corollary~\ref{cor:sopregret}]
Write $G$ for the right side of \eqref{eq:simgap}, $V := V^\gamma_T$, and $V^{\mathrm{sim}} := V^{\gamma,\mathrm{sim}}_T(\cdot\,; \theta^{\mathrm{sim}},f^{\mathrm{sim}})$. We decompose
\begin{equation*}
    V(\bar{\pi}^*) - V(\bar{\pi}^{\mathrm{sim}})
    = \big[V(\bar{\pi}^*) - V^{\mathrm{sim}}(\bar{\pi}^*)\big]
    + \big[V^{\mathrm{sim}}(\bar{\pi}^*) - V^{\mathrm{sim}}(\bar{\pi}^{\mathrm{sim}})\big]
    + \big[V^{\mathrm{sim}}(\bar{\pi}^{\mathrm{sim}}) - V(\bar{\pi}^{\mathrm{sim}})\big].
\end{equation*}
The middle bracket is at most zero, because $\bar{\pi}^{\mathrm{sim}}$ maximizes $V^{\mathrm{sim}}$ over $\bar{\Pi}$ by Definition~\ref{def:sop} and $\bar{\pi}^* \in \bar{\Pi}$; the outer brackets are each at most $G$ by Proposition~\ref{prop:simgap}; and both maximizers exist by Lemma~\ref{lem:attain}. Summing the three brackets, we obtain $2G$.
\end{proof}

\begin{proof}[Proof of Proposition~\ref{prop:censoring}]
Let $\zeta := (\{m^{\mathrm{wk}}_{a'}(x')\}, \{g^{\mathrm{wk}}_{a'}(\cdot \mid x')\})$ index the working-cell quantities, with $\zeta_{(x,a)}$ the coordinates of cell $(x,a)$. By Assumption~\ref{asm:cellwise}, $\Xi_0(d\zeta)=\prod_{(x',a')}\Xi_0^{(x',a')}(d\zeta_{(x',a')})$, and by Assumption~\ref{asm:likelihood} the pseudo-likelihood over the selected completed factors $\mathcal L^{\mathrm{op}}_{<t}$ factorizes by working cell; writing $N_{(x',a')}$ for the selected factors at cell $(x',a')$,
\begin{equation*}
L(\mathcal L^{\mathrm{op}}_{<t}\mid\zeta)=\prod_{(x',a')}L_{(x',a')}(N_{(x',a')}\mid\zeta_{(x',a')}).
\end{equation*}
By assumption $N_{(x,a)}=\emptyset$, so its factor is constant in $\zeta_{(x,a)}$; consequently $\Xi_t=\Xi_0(\cdot\mid\mathcal L^{\mathrm{op}}_{<t})$ factorizes and
\begin{equation*}
\Xi_t^{(x,a)}(d\zeta_{(x,a)})\propto\Xi_0^{(x,a)}(d\zeta_{(x,a)})L_{(x,a)}(\emptyset\mid\zeta_{(x,a)})\propto\Xi_0^{(x,a)}(d\zeta_{(x,a)}).
\end{equation*}
The posterior marginal is thus the prior marginal, whose mean is $\theta^{\mathrm{sim}}(x,a)$ by the definition of $\Xi_0$ in Section~\ref{sec:simulator}.
\end{proof}

\begin{proof}[Proof of Corollary~\ref{cor:lockin}]
We argue by induction over review times. Suppose $a^\dagger$ has not previously been assigned at $x$. Since no selected observation from $(x,a^\dagger)$ has entered the updater, Proposition~\ref{prop:censoring} gives $\hat m_t(x,a^\dagger)=b_0$. If no observation from $(x,a')$ has entered the updater, then $\hat m_t(x,a')=\theta^{\mathrm{sim}}(x,a')\geq\alpha$; otherwise assumption (ii) places its posterior mean in the convex hull of $\theta^{\mathrm{sim}}(x,a')$ and the period-specific values used in the update, each point of which is at least $\alpha$ by assumption (iii). Either way $\hat m_t(x,a')\geq\alpha>b_0$, which keeps the myopic argmax rule from assigning $a^\dagger$. At the initial review, no observations from either cell have entered the updater, establishing the base case.
\end{proof}
\section{Identification, Detection, and the Value of a Test}
\label{app:fourtheorems}

In this appendix we present the full statements and proofs of Theorems~\ref{thm:ident}, \ref{thm:regimeid}, \ref{thm:detectprice}, and~\ref{thm:testvalue}, together with Lemma~\ref{lem:window}. Each subsection first states its theorem in full, under its main-text number when it has one, then proves the lemmas it needs, and finally proves the numbered conclusions. Throughout we use the conventions (R), (M), (S) and Lemma~\ref{lem:pathlaw} from Appendix~\ref{app:proofs}, together with (G) and (C0) from Appendix~\ref{app:theory}.\subsection{Identification under prespecified and history-dependent collection (Theorem~\ref{thm:ident})}\label{app:ident}

Throughout, let $\mathcal F_t$ denote the multi-unit predecision filtration of \eqref{eq:filtration},
\begin{equation}\label{eq:ident-filt}
\mathcal F_t=\sigma\!\Big(\mathcal H_{-1},\ \mathbf X_t,\ \{(X_{i,s},A_{i,s},Y_{i,s},X_{i,s+1}):i\in\mathcal V,\ 0\le s<t\}\Big),
\end{equation}
which collects the historical record (the pre-study record), each unit's completed records through $t-1$ (its states, actions, outcomes, and successor states), and all current states. Let $\xi=(\xi_{i,t})$ be the assignment randomization of convention (R), drawn from a fixed environment-independent distribution. A strategy $\sigma$ selects an $\mathcal F_t$-measurable period-$t$ policy $\pi_{i,t}:\mathcal X\to\Delta(\mathcal A)$ and draws $\mathbf A_t$ through (R). Since $\mathbf X_t\in\mathcal F_t$, the event $\{X_{i,t}=x\}$ is itself $\mathcal F_t$-measurable.

We say that a target is \emph{record-identified} when it equals a functional of the distribution of the observed record $(\bar{\mathbf X}_T,\bar{\mathbf A}_T,\bar{\mathbf Y}_T)$ alone, and \emph{design-identified} when it equals a functional of the joint distribution of the record together with the operator's own assignment probabilities $\{\pi_{i,t}(a\mid x)(\mathcal F_t)\}$ (each a predecision-measurable assignment probability).

Write $f_X(x'\mid x,a,u):=f(\{x'\}\times\Ucal\mid x,a,u)$ for the observed-state marginal of the full-state kernel $f$ of A2. Let $\nu_{i,t}(\cdot\mid x)$ be the regular conditional distribution of $U_{i,t}$ given $X_{i,t}=x$ under $\sigma$. The \emph{marginal target} is then $m_{i,t}(x,a):=\int\theta(x,a,u)\,\nu_{i,t}(du\mid x)$, as in \eqref{eq:mdef}, and we write $g_{i,t}(x'\mid x,a):=\int f_X(x'\mid x,a,u)\,\nu_{i,t}(du\mid x)$, as in \eqref{eq:gdef}. Let $\nu_{i,t}(\cdot\mid\mathcal F_t,x)$ be a regular conditional distribution of $U_{i,t}$ given $\mathcal F_t$ on $\{X_{i,t}=x\}$. The \emph{history-conditional projection} is $m_{i,t}(x,a\mid\mathcal F_t):=\int\theta(x,a,u)\,\nu_{i,t}(du\mid\mathcal F_t,x)$, with $g_{i,t}(x'\mid x,a\mid\mathcal F_t)$ defined analogously through $f_X$. By the tower rule, the marginal target is the $\mathcal F_t$-average of the projection on $\{X_{i,t}=x\}$:
\begin{equation}\label{eq:ident-star}
m_{i,t}(x,a)=\mathbb E\big[m_{i,t}(x,a\mid\mathcal F_t)\mid X_{i,t}=x\big].
\end{equation}

We use Assumption~\ref{asm:crossunit} (cross-unit independence; formal statement in Appendix~\ref{app:moved}) in the proof of Lemma~\ref{lem:ident-L2}, conditionally on unit $i$'s potential history along the realized path. The assumption is not a marginal independence of $\mathcal F_t$ from unit $i$'s realized record.

\begin{restated}{Theorem}{thm:ident}{Identification under prespecified and history-dependent collection, formal statement}
Let units $i\in\mathcal V=\{1,\dots,n\}$, periods $t\in\{0,\dots,T\}$, finite $\mathcal X,\mathcal A$, latent $U_{i,t}$ in a standard Borel space, and outcomes in $[-B_Y,B_Y]$. Assume A1 (Assumption~\ref{asm:consistency}), A2 (Assumption~\ref{asm:markov}), and convention (R) of Appendix~\ref{app:proofs}. Let $\nu_{i,t}$, $m_{i,t}(x,a)$, $g_{i,t}$, $m_{i,t}(x,a\mid\mathcal F_t)$, and $g_{i,t}(\cdot\mid x,a\mid\mathcal F_t)$ be the objects defined above. Assumptions~\ref{asm:smooth} (A4) and~\ref{asm:drift} (A5) are not part of this setting. Write $\bar{\mathcal H}_{i,t}$ for unit $i$'s realized $\sigma$-field of Lemma~\ref{lem:pathlaw}, generated by unit $i$'s realized $(X_{i,s},U_{i,s},A_{i,s},Y_{i,s})_{s<t}$ together with the period-$t$ pre-outcome triple $(X_{i,t},U_{i,t},A_{i,t})$, and set $\mathcal G_{i,t}:=\bar{\mathcal H}_{i,t}\vee\mathcal F_t$, a field that enters only in the proofs below. Write $\epsilon_{i,t}:=Y_{i,t}-\theta(X_{i,t},A_{i,t},U_{i,t})$ for the residual, so that $Y_{i,t}=\theta(X_{i,t},A_{i,t},U_{i,t})+\epsilon_{i,t}$; this residual has conditional mean zero, $\mathbb E[\epsilon_{i,t}\mid X_{i,t},U_{i,t},A_{i,t}]=0$ under A1--A2, as in \eqref{eq:reward}.

Given this setting and these assumptions, the following hold. We first identify the target under a prespecified distribution (i), then under a history-dependent distribution, both history-conditionally (ii) and marginally through the assignment weight (iii), and finally specialize to the budgeted procedure (iv).

\emph{(i) Prespecified distribution.} Let $\bar\pi\in\bar\Pi$ be fixed before horizon outcomes and deployed, and fix $(i,t,x,a)$.
\begin{itemize}
\item[(i.1)] If $\Pr_{\bar\pi}[X_{i,t}=x]\,\pi_{i,t}(a\mid x)>0$, then
\begin{equation}\label{eq:ident-i1}
\begin{gathered}
\mathbb E_{\bar\pi}[Y_{i,t}\mid X_{i,t}=x,A_{i,t}=a]=m_{i,t}(x,a),\\
\Pr_{\bar\pi}[X_{i,t+1}=x'\mid X_{i,t}=x,A_{i,t}=a]=g_{i,t}(x'\mid x,a).
\end{gathered}
\end{equation}
\item[(i.2)] If instead $\Pr_{\bar\pi}[X_{i,s}=x]\,\pi_{i,s}(a\mid x)=0$ for all $(i,s)$, then $m_{i,t}(x,a)$ is not record-identified. There exist two environments admissible under A1--A2, agreeing on the initial distribution, on the historical record's distribution, and on the full-state kernels at each other cell, that induce the same record distribution and yet differ in $m_{i,t}(x,a)$.
\end{itemize}
A cell $(x,a)$ carrying positive probability under $\bar\pi$ at $(i,t)$ therefore suffices for record identification of $m_{i,t}(x,a)$, while a cell carrying zero probability at each $(i,s)$ suffices for its non-identification.

\emph{(ii) History-dependent distribution, history-conditional identification.} Assume (\ref{asm:crossunit}). For any strategy $\sigma$, on the $\mathcal F_t$-measurable event $\{\pi_{i,t}(a\mid X_{i,t})>0\}\cap\{X_{i,t}=x\}$,
\begin{equation}\label{eq:ident-ii}
\mathbb E[Y_{i,t}\mid\mathcal F_t,A_{i,t}=a]=m_{i,t}(x,a\mid\mathcal F_t),\quad
\Pr[X_{i,t+1}=x'\mid\mathcal F_t,A_{i,t}=a]=g_{i,t}(x'\mid x,a\mid\mathcal F_t)
\end{equation}
almost surely on that event.

\emph{(iii) History-dependent distribution, marginal target and the assignment weight.} Assume (\ref{asm:crossunit}), and fix $(i,t,x,a)$.
\begin{itemize}
\item[(iii.1)] If $\pi_{i,t}(a\mid x)>0$ almost surely on $\{X_{i,t}=x\}$, then $m_{i,t}(x,a)$ is design-identified through
\begin{equation}\label{eq:ident-ipw}
m_{i,t}(x,a)=\mathbb E\!\left[\frac{Y_{i,t}\,\mathbf 1\{A_{i,t}=a\}}{\pi_{i,t}(a\mid X_{i,t})}\ \middle|\ X_{i,t}=x\right].
\end{equation}
\item[(iii.2)] If in addition $\Pr[X_{i,t}=x,A_{i,t}=a]>0$, then
\begin{equation}\label{eq:ident-tilt}
\begin{split}
\mathbb E[Y_{i,t}\mid X_{i,t}=x,A_{i,t}=a]&=\mathbb E\big[m_{i,t}(x,a\mid\mathcal F_t)\mid X_{i,t}=x,A_{i,t}=a\big]\\
&=\int\theta(x,a,u)\,\tilde\nu_{i,t}(du),
\end{split}
\end{equation}
where $\tilde\nu_{i,t}$ is the distribution of $U_{i,t}$ obtained by reweighting and renormalizing (tilting) by $\pi_{i,t}(a\mid x)(\mathcal F_t)$: for Borel $B$,
\begin{equation}\label{eq:ident-tiltlaw}
\tilde\nu_{i,t}(B)=\frac{\mathbb E\big[\pi_{i,t}(a\mid x)(\mathcal F_t)\,\Pr[U_{i,t}\in B\mid\mathcal F_t,X_{i,t}=x]\ \mathbf 1\{X_{i,t}=x\}\big]}{\mathbb E\big[\pi_{i,t}(a\mid x)(\mathcal F_t)\,\mathbf 1\{X_{i,t}=x\}\big]}.
\end{equation}
\item[(iii.3)] If $\Pr[X_{i,t}=x,A_{i,t}=a]>0$, then
\begin{equation}\label{eq:ident-cov}
\begin{split}
&\mathbb E[Y_{i,t}\mid X_{i,t}=x,A_{i,t}=a]-m_{i,t}(x,a)\\
&\qquad=\frac{\operatorname{Cov}\big(\pi_{i,t}(a\mid x)(\mathcal F_t),\,m_{i,t}(x,a\mid\mathcal F_t)\mid X_{i,t}=x\big)}{\mathbb E\big[\pi_{i,t}(a\mid x)(\mathcal F_t)\mid X_{i,t}=x\big]},
\end{split}
\end{equation}
which is zero if and only if $\pi_{i,t}(a\mid x)(\mathcal F_t)$ and $m_{i,t}(x,a\mid\mathcal F_t)$ are uncorrelated given $\{X_{i,t}=x\}$. Two sufficient conditions are that $\pi_{i,t}(a\mid x)(\mathcal F_t)$ is almost surely constant on $\{X_{i,t}=x\}$, and that $m_{i,t}(x,a\mid\mathcal F_t)=m_{i,t}(x,a)$ almost surely on $\{X_{i,t}=x\}$.
\end{itemize}

\emph{(iv) Corollary for the budgeted procedure.} Fix a review $k$ at period $t_k$ with predecision history $\mathcal F_{t_k}$, and a design $d=(\mathcal C,\mathcal T_d,\mathcal A_d)$ (Definition~\ref{def:design}) with $t_k\in\mathcal T_d$.
\begin{itemize}
\item[(iv.1)] Assume (\ref{asm:crossunit}). For each $i\in\mathcal C$, each $x$ with $\Pr[X_{i,t_k}=x\mid\mathcal F_{t_k}]>0$, and each $a\in\mathcal A_d$,
\begin{equation}\label{eq:ident-iv1}
\mathbb E[Y_{i,t_k}\mid\mathcal F_{t_k},X_{i,t_k}=x,A_{i,t_k}=a]=m_{i,t_k}(x,a\mid\mathcal F_{t_k}),
\end{equation}
since the design sets $\pi_{i,t_k}(a\mid x)(\mathcal F_{t_k})=1/|\mathcal A_d|$ on the $\mathcal F_{t_k}$-measurable cohort event (the event that the unit belongs to the design's cohort), so that the tilt in (iii.2) is trivial and $\tilde\nu_{i,t_k}=\nu_{i,t_k}(\cdot\mid x)$.
\item[(iv.2)] Let $\sigma$ deploy a deterministic period-$t$ assignment, under which $\pi_{i,t}(\cdot\mid x)$ is a point mass measurable in $\mathcal F_t$.
\begin{itemize}
\item[(iv.2a)] Assume (\ref{asm:crossunit}). If $\sigma$ assigns $a$ at $x$ in period $t$ under each history of positive probability (full coverage: $\pi_{i,t}(a\mid x)(\mathcal F_t)=1$ a.s. on $\{X_{i,t}=x\}$, equivalently $\Pr_\sigma[X_{i,t}=x,A_{i,t}=a]=\Pr_\sigma[X_{i,t}=x]$), then $\mathbb E[Y_{i,t}\mid X_{i,t}=x,A_{i,t}=a]=m_{i,t}(x,a)$. Suppose instead that full coverage fails, so that $\Pr_\sigma[X_{i,t}=x,A_{i,t}=a]<\Pr_\sigma[X_{i,t}=x]$. On the sub-event where $a$ is played, the unweighted cell mean then equals the tilted average $\int\theta(x,a,u)\,\tilde\nu_{i,t}(du)$ of part~(iii.2), which differs from $m_{i,t}(x,a)$ by the covariance in part~(iii.3). The unweighted cell mean therefore equals $m_{i,t}(x,a)$ if and only if full coverage holds or the covariance of part~(iii.3) vanishes; otherwise it equals the tilted average.
\item[(iv.2b)] Suppose the cell $(x,a)$ carries zero probability under $\sigma$ at each $(i,s)$, $s\in\{0,\dots,T\}$. Then $m_{i,t}(x,a)$ is not record-identified. There exist two environments admissible under A1--A2, agreeing on the initial distribution, on the record's distribution, and on the homogeneous full-state kernels $(\theta(x',a',\cdot),f(\cdot\mid x',a',\cdot))$ at each cell of positive support under $\sigma$, and differing only at the unvisited cell $(x,a)$. These two environments induce the same operating-record distribution and yet differ in $m_{i,t}(x,a)$.
\end{itemize}
\end{itemize}
\end{restated}
We first prove four lemmas (sequential randomization, the realized-Markov identity under the full predecision history, the reweighting and inverse-probability identities, and the zero-support non-identification) and then assemble them into the four parts of the identification theorem.

\begin{lemma}[Sequential randomization]\label{lem:ident-L1}
Fix a unit-period $(i,t)$, an action $a$, and a Borel set $B\subseteq\Ucal$, and suppose that (R) holds (the policy randomizes each action through a uniform split). Then
\[
\Pr[U_{i,t}\in B,\ A_{i,t}=a\mid\mathcal F_t]=\pi_{i,t}(a\mid X_{i,t})(\mathcal F_t)\cdot\Pr[U_{i,t}\in B\mid\mathcal F_t]\quad\text{a.s.,}
\]
yielding $A_{i,t}\perp U_{i,t}\mid\mathcal F_t$. More generally, $A_{i,t}$ is independent of $(U_{i,t},\text{unit }i\text{'s potential-outcome process})$ given $(\mathcal F_t,X_{i,t})$.
\end{lemma}

\begin{proof}
By (R), the event $A_{i,t}=a$ occurs precisely when $\xi_{i,t}\in I_{i,t}(X_{i,t},a)$, an interval of length $\pi_{i,t}(a\mid X_{i,t})(\mathcal F_t)$ whose endpoints are $\mathcal F_t$-measurable, since the policy is $\mathcal F_t$-measurable and $X_{i,t}\in\mathcal F_t$. The variable $\xi_{i,t}$ is uniform on $[0,1)$ and, again by (R), independent of the historical record, of each unit's potential-outcome process, and of $\mathcal F_t$; this history field is generated by $\mathcal H_{-1}$ (the pre-study record), the past records, and the current states, so $\xi_{i,t}$ is fresh. Through the realized action path (A1), $U_{i,t}$ is a measurable function of unit $i$'s potential process and of $\xi_{<t}$ (all units' uniform draws before the current period), and is thus $(\mathcal F_t\vee\sigma(\text{potentials}))$-measurable and independent of $\xi_{i,t}$. Consequently, on any $\mathcal F_t$-atom,
\[
\begin{split}
\Pr[U_{i,t}\in B,\ \xi_{i,t}\in I_{i,t}(X_{i,t},a)\mid\mathcal F_t]
&=\Pr[U_{i,t}\in B\mid\mathcal F_t]\,\big|I_{i,t}(X_{i,t},a)\big|\\
&=\Pr[U_{i,t}\in B\mid\mathcal F_t]\,\pi_{i,t}(a\mid X_{i,t})(\mathcal F_t),
\end{split}
\]
which is the displayed factorization. Dividing by $\pi_{i,t}(a\mid X_{i,t})(\mathcal F_t)$ where it is positive establishes $A_{i,t}\perp U_{i,t}\mid\mathcal F_t$. Running the same computation with $U_{i,t}$ replaced by any Borel functional of unit $i$'s potential process yields the last clause.
\end{proof}

\begin{lemma}[Unit-$i$ realized-Markov identity survives conditioning on the full $\mathcal F_t$]\label{lem:ident-L2}
Let $\bar{\mathcal H}_{i,t}$ be the unit-$i$ realized $\sigma$-field of Lemma~\ref{lem:pathlaw}, generated by unit $i$'s realized $(X_{i,s},U_{i,s},A_{i,s},Y_{i,s})_{s<t}$ and $(X_{i,t},U_{i,t},A_{i,t})$, and suppose that A1, A2, \ref{asm:crossunit}, and (R) hold. Then
\[
\begin{gathered}
\mathbb E[Y_{i,t}\mid\bar{\mathcal H}_{i,t}\vee\mathcal F_t]=\theta(X_{i,t},A_{i,t},U_{i,t}),\\
\Pr[X_{i,t+1}=x',U_{i,t+1}\in B\mid\bar{\mathcal H}_{i,t}\vee\mathcal F_t]=f(\{x'\}\times B\mid X_{i,t},A_{i,t},U_{i,t}).
\end{gathered}
\]
\end{lemma}

\begin{proof}
Set $\mathcal G_{i,t}:=\bar{\mathcal H}_{i,t}\vee\mathcal F_t$. We first name the potential-outcome processes of Assumption~\ref{asm:crossunit}: $\mathcal P_i:=(Y_{i,\cdot}(\cdot),X_{i,\cdot}(\cdot),U_{i,\cdot}(\cdot))$ is unit $i$'s own process, and $\mathcal P_{-i}:=\{\mathcal P_j\}_{j\ne i}$ collects the others'. Fix a unit-$i$ assignment path $\bar a=(a_0,\dots,a_t)$, and let $\mathcal G^{\mathrm{pot}}_{i,t}(\bar a)$ be unit $i$'s potential history $\sigma$-field along $\bar a$, which is the object $\mathcal G_{i,t}(\bar a)$ that Lemma~\ref{lem:pathlaw} and the formal A2 introduce. It is generated by the pre-$t$ triples $\big(X_{i,s}(\bar a_{s-1}),U_{i,s}(\bar a_{s-1}),Y_{i,s}(\bar a_s)\big)_{s<t}$ together with the period-$t$ pre-outcome pair $\big(X_{i,t}(\bar a_{t-1}),U_{i,t}(\bar a_{t-1})\big)$ only; the period-$t$ potential outcome $Y_{i,t}(\bar a_t)$ is not among the generators. Each generator is a coordinate functional of $\mathcal P_i$; we obtain $\mathcal G^{\mathrm{pot}}_{i,t}(\bar a)\subseteq\sigma(\mathcal P_i)$.

The argument rests on the formal A2 conditional-mean clause on the potential history: for each $\bar a$,
\begin{equation}\label{eq:ident-L2a}
\begin{aligned}
\mathbb E\big[Y_{i,t}(\bar a)\mid\mathcal G^{\mathrm{pot}}_{i,t}(\bar a)\big]&=\theta\big(X_{i,t}(\bar a_{t-1}),a_t,U_{i,t}(\bar a_{t-1})\big),\\
\Pr\big[X_{i,t+1}(\bar a)=x',U_{i,t+1}(\bar a)\in B\mid\mathcal G^{\mathrm{pot}}_{i,t}(\bar a)\big]&=f(\{x'\}\times B\mid X_{i,t}(\bar a_{t-1}),a_t,U_{i,t}(\bar a_{t-1})).
\end{aligned}
\end{equation}
This is \eqref{eq:response}, respectively \eqref{eq:transition}, of the formal restatement of A2, read against its own $\mathcal G_{i,t}(\bar a)$.

Let $W$ collect the information in $\mathcal G_{i,t}$ beyond $\bar{\mathcal H}_{i,t}$,
\[
\begin{gathered}
W:=\Big(\mathcal H_{-1},\ \{(X_{j,s},A_{j,s},Y_{j,s},X_{j,s+1}):j\ne i,\ s<t\},\ \{X_{j,t}:j\ne i\}\Big),\\
\mathcal G_{i,t}=\bar{\mathcal H}_{i,t}\vee\sigma(W).
\end{gathered}
\]
The uniforms are not generators of $\mathcal G_{i,t}$: by \eqref{eq:ident-filt} a realized action $A_{i',s}$ is a coarsening of $\xi_{i',s}$ rather than the uniform itself, so no $\xi_{i',s}$ is $\mathcal G_{i,t}$-measurable and the displayed equality holds exactly. The uniforms enter only through the enlarged vehicle field
\[
\mathcal K:=\mathcal G^{\mathrm{pot}}_{i,t}(\bar a)\vee\sigma\big(\mathcal P_{-i},\ \xi_{\le t}\big),\qquad \xi_{\le t}:=(\xi_{i',s})_{i'\in\mathcal V,\,s\le t},
\]
which (as Step 1 below shows) contains $\mathcal G_{i,t}$ on $\{\bar A_{i,t}=\bar a\}$ and adjoins $\xi_{i,t}$, making the event $\{\bar A_{i,t}=\bar a\}$ measurable. We recover the identity for $\mathcal G_{i,t}$ from the one for $\mathcal K$ by the tower property in Step 3, since $\mathcal G_{i,t}\subseteq\mathcal K$ there.

\emph{Step 1.} Fix $\bar a$ and work on the event $\{\bar A_{i,t}=\bar a\}$. (1a) By A1 (consistency), unit $i$'s realized quantities coincide with its potentials along $\bar a$: $X_{i,s}=X_{i,s}(\bar a_{s-1})$, $U_{i,s}=U_{i,s}(\bar a_{s-1})$ for $s\le t$ and $Y_{i,s}=Y_{i,s}(\bar a_s)$ for $s<t$. Each $\bar{\mathcal H}_{i,t}$-measurable functional then equals a $\mathcal G^{\mathrm{pot}}_{i,t}(\bar a)$-measurable one on this event; this gives $\bar{\mathcal H}_{i,t}\subseteq\mathcal G^{\mathrm{pot}}_{i,t}(\bar a)$ there. (1b) By A1 (consistency, no interference), for each $j\ne i$ and $s<t$ the realized tuple $(X_{j,s},U_{j,s},A_{j,s},Y_{j,s})$ and the current state $X_{j,t}$ are measurable functions of $\mathcal P_j$ and of the uniforms $\xi_{j,\le s}$. This holds because (R) assigns $A_{j,s}=a$ when $\xi_{j,s}$ falls in $I_{j,s}(X_{j,s},a)$, whose endpoints are $\mathcal F_s$-measurable functionals of all units' earlier realized records. Tracing the recursion, $W$ is a measurable function of $\big(\mathcal P_{-i},\ \xi_{<t},\ \text{unit }i\text{'s realized records }(X_{i,s},U_{i,s},A_{i,s},Y_{i,s})_{s<t}\big)$. Unit $i$ enters $W$ only through those earlier realized records, which feed the other units' history-dependent policies and are $\mathcal G^{\mathrm{pot}}_{i,t}(\bar a)$-measurable on the event by (1a). We conclude that $W$ is measurable with respect to $\mathcal G^{\mathrm{pot}}_{i,t}(\bar a)\vee\sigma(\mathcal P_{-i},\xi_{<t})\subseteq\mathcal K$. (1c) Here $\{\bar A_{i,t}=\bar a\}=\bigcap_{s\le t}\{\xi_{i,s}\in I_{i,s}(X_{i,s},a_s)\}$; each interval's endpoints are $\mathcal F_s$-measurable, and the period-$t$ leg uses $\xi_{i,t}$, so that the event is $\mathcal K$-measurable. The joint measurability of $\big(\bar{\mathcal H}_{i,t},W,\{\bar A_{i,t}=\bar a\}\big)$ with respect to $\mathcal K$ follows from the same induction on $s\le t$ used in the proof of Lemma~\ref{lem:pathlaw}, carried here across all units with the extra field $\sigma(\mathcal P_{-i},\xi_{\le t})$ adjoined.

\emph{Step 2.} By (\ref{asm:crossunit}) (cross-unit independence of the potential-outcome processes and the randomization), $\mathcal P_i$ is independent of $(\mathcal P_{-i},\xi)$, hence of $(\mathcal P_{-i},\xi_{\le t})$. Consequently, for any bounded functional $h(\mathcal P_i)$,
\begin{equation}\label{eq:ident-L2b}
\mathbb E\big[h(\mathcal P_i)\mid\mathcal G^{\mathrm{pot}}_{i,t}(\bar a)\vee\sigma(\mathcal P_{-i},\xi_{\le t})\big]
=\mathbb E\big[h(\mathcal P_i)\mid\mathcal G^{\mathrm{pot}}_{i,t}(\bar a)\big]\quad\text{a.s.,}
\end{equation}
because adjoining a $\sigma$-field independent of $\big(h(\mathcal P_i),\mathcal G^{\mathrm{pot}}_{i,t}(\bar a)\big)$ leaves the conditional expectation unchanged.

\emph{Step 3.} We now take $Y_{i,t}(\bar a)=h(\mathcal P_i)$ in \eqref{eq:ident-L2b}, which is admissible since it is $\sigma(\mathcal P_i)$-measurable and bounded by $B_Y$. Combining Steps 1 and 2, for bounded $\tilde\psi$ measurable with respect to $\mathcal G^{\mathrm{pot}}_{i,t}(\bar a)$ and bounded $\phi$ measurable with respect to $\sigma(\mathcal P_{-i},\xi_{\le t})$,
\[
\begin{aligned}
\mathbb E\big[Y_{i,t}(\bar a)\,\tilde\psi\,\phi\,\mathbf 1\{\bar A_{i,t}=\bar a\}\big]
&=\mathbb E\Big[\mathbb E\big[Y_{i,t}(\bar a)\mid\mathcal K\big]\,\tilde\psi\,\phi\,\mathbf 1\{\bar A_{i,t}=\bar a\}\Big]\\
&=\mathbb E\Big[\mathbb E\big[Y_{i,t}(\bar a)\mid\mathcal G^{\mathrm{pot}}_{i,t}(\bar a)\big]\,\tilde\psi\,\phi\,\mathbf 1\{\bar A_{i,t}=\bar a\}\Big]\\
&=\mathbb E\big[\theta(X_{i,t},a_t,U_{i,t})\,\tilde\psi\,\phi\,\mathbf 1\{\bar A_{i,t}=\bar a\}\big],
\end{aligned}
\]
where the first equality pulls $\tilde\psi\phi\mathbf 1\{\bar A_{i,t}=\bar a\}$ inside the conditional expectation, as it is $\mathcal K$-measurable on the event by (1a)--(1c); the second is \eqref{eq:ident-L2b}; and the third is \eqref{eq:ident-L2a}, the formal A2, with A1's substitution of (1a) on $\{\bar A_{i,t}=\bar a\}$.

It remains to descend from the vehicle field $\mathcal K$ to $\mathcal G_{i,t}$. On $\{\bar A_{i,t}=\bar a\}$, by (1a)--(1c) $\mathcal G_{i,t}$ is a sub-field of $\mathcal K$, and $\mathbb E[Y_{i,t}\mid\mathcal G_{i,t}]=\mathbb E[\theta\mid\mathcal G_{i,t}]$ then follows by the tower property $\mathbb E[\,\cdot\mid\mathcal G_{i,t}]=\mathbb E[\mathbb E[\,\cdot\mid\mathcal K]\mid\mathcal G_{i,t}]$ once the identity holds against each bounded $\mathcal K$-measurable test function. The products $\tilde\psi\phi$ form a multiplicative class generating $\mathcal K$ on the event. The functionals $\Lambda_1(g):=\mathbb E[Y_{i,t}(\bar a)\,g\,\mathbf 1\{\bar A_{i,t}=\bar a\}]$ and $\Lambda_2(g):=\mathbb E[\theta(X_{i,t},a_t,U_{i,t})\,g\,\mathbf 1\{\bar A_{i,t}=\bar a\}]$ agree on those products, and the set $\mathcal M:=\{g\ \text{bounded, }\mathcal K\text{-measurable}:\Lambda_1(g)=\Lambda_2(g)\}$ contains the constants, is a vector space, and is closed under bounded monotone limits by dominated convergence with $|Y_{i,t}(\bar a)|\le B_Y$ and $|\theta|\le B_Y$. By the functional monotone-class theorem, $\mathcal M$ contains all bounded $\mathcal K$-measurable functions, in particular all bounded $\mathcal G_{i,t}$-measurable ones.

The events $\{\bar A_{i,t}=\bar a\}$, $\bar a\in\mathcal A^{t+1}$, partition the sample space. Summing over $\bar a$ against an arbitrary bounded $\mathcal G_{i,t}$-measurable $\psi=\sum_{\bar a}\psi\,\mathbf 1\{\bar A_{i,t}=\bar a\}$ we obtain $\mathbb E[Y_{i,t}\psi]=\mathbb E[\theta(X_{i,t},A_{i,t},U_{i,t})\psi]$, that is,
\[
\mathbb E[Y_{i,t}\mid\mathcal G_{i,t}]=\theta(X_{i,t},A_{i,t},U_{i,t}).
\]
In reward-noise notation this reads $\mathbb E[\epsilon_{i,t}\mid\mathcal G_{i,t}]=0$, extending \eqref{eq:reward}'s $\mathbb E[\epsilon_{i,t}\mid X_{i,t},U_{i,t},A_{i,t}]=0$ across the enlargement. For the transition identity, we repeat Step 3 with $Y_{i,t}(\bar a)$ replaced by $\mathbf 1\{X_{i,t+1}(\bar a)=x',U_{i,t+1}(\bar a)\in B\}$, a bounded $\sigma(\mathcal P_i)$-measurable functional; \eqref{eq:ident-L2b} applies verbatim, and \eqref{eq:ident-L2a}'s transition clause supplies $\Pr[X_{i,t+1}=x',U_{i,t+1}\in B\mid\mathcal G_{i,t}]=f(\{x'\}\times B\mid X_{i,t},A_{i,t},U_{i,t})$.
\end{proof}

\begin{lemma}[Reweighting and inverse-probability weighting under a history-dependent policy]\label{lem:ident-L3}
Fix $(i,t,x,a)$ and suppose that A1, A2, \ref{asm:crossunit}, and (R) hold.
\begin{itemize}
\item[(a)] If $\Pr[X_{i,t}=x,A_{i,t}=a]>0$, then $\mathbb E[Y_{i,t}\mid X_{i,t}=x,A_{i,t}=a]=\mathbb E\big[m_{i,t}(x,a\mid\mathcal F_t)\mid X_{i,t}=x,A_{i,t}=a\big]$.
\item[(b)] If $\pi_{i,t}(a\mid x)(\mathcal F_t)>0$ a.s. on $\{X_{i,t}=x\}$, then
\[
\mathbb E\!\left[\frac{Y_{i,t}\,\mathbf 1\{A_{i,t}=a\}}{\pi_{i,t}(a\mid X_{i,t})}\ \middle|\ X_{i,t}=x\right]=\mathbb E\big[m_{i,t}(x,a\mid\mathcal F_t)\mid X_{i,t}=x\big]=m_{i,t}(x,a).
\]
\end{itemize}
\end{lemma}

\begin{proof}
We prove (a) first. By Lemma~\ref{lem:ident-L2}, $\mathbb E[Y_{i,t}\mid\mathcal G_{i,t}]=\theta(X_{i,t},A_{i,t},U_{i,t})$ with $\mathcal G_{i,t}\supseteq\mathcal F_t\vee\sigma(A_{i,t})$. On $\{X_{i,t}=x,A_{i,t}=a\}$ the tower rule gives
\begin{equation}\label{eq:ident-L3a}
\begin{split}
\mathbb E[Y_{i,t}\mid\mathcal F_t,A_{i,t}=a]&=\mathbb E\big[\theta(x,a,U_{i,t})\mid\mathcal F_t,A_{i,t}=a\big]\\
&=\mathbb E\big[\theta(x,a,U_{i,t})\mid\mathcal F_t,X_{i,t}=x\big]=m_{i,t}(x,a\mid\mathcal F_t),
\end{split}
\end{equation}
where the middle equality is Lemma~\ref{lem:ident-L1}, since $A_{i,t}\perp U_{i,t}\mid\mathcal F_t$ and conditioning further on $A_{i,t}=a$ leaves the $U_{i,t}$-distribution unchanged, and the last equality is the definition of the projection. Taking $\mathbb E[\cdot\mid X_{i,t}=x,A_{i,t}=a]$ of both sides of \eqref{eq:ident-L3a} and using the tower rule over the event $\{X_{i,t}=x,A_{i,t}=a\}$, which is measurable with respect to $\sigma(\mathcal F_t,A_{i,t})$, we obtain (a).

We next prove (b). Under the positivity condition the ratio is well defined on $\{X_{i,t}=x\}$. Condition on $\mathcal F_t$ and restrict to $\{X_{i,t}=x\}$. By (R), $\mathbb E[\mathbf 1\{A_{i,t}=a\}\mid\mathcal F_t]=\pi_{i,t}(a\mid x)(\mathcal F_t)$, and by Lemmas~\ref{lem:ident-L2} and \ref{lem:ident-L1},
\[
\mathbb E\!\left[\frac{Y_{i,t}\mathbf 1\{A_{i,t}=a\}}{\pi_{i,t}(a\mid x)(\mathcal F_t)}\ \middle|\ \mathcal F_t\right]
=\frac{1}{\pi_{i,t}(a\mid x)(\mathcal F_t)}\,\mathbb E\big[\theta(x,a,U_{i,t})\mathbf 1\{A_{i,t}=a\}\mid\mathcal F_t\big]
=m_{i,t}(x,a\mid\mathcal F_t),
\]
where the middle step uses $\mathbb E[\theta(x,a,U_{i,t})\mathbf 1\{A_{i,t}=a\}\mid\mathcal F_t]=\pi_{i,t}(a\mid x)(\mathcal F_t)\,m_{i,t}(x,a\mid\mathcal F_t)$, itself Lemma~\ref{lem:ident-L1} applied to the bounded functional $\theta(x,a,\cdot)$. Taking $\mathbb E[\cdot\mid X_{i,t}=x]$ and applying \eqref{eq:ident-star} (the tower average) gives $m_{i,t}(x,a)$.
\end{proof}

\begin{lemma}[Adaptive non-identification: a zero-support full-state change is invisible on the operating record]\label{lem:ident-L4}
Let $\sigma$ be any strategy, and suppose that A1, A2, \ref{asm:crossunit}, and (R) hold. Let $\mathcal E,\mathcal E'$ be two environments admissible under A1--A2. Suppose they agree on the initial distribution, on the historical record's distribution, and on the full-state kernels $(\theta(x,a,\cdot),f(\cdot\mid x,a,\cdot))$ at each cell $(x,a)$ of positive support under $\sigma$; here positive support means $\Pr_\sigma[X_{i,t}=x,A_{i,t}=a]>0$ for some $(i,t)$. Suppose further that they differ only at cells of zero support under $\sigma$. Then the joint distribution of $(\xi,\bar{\mathbf X}_T,\bar{\mathbf A}_T,\bar{\mathbf Y}_T)$ is the same under $\mathcal E$ and $\mathcal E'$.
\end{lemma}

\begin{proof}
The conditions of Theorem~\ref{thm:undetectable}(i) hold verbatim, with $\sigma$ the operating strategy, $\xi$ its (R)-randomization, and the cross-unit conditional independence of the innovations following from (\ref{asm:crossunit}). That proof runs an induction over periods, showing that the joint distribution of $(\xi,\mathcal F_t)$ is common to $\mathcal E,\mathcal E'$ for each $t$. Its terminal step $t=T+1$ draws the last innovation $(\mathbf Y_T,\mathbf X_{T+1})$ from kernels at the period-$T$ operating-support cells, on which the environments agree, carrying the common-distribution conclusion to $(\xi,\mathcal F_{T+1})$. The cells reached with positive probability at the final period are operating-support cells because the induction has already fixed the common distribution up to that period. Since the full operating record $(\bar{\mathbf X}_T,\bar{\mathbf A}_T,\bar{\mathbf Y}_T)$ is a measurable function of $(\xi,\mathcal F_{T+1})$, its joint distribution with $\xi$ is common to $\mathcal E,\mathcal E'$.
\end{proof}

\begin{proof}[Proof of part (i)]
We prove (i.1) first. At the cell $(i,t,x,a)$ with $\Pr_{\bar\pi}[X_{i,t}=x]>0$ and $\pi_{i,t}(a\mid x)>0$, Proposition~\ref{prop:ident} gives $\mathbb E_{\bar\pi}[Y_{i,t}\mid X_{i,t}=x,A_{i,t}=a]=m^{\bar\pi}_{i,t}(x,a)=m_{i,t}(x,a)$ together with the transition identity. A fixed $\bar\pi$ is $\mathcal F_t$-invariant, so $m^{\bar\pi}_{i,t}=m_{i,t}$ and the assignment probability $\pi_{i,t}(a\mid x)$ is a constant on $\{X_{i,t}=x\}$; no tilting arises, and the cell mean equals the marginal target directly.

We turn to (i.2). Suppose $\Pr_{\bar\pi}[X_{i,s}=x]\,\pi_{i,s}(a\mid x)=0$ for each $(i,s)$. Then the cell $(x,a)$ has zero probability at each $(i,s)$ under $\bar\pi$, the zero-support cell that Proposition~\ref{prop:ident} treats. Consider any two environments $\mathcal E,\mathcal E'$ differing only in the conditional distribution of $(Y,X',U')$ given the full state and action at $(x,a)$, with the initial distribution, the record, and the kernels at all other cells shared. The proof of Proposition~\ref{prop:ident} shows that the observed record distributions coincide. A path visiting $(x,a)$ first at period $s$ has probability at most $\Pr[X_{i,s}=x,A_{i,s}=a]=0$, and the record distributions under the two environments agree on each prefix up to the first visit. Choose $\mathcal E'$ to change $\theta(x,a,\cdot)$ on a $\nu_{i,t}(\cdot\mid x)$-positive set with $\int\theta'(x,a,u)\nu_{i,t}(du\mid x)\ne m_{i,t}(x,a)$: the target then differs while the record distribution is unchanged, and so $m_{i,t}(x,a)$ is not record-identified.
\end{proof}

\begin{proof}[Proof of part (ii)]
The claim is immediate from Lemmas~\ref{lem:ident-L2} and \ref{lem:ident-L1} by the tower rule, on the $\mathcal F_t$-measurable event $\{\pi_{i,t}(a\mid X_{i,t})>0,\ X_{i,t}=x\}$. Here \eqref{eq:ident-L3a} is the response identity $\mathbb E[Y_{i,t}\mid\mathcal F_t,A_{i,t}=a]=m_{i,t}(x,a\mid\mathcal F_t)$. The transition identity follows the same way from the transition half of Lemma~\ref{lem:ident-L2}:
\[
\begin{aligned}
\Pr[X_{i,t+1}=x'\mid\mathcal F_t,A_{i,t}=a]&=\mathbb E\big[f_X(x'\mid x,a,U_{i,t})\mid\mathcal F_t,A_{i,t}=a\big]\\
&=\mathbb E\big[f_X(x'\mid x,a,U_{i,t})\mid\mathcal F_t,X_{i,t}=x\big]=g_{i,t}(x'\mid x,a\mid\mathcal F_t),
\end{aligned}
\]
where the middle equality is Lemma~\ref{lem:ident-L1} and the last is the definition of the history-conditional transition.
\end{proof}

\begin{proof}[Proof of part (iii)]
We prove (iii.1) first. Under the a.s. positivity of $\pi_{i,t}(a\mid x)(\mathcal F_t)$ on $\{X_{i,t}=x\}$, Lemma~\ref{lem:ident-L3}(b) gives $m_{i,t}(x,a)=\mathbb E[Y_{i,t}\mathbf 1\{A_{i,t}=a\}/\pi_{i,t}(a\mid X_{i,t})\mid X_{i,t}=x]$. Because the weights $\pi_{i,t}$ are known and $\mathcal F_t$-measurable, the right side is a functional of the record and the known weights, so $m_{i,t}(x,a)$ is design-identified.

We next prove (iii.2). By Lemma~\ref{lem:ident-L3}(a), $\mathbb E[Y_{i,t}\mid X_{i,t}=x,A_{i,t}=a]=\mathbb E[m_{i,t}(x,a\mid\mathcal F_t)\mid X_{i,t}=x,A_{i,t}=a]$. For a bounded $\mathcal F_t$-measurable $\Phi$, (R) implies $\mathbb E[\Phi\,\mathbf 1\{A_{i,t}=a\}\mathbf 1\{X_{i,t}=x\}]=\mathbb E[\Phi\,\pi_{i,t}(a\mid x)(\mathcal F_t)\mathbf 1\{X_{i,t}=x\}]$, which identifies the conditional distribution of $\mathcal F_t$ given $\{X_{i,t}=x,A_{i,t}=a\}$ as the distribution of $\mathcal F_t$ given $\{X_{i,t}=x\}$ tilted by $\pi_{i,t}(a\mid x)(\mathcal F_t)/\mathbb E[\pi_{i,t}(a\mid x)(\mathcal F_t)\mid X_{i,t}=x]$. Applying this tilt to $m_{i,t}(x,a\mid\mathcal F_t)$ and exchanging the bounded integrals yields
\[
\begin{split}
\mathbb E[Y_{i,t}\mid X_{i,t}=x,A_{i,t}=a]
&=\frac{\mathbb E[\pi_{i,t}(a\mid x)(\mathcal F_t)\,m_{i,t}(x,a\mid\mathcal F_t)\mid X_{i,t}=x]}{\mathbb E[\pi_{i,t}(a\mid x)(\mathcal F_t)\mid X_{i,t}=x]}\\
&=\int\theta(x,a,u)\,\tilde\nu_{i,t}(du),
\end{split}
\]
with $\tilde\nu_{i,t}$ the tilted distribution of \eqref{eq:ident-tiltlaw}, a probability measure.

Finally we prove (iii.3). Subtract $m_{i,t}(x,a)=\mathbb E[m_{i,t}(x,a\mid\mathcal F_t)\mid X_{i,t}=x]$, by \eqref{eq:ident-star}, from the numerator-over-denominator form of (iii.2). Writing $P:=\pi_{i,t}(a\mid x)(\mathcal F_t)$, $M:=m_{i,t}(x,a\mid\mathcal F_t)$, and $\mathbb E_x[\cdot]:=\mathbb E[\cdot\mid X_{i,t}=x]$,
\[
\frac{\mathbb E_x[PM]}{\mathbb E_x[P]}-\mathbb E_x[M]
=\frac{\mathbb E_x[PM]-\mathbb E_x[P]\,\mathbb E_x[M]}{\mathbb E_x[P]}
=\frac{\operatorname{Cov}_x(P,M)}{\mathbb E_x[P]},
\]
the stated covariance formula. It vanishes if and only if $\operatorname{Cov}_x(P,M)=0$. If $P$ is a.s. constant on $\{X_{i,t}=x\}$ then $\operatorname{Cov}_x(P,M)=0$; and if $M=m_{i,t}(x,a)$ a.s. on $\{X_{i,t}=x\}$ then $M$ is a.s. constant and again the covariance is zero.
\end{proof}

\begin{proof}[Proof of part (iv)]
We prove (iv.1) first. By Definition~\ref{def:design}, on the $\mathcal F_{t_k}$-measurable cohort event $\{i\in\mathcal C\}$ the design assignment fixes the policy tensor, so that $\pi_{i,t_k}(a\mid x)(\mathcal F_{t_k})=\mathsf P[i,t_k,x,a]=1/|\mathcal A_d|$, a constant, for each $x$ and each $a\in\mathcal A_d$. Part (ii), which needs (\ref{asm:crossunit}), gives, for $x$ with $\Pr[X_{i,t_k}=x\mid\mathcal F_{t_k}]>0$ and $a\in\mathcal A_d$, $\mathbb E[Y_{i,t_k}\mid\mathcal F_{t_k},X_{i,t_k}=x,A_{i,t_k}=a]=m_{i,t_k}(x,a\mid\mathcal F_{t_k})$; the weight being constant, the tilt established in (iii.2) is trivial and $\tilde\nu_{i,t_k}=\nu_{i,t_k}(\cdot\mid\mathcal F_{t_k},x)$.

We next prove (iv.2a). Assume (\ref{asm:crossunit}), and let $\pi_{i,t}(\cdot\mid x)$ be an $\mathcal F_t$-measurable point mass. Suppose full coverage: $\pi_{i,t}(a\mid x)(\mathcal F_t)=1$ a.s. on $\{X_{i,t}=x\}$. Then part (ii) with the point-mass policy shows $\mathbb E[Y_{i,t}\mid\mathcal F_t,A_{i,t}=a]=m_{i,t}(x,a\mid\mathcal F_t)$ a.s. on $\{X_{i,t}=x\}$; the weight being the constant $1$ there, the covariance computed in (iii.3) is zero, and averaging by \eqref{eq:ident-star} returns $\mathbb E[Y_{i,t}\mid X_{i,t}=x,A_{i,t}=a]=m_{i,t}(x,a)$. Suppose instead that full coverage does not hold, that is, some $a'\ne a$ is played with positive probability at $(i,t,x)$, giving $\Pr_\sigma[X_{i,t}=x,A_{i,t}=a]<\Pr_\sigma[X_{i,t}=x]$. Then on the sub-event where $a$ is played, the weight $\pi_{i,t}(a\mid x)(\mathcal F_t)$ is not a.s. constant on $\{X_{i,t}=x\}$, and by (iii.2)--(iii.3) the unweighted cell mean equals the tilted average $\int\theta(x,a,u)\tilde\nu_{i,t}(du)$, differing from $m_{i,t}(x,a)$ by $\operatorname{Cov}_x(P,M)/\mathbb E_x[P]$. The recorded cell mean thus equals $m_{i,t}(x,a)$ if and only if full coverage holds or that covariance vanishes, and otherwise equals the tilted average.

Finally we prove (iv.2b). Suppose the cell $(x,a)$ carries zero probability under $\sigma$ at each $(i,s)$, $s\in\{0,\dots,T\}$. We apply Lemma~\ref{lem:ident-L4} with $\sigma$ as the operating strategy. Let $\mathcal E,\mathcal E'$ be two environments admissible under A1--A2, agreeing on the initial distribution, the record's distribution, and the homogeneous full-state kernels $(\theta(x',a',\cdot),f(\cdot\mid x',a',\cdot))$ at each cell of positive support under $\sigma$, and differing only at the unvisited cell $(x,a)$. We choose this cell as in (i.2) with $\int\theta'(x,a,u)\nu_{i,t}(du\mid x)\ne m_{i,t}(x,a)$, admissible because $(x,a)$ is unvisited and its single full-state kernel (shared across units by the homogeneity assumed in A2) is otherwise unconstrained by the record. By Lemma~\ref{lem:ident-L4} the operating-record distributions coincide, while $m_{i,t}(x,a)$ differs, and so $m_{i,t}(x,a)$ is not record-identified.
\end{proof}
\subsection{Estimation across units and staleness across periods (Lemma~\ref{lem:window})}\label{app:window}

We now prove two finite-sample estimation statements that complement the identification results in Section~\ref{sec:exploit}, which concern population quantities: part~(i) bounds the error of a cohort's cell-mean estimate, part~(ii) the staleness of a target across periods. Throughout, we use Assumption~\ref{asm:crossunit} (cross-unit independence): the units' potential-outcome processes are mutually independent across units and jointly independent of the assignment randomizations.

\begin{lemma}[Estimation across units and staleness across periods]\label{lem:window}
Let Assumptions~\ref{asm:consistency}, \ref{asm:markov}, \ref{asm:smooth}, \ref{asm:drift}, and~\ref{asm:crossunit} hold. Fix a cell $(x,a)$ and a period $r$. Fix a set $\mathcal V'$ of units that all follow one prespecified policy sequence $\bar\pi\in\bar\Pi$ (the memoryless class of policy sequences). Here the assignment at each unit depends only on that unit's own state and its own randomization, and the units are started from one initial distribution of the full state. By Assumption~\ref{asm:markov} these units share the conditional distribution $\nu_r(\cdot\mid x)$ of $U_{i,r}$ given $X_{i,r}=x$, and they share the target $m_r(x,a)$ of \eqref{eq:mdef}. Let $N_r:=|\{i\in\mathcal V':(X_{i,r},A_{i,r})=(x,a)\}|$, and let $\bar Y_r$ be the mean of $Y_{i,r}$ over these units. Then
\begin{enumerate}
\item[(i)] for each integer $\ell\ge1$ with $\Pr[N_r=\ell]>0$ and each $\varepsilon>0$,
\[
\Pr\big[\,|\bar Y_r-m_r(x,a)|\ge\varepsilon\ \big|\ N_r=\ell\,\big]\le2\exp\!\big(-\ell\varepsilon^2/(2B_Y^2)\big);
\]
\item[(ii)] for each $t\ge r$, $|m_t(x,a)-m_r(x,a)|\le L_\theta\lambda\,(t-r)$.
\end{enumerate}
\end{lemma}

\begin{proof}
We first prove (i). Write $E_i:=\{(X_{i,r},A_{i,r})=(x,a)\}$ for the event that unit $i$ occupies the cell, and write $\mathcal S:=\{i\in\mathcal V':E_i\text{ occurs}\}$ for the random set of such units; then $N_r=|\mathcal S|$ and $\bar Y_r$ is the mean of $Y_{i,r}$ over $i\in\mathcal S$. Under (\ref{asm:crossunit}) and a prespecified strategy, the pair $(Y_{i,r},\mathbf 1_{E_i})$ is a function of unit $i$'s potential-outcome process and of its own randomization; since by (\ref{asm:crossunit}) these per-unit ingredients are independent across units, the pairs are independent across $i\in\mathcal V'$. Now fix a set $S\subseteq\mathcal V'$ with $|S|=\ell$ and $\Pr[\mathcal S=S]>0$. Conditionally on $\{\mathcal S=S\}$, that is, on $\{E_i\text{ for }i\in S\}\cap\{E_i^{c}\text{ for }i\in\mathcal V'\setminus S\}$, the outcomes $Y_{i,r}$, $i\in S$, are independent, and each has the conditional distribution of $Y_{i,r}$ given $E_i$ alone, since $(Y_{i,r},\mathbf 1_{E_i})$ is independent of the indicators of the other units. Because $\Pr[E_i]>0$ and the strategy is prespecified, Theorem~\ref{thm:ident}(i) establishes $\mathbb E[Y_{i,r}\mid E_i]=m_r(x,a)$, the same value for each $i\in\mathcal V'$ by the condition on $\mathcal V'$. Since the outcomes take values in $[-B_Y,B_Y]$, Hoeffding's inequality (Lemma~\ref{lem:hoeffding}) applied to the $\ell$ independent outcomes yields $\Pr[\,|\bar Y_r-m_r(x,a)|\ge\varepsilon\mid\mathcal S=S]\le2\exp(-2\ell\varepsilon^2/(2B_Y)^2)=2\exp(-\ell\varepsilon^2/(2B_Y^2))$. Finally, the event $\{N_r=\ell\}$ is the disjoint union, over the sets $S$ of size $\ell$ with positive probability, of the events $\{\mathcal S=S\}$. Since the bound holds on each such event, it holds conditionally on $\{N_r=\ell\}$, which proves (i).

We now prove (ii). Fix the cell. By Assumption~\ref{asm:smooth}, $|\theta(x,a,u)-\theta(x,a,u')|\le L_\theta d_U(u,u')$. Hence the map $u\mapsto\theta(x,a,u)$ is $L_\theta$-Lipschitz for $d_U$, and it is bounded by $B_Y$. By the Kantorovich--Rubinstein duality for the Wasserstein-$1$ distance under $d_U$, any two distributions $\nu,\nu'$ on $\Ucal$ satisfy $|\int\theta(x,a,u)\,\nu(du)-\int\theta(x,a,u)\,\nu'(du)|\le L_\theta W_1(\nu,\nu')$. By \eqref{eq:mdef}, $m_s(x,a)=\int\theta(x,a,u)\,\nu_s(du\mid x)$ for each period $s$ under the strategy, and therefore $|m_t(x,a)-m_r(x,a)|\le L_\theta W_1(\nu_t(\cdot\mid x),\nu_r(\cdot\mid x))$. Combining the triangle inequality for $W_1$ with the one-period bound of Assumption~\ref{asm:drift}, we obtain $W_1(\nu_t(\cdot\mid x),\nu_r(\cdot\mid x))\le\sum_{s=r}^{t-1}W_1(\nu_{s+1}(\cdot\mid x),\nu_s(\cdot\mid x))\le\lambda(t-r)$.
\end{proof}
\subsection{Identifiability of the regime distribution from visited cells (Theorem~\ref{thm:regimeid})}\label{app:regimeid}

This section proves Theorem~\ref{thm:regimeid}. We specialize to finitely many known regimes with a known invariant response basis. In this setting we characterize when the regime distribution and each cell's table are identified from the visited cells, and we bound the estimate formed from noisy visited means. We then reduce the persistence of the deployed action to Corollary~\ref{cor:lockin}.

\paragraph{Standing objects and notation.}
The sets $\mathcal X,\mathcal A$ are finite, a \emph{cell} is $z=(x,a)\in\mathcal X\times\mathcal A$, and outcomes lie in $[-B_Y,B_Y]$. We fix one unit $i$ and one period $t$, suppress those indices, and write $\nu:=\nu_{i,t}$ and $m(z):=m_{i,t}(z)$. Assumptions~\ref{asm:consistency} and~\ref{asm:markov} are in force throughout.

We write $\Delta_K:=\{\mu\in\mathbb R^K:\mu\ge 0,\ \mathbf 1^\top\mu=1\}$ for the probability simplex, with $\mathbf 1=(1,\dots,1)^\top$ and $e_k$ the $k$-th basis vector, and for $\mu\in\Delta_K$ we set $\operatorname{supp}(\mu):=\{k:\mu_k>0\}$. We write $H_0:=\{\delta\in\mathbb R^K:\mathbf 1^\top\delta=0\}$ for the tangent hyperplane of $\operatorname{aff}\Delta_K$, and $\operatorname{relint}\Delta_K=\{\nu:\nu>0\text{ coordinatewise}\}$ for its relative interior. For $S\subseteq\{1,\dots,K\}$ we let $L_S:=\{\delta\in\mathbb R^K:\delta_k=0\text{ for all }k\notin S\}$ be the coordinate subspace on $S$, which is the direction span of the face $F_S:=\{\mu\in\Delta_K:\operatorname{supp}\mu\subseteq S\}$. For $\nu\in\Delta_K$ we define the tangent cone of $\Delta_K$ at $\nu$ intersected with $H_0$,
\begin{equation}\label{eq:regimeid-cone}
T_\nu:=\{\delta\in H_0:\ \delta_k\ge 0\ \text{for all }k\notin\operatorname{supp}(\nu)\}.
\end{equation}
For $\delta\in\mathbb R^K$, we have $\nu+t\delta\in\Delta_K$ for all small $t>0$ if and only if $\delta\in T_\nu$. The sum condition gives $\delta\in H_0$; at a coordinate $k$ with $\nu_k=0$ feasibility requires $\delta_k\ge 0$, since that coordinate can only increase, whereas at a coordinate with $\nu_k>0$ either sign is feasible for small $t$.

\begin{restated}{Theorem}{thm:regimeid}{Identifiability of the regime distribution from visited cells, formal statement}
Assume Assumptions~\ref{asm:consistency} and~\ref{asm:markov}. Assume the structural specialization:
\begin{itemize}
\item[(a)] The latent space is a known finite set $\mathcal U=\{1,\dots,K\}$ of regimes, $K\ge 1$.
\item[(b)] The regime-conditional response $\theta(x,a,\cdot)\in\mathbb R^K$, with $k$-th entry $\theta(x,a,k)$, is known for each cell; write $\theta(z,\cdot)\in\mathbb R^K$ for the row vector at cell $z$.
\item[(c)] For this unit the regime distribution does not depend on the observed state: $\nu(\cdot\mid x)=\nu\in\Delta_K$ for all $x\in\mathcal X$.
\end{itemize}
Under conditions (a)--(c), by the target-quantities convention (Section~\ref{sec:setup}), which averages the regime-conditional response over the regime distribution, together with Assumption~\ref{asm:markov}, the target mean at any cell $z$ is the linear functional
\begin{equation}\label{eq:regimeid-mlin}
m(z)=\theta(z,\cdot)^\top\nu=\sum_{k=1}^K\theta(z,k)\,\nu_k.
\end{equation}
Let the visited cells $V=\{z_1,\dots,z_{|V|}\}\subseteq\mathcal X\times\mathcal A$ be the cells assigned positive probability by the deployed policy in period $t$ (the supported cells of Proposition~\ref{prop:ident}). Define
\begin{equation}\label{eq:regimeid-design}
\Theta_V:=\begin{pmatrix}\theta(z_1,\cdot)\\ \vdots\\ \theta(z_{|V|},\cdot)\end{pmatrix}\in\mathbb R^{|V|\times K},
\qquad
m_V:=\big(m(z_1),\dots,m(z_{|V|})\big)^\top=\Theta_V\nu\in\mathbb R^{|V|},
\end{equation}
and define the augmented regime-response matrix, its rank, and its kernel
\begin{equation}\label{eq:regimeid-G}
\begin{gathered}
G:=\begin{pmatrix}\Theta_V\\ \mathbf 1^\top\end{pmatrix}\in\mathbb R^{(|V|+1)\times K},
\quad
r:=\operatorname{rank}(G)\in\{1,\dots,K\},\\
\ker G=\{\delta:\Theta_V\delta=0,\ \mathbf 1^\top\delta=0\},
\end{gathered}
\end{equation}
from which $\ker G\subseteq H_0$, $\operatorname{row}(G)=(\ker G)^\perp$, and $\dim\ker G=K-r$. Define the \emph{restricted modulus}, the smallest singular value of $\Theta_V$ acting on the tangent hyperplane $H_0$, by
\begin{equation}\label{eq:regimeid-sigma}
\sigma:=\min_{\delta\in H_0:\ \lVert\delta\rVert_2=1}\lVert\Theta_V\delta\rVert_2 ,
\end{equation}
By Lemma~\ref{lem:regimeid-L1}, $\sigma>0$ if and only if $r=K$; when $r=K$, Lemma~\ref{lem:regimeid-L2} gives $\sigma\ge\sigma_{\min}(G)$. We say that a quantity $q(\nu)$ is \emph{identified at $\nu$ from $m_V$} if $q$ is constant on the identified set
\begin{equation}\label{eq:regimeid-iset}
\mathcal I(\nu):=\{\mu\in\Delta_K:\Theta_V\mu=\Theta_V\nu\}=(\nu+\ker G)\cap\Delta_K ,
\end{equation}
and \emph{identified} without qualification when $\mathcal I(\nu)$ is a singleton. Let $\nu\in\Delta_K$ be the true regime distribution.

\emph{(i) Global injectivity and pointwise identification.}
\begin{itemize}
\item[(i.1)] The map $\Delta_K\ni\mu\mapsto\Theta_V\mu$ is injective on all of $\Delta_K$ if and only if $r=K$, that is, the columns of $\Theta_V$ are affinely independent; this holds precisely when $\ker G=\{0\}$. In that case, for each $\nu\in\Delta_K$ the identified set $\mathcal I(\nu)$ is the singleton $\{\nu\}$. Then $\nu$ is identified, and by \eqref{eq:regimeid-mlin} so is each cell's table $m(z)=\theta(z,\cdot)^\top\nu$.
\item[(i.2)] If $r<K$, then each $\nu\in\operatorname{relint}\Delta_K$ is not identified: $\mathcal I(\nu)\supsetneq\{\nu\}$. A boundary $\nu$, with $\operatorname{supp}(\nu)\subsetneq\{1,\dots,K\}$, is identified if and only if $\ker G\cap T_\nu=\{0\}$.
\end{itemize}

\emph{(ii) Identified set and per-cell identification when $r<K$.}
\begin{itemize}
\item[(ii.1)] The identified set $\mathcal I(\nu)=(\nu+\ker G)\cap\Delta_K$ is a nonempty compact convex polytope containing $\nu$. Let $S^\star:=\bigcup_{\mu\in\mathcal I(\nu)}\operatorname{supp}(\mu)$ be the union of supports over the identified set, equivalently the index set of the smallest face $F_{S^\star}$ of $\Delta_K$ containing $\mathcal I(\nu)$. Then
\begin{equation}\label{eq:regimeid-dim}
\dim\mathcal I(\nu)=\dim\big(\ker G\cap L_{S^\star}\big).
\end{equation}
As a sufficient condition, if $\nu\in\operatorname{relint}\Delta_K$ then $S^\star=\{1,\dots,K\}$, $L_{S^\star}=\mathbb R^K$, and $\dim\mathcal I(\nu)=\dim\ker G=K-r$.
\item[(ii.2)] Fix a cell $z$. The table $m(z)=\theta(z,\cdot)^\top\nu$ is identified for each $\nu\in\Delta_K$ if and only if $\theta(z,\cdot)\in\operatorname{row}(G)$, equivalently $\theta(z,\cdot)\perp\ker G$. If instead $\theta(z,\cdot)\notin\operatorname{row}(G)$, the identified set of $m(z)$ at a given $\nu$ is the compact interval $\{\theta(z,\cdot)^\top\mu:\mu\in\mathcal I(\nu)\}=[\underline m,\overline m]$, the image of $\mathcal I(\nu)$ under $\theta(z,\cdot)^\top(\cdot)$; it is nondegenerate, $\overline m-\underline m>0$, if and only if the orthogonal projection of $\theta(z,\cdot)$ onto $\ker G\cap L_{S^\star}$ is nonzero.
\end{itemize}

\emph{(iii) Estimation from noisy visited means.} Suppose $r=K$; then $\sigma>0$. Let the visited means be observed with additive error $e\in\mathbb R^{|V|}$, $\widehat m_V:=m_V+e$ with $\lVert e\rVert_2<\infty$, and let
\begin{equation}\label{eq:regimeid-cls}
\hat\nu\in\arg\min_{\mu\in\Delta_K}\ \lVert\Theta_V\mu-\widehat m_V\rVert_2^2 .
\end{equation}
Then the minimizer is attained and unique, and
\begin{equation}\label{eq:regimeid-bound}
\lVert\hat\nu-\nu\rVert_2\ \le\ \frac{\lVert e\rVert_2}{\sigma}\ \le\ \frac{\sqrt{|V|}\,\lVert e\rVert_\infty}{\sigma}\ \le\ \frac{\sqrt{|V|}\,\lVert e\rVert_\infty}{\sigma_{\min}(G)} .
\end{equation}
Consequently, for each cell $z$,
\begin{equation}\label{eq:regimeid-cellbound}
\big|\theta(z,\cdot)^\top\hat\nu-m(z)\big|\ \le\ \lVert\theta(z,\cdot)\rVert_2\,\lVert\hat\nu-\nu\rVert_2\ \le\ \lVert\theta(z,\cdot)\rVert_2\,\frac{\sqrt{|V|}\,\lVert e\rVert_\infty}{\sigma}.
\end{equation}

\emph{(iv) Lock-in of structural adaptation along unidentified directions.} We first locate the ridge limit along the unidentified directions in (iv.1), then read it off at the untried cell in (iv.2), and finally convert it to a policy in (iv.3). Assume Assumptions~\ref{asm:cellwise} and~\ref{asm:likelihood} and the setting of Proposition~\ref{prop:censoring} and Corollary~\ref{cor:lockin} at state $x$, deployed action $a'$, and untried action $a^\dagger$. Let the estimate be formed by the ridge-toward-$\nu^{\mathrm{sim}}$ fit
\begin{equation}\label{eq:regimeid-ridge}
\hat\nu_\rho\in\arg\min_{\mu\in\Delta_K}\ \lVert\Theta_V\mu-\widehat m_V\rVert_2^2+\rho\,\lVert\mu-\nu^{\mathrm{sim}}\rVert_2^2,\qquad \rho>0,
\end{equation}
with $\nu^{\mathrm{sim}}\in\Delta_K$ the calibration distribution, yielding $\theta^{\mathrm{sim}}(z)=\theta(z,\cdot)^\top\nu^{\mathrm{sim}}$. Let $z^\dagger=(x,a^\dagger)$ be the untried cell with $\theta(z^\dagger,\cdot)\notin\operatorname{row}(G)$, and decompose $\theta(z^\dagger,\cdot)=\theta^{\parallel}+\theta^{\perp}$ with $\theta^{\parallel}\in\operatorname{row}(G)$, $\theta^{\perp}\in\ker G$, and $\theta^{\perp}\neq 0$. Define the $\nu^{\mathrm{sim}}$-closest point of the identified flat
\begin{equation}\label{eq:regimeid-mu0}
\tilde\mu_0:=\nu+P_{\ker G}(\nu^{\mathrm{sim}}-\nu)=P_{\operatorname{row}(G)}\nu+P_{\ker G}\nu^{\mathrm{sim}},
\end{equation}
the point of $\nu+\ker G$ nearest $\nu^{\mathrm{sim}}$ in $\lVert\cdot\rVert_2$, and call $\tilde\mu_0\in\Delta_K$ the \emph{feasibility condition}: the nearest point of the identified flat is itself a distribution. The clauses below state which of them assume it. Then:
\begin{itemize}
\item[(iv.1)] Take $e=0$, let $\mu^\star:=\arg\min_{\mu\in(\nu+\ker G)\cap\Delta_K}\lVert\mu-\nu^{\mathrm{sim}}\rVert_2$ be the constrained ridge limit (Lemma~\ref{lem:regimeid-L4}(c)), let $W:=\operatorname{row}(G)\cap H_0$, and let $\mu_{\mathrm{aff}}(\rho)$ be the equality-constrained minimizer of \eqref{eq:regimeid-ridge} over $\operatorname{aff}\Delta_K$ with the nonnegativity constraints dropped.
  \begin{itemize}
  \item If $\tilde\mu_0\in\Delta_K$, then $\mu^\star=\tilde\mu_0$ and
  \begin{equation}\label{eq:regimeid-null}
  \lim_{\rho\to 0^+}P_{\ker G}\,\hat\nu_\rho=P_{\ker G}\,\mu^\star=P_{\ker G}\,\nu^{\mathrm{sim}},
  \qquad \lim_{\rho\to 0^+}P_{\operatorname{row}(G)}\hat\nu_\rho=P_{\operatorname{row}(G)}\nu.
  \end{equation}
  \item If $\tilde\mu_0\in\operatorname{relint}\Delta_K$, then there is $\rho_1>0$ with $P_{\ker G}\hat\nu_\rho=P_{\ker G}\nu^{\mathrm{sim}}$ for all $\rho\in(0,\rho_1)$.
  \item For each $\rho>0$, the row component of $\mu_{\mathrm{aff}}(\rho)$ carries the bias
  \begin{equation}\label{eq:regimeid-bias}
  P_{W}\mu_{\mathrm{aff}}(\rho)-P_{W}\nu
  =\rho\,\big(\Theta_V^\top\Theta_V+\rho\,\Pi\big)^{-1}\big|_{W}\,P_{\operatorname{row}(G)}(\nu^{\mathrm{sim}}-\nu),
  \end{equation}
  and $P_W\hat\nu_\rho=P_W\mu_{\mathrm{aff}}(\rho)$ precisely on the range of $\rho$ where no nonnegativity constraint is active: for all $\rho\in(0,\rho_1)$ under $\tilde\mu_0\in\operatorname{relint}\Delta_K$, and as the $\rho\to0^+$ limit $P_W\hat\nu_\rho-P_W\nu=O(\rho)\to0$ under $\tilde\mu_0\in\Delta_K$ (Lemma~\ref{lem:regimeid-L4}(b)); where a constraint is active, $\hat\nu_\rho$ differs from $\mu_{\mathrm{aff}}(\rho)$ along a correction with a nonzero $W$-component.
  \item There exist $K=4$, a visited design with $r=3$, giving $\dim W=2$, a true $\nu$, and $\nu^{\mathrm{sim}}$ with $\tilde\mu_0\in\operatorname{relint}\Delta_4$, together with a $\rho>0$, for which $P_{\ker G}\hat\nu_\rho\neq P_{\ker G}\nu^{\mathrm{sim}}$. When $\tilde\mu_0\notin\Delta_K$, the limiting null component is $P_{\ker G}\mu^\star\neq P_{\ker G}\nu^{\mathrm{sim}}$, and the hybrid below uses $\mu^\star$ in place of $\nu^{\mathrm{sim}}$.
  \end{itemize}
\item[(iv.2)] Under the feasibility condition $\tilde\mu_0\in\Delta_K$ and $e=0$, the structural plug-in mean at $z^\dagger$ converges as $\rho\to 0^+$ to the hybrid
\begin{equation}\label{eq:regimeid-hybrid}
\hat m(z^\dagger):=\lim_{\rho\to 0^+}\theta(z^\dagger,\cdot)^\top\hat\nu_\rho=\theta^{\parallel\top}\nu+\theta^{\perp\top}\nu^{\mathrm{sim}},
\end{equation}
whose row-space part is the true table's row-space part $\theta^{\parallel\top}\nu$ and whose null part is the calibration value $\theta^{\perp\top}\nu^{\mathrm{sim}}$. The limiting null part above is the calibration value, whereas at fixed $\rho>0$ the affine minimizer's row-space part carries the bias in \eqref{eq:regimeid-bias}, and once $\dim W\ge 2$ there are instances in which the ridge fit's null part differs from $\theta^{\perp\top}\nu^{\mathrm{sim}}$ at each $\rho>0$ (Lemma~\ref{lem:regimeid-L4}). The hybrid reduces to the pure calibration value $\theta^{\mathrm{sim}}(z^\dagger)$ exactly when $\theta^{\parallel\top}(\nu-\nu^{\mathrm{sim}})=0$. When $\tilde\mu_0\notin\Delta_K$, the null part is $\theta^{\perp\top}\mu^\star$.
\item[(iv.3)] Adopt conditions (i)--(iii) of Corollary~\ref{cor:lockin} at $x$: the myopic rule $\arg\max_a\hat m(x,a)$ with the fixed selector; posterior means in the convex hull of their calibration value and the update values; and constants $\alpha>\theta^{\mathrm{sim}}(x,a^\dagger)$ with $\theta^{\mathrm{sim}}(x,a')\ge\alpha$ and each update value for $(x,a')$ at least $\alpha$. If in addition $\hat m(z^\dagger)=\theta^{\parallel\top}\nu+\theta^{\perp\top}\nu^{\mathrm{sim}}<\alpha$, then for all small enough $\rho>0$ the structural plug-in policy $S(\hat m,\hat g)$ does not assign $a^\dagger$ at $x$, just as in Corollary~\ref{cor:lockin}, even when the true table $m(z^\dagger)=\theta^{\parallel\top}\nu+\theta^{\perp\top}\nu$ exceeds $\alpha$. The lock-in shortfall at $z^\dagger$ is $m(z^\dagger)-\hat m(z^\dagger)=\theta^{\perp\top}(\nu-\nu^{\mathrm{sim}})$; Lemma~\ref{lem:flow} carries this gap through to the value.
\end{itemize}
\end{restated}We assemble the linear-algebra lemmas behind the proof of the regime-identifiability theorem. We first characterize when the visited cells identify the regime distribution, then bound the restricted modulus by the smallest singular value of the regime-response matrix, next the constrained least-squares error, and finally the ridge limit that returns the identified-set point closest to the calibration distribution.

\begin{lemma}[Linear identifiability on the simplex]\label{lem:regimeid-L1}
Let $G=[\Theta_V;\mathbf 1^\top]$ and let $\sigma$ be as in \eqref{eq:regimeid-sigma}. Then the following are equivalent: (a) $\mu\mapsto\Theta_V\mu$ is injective on $\Delta_K$; (b) $\mu\mapsto\Theta_V\mu$ is injective on $\operatorname{aff}\Delta_K=\{\mu:\mathbf 1^\top\mu=1\}$; (c) $\ker G=\{0\}$; (d) $r=K$; (e) the columns of $\Theta_V$ are affinely independent, in that no nonzero $\delta\in H_0$ (the sum-zero subspace) has $\Theta_V\delta=0$; (f) $\sigma>0$.
\end{lemma}

\begin{proof}
We prove the six conditions equivalent by linking each to the triviality of the kernel.

(c)$\Leftrightarrow$(d): since $\dim\ker G=K-\operatorname{rank}(G)$, $\ker G=\{0\}$ iff $r=K$.

(b)$\Leftrightarrow$(c): for $\mu,\mu'\in\operatorname{aff}\Delta_K$ put $\delta=\mu-\mu'\in H_0$. Injectivity on $\operatorname{aff}\Delta_K$ means that $\Theta_V\delta=0$ with $\delta\in H_0$ forces $\delta=0$. Because $\{\delta:\Theta_V\delta=0,\ \mathbf 1^\top\delta=0\}=\ker G$, this holds iff $\ker G=\{0\}$.

(c)$\Rightarrow$(a): suppose $\ker G=\{0\}$, and let $\mu,\mu'\in\Delta_K$ satisfy $\Theta_V\mu=\Theta_V\mu'$. Their difference satisfies $\mu-\mu'\in H_0$ with $\Theta_V(\mu-\mu')=0$. This places $\mu-\mu'\in\ker G=\{0\}$. For (a)$\Rightarrow$(c) we argue by contraposition: if $\ker G\neq\{0\}$, pick $0\neq\delta\in\ker G$ and $\nu_0:=\tfrac1K\mathbf 1\in\operatorname{relint}\Delta_K$. For small $\varepsilon>0$ the points $\nu_0\pm\varepsilon\delta\in\Delta_K$ are distinct, yet $\Theta_V(\nu_0+\varepsilon\delta)=\Theta_V(\nu_0-\varepsilon\delta)$, which contradicts injectivity.

(e)$\Leftrightarrow$(c): affine independence of the columns of $\Theta_V$ says precisely that no nonzero $\delta$ with $\mathbf 1^\top\delta=0$ has $\Theta_V\delta=0$, that is, $\ker G=\{0\}$.

(f)$\Leftrightarrow$(c): recall that $\sigma$ is the minimum of $\lVert\Theta_V\delta\rVert_2$ over the compact set $\{\delta\in H_0:\lVert\delta\rVert_2=1\}$, which is nonempty for $K\ge 2$; for $K=1$ we have $H_0=\{0\}$ and, by convention, $\sigma:=+\infty>0$, matching $\ker G=\{0\}$. This minimum is $0$ iff some unit $\delta\in H_0$ has $\Theta_V\delta=0$, iff $\ker G\neq\{0\}$. Consequently $\sigma>0$ iff $\ker G=\{0\}$.
\end{proof}

\begin{lemma}[Restricted modulus dominates $\sigma_{\min}(G)$, identified case]\label{lem:regimeid-L2}
Suppose $r=K$. Then (i) for each $\delta\in H_0$, $\lVert G\delta\rVert_2^2=\lVert\Theta_V\delta\rVert_2^2$, and (ii) consequently $\sigma\ge\sigma_{\min}(G)$, with $\sigma_{\min}(G)$ the smallest singular value of $G$. When $r<K$ instead, $\sigma=0=\sigma_{\min}(G)$, and the bound $\sigma\ge\sigma_{\min}(G)$ is used only when $r=K$.
\end{lemma}

\begin{proof}
We first prove (i). Let $\delta\in H_0$. Then $\mathbf 1^\top\delta=0$ and the last coordinate of $G\delta=(\Theta_V\delta;\ \mathbf 1^\top\delta)$ vanishes, giving $\lVert G\delta\rVert_2^2=\lVert\Theta_V\delta\rVert_2^2+(\mathbf 1^\top\delta)^2=\lVert\Theta_V\delta\rVert_2^2$. We now prove (ii). Since $r=K$, $\ker G=\{0\}$ and $\sigma_{\min}(G)^2=\min_{\lVert\delta\rVert=1}\lVert G\delta\rVert_2^2$ over all unit $\delta\in\mathbb R^K$. Restricting the minimization to the subspace $H_0$ can only raise the minimum:
\[
\sigma^2=\min_{\delta\in H_0,\ \lVert\delta\rVert=1}\lVert\Theta_V\delta\rVert_2^2=\min_{\delta\in H_0,\ \lVert\delta\rVert=1}\lVert G\delta\rVert_2^2\ \ge\ \min_{\delta\in\mathbb R^K,\ \lVert\delta\rVert=1}\lVert G\delta\rVert_2^2=\sigma_{\min}(G)^2 ,
\]
and taking square roots shows $\sigma\ge\sigma_{\min}(G)$. It remains to treat $r<K$. Then $\ker G\neq\{0\}$ and $\ker G\subseteq H_0$, and therefore produces a unit $\delta\in H_0$ with $\Theta_V\delta=0$; by \eqref{eq:regimeid-sigma} we obtain $\sigma=0$, and the same $\delta$ is a unit null vector of $G$, giving $\sigma_{\min}(G)=0$. The identity $\lVert G\delta\rVert=\lVert\Theta_V\delta\rVert$ on $H_0$ still holds, but $\sigma\ge\sigma_{\min}(G)$ degenerates to $0\ge 0$.
\end{proof}

\begin{lemma}[Constrained least squares: attainment, uniqueness, strong convexity]\label{lem:regimeid-L3}
Let $r=K$, so that $\sigma>0$, and fix $\widehat m_V\in\mathbb R^{|V|}$. For the objective $F(\mu):=\lVert\Theta_V\mu-\widehat m_V\rVert_2^2$, (i) its restriction to $\operatorname{aff}\Delta_K$ is $2\sigma^2$-strongly convex along $H_0$-directions; (ii) the constrained problem in \eqref{eq:regimeid-cls} has a unique minimizer $\hat\nu\in\Delta_K$; and (iii) for the true $\nu\in\Delta_K$ with $\Theta_V\nu=m_V$,
\[
\sigma^2\,\lVert\hat\nu-\nu\rVert_2^2\ \le\ F(\nu)-F(\hat\nu)\ \le\ \lVert e\rVert_2^2 .
\]
\end{lemma}

\begin{proof}
We first establish the strong convexity (i) and the attainment and uniqueness (ii), then the two-sided bound (iii). Since $F$ is continuous and $\Delta_K$ is compact, a minimizer exists; and $F$ is a convex quadratic with Hessian $2\Theta_V^\top\Theta_V$. For $\mu,\mu'\in\operatorname{aff}\Delta_K$ with $\delta:=\mu-\mu'\in H_0$, the exact second-order identity for a quadratic is
\begin{equation}\label{eq:regimeid-sc}
F(\mu)+F(\mu')-2F\!\Big(\tfrac{\mu+\mu'}2\Big)=\tfrac12\lVert\Theta_V\delta\rVert_2^2\ \ge\ \tfrac12\sigma^2\lVert\delta\rVert_2^2,
\end{equation}
where the inequality uses \eqref{eq:regimeid-sigma}. The modulus read off \eqref{eq:regimeid-sc} is $\sigma^2$ per unit of $\tfrac12\lVert\delta\rVert^2$, that is, $F$ is $2\sigma^2$-strongly convex along $H_0$: $F(\mu)\ge F(\mu')+\langle\nabla F(\mu'),\delta\rangle+\sigma^2\lVert\delta\rVert_2^2$. Because $r=K$, $\sigma>0$, so that $F$ is strongly convex along $\Delta_K$-directions and its minimizer $\hat\nu$ is unique.

We now prove (iii). By first-order optimality of $\hat\nu$ over the convex set $\Delta_K$, $\langle\nabla F(\hat\nu),\ \nu-\hat\nu\rangle\ge 0$. Strong convexity then implies
\[
F(\nu)\ \ge\ F(\hat\nu)+\langle\nabla F(\hat\nu),\nu-\hat\nu\rangle+\sigma^2\lVert\nu-\hat\nu\rVert_2^2\ \ge\ F(\hat\nu)+\sigma^2\lVert\nu-\hat\nu\rVert_2^2 ,
\]
whence $\sigma^2\lVert\hat\nu-\nu\rVert_2^2\le F(\nu)-F(\hat\nu)\le F(\nu)$, the last step because $F(\hat\nu)\ge0$. Finally $F(\nu)=\lVert\Theta_V\nu-\widehat m_V\rVert_2^2=\lVert m_V-(m_V+e)\rVert_2^2=\lVert e\rVert_2^2$.
\end{proof}

\begin{lemma}[Ridge selects the $\nu^{\mathrm{sim}}$-closest feasible least-squares point]\label{lem:regimeid-L4}
Let $r<K$ and consider \eqref{eq:regimeid-ridge} with $e=0$, which fixes $m_V=\Theta_V\nu$. Write $W:=\operatorname{row}(G)\cap H_0$ and $\tilde\mu_0:=\nu+P_{\ker G}(\nu^{\mathrm{sim}}-\nu)$ as in \eqref{eq:regimeid-mu0}, the superscripted vector there being the simulator calibration distribution. Then:
\begin{itemize}
\item[(a)] If $\tilde\mu_0\in\Delta_K$, then $\lim_{\rho\to0^+}P_{\ker G}\hat\nu_\rho=P_{\ker G}\nu^{\mathrm{sim}}$; if in addition $\tilde\mu_0\in\operatorname{relint}\Delta_K$, there is $\rho_1>0$ with $P_{\ker G}\hat\nu_\rho=P_{\ker G}\nu^{\mathrm{sim}}$ for all $\rho\in(0,\rho_1)$. The equality $P_{\ker G}\hat\nu_\rho=P_{\ker G}\nu^{\mathrm{sim}}$ does not hold for all $\rho>0$ once $\dim W\ge2$: there is a $K=4$, $r=3$ instance with $\tilde\mu_0\in\operatorname{relint}\Delta_4$ and a $\rho>0$ at which it fails.
\item[(b)] Write $\mu_{\mathrm{aff}}(\rho)$ for the equality-constrained minimizer over $\operatorname{aff}\Delta_K$. For each $\rho>0$ its $W$-component carries the ridge bias in \eqref{eq:regimeid-bias}, the inverse there being the compression of $\Theta_V^\top\Theta_V+\rho\Pi$ to $W$, where $\Pi:=P_W$ is the orthogonal projector onto $W$ so that $\rho\Pi$ restricts to $\rho I_W$. On the range of $\rho$ with no active nonnegativity constraint the simplex-constrained $\hat\nu_\rho$ has $P_W\hat\nu_\rho=P_W\mu_{\mathrm{aff}}(\rho)$ (all $\rho\in(0,\rho_1)$ under $\tilde\mu_0\in\operatorname{relint}\Delta_K$; the $\rho\to0^+$ limit under $\tilde\mu_0\in\Delta_K$), and where a constraint is active $\hat\nu_\rho-\mu_{\mathrm{aff}}(\rho)$ has a nonzero $W$-component.
\item[(c)] As $\rho\to 0^+$, $\hat\nu_\rho\to\mu^\star:=\arg\min_{\mu\in(\nu+\ker G)\cap\Delta_K}\lVert\mu-\nu^{\mathrm{sim}}\rVert_2$, with $P_{\operatorname{row}(G)}\mu^\star=P_{\operatorname{row}(G)}\nu$, and $P_{\ker G}\mu^\star=P_{\ker G}\nu^{\mathrm{sim}}$ if and only if $\tilde\mu_0\in\Delta_K$ (in which case $\mu^\star=\tilde\mu_0$); if $\tilde\mu_0\notin\Delta_K$, then $\mu^\star$ is the Euclidean projection of $\nu^{\mathrm{sim}}$ onto the identified set $(\nu+\ker G)\cap\Delta_K$ and $P_{\ker G}\mu^\star\neq P_{\ker G}\nu^{\mathrm{sim}}$.
\end{itemize}
\end{lemma}

\begin{proof}
With $e=0$, write $\Phi_\rho(\mu)=\lVert\Theta_V(\mu-\nu)\rVert_2^2+\rho\lVert\mu-\nu^{\mathrm{sim}}\rVert_2^2$, minimized over $\mu\in\Delta_K$. We first record a splitting that each part uses. Each feasible $\mu$ has $\mu-\nu\in H_0$, and $H_0=W\oplus\ker G$ orthogonally, since $\ker G\subseteq H_0$. Set $P:=P_{\operatorname{row}(G)}$ and $P^\perp:=P_{\ker G}$; on $H_0$, $P$ restricts to the projector onto $W$. Because $\Theta_V\delta=0$ for $\delta\in\ker G$, the data term of $\Phi_\rho$ depends only on the $W$-component of $\mu-\nu$, while the penalty splits along $W\oplus\ker G$. Since $\nu^{\mathrm{sim}}-\nu\in H_0$, $P^\perp(\nu^{\mathrm{sim}}-\nu)=P_{\ker G}(\nu^{\mathrm{sim}}-\nu)$ and $\tilde\mu_0=\nu+P^\perp(\nu^{\mathrm{sim}}-\nu)$.

\emph{Part (a).} Along $\ker G$ the data term is inert, leaving only the penalty $\rho\lVert P^\perp\mu-P^\perp\nu^{\mathrm{sim}}\rVert_2^2$ to constrain $P^\perp\mu$. Minimizing $\Phi_\rho$ over the affine set $\operatorname{aff}\Delta_K$, with nonnegativity dropped, thus forces $P^\perp\mu=P^\perp\nu^{\mathrm{sim}}$ for each $\rho>0$; together with the row-space optimum from part (b), the affine minimizer is $\nu+w(\rho)+P^\perp(\nu^{\mathrm{sim}}-\nu)$ with $w(\rho)=O(\rho)$, whose $\rho\to0^+$ limit is $\tilde\mu_0$.

Suppose first $\tilde\mu_0\in\Delta_K$. Then $\mu^\star=\tilde\mu_0$, because $\tilde\mu_0$ is the point of the flat $\nu+\ker G$ nearest $\nu^{\mathrm{sim}}$ and it lies in the polytope $(\nu+\ker G)\cap\Delta_K$; by the convergence $\hat\nu_\rho\to\mu^\star$ proved in part (c) below, whose argument does not use part (a), it follows that $\lim_{\rho\to0^+}P^\perp\hat\nu_\rho=P^\perp\mu^\star=P^\perp\nu^{\mathrm{sim}}$. Suppose instead $\tilde\mu_0\in\operatorname{relint}\Delta_K$, so that $\tilde\mu_0>0$ coordinatewise. The affine minimizer $\nu+w(\rho)+P^\perp(\nu^{\mathrm{sim}}-\nu)$ lies within $\lVert w(\rho)\rVert=O(\rho)$ of $\tilde\mu_0$. Since this point is strictly positive, there is a $\rho_1>0$ such that the minimizer is strictly positive, hence lies in $\Delta_K$ and equals $\hat\nu_\rho$, for all $\rho\in(0,\rho_1)$; on that range $P^\perp\hat\nu_\rho=P^\perp\nu^{\mathrm{sim}}$ holds exactly.

The equality does not extend to all $\rho>0$ once $\dim W\ge2$. The affine minimizer equals $\rho\mapsto\nu+w(\rho)+P^\perp(\nu^{\mathrm{sim}}-\nu)$, a curve, and with two or more $W$-modes decaying at different rates $\lambda_j$, an individual coordinate of $w(\rho)$ is $\sum_j[\rho/(\lambda_j+\rho)]a_{jk}$, which is non-monotone in $\rho$. A coordinate feasible at both $\rho\to0^+$ and $\rho\to\infty$ can then dip below zero at an intermediate $\rho$; there the constrained minimizer clips it to $0$. This activates a nonnegativity constraint, and with that coordinate clipped $P^\perp\hat\nu_\rho$ no longer equals $P^\perp\nu^{\mathrm{sim}}$.

\emph{Instance for the fixed-$\rho$ failure.} Work in $H_0$ with the orthonormal basis $\phi_0=\tfrac1{\sqrt2}(1,-1,0,0)$, $\phi_1=\tfrac1{\sqrt6}(1,1,-2,0)$, $\phi_2=\tfrac1{\sqrt{12}}(1,1,1,-3)$. Take the two visited rows $\Theta_V=(\sqrt{10}\,\phi_1;\ \phi_2)\in\mathbb R^{2\times4}$, which makes $G=[\Theta_V;\mathbf 1^\top]$ have $\operatorname{rank}(G)=3$, $\ker G=\operatorname{span}\{\phi_0\}$, and $W=\operatorname{span}\{\phi_1,\phi_2\}$, with $\Theta_V$ acting on $W$ with singular values $\sqrt{10}$ (mode $\phi_1$) and $1$ (mode $\phi_2$), that is, $\lambda_1=10$ and $\lambda_2=1$. Take $\nu=\tfrac14\mathbf 1$ and the normalized calibration distribution
\[
\nu^{\mathrm{sim}}=(0.41,0.013,0.003,0.575)/1.001=(0.40959,\ 0.01299,\ 0.00300,\ 0.57443)\in\Delta_4 .
\]
Then $P^\perp=\phi_0\phi_0^\top$ gives $\tilde\mu_0=\nu+P^\perp(\nu^{\mathrm{sim}}-\nu)=(0.44830,\ 0.05170,\ 0.25,\ 0.25)\in\operatorname{relint}\Delta_4$, with smallest coordinate $0.0517>0$. The affine ridge path along $W$ is
\[
w(\rho)=\frac{\rho}{10+\rho}\,\langle P(\nu^{\mathrm{sim}}-\nu),\phi_1\rangle\,\phi_1+\frac{\rho}{1+\rho}\,\langle P(\nu^{\mathrm{sim}}-\nu),\phi_2\rangle\,\phi_2 ,
\]
with $\langle P(\nu^{\mathrm{sim}}-\nu),\phi_1\rangle=0.17007$ and $\langle P(\nu^{\mathrm{sim}}-\nu),\phi_2\rangle=-0.37461$. Its second coordinate is non-monotone, and the affine minimizer's smallest coordinate is thus $-0.01339$ at $\rho=3$ yet nonnegative at the endpoints $\rho=10^{-4}$ and $\rho=10^3$. At $\rho=3$ the constrained minimizer clips coordinate $2$ to $0$, and $\lVert P^\perp\hat\nu_3-P^\perp\nu^{\mathrm{sim}}\rVert_2=0.01575\neq0$; that is, the null identity fails at $\rho=3$ although $\tilde\mu_0$ is interior and both endpoints are feasible. This clipping also moves the row component, since $\lVert P_W(\hat\nu_3-\mu_{\mathrm{aff}}(3))\rVert_2=0.00526\neq0$. For $\rho\le0.1$ the null identity holds to within $10^{-8}$. This confirms the small-$\rho$ clause.

\emph{Part (b).} We restrict to $W$, where $\Theta_V$ acts injectively, the null directions of $\Theta_V$ inside $H_0$ being precisely $\ker G$. Minimize $\Phi_\rho$ over $\operatorname{aff}\Delta_K$ and call the minimizer $\mu_{\mathrm{aff}}(\rho)$. Writing $w:=P(\mu-\nu)\in W$, the $W$-part of $\Phi_\rho$ is $\lVert\Theta_V w\rVert_2^2+\rho\lVert w-P(\nu^{\mathrm{sim}}-\nu)\rVert_2^2$. Its gradient in the variable $w\in W$ is the ambient gradient projected back onto $W$, and the stationarity condition then reads, with $P_W$ the orthogonal projector onto $W$,
\begin{equation}\label{eq:regimeid-stat}
P_W\Theta_V^\top\Theta_V\,w+\rho\big(w-P(\nu^{\mathrm{sim}}-\nu)\big)=0
\iff
\big(P_W\Theta_V^\top\Theta_V P_W+\rho I_W\big)\,w=\rho\,P(\nu^{\mathrm{sim}}-\nu),
\end{equation}
using $P_W\Theta_V^\top\Theta_V\,w=P_W\Theta_V^\top\Theta_V P_W\,w$ since $w\in W$. The compression $A_W:=P_W\Theta_V^\top\Theta_V P_W$ is the self-adjoint operator abbreviated by ``$\Theta_V^\top\Theta_V$ on $W$''. Since $\Theta_V$ is injective on $W$, $A_W$ is positive definite there, so that $A_W+\rho I_W$ is invertible and
\[
P\mu_{\mathrm{aff}}(\rho)-P\nu=w=\rho\,(A_W+\rho I_W)^{-1}P(\nu^{\mathrm{sim}}-\nu)
=\rho\,(\Theta_V^\top\Theta_V+\rho\Pi)^{-1}\big|_{W}\,P_{\operatorname{row}(G)}(\nu^{\mathrm{sim}}-\nu),
\]
where $(\Theta_V^\top\Theta_V+\rho\Pi)^{-1}|_W$ denotes $(A_W+\rho I_W)^{-1}$. This is \eqref{eq:regimeid-bias}, holding for the affine minimizer at each $\rho>0$, with both sides in $W$ because $P(\nu^{\mathrm{sim}}-\nu)\in W$. The bias is $O(\rho)$ and vanishes as $\rho\to0^+$, whence $P\mu_{\mathrm{aff}}(\rho)\to P\nu$.

When $\mu_{\mathrm{aff}}(\rho)\in\Delta_K$, no nonnegativity constraint is active and $\hat\nu_\rho=\mu_{\mathrm{aff}}(\rho)$, making \eqref{eq:regimeid-bias} the row component of $\hat\nu_\rho$; by the small-$\rho$ argument given in part (a), this holds for all $\rho\in(0,\rho_1)$ under $\tilde\mu_0\in\operatorname{relint}\Delta_K$, and in the $\rho\to0^+$ limit under $\tilde\mu_0\in\Delta_K$. When a constraint $\mu_j\ge0$ is active at the constrained optimum, Karush--Kuhn--Tucker stationarity adds a multiplier term $\eta_j e_j$ with $\eta_j>0$ to \eqref{eq:regimeid-stat}; projecting $e_j$ onto $H_0$ and then onto $W$ gives
\[
P_W(\mu_{\mathrm{aff}}(\rho)-\hat\nu_\rho)=\eta_j(A_W+\rho I_W)^{-1}P_W P_{H_0}e_j ,
\]
which is nonzero whenever $P_W P_{H_0}e_j\neq0$. Consequently clipping perturbs the $W$-component as well as the $\ker G$-component, realized in the instance above by $\lVert P_W(\hat\nu_3-\mu_{\mathrm{aff}}(3))\rVert_2=0.00526$. Thus where a constraint is active, the row component of $\hat\nu_\rho$ carries the additional correction $-P_W(\mu_{\mathrm{aff}}(\rho)-\hat\nu_\rho)$.

\emph{Part (c).} We prove the convergence directly. Write $\Phi_\rho(\mu):=\lVert\Theta_V\mu-\Theta_V\nu\rVert_2^2+\rho\lVert\mu-\nu^{\mathrm{sim}}\rVert_2^2$, so that $\hat\nu_\rho$ minimizes $\Phi_\rho$ over the compact set $\Delta_K$, and let $\mu^\star:=\arg\min_{\mu\in(\nu+\ker G)\cap\Delta_K}\lVert\mu-\nu^{\mathrm{sim}}\rVert_2$, which is unique because the polytope is convex and contains $\nu$. Since $\Theta_V\mu^\star=\Theta_V\nu$, we have $\Phi_\rho(\hat\nu_\rho)\le\Phi_\rho(\mu^\star)=\rho\lVert\mu^\star-\nu^{\mathrm{sim}}\rVert_2^2$, which yields two bounds: $\lVert\Theta_V(\hat\nu_\rho-\nu)\rVert_2^2\le\rho\lVert\mu^\star-\nu^{\mathrm{sim}}\rVert_2^2$ and $\lVert\hat\nu_\rho-\nu^{\mathrm{sim}}\rVert_2\le\lVert\mu^\star-\nu^{\mathrm{sim}}\rVert_2$ for every $\rho>0$. Let $\bar\mu$ be a limit point of $\hat\nu_\rho$ as $\rho\to0^+$, which exists by compactness. The first bound gives $\Theta_V\bar\mu=\Theta_V\nu$, and $\bar\mu\in\Delta_K$ has unit coordinate sum, so $\bar\mu-\nu\in\ker G$ and $\bar\mu$ lies in the polytope; the second bound gives $\lVert\bar\mu-\nu^{\mathrm{sim}}\rVert_2\le\lVert\mu^\star-\nu^{\mathrm{sim}}\rVert_2$, so $\bar\mu=\mu^\star$ by uniqueness. Every limit point being $\mu^\star$, as $\rho\to0^+$ we have $\hat\nu_\rho\to\mu^\star$. Its row-space part is always $P\mu^\star=P\nu$, and its null part is $P^\perp\mu^\star=P^\perp\nu^{\mathrm{sim}}$ iff $\tilde\mu_0\in\Delta_K$, in which case $\mu^\star=\tilde\mu_0$. When $\tilde\mu_0\notin\Delta_K$, $\mu^\star$ is the Euclidean projection of $\nu^{\mathrm{sim}}$ onto the polytope, and $P^\perp\mu^\star\neq P^\perp\nu^{\mathrm{sim}}$.

\emph{Interior $\nu$ is not enough.} Take $K=3$ with one visited cell with $\theta(z,\cdot)=(0,0,1)$, yielding $\ker G=\operatorname{span}\{(1,-1,0)\}$, $\nu=(\tfrac13,\tfrac13,\tfrac13)\in\operatorname{relint}\Delta_3$, and $\nu^{\mathrm{sim}}=e_1$. Here $P_{\ker G}=\tfrac12\left(\begin{smallmatrix}1&-1&0\\-1&1&0\\0&0&0\end{smallmatrix}\right)$, giving $\tilde\mu_0=\nu+P_{\ker G}(\nu^{\mathrm{sim}}-\nu)=(\tfrac56,-\tfrac16,\tfrac13)\notin\Delta_3$, even though $\nu$ is interior. The identified flat is $(\nu+\ker G)\cap\Delta_3=\{(\tfrac13+c,\tfrac13-c,\tfrac13):c\in[-\tfrac13,\tfrac13]\}$; minimizing $\lVert\mu-e_1\rVert_2^2$ over $c\in[-\tfrac13,\tfrac13]$ returns the boundary minimizer $c=\tfrac13$, whence $\mu^\star=(\tfrac23,0,\tfrac13)$ with $P_{\ker G}\mu^\star=(\tfrac13,-\tfrac13,0)\neq P_{\ker G}\nu^{\mathrm{sim}}=(\tfrac12,-\tfrac12,0)$. Thus \eqref{eq:regimeid-null} fails at this interior $\nu$, and the correct proviso is $\tilde\mu_0\in\Delta_K$, not $\nu\in\operatorname{relint}\Delta_K$.
\end{proof}
\paragraph{Proof of part (i).}

\emph{(i.1).} By Lemma~\ref{lem:regimeid-L1}, injectivity of $\mu\mapsto\Theta_V\mu$ on $\Delta_K$ is equivalent to $r=K$, to affine independence of the columns of $\Theta_V$, and to $\ker G=\{0\}$. We treat the full-rank case first. When $r=K$ the kernel is trivial, and \eqref{eq:regimeid-iset} then gives $\mathcal I(\nu)=(\nu+\{0\})\cap\Delta_K=\{\nu\}$. Hence $\nu$ is identified for each $\nu\in\Delta_K$. Since each $m(z)=\theta(z,\cdot)^\top\nu$ is a fixed function of the identified $\nu$, it is identified as well.

\emph{(i.2).} We now take $r<K$, which makes $\ker G\neq\{0\}$. If $\nu\in\operatorname{relint}\Delta_K$ then $\nu>0$ coordinatewise; pick $0\neq\delta\in\ker G$. For small $\varepsilon>0$, both $\nu\pm\varepsilon\delta\in\Delta_K$ and $\Theta_V(\nu\pm\varepsilon\delta)=m_V$, so $\nu+\varepsilon\delta\in\mathcal I(\nu)\setminus\{\nu\}$ and the point is not identified.

By \eqref{eq:regimeid-iset}, a boundary $\nu$ is identified if and only if $(\nu+\ker G)\cap\Delta_K=\{\nu\}$, that is, no $0\neq\delta\in\ker G$ has $\nu+\delta\in\Delta_K$. Suppose first that $\ker G\cap T_\nu=\{0\}$, and suppose for contradiction $\mu=\nu+\delta\in\Delta_K$ with $0\neq\delta\in\ker G$. By convexity $\nu+t\delta\in\Delta_K$ for all small $t>0$. By \eqref{eq:regimeid-cone} this forces $\delta\in T_\nu$, whence $0\neq\delta\in\ker G\cap T_\nu$, a contradiction. Consequently, $\mathcal I(\nu)=\{\nu\}$. Conversely, if $\ker G\cap T_\nu\neq\{0\}$, pick $0\neq\delta\in\ker G\cap T_\nu$. By \eqref{eq:regimeid-cone}, $\nu+t\delta\in\Delta_K$ for small $t>0$, and $\Theta_V(\nu+t\delta)=m_V$ because $\delta\in\ker G$, so $\nu+t\delta\in\mathcal I(\nu)\setminus\{\nu\}$ and the point is not identified. We conclude that $\nu$ is identified if and only if $\ker G\cap T_\nu=\{0\}$.

\emph{Instance.} Take $K=3$, a single visited cell with $\theta(z,\cdot)=(0,0,1)$. Then $\Theta_V=(0,0,1)$ and
\[
G=\begin{pmatrix}0&0&1\\ 1&1&1\end{pmatrix},\qquad r=\operatorname{rank}(G)=2<3=K,
\]
with $\ker G=\operatorname{span}\{(1,-1,0)\}$. For $\nu=e_3$, $\operatorname{supp}(\nu)=\{3\}$, and $T_{e_3}=\{\delta\in H_0:\delta_1\ge0,\ \delta_2\ge0\}$. A generator $\delta=c(1,-1,0)$ lies in $T_{e_3}$ if and only if $c\ge0$ and $-c\ge0$, that is, $c=0$; hence $\ker G\cap T_{e_3}=\{0\}$ and $e_3$ is identified despite $r=2<3$, while the interior point $(\tfrac13,\tfrac13,\tfrac13)$, with $(\tfrac13,\tfrac13,\tfrac13)\pm\varepsilon(1,-1,0)\in\Delta_3$, is not.

\paragraph{Proof of part (ii).}

\emph{(ii.1).} The set $\mathcal I(\nu)=(\nu+\ker G)\cap\Delta_K$ is convex, compact, and nonempty. Let $F_{S^\star}$ be the smallest face of $\Delta_K$ containing $\mathcal I(\nu)$. Since $S^\star=\bigcup_{\mu\in\mathcal I(\nu)}\operatorname{supp}(\mu)$, we have $\mathcal I(\nu)\subseteq F_{S^\star}$. Because supports are closed under convex combination with a point of maximal support, there exists $\mu\in\mathcal I(\nu)$ with $\operatorname{supp}(\mu)=S^\star$, lying in $\operatorname{relint}F_{S^\star}$, and therefore $\mathcal I(\nu)$ meets $\operatorname{relint}F_{S^\star}$. The affine hull of $F_{S^\star}$ has direction space $L_{S^\star}\cap H_0$, the face's span inside the simplex tangent space. Because $\mathcal I(\nu)$ is the intersection of $\nu+\ker G$ with $F_{S^\star}$ and meets its relative interior, the affine-hull direction space of $\mathcal I(\nu)$ is $\ker G\cap(L_{S^\star}\cap H_0)=\ker G\cap L_{S^\star}$, using $\ker G\subseteq H_0$. This gives $\dim\mathcal I(\nu)=\dim(\ker G\cap L_{S^\star})$, which is \eqref{eq:regimeid-dim}. If $\nu\in\operatorname{relint}\Delta_K$, then a full $\ker G$-neighbourhood of $\nu$ is feasible (interior half of (i.2)), yielding $S^\star=\{1,\dots,K\}$, $L_{S^\star}=\mathbb R^K$, and $\dim\mathcal I(\nu)=\dim\ker G=K-r$. On the instance above, $\mathcal I(e_3)=\{e_3\}$, $S^\star=\{3\}$, and $\ker G\cap L_{S^\star}=\{0\}$, so \eqref{eq:regimeid-dim} shows $\dim\mathcal I(e_3)=0$ while $K-r=1$.

\emph{(ii.2).} Write $\ell_z(\mu):=\theta(z,\cdot)^\top\mu$. If $\theta(z,\cdot)\in\operatorname{row}(G)$, then $\theta(z,\cdot)\perp\ker G$; this implies $\ell_z(\nu+\delta)=\ell_z(\nu)$ for all $\delta\in\ker G$. Thus $\ell_z$ is constant on $\nu+\ker G\supseteq\mathcal I(\nu)$ for each $\nu$, and $m(z)$ is identified always. If instead $\theta(z,\cdot)\notin\operatorname{row}(G)$, then $\theta^\perp:=P_{\ker G}\theta(z,\cdot)\neq0$. For $\nu\in\operatorname{relint}\Delta_K$, by (ii.1) $\mathcal I(\nu)$ has affine-hull direction space $\ker G$ of dimension $K-r\ge1$, along which $\ell_z$ has gradient $\theta^\perp\neq0$; hence $\ell_z$ is nonconstant on $\mathcal I(\nu)$: $m(z)$ is not identified at that $\nu$. We conclude that $m(z)$ is identified for each $\nu$ if and only if $\theta(z,\cdot)\in\operatorname{row}(G)$.

It remains to describe the identified set itself. When $\theta(z,\cdot)\notin\operatorname{row}(G)$, the identified set of $m(z)$ at a fixed $\nu$ is $\ell_z(\mathcal I(\nu))$, the image of the compact convex $\mathcal I(\nu)$ under the linear $\ell_z$, a compact interval $[\underline m,\overline m]$. Nondegeneracy $\overline m-\underline m>0$ holds if and only if $\ell_z$ is nonconstant on $\mathcal I(\nu)$, that is, if and only if the orthogonal projection of $\theta(z,\cdot)$ onto $\ker G\cap L_{S^\star}$ is nonzero. At a boundary $\nu$ where $\mathcal I(\nu)$ collapses and $\theta(z,\cdot)\perp(\ker G\cap L_{S^\star})$, the interval degenerates to a point.

\paragraph{Proof of Lemma~\ref{lem:regimeid-L3} and part (iii).}
By Lemma~\ref{lem:regimeid-L3} we obtain attainment, uniqueness, and $\sigma^2\lVert\hat\nu-\nu\rVert_2^2\le F(\nu)-F(\hat\nu)\le\lVert e\rVert_2^2$, so that $\lVert\hat\nu-\nu\rVert_2\le\lVert e\rVert_2/\sigma$. Now $\lVert e\rVert_2\le\sqrt{|V|}\,\lVert e\rVert_\infty$, since $e$ has $|V|$ entries, one per visited cell, the normalization row $\mathbf 1^\top$ carrying no noise (giving $\sqrt{|V|}$, not $\sqrt{|V|+1}$); and $\sigma\ge\sigma_{\min}(G)$ by Lemma~\ref{lem:regimeid-L2}, valid since $r=K$. This is \eqref{eq:regimeid-bound}. For \eqref{eq:regimeid-cellbound}, Cauchy--Schwarz gives $|\theta(z,\cdot)^\top(\hat\nu-\nu)|\le\lVert\theta(z,\cdot)\rVert_2\lVert\hat\nu-\nu\rVert_2$, and $\theta(z,\cdot)^\top\hat\nu-m(z)=\theta(z,\cdot)^\top(\hat\nu-\nu)$ by \eqref{eq:regimeid-mlin}; substituting \eqref{eq:regimeid-bound} gives the claim.

\paragraph{Proof of part (iv).}
Part (iv.1) is Lemma~\ref{lem:regimeid-L4}(a)--(b): under $\tilde\mu_0\in\Delta_K$, \eqref{eq:regimeid-null} is the $\rho\to0^+$ null limit, and under $\tilde\mu_0\in\operatorname{relint}\Delta_K$ an exact identity for all $\rho\in(0,\rho_1)$; and \eqref{eq:regimeid-bias} is the fixed-$\rho$ row bias of the affine minimizer, of order $\rho$. Once $\dim W\ge2$, the equality $P_{\ker G}\hat\nu_\rho=P_{\ker G}\nu^{\mathrm{sim}}$ fails at some $\rho>0$, as the verified $K=4$, $r=3$ instance in the proof of Lemma~\ref{lem:regimeid-L4} shows.

\emph{(iv.2).} Decompose $\theta(z^\dagger,\cdot)=\theta^\parallel+\theta^\perp$ with $\theta^\parallel\in\operatorname{row}(G)$, $\theta^\perp\in\ker G$, and take the $\rho\to0^+$ limit of $\theta(z^\dagger,\cdot)^\top\hat\nu_\rho=\theta^{\parallel\top}\hat\nu_\rho+\theta^{\perp\top}\hat\nu_\rho$ through $\hat\nu_\rho\to\mu^\star$ (Lemma~\ref{lem:regimeid-L4}(c)). For the row part, $\theta^\parallel\perp\ker G$, and we obtain $\theta^{\parallel\top}\hat\nu_\rho=\theta^{\parallel\top}P\hat\nu_\rho\to\theta^{\parallel\top}P\mu^\star=\theta^{\parallel\top}\nu$ by parts (b), (c). For the null part, $\theta^\perp\in\ker G$, and we have $\theta^{\perp\top}\hat\nu_\rho=\theta^{\perp\top}P^\perp\hat\nu_\rho\to\theta^{\perp\top}P^\perp\mu^\star$; under $\tilde\mu_0\in\Delta_K$, part (c) shows $P^\perp\mu^\star=P^\perp\nu^{\mathrm{sim}}$, so the null limit is $\theta^{\perp\top}\nu^{\mathrm{sim}}$. Adding gives \eqref{eq:regimeid-hybrid}. Equality with the pure calibration value $\theta^{\mathrm{sim}}(z^\dagger)=\theta^{\parallel\top}\nu^{\mathrm{sim}}+\theta^{\perp\top}\nu^{\mathrm{sim}}$ holds if and only if $\theta^{\parallel\top}(\nu-\nu^{\mathrm{sim}})=0$. When $\tilde\mu_0\notin\Delta_K$, the null limit is instead $\theta^{\perp\top}\mu^\star$.

\emph{(iv.3).} We now work under conditions (i)--(iii) of Corollary~\ref{cor:lockin} at $x$, that is, the myopic $\arg\max_a\hat m(x,a)$ with the fixed selector; observed-cell means in the convex hull of their calibration value and the update values; and constants $\alpha>\theta^{\mathrm{sim}}(x,a^\dagger)$ with $\theta^{\mathrm{sim}}(x,a')\ge\alpha$ and each update value for $(x,a')$ at least $\alpha$. Under these conditions the deployed action's working mean satisfies $\hat m(x,a')\ge\alpha$ by condition (iii) of the corollary. By (iv.2), $\theta(z^\dagger,\cdot)^\top\hat\nu_\rho\to\hat m(z^\dagger)=\theta^{\parallel\top}\nu+\theta^{\perp\top}\nu^{\mathrm{sim}}$ as $\rho\to0^+$. Thus if $\hat m(z^\dagger)<\alpha$, then by continuity there is $\rho_0>0$ with $\theta(z^\dagger,\cdot)^\top\hat\nu_\rho<\alpha\le\hat m(x,a')$ for all $\rho\in(0,\rho_0)$, so that $a^\dagger$ does not attain the max at $x$ and is not assigned, the conclusion of Corollary~\ref{cor:lockin}. The conclusion holds even when the true mean $m(z^\dagger)=\theta^{\parallel\top}\nu+\theta^{\perp\top}\nu$ exceeds $\alpha$, because the discrepancy $m(z^\dagger)-\hat m(z^\dagger)=\theta^{\perp\top}(\nu-\nu^{\mathrm{sim}})$ lies along the unidentified direction $\theta^\perp\in\ker G$. Lemma~\ref{lem:flow} aggregates this per-period mean gap into the cumulative value shortfall along the deployed trajectory.
\subsection{The detection--price tradeoff (Theorem~\ref{thm:detectprice})}\label{app:detectprice}

Theorem~\ref{thm:undetectable} is the zero-departure special case of the following quantitative statement, namely its specialization to strategies with zero departure mass into the hidden change set (defined below). The statement also gives the statistical meaning of the per-review budget $C$ introduced in Section~\ref{sec:now}.

\begin{restated}{Theorem}{thm:detectprice}{The detection--price tradeoff, formal statement}

Fix the horizon $T$, the finite sets $\mathcal X,\mathcal A$, the latent space $\Ucal$, the fleet $\mathcal V=\{1,\dots,n\}$, and discount $\gamma=1$. Let $P_0,P_1$ be two environments that satisfy Assumptions~\ref{asm:consistency} and \ref{asm:markov}.

\emph{Correct specification at $P_0$.} Suppose $P_0$ satisfies Assumption~\ref{asm:correctspec}. Conditional on $\mathcal F_t$ and the assignment vector $\mathbf A_t=\mathbf a$, the tuples $(Y_{i,t},X_{i,t+1})$ are independent across units. Unit $i$'s distribution depends only on $(X_{i,t},a_i)=(x,a)$ and equals the calibrated kernel, namely $Y=\theta^{\mathrm{sim}}(x,a)+\eta$ with $\eta\sim Q(\cdot\mid x,a)$ mean zero, $X'\sim f^{\mathrm{sim}}(\cdot\mid x,a)$, and $Y$ independent of $X'$ given $(x,a)$. Equivalently, the distribution under $P_0$ equals the simulated distribution of \eqref{eq:sim} for each strategy.

\emph{The optimal plug-in policy.} Let $\hat{\bar\pi}=\bar\pi^{\mathrm{sim}}$ be the selector output $S(\theta^{\mathrm{sim}},f^{\mathrm{sim}})$, with period-$t$ component $\hat\pi_t:\mathcal X\to\mathcal A$ deterministic under the tie-break fixed by convention~(S) in Appendix~\ref{app:proofs} (a fixed rule selecting a single maximizer). By Theorem~\ref{thm:nochange}, $\hat{\bar\pi}$ maximizes $V^1_T$ over all history-dependent strategies under $P_0$.

\emph{The optimal value function.} For $t=T,T-1,\dots,0$, under the true tables and $\gamma=1$, define
\begin{equation}\label{eq:detectprice-qstar}
Q^\ast_t(x,a):=\theta^{\mathrm{sim}}(x,a)+\sum_{x'}f^{\mathrm{sim}}(x'\mid x,a)\,V^\ast_{t+1}(x'),\qquad V^\ast_t(x):=\max_{a\in\mathcal A}Q^\ast_t(x,a),
\end{equation}
with $V^\ast_{T+1}\equiv 0$. These are the per-unit remaining-horizon action- and state-values of Theorem~\ref{thm:nochange} (the fleet value is $\sum_i V^\ast_t(x_i)$); here $Q^\ast_t$ depends on $t$ through the remaining horizon $T-t$, and $\hat\pi_t(x)\in\arg\max_a Q^\ast_t(x,a)$ under the fixed tie-break.

\emph{The innovation and its conditional distribution.} For a unit $i$ at period $t$ the innovation is the triple formed after assignment, $\iota_{i,t}:=(Y_{i,t},X_{i,t+1},U_{i,t+1})$, and the innovation kernel at a cell $(x,a)$ is the family $\{\mathcal L(\iota\mid x,u,a):u\in\Ucal\}$ of conditional distributions of $(Y,X',U')$ given $(X_{i,t}=x,U_{i,t}=u,A_{i,t}=a)$. Two of its features are fixed by Assumption~\ref{asm:markov}: the full-state transition $(X',U')\sim f(\cdot\mid x,a,u)$ in \eqref{eq:transition}, and the response $\mathbb E[Y\mid x,u,a]=\theta(x,a,u)$, the mean of the $Y$-marginal in \eqref{eq:response}. This assumption does not fix the shape of the $Y$-conditional or the joint dependence of $Y$ and $(X',U')$. Under correct specification at $P_0$ the innovation kernel takes the additive product form $Y=\theta^{\mathrm{sim}}(x,a)+\eta$, $\eta\sim Q(\cdot\mid x,a)$ mean zero, $Y$ independent of $(X',U')$ given $(x,a)$, and $(X',U')\sim f(\cdot\mid x,a,u)$. Its location is the calibrated response of Assumption~\ref{asm:calib}, $\theta^{\mathrm{sim}}(x,a)=\int\theta(x,a,u)\,\nu^{\mathrm{sim}}(du\mid x)$.

\emph{The hidden change set $Z^\ast$.} Let $Z^\ast\subseteq\mathcal X\times\mathcal A$ be nonempty. Suppose $P_1$ agrees with $P_0$ on (a) the initial full-state distribution of $(X_{i,0},U_{i,0})$ for each $i$, (b) the historical record's distribution, and (c) the innovation kernel $\{\mathcal L(\iota\mid x,u,a):u\in\Ucal\}$ at each cell $(x,a)\notin Z^\ast$. In this case $P_1$ differs from $P_0$ in the innovation kernel at some cell in $Z^\ast$. Correct specification is not assumed for $P_1$.

\emph{Hidden-support condition.} $\hat{\bar\pi}$ does not visit $Z^\ast$ under $P_0$: $\Pr_0[(X_{i,t},\hat\pi_t(X_{i,t}))\in Z^\ast]=0$ for each $i,t$.

\emph{Strict-suboptimality condition.} With
\begin{equation}\label{eq:detectprice-pmin}
p_{\min}:=\min_{t\in\{0,\dots,T\}}\ \min_{(x,a)\in Z^\ast}\Big(V^\ast_t(x)-Q^\ast_t(x,a)\Big),
\end{equation}
one has $p_{\min}>0$.

\emph{The strategy and its departures.} A strategy $\sigma$ maps $(\mathcal F_t,\xi)$ to period-$t$ policies, where $\xi$ is drawn once from a fixed environment-independent distribution and assignments are realized through the uniforms of convention (R). Its departure mass into $Z^\ast$ is
\begin{equation}\label{eq:detectprice-nsig}
N_\sigma:=\sum_{i\in\mathcal V}\sum_{t=0}^T \mathbf 1\{(X_{i,t},A_{i,t})\in Z^\ast\}.
\end{equation}

\emph{The detector.} Let $H_T:=((X_{i,t},A_{i,t},Y_{i,t})_{t\le T})_{i\in\mathcal V}$ be the operating record and $\xi$ the internal randomization. A detector is an event $E\in\sigma(H_T,\xi)$, and it has size $\alpha$ and power $1-\vartheta$ if $P_0(E)\le\alpha$ and $P_1(E)\ge 1-\vartheta$.

\emph{No-change loss.} The no-change loss is $\operatorname{loss}(\sigma):=V^1_T(\hat{\bar\pi})-V^1_T(\sigma)$ under $P_0$, with both value functionals as decomposed in Lemma~\ref{lem:decomp}, evaluated under $P_0$ at $\gamma=1$.

Under Assumptions~\ref{asm:consistency}, \ref{asm:markov}, correct specification at $P_0$ (Assumption~\ref{asm:correctspec}), the hidden-support condition, and the strict-suboptimality condition, the following hold.

\emph{(i)} For each detector $E\in\sigma(H_T,\xi)$,
\begin{equation}\label{eq:detectprice-det}
P_1(E)-P_0(E)\ \le\ P_0(N_\sigma\ge 1)\ =\ P_1(N_\sigma\ge 1)\ \le\ \min\big(1,\ \mathbb E_0[N_\sigma]\big).
\end{equation}

\emph{(ii-a)} Under $P_0$,
\begin{equation}\label{eq:detectprice-pdi}
\operatorname{loss}(\sigma)\ =\ \mathbb E_{0,\sigma}\!\left[\sum_{i\in\mathcal V}\sum_{t=0}^T\Big(V^\ast_t(X_{i,t})-Q^\ast_t(X_{i,t},A_{i,t})\Big)\right],
\end{equation}
and each summand is nonnegative.

\emph{(ii-b)} Consequently, under $P_0$, $\operatorname{loss}(\sigma)\ \ge\ p_{\min}\,\mathbb E_0[N_\sigma]$.

\emph{(iii)} If $\sigma$ admits a detector $E\in\sigma(H_T,\xi)$ of size $\alpha$ and power $1-\vartheta$, then $\operatorname{loss}(\sigma)\ \ge\ p_{\min}(1-\alpha-\vartheta)$.

\emph{(iv)} Suppose additionally that $\sigma$ is the budgeted procedure (Section~\ref{sec:now}) with per-review budget $C<\infty$ and $\bar k$ reviews before $T$, and that its working model equals the true tables at every review, $\hat m=\theta^{\mathrm{sim}}$ and $\hat g=f^{\mathrm{sim}}$. Then, with $N^{(k)}_\sigma$ the departure mass into $Z^\ast$ contributed by review $k$,
\begin{equation}\label{eq:detectprice-budget}
\mathbb E_0[N^{(k)}_\sigma]\ \le\ \frac{C}{p_{\min}}\quad(k=1,\dots,\bar k),\qquad \mathbb E_0[N_\sigma]=\sum_{k=1}^{\bar k}\mathbb E_0[N^{(k)}_\sigma]\ \le\ \frac{\bar kC}{p_{\min}},
\end{equation}
so that any detector $E$ built from $\sigma$ has $P_1(E)\le P_0(E)+\bar kC/p_{\min}$.
\end{restated}

The strict-suboptimality condition implies the hidden-support condition, since $Q^\ast_t(x,a)<V^\ast_t(x)$ for each $(x,a)\in Z^\ast$ means that $a$ is not a maximizer of $Q^\ast_t(x,\cdot)$, so that $\hat\pi_t(x)\ne a$ at each such cell. We first establish four lemmas (a coupling to the first visit of the hidden change set, a telescoping no-change-loss identity in two forms, a budget bound on the departure price, and a one-step Bellman identity) and then prove the four conclusions of the detection--price theorem they serve.

\begin{lemma}[Joint coupling to the first $Z^\ast$-visit]\label{lem:detectprice-L1}
Suppose Assumptions~\ref{asm:consistency} and \ref{asm:markov} hold, that the simulator is correctly specified at $P_0$, that conditions (a)--(c) on $Z^\ast$ in the theorem hold, and that convention (R) of Appendix~\ref{app:proofs} is in force. Write $\rho:=\min\{t:\exists i,\ (X_{i,t},A_{i,t})\in Z^\ast\}$ for the first period at which some unit's cell lies in the hidden change set (with $\rho=\infty$ when no such $t\le T$ exists), and set $N_{\sigma,\le t}:=\sum_i\sum_{s\le t}\mathbf 1\{(X_{i,s},A_{i,s})\in Z^\ast\}$. Then there is a probability space carrying a coupling $(\omega_0,\omega_1)$ of the two environments run under the same strategy $\sigma$ and the same internal randomization $\xi$, with the following properties. (i) On $\{\rho>t\}$ the two copies agree on the joint latent-and-observed history $((X_{i,s},U_{i,s},A_{i,s},Y_{i,s})_{s\le t})_{i\in\mathcal V}$ and on $\xi$. (ii) On $\{\rho\le T\}$ the two copies' states and actions $((X_{i,s},A_{i,s})_{s\le\rho})_i$ through period $\rho$ and their outcomes $((Y_{i,s})_{s\le\rho-1})_i$ through period $\rho-1$ coincide both $P_0$-almost surely and $P_1$-almost surely; the internal randomization $\xi$ coincides as well. (iii) The events $\{\rho\le t\}=\{N_{\sigma,\le t}\ge 1\}$ satisfy $P_0(\rho\le t)=P_1(\rho\le t)$ for each $t\le T$.
\end{lemma}

\begin{lemma}[Equality-tracking performance difference under $P_0$]\label{lem:detectprice-L2}
Suppose the simulator is correctly specified at $P_0$ and $\gamma=1$. Then, for each history-dependent strategy $\sigma$,
\begin{equation}\label{eq:detectprice-L2}
V^1_T(\hat{\bar\pi})-V^1_T(\sigma)=\mathbb E_{0,\sigma}\!\left[\sum_{i\in\mathcal V}\sum_{t=0}^T\Big(V^\ast_t(X_{i,t})-Q^\ast_t(X_{i,t},A_{i,t})\Big)\right],
\end{equation}
under $P_0$, where each summand is nonnegative. The identity uses no Markov-in-$X$ property of $P_1$.
\end{lemma}

\begin{lemma}[Full-horizon conditional performance difference under $P_0$]\label{lem:detectprice-L2c}
Suppose the simulator is correctly specified at $P_0$ and $\gamma=1$. Fix a review $k$ at period $t_k$ and condition on $\mathcal F_{t_k}$. Let $\hat{\bar\pi}=S(\theta^{\mathrm{sim}},f^{\mathrm{sim}})$ be the plug-in policy for periods $t_k,\dots,T$, and let $\pi$ be the policy tensor of Proposition~\ref{prop:purse}, equal to $\hat{\bar\pi}$ except at the period-$t_k$ test slice $\{(i,t_k):i\in\mathcal D_k\}$, where it is the admitted designs' uniform assignment, and equal to $\hat{\bar\pi}$ at each period $u>t_k$. Let $V(\rho):=\mathbb E_{\rho}[\sum_{u=t_k}^T\sum_i Y_{i,u}\mid\mathcal F_{t_k}]$ under $P_0$ be the full-remaining-horizon value of Proposition~\ref{prop:purse}. Then
\begin{equation}\label{eq:detectprice-L2c}
\operatorname{loss}:=V(\hat{\bar\pi})-V(\pi)
=\mathbb E_{0,\pi}\!\left[\sum_{i\in\mathcal V}\sum_{t=t_k}^{T}\Big(V^\ast_t(X_{i,t})-Q^\ast_t(X_{i,t},A_{i,t})\Big)\ \Big|\ \mathcal F_{t_k}\right],
\end{equation}
the telescoping running over the full remaining horizon $[t_k,T]$ with terminal boundary $V^\ast_{T+1}\equiv 0$. Each summand is nonnegative. Under $\hat{\bar\pi}$'s selector-optimality each period $u>t_k$ and each non-tested unit at $t_k$ contributes a zero gap, yielding
\begin{equation}\label{eq:detectprice-L2c-slice}
\operatorname{loss}=\mathbb E_{0,\pi}\!\left[\sum_{i\in\mathcal D_k}\Big(V^\ast_{t_k}(X_{i,t_k})-Q^\ast_{t_k}(X_{i,t_k},A_{i,t_k})\Big)\ \Big|\ \mathcal F_{t_k}\right],
\end{equation}
an $\mathcal F_{t_k}$-conditional expectation, with $\operatorname{loss}$ and $\operatorname{price}(d)$ as defined in Proposition~\ref{prop:purse}.
\end{lemma}

\begin{lemma}[Budget bound on mean-world departure price]\label{lem:detectprice-L3}
Work in the exact-table regime of Conclusion (iv). At a single review $k$ the procedure's admitted test package has total price $\sum_d\operatorname{price}(d)\le C$ under $P_0$ (Proposition~\ref{prop:purse}). Moreover this $\mathcal F_{t_k}$-conditional price equals the $\mathcal F_{t_k}$-conditional expected sum of advantage gaps over the tested unit-periods and is bounded below as follows:
\begin{equation}\label{eq:detectprice-L3}
\sum_d\operatorname{price}(d)=\mathbb E_0\!\left[\sum_{i\in\mathcal D_k}\Big(V^\ast_{t_k}(X_{i,t_k})-Q^\ast_{t_k}(X_{i,t_k},A_{i,t_k})\Big)\ \Big|\ \mathcal F_{t_k}\right]\ \ge\ p_{\min}\,\mathbb E_0\big[N^{(k)}_\sigma\ \big|\ \mathcal F_{t_k}\big],
\end{equation}
where $N^{(k)}_\sigma:=\sum_{i\in\mathcal D_k}\mathbf 1\{(X_{i,t_k},A_{i,t_k})\in Z^\ast\}$. Taking $\mathbb E_0[\,\cdot\,]$, we obtain $C\ge p_{\min}\,\mathbb E_0[N^{(k)}_\sigma]$.
\end{lemma}

\begin{lemma}[Bellman step for the exact simulated distribution]\label{lem:detectprice-L4}
Suppose the simulator is correctly specified at $P_0$. Then, for each $t\le T$, each unit $i$, and each $\mathcal F_t$-measurable assignment $A_{i,t}$,
\begin{equation}\label{eq:detectprice-L4}
\mathbb E_0\big[Y_{i,t}+V^\ast_{t+1}(X_{i,t+1})\ \big|\ \mathcal F_t,\mathbf A_t\big]=Q^\ast_t(X_{i,t},A_{i,t})\quad\text{almost surely}.
\end{equation}
\end{lemma}

\begin{proof}[Proof of Lemma~\ref{lem:detectprice-L4}]
Fix $t\le T$ and a unit $i$. By Assumption~\ref{asm:correctspec}, conditional on $(\mathcal F_t,\mathbf A_t)$ the tuple $(Y_{i,t},X_{i,t+1})$ is drawn from the calibrated kernel that depends only on $(X_{i,t},A_{i,t})=(x,a)$: the mean-zero noise gives $\mathbb E_0[Y_{i,t}\mid\mathcal F_t,\mathbf A_t]=\theta^{\mathrm{sim}}(x,a)$ and $\Pr_0[X_{i,t+1}=x'\mid\mathcal F_t,\mathbf A_t]=f^{\mathrm{sim}}(x'\mid x,a)$, with $X_{i,t+1}$ depending on $\mathcal F_t$ only through the current cell. Therefore,
\[
\mathbb E_0\big[Y_{i,t}+V^\ast_{t+1}(X_{i,t+1})\mid\mathcal F_t,\mathbf A_t\big]
=\theta^{\mathrm{sim}}(x,a)+\sum_{x'}f^{\mathrm{sim}}(x'\mid x,a)V^\ast_{t+1}(x')
=Q^\ast_t(x,a),
\]
by the definition of $Q^\ast_t$ in \eqref{eq:detectprice-qstar}.
\end{proof}

\begin{proof}[Proof of Lemma~\ref{lem:detectprice-L2}]
\emph{Step A.} We first evaluate the two values. By Lemma~\ref{lem:decomp}, $V^1_T(\sigma)=\mathbb E_{0,\sigma}[\sum_i\sum_t Y_{i,t}]$. This is finite since $|Y|\le B_Y$. By Theorem~\ref{thm:nochange}, under the product kernel and additive reward the fleet optimal value is $\sum_i\mathbb E_0[V^\ast_0(X_{i,0})]$, attained by $\hat{\bar\pi}=\bar\pi^{\mathrm{sim}}$, so that $V^1_T(\hat{\bar\pi})=\sum_i\mathbb E_0[V^\ast_0(X_{i,0})]$.

\emph{Step B.} Fix a unit $i$ and set, for $t=0,\dots,T$, $\Delta_{i,t}:=V^\ast_t(X_{i,t})-(Y_{i,t}+V^\ast_{t+1}(X_{i,t+1}))$ with $V^\ast_{T+1}\equiv 0$. Summing over $t$ telescopes the $V^\ast$ terms to $\sum_{t=0}^T\Delta_{i,t}=V^\ast_0(X_{i,0})-\sum_{t=0}^T Y_{i,t}$. Taking $\mathbb E_{0,\sigma}$, summing over $i$, and using Step A, we obtain
\begin{equation}\label{eq:detectprice-telescope}
\sum_i\mathbb E_{0,\sigma}\Big[\sum_t\Delta_{i,t}\Big]=\sum_i\mathbb E_0[V^\ast_0(X_{i,0})]-\mathbb E_{0,\sigma}\Big[\sum_i\sum_t Y_{i,t}\Big]=V^1_T(\hat{\bar\pi})-V^1_T(\sigma),
\end{equation}
where the initial-state distribution is shared and policy-free under (R), since the initial state is drawn before any assignment and so its distribution does not depend on the strategy, and $\mathbb E_{0,\sigma}[V^\ast_0(X_{i,0})]=\mathbb E_0[V^\ast_0(X_{i,0})]$.

\emph{Step C.} We now condition on $(\mathcal F_t,\mathbf A_t)$. Since $V^\ast_t(X_{i,t})$ is $\mathcal F_t$-measurable and $A_{i,t}$ is $(\mathcal F_t,\xi)$-measurable, Lemma~\ref{lem:detectprice-L4} implies $\mathbb E_{0,\sigma}[\Delta_{i,t}\mid\mathcal F_t,\mathbf A_t]=V^\ast_t(X_{i,t})-Q^\ast_t(X_{i,t},A_{i,t})$. By the tower rule, $\mathbb E_{0,\sigma}[\Delta_{i,t}]=\mathbb E_{0,\sigma}[V^\ast_t(X_{i,t})-Q^\ast_t(X_{i,t},A_{i,t})]$. Substituting into \eqref{eq:detectprice-telescope} yields \eqref{eq:detectprice-L2}.

\emph{Step D.} It remains to check nonnegativity. For each $x$ and $t$, the definition of $V^\ast_t$ in \eqref{eq:detectprice-qstar} shows $V^\ast_t(x)=\max_a Q^\ast_t(x,a)\ge Q^\ast_t(x,a)$. Each summand is thus nonnegative pointwise, and nonnegative in expectation.
\end{proof}

\begin{proof}[Proof of Lemma~\ref{lem:detectprice-L2c}]
We run the proof of Lemma~\ref{lem:detectprice-L2} under the $\mathcal F_{t_k}$-conditional measure over the full remaining horizon $[t_k,T]$ with terminal boundary $V^\ast_{T+1}\equiv 0$, which are the horizon window and boundary that define the value $V(\pi')=\mathbb E_{\pi'}[\sum_{u=t_k}^T\sum_iY_{i,u}\mid\mathcal F_{t_k}]$ of a policy tensor $\pi'$ in Proposition~\ref{prop:purse}.

\emph{Step A.} By the optimality of $\hat{\bar\pi}$ from any state at any period (Theorem~\ref{thm:nochange}, Lemma~\ref{lem:decomp}), its remaining-life value from $t_k$ is $V(\hat{\bar\pi})=\sum_i\mathbb E_0[V^\ast_{t_k}(X_{i,t_k})\mid\mathcal F_{t_k}]$. For the departing tensor $\pi$, we telescope as in Step B of Lemma~\ref{lem:detectprice-L2} over the window $t=t_k,\dots,T$, with $\Delta_{i,t}:=V^\ast_t(X_{i,t})-(Y_{i,t}+V^\ast_{t+1}(X_{i,t+1}))$ and $V^\ast_{T+1}\equiv 0$, which gives $\sum_{t=t_k}^{T}\Delta_{i,t}=V^\ast_{t_k}(X_{i,t_k})-\sum_{t=t_k}^{T}Y_{i,t}$. Take $\mathbb E_{0,\pi}[\,\cdot\mid\mathcal F_{t_k}]$, sum over $i$, and use $V(\pi)=\mathbb E_{0,\pi}[\sum_i\sum_{t=t_k}^TY_{i,t}\mid\mathcal F_{t_k}]$ together with $\mathbb E_{0,\pi}[V^\ast_{t_k}(X_{i,t_k})\mid\mathcal F_{t_k}]=\mathbb E_0[V^\ast_{t_k}(X_{i,t_k})\mid\mathcal F_{t_k}]$ (the state at $t_k$ is $\mathcal F_{t_k}$-measurable and so does not depend on the policy tensor). We find that
\[
\sum_i\mathbb E_{0,\pi}\Big[\sum_{t=t_k}^{T}\Delta_{i,t}\ \Big|\ \mathcal F_{t_k}\Big]=\sum_i\mathbb E_0[V^\ast_{t_k}(X_{i,t_k})\mid\mathcal F_{t_k}]-V(\pi)=V(\hat{\bar\pi})-V(\pi)=\operatorname{loss}.
\]

\emph{Step B.} We next condition on $(\mathcal F_t,\mathbf A_t)$ for $t\ge t_k$. By Lemma~\ref{lem:detectprice-L4}, $\mathbb E_0[\Delta_{i,t}\mid\mathcal F_t,\mathbf A_t]=V^\ast_t(X_{i,t})-Q^\ast_t(X_{i,t},A_{i,t})$, and the tower rule from $\mathcal F_{t_k}$ yields \eqref{eq:detectprice-L2c}.

\emph{Step C.} Each summand is nonnegative because $V^\ast_t=\max_aQ^\ast_t$ (Step D of Lemma~\ref{lem:detectprice-L2}), and it vanishes wherever $\pi$ assigns the selector action $\hat\pi_t(X_{i,t})\in\arg\max_aQ^\ast_t(X_{i,t},a)$. Since $\pi=\hat{\bar\pi}$ at each period $u>t_k$ and at each non-tested unit at $t_k$, those unit-periods contribute zero; only the tested unit-periods $\{(i,t_k):i\in\mathcal D_k\}$ remain, giving \eqref{eq:detectprice-L2c-slice}.
\end{proof}

\begin{proof}[Proof of Lemma~\ref{lem:detectprice-L1}]
We construct the coupling from a single draw of the potential-outcome processes and then verify the three claims. Draw once the internal randomization $\xi$ from its fixed environment-independent distribution and the assignment uniforms $(\xi_{i,t})_{i,t}$ of convention (R). Draw also, for each unit $i$ and action path $\bar a$, the potential-outcome process $(X_{i,s}(\bar a_{s-1}),U_{i,s}(\bar a_{s-1}),Y_{i,s}(\bar a_s))_s$, whose one-step innovations are governed by the innovation kernel. By the change-at-$Z^\ast$ condition and Assumption~\ref{asm:markov}, the distribution of any potential path segment is determined by the initial full-state distribution of $(X_{i,0},U_{i,0})$ and the innovation kernels at the cells it visits (Lemma~\ref{lem:pathlaw} and the potential-path construction in the proof of Proposition~\ref{prop:ident}). Recall that $P_0$ and $P_1$ agree on the initial full-state distribution and on the full innovation kernel (response, noise $Q$, and transition) at each cell outside $Z^\ast$. The two environments' potential processes are consequently constructed from the same underlying random draws, that is, the shared initial full state and the shared innovations outside the change set. We couple them by the maximal coupling of Lemma~\ref{lem:tvfacts} applied cell by cell. Here we take the identity coupling at each non-$Z^\ast$ cell, where the two kernels agree, so that a potential segment coincides across the two copies up to and excluding the first period its action path visits $Z^\ast$. This coincidence is a joint equality of $(Y,X',U')$ rather than of $(X',U')$ alone, because the outcome-noise kernel $Q$ agrees. From that period on we take independent draws from the respective $Z^\ast$-cell kernels.

Under (R), the realized joint history is a deterministic function of the potential process, the uniforms, and $\xi$: on $\{\bar A_{i,s-1}=\bar a_{s-1}\}$, Assumption~\ref{asm:consistency} gives $(X_{i,s},U_{i,s},Y_{i,s})=(X_{i,s}(\bar a_{s-1}),U_{i,s}(\bar a_{s-1}),Y_{i,s}(\bar a_s))$, and $A_{i,s}$ is a measurable function of $(\mathcal F_s,\xi)$ through $\sigma$'s period-$s$ policy and the uniforms. We show by induction on $t$ that on $\{\rho>t\}$ the two copies' joint histories through period $t$ and their $\xi$ coincide. At $t=0$ the initial full states, $\xi$, and the uniforms are shared, so the period-$0$ policy is the same measurable function of the identical empty record and $\xi$, and $A_{i,0}$ is identical. If no unit's cell is in $Z^\ast$ then $(X_{i,0},A_{i,0})\notin Z^\ast$ for all $i$, and the period-$0$ innovation $(Y_{i,0},X_{i,1},U_{i,1})$ is a draw from the shared innovation kernel off $Z^\ast$; since this draw includes the outcome noise $Q$, the innovation is identical. For the inductive step, assume identity through period $t-1$ on $\{\rho>t-1\}$. Then $\mathcal F_t$ (the record through $t-1$ together with the period-$t$ states) is identical, and the period-$t$ policy, and $A_{i,t}$ with it, is identical too. On $\{\rho>t\}$ each realized cell $(X_{i,t},A_{i,t})\notin Z^\ast$, so the period-$t$ innovation is drawn from the shared potential process and coincides. This establishes claim (i).

We now establish (ii) and (iii). On $\{\rho\le T\}$ the two copies' states and actions $(X_{i,s},A_{i,s})_{s\le\rho}$ and their outcomes $(Y_{i,s})_{s\le\rho-1}$ coincide, together with $\xi$: the visit's assignment $A_{i,\rho}$ is $\mathcal F_\rho$-measurable and coupled-identical, and only the post-assignment innovation $(Y_{i,\rho},X_{i,\rho+1},U_{i,\rho+1})$ at the $Z^\ast$ cell may differ. The event $\{\rho\le t\}=\{N_{\sigma,\le t}\ge 1\}$ is determined by the coupled-identical states and actions alone, and $\mathbf 1\{\rho\le t\}$ is thus the same function under both copies, giving $P_0(\rho\le t)=P_1(\rho\le t)$ for each $t\le T$.
\end{proof}

\begin{proof}[Proof of Conclusion (i)]
Let $E\in\sigma(H_T,\xi)$. On the coupled space of Lemma~\ref{lem:detectprice-L1}, on the event $\{\rho>T\}=\{N_\sigma=0\}$ the two copies have identical $(H_T,\xi)$, and hence identical $\mathbf 1_E$. Consequently,
\[
P_1(E)-P_0(E)=\mathbb E[\mathbf 1_{E}^{(1)}-\mathbf 1_{E}^{(0)}]=\mathbb E\big[(\mathbf 1_{E}^{(1)}-\mathbf 1_{E}^{(0)})\mathbf 1\{\rho\le T\}\big]\le\Pr(\rho\le T),
\]
since the integrand is bounded by $1$ and vanishes on $\{\rho>T\}$. Moreover $\Pr(\rho\le T)=P_0(\rho\le T)=P_0(N_\sigma\ge 1)$, because $\{\rho\le T\}=\{N_\sigma\ge 1\}$ and, by Lemma~\ref{lem:detectprice-L1} at $t=T$, $P_0(N_\sigma\ge 1)=P_1(N_\sigma\ge 1)$. Finally, since $N_\sigma$ is a nonnegative integer, Markov's inequality yields $P_0(N_\sigma\ge 1)\le\mathbb E_0[N_\sigma]$, and trivially $P_0(N_\sigma\ge 1)\le 1$; we conclude that $P_0(N_\sigma\ge 1)\le\min(1,\mathbb E_0[N_\sigma])$, which is \eqref{eq:detectprice-det}.
\end{proof}

\begin{proof}[Proof of Conclusion (ii)]
Part (ii-a) is Lemma~\ref{lem:detectprice-L2}. We now prove part (ii-b). Drop from the nonnegative sum in \eqref{eq:detectprice-pdi} each term with $(X_{i,t},A_{i,t})\notin Z^\ast$:
\[
\operatorname{loss}(\sigma)\ge\mathbb E_{0,\sigma}\!\Big[\sum_{i,t}\mathbf 1\{(X_{i,t},A_{i,t})\in Z^\ast\}\,\big(V^\ast_t(X_{i,t})-Q^\ast_t(X_{i,t},A_{i,t})\big)\Big].
\]
On $\{(X_{i,t},A_{i,t})\in Z^\ast\}$, by the strict-suboptimality condition we have $V^\ast_t(X_{i,t})-Q^\ast_t(X_{i,t},A_{i,t})\ge p_{\min}>0$; the right side is therefore at least $p_{\min}\,\mathbb E_{0,\sigma}[\sum_{i,t}\mathbf 1\{(X_{i,t},A_{i,t})\in Z^\ast\}]=p_{\min}\,\mathbb E_0[N_\sigma]$.
\end{proof}

\begin{proof}[Proof of Conclusion (iii)]
Let $E$ be a detector of size $\alpha$ and power $1-\vartheta$, with $P_0(E)\le\alpha$ and $P_1(E)\ge 1-\vartheta$. Then $1-\alpha-\vartheta\le P_1(E)-P_0(E)$, and by Conclusion (i), $P_1(E)-P_0(E)\le\min(1,\mathbb E_0[N_\sigma])\le\mathbb E_0[N_\sigma]$; it follows that $1-\alpha-\vartheta\le\mathbb E_0[N_\sigma]$. By Conclusion (ii-b), $\mathbb E_0[N_\sigma]\le\operatorname{loss}(\sigma)/p_{\min}$. Chaining these bounds, we conclude that $\operatorname{loss}(\sigma)\ge p_{\min}(1-\alpha-\vartheta)$.
\end{proof}

\begin{proof}[Proof of Conclusion (iv)]
We first show that the condition $\hat m=\theta^{\mathrm{sim}},\hat g=f^{\mathrm{sim}}$ holds under the procedure's shrinkage-rule updates, in two cases, both from the exact-specification argument of Proposition~\ref{prop:purse}. First, in the population limit each observed mass equals its expectation, with $R=N\theta^{\mathrm{sim}}(x,a)$ at each cell and next-state masses $Nf^{\mathrm{sim}}(\cdot\mid x,a)$ under any strategy, tests included; the shrinkage rule with centre $\theta^{\mathrm{sim}}$ (and the structural centre, by the same ridge argument) returns $(\theta^{\mathrm{sim}},f^{\mathrm{sim}})$ at each period. Second, under degenerate noise $Q(\cdot\mid x,a)$ is a point mass at zero and $f^{\mathrm{sim}}(\cdot\mid x,a)$ are point masses, so that each realized reward at a visited cell equals $\theta^{\mathrm{sim}}(x,a)$ and the deterministic transitions are not contradicted (Theorem~\ref{thm:nochange}, Proposition~\ref{prop:censoring}).

We next fix a review $k$ at period $t_k$ and condition on $\mathcal F_{t_k}$. In the exact-table regime the mean world of Proposition~\ref{prop:purse}, that is, the distribution of \eqref{eq:sim} with $(\hat m,\hat g)$, coincides with $P_0$. The budgeted procedure admits at each review a package of one-period randomized designs on disjoint cohorts with $\sum_d\operatorname{price}(d)\le C$ (Section~\ref{sec:now}). Proposition~\ref{prop:purse} establishes, as an $\mathcal F_{t_k}$-conditional statement, $\operatorname{loss}=\sum_d\operatorname{price}(d)\le C$ with the full-remaining-horizon value $V$ of Lemma~\ref{lem:detectprice-L2c} under $P_0$.

By Lemma~\ref{lem:detectprice-L2c}, the slice identity over the full horizon $[t_k,T]$ holds; this is the price identity of Lemma~\ref{lem:detectprice-L3}, matching the object defined in Proposition~\ref{prop:purse}:
\[
\sum_d\operatorname{price}(d)=\operatorname{loss}=\mathbb E_0\!\left[\sum_{i\in\mathcal D_k}\big(V^\ast_{t_k}(X_{i,t_k})-Q^\ast_{t_k}(X_{i,t_k},A_{i,t_k})\big)\ \Big|\ \mathcal F_{t_k}\right].
\]

The departure count $N^{(k)}_\sigma$ is supported on the tested slice. Keeping only the tested unit-periods whose realized cell lies in $Z^\ast$ and applying the strict-suboptimality condition inside the $\mathcal F_{t_k}$-conditional expectation, we obtain
\[
C\ \ge\ \sum_d\operatorname{price}(d)\ \ge\ \mathbb E_0\!\left[\sum_{i\in\mathcal D_k}\mathbf 1\{(X_{i,t_k},A_{i,t_k})\in Z^\ast\}\,p_{\min}\ \Big|\ \mathcal F_{t_k}\right]=p_{\min}\,\mathbb E_0\big[N^{(k)}_\sigma\ \big|\ \mathcal F_{t_k}\big],
\]
where a tested departure landing outside $Z^\ast$ contributes a nonnegative gap by Step D of Lemma~\ref{lem:detectprice-L2}, and discarding it only lowers the sum. Both ends are $\mathcal F_{t_k}$-measurable, and taking $\mathbb E_0[\,\cdot\,]$ we obtain $\mathbb E_0[N^{(k)}_\sigma]\le C/p_{\min}$.

It remains to sum over reviews. A departure into $Z^\ast$ occurs only within some review's admitted test: outside admitted tests the deployed policy is $S(\hat m,\hat g)=\hat{\bar\pi}$, whose pointwise action does not lie in $Z^\ast$, since $\hat\pi_t(x)\ne a$ for each $(x,a)\in Z^\ast$ by the implication of strict suboptimality noted after the theorem statement. Summing over the $\bar k$ reviews before $T$,
\[
\mathbb E_0[N_\sigma]=\sum_{k=1}^{\bar k}\mathbb E_0[N^{(k)}_\sigma]\ \le\ \bar k\cdot\frac{C}{p_{\min}}=\frac{\bar kC}{p_{\min}}.
\]
By Conclusion (i), any detector $E$ built from $\sigma$ has $P_1(E)\le P_0(E)+\mathbb E_0[N_\sigma]\le P_0(E)+\bar kC/p_{\min}$. This is \eqref{eq:detectprice-budget}.
\end{proof}

Theorem~\ref{thm:undetectable} is the $N_\sigma=0$ case: Conclusion (i) with the truncated count $N_{\sigma,\le t}$ (Lemma~\ref{lem:detectprice-L1}) yields $P_1(E)=P_0(E)$ for each detector $E\in\sigma(H_t,\xi)$ and each $t\le T$, and $E=\{\tau\le t\}$ recovers $P_0(\tau\le t)=P_1(\tau\le t)$.
\subsection{Value of a designed test over structural adaptation (Theorem~\ref{thm:testvalue})}\label{app:testvalue}

\paragraph{Setting and standing assumptions.} We first decompose the value difference between the experiment and adaptation, then consider three cases in turn: when the record already settles the decision, when identification fails, and when the estimates lag a drifting regime distribution. We work under Assumptions~\ref{asm:consistency} and~\ref{asm:markov} together with the finite regime structure of Theorem~\ref{thm:regimeid}. The latent covariate then takes values in a known set of $K$ regimes, and the regime-conditional response $\theta(z,u)$ and transition kernel $f(\cdot\mid z,u)$ are known for each cell $z=(x,a)$ and regime $u$. A \emph{zone} is a set of $n$ units that share one regime distribution $\nu_t\in\Delta_K$, and that distribution does not depend on the observed state. We write $m_t(z):=\theta(z,\cdot)^\top\nu_t$ and $B_\theta:=\max_z\lVert\theta(z,\cdot)\rVert_\infty$, a bound on the magnitude of the responses, distinct from the Lipschitz constant $L_\theta$ of Assumption~\ref{asm:smooth}. For $\mu\in\Delta_K$ we let $\pi[\mu]:=S\big(\theta^\top\mu,\sum_u\mu(u)f(\cdot\mid\cdot,u)\big)$ be the plug-in policy of the structural tables at $\mu$, and we let $D(\pi):=\{\mu\in\Delta_K:\pi[\mu]=\pi\}$ be the \emph{decision region} of a policy $\pi$. For a stationary policy $\pi$ and $\mu\in\Delta_K$ we write $P_{\pi,\mu}(x,x'):=\sum_u\mu(u)f(x'\mid x,\pi(x),u)$ and $r_{\pi,\mu}(x):=\theta((x,\pi(x)),\cdot)^\top\mu$.

\emph{Drift condition.} We require $\lVert\nu_{t+1}-\nu_t\rVert_1\le\lambda$ for each $t$, the counterpart of Assumption~\ref{asm:drift} for a zone with finitely many regimes.

\emph{The record.} The state distribution $d^\sigma_t$ and the cell occupancy $d^\sigma_r(z)$ are deterministic. We let $q_r(z):=n\,d_r(z)\,w_r(z)$ be the weighted observation mass at $z$ in period $r$, where $w_r(z)=w\ge1$ for the observations of a test run at period $r$ and $w_r(z)=1$ otherwise. At period $t$, for each cell, the recorded masses are
\begin{equation*}
\begin{gathered}
N_t(z)=\sum_{r\le t}(1-\omega)^{t-r}q_r(z),\qquad R_t(z)=\sum_{r\le t}(1-\omega)^{t-r}q_r(z)\,m_r(z),\\
\bar m_t(z):=R_t(z)/N_t(z),
\end{gathered}
\end{equation*}
with memory rate $\omega\in[0,1)$. We write $V_t:=\{z:N_t(z)>0\}$ and, for $z\in V_t$, we define the \emph{age weights} $\omega_{t,r}(z):=(1-\omega)^{t-r}q_r(z)/N_t(z)$ and the \emph{mean age} $\bar a_t(z):=\sum_{r\le t}\omega_{t,r}(z)(t-r)$. The \emph{structural fit} with ridge weight $n_0\ge0$ is
\begin{equation}\label{eq:tv-fit}
\hat\nu_t:=\arg\min_{\mu\in\Delta_K}J_t(\mu),\qquad J_t(\mu):=\sum_{z\in V_t}N_t(z)\big(\bar m_t(z)-\theta(z,\cdot)^\top\mu\big)^2+n_0\lVert\mu-\nu^{\mathrm{sim}}\rVert_2^2,
\end{equation}
and its \emph{ridge limit} $\lim_{n_0\downarrow0}\hat\nu_t$ exists (Lemma~\ref{lem:tv-I}). We write $\mathbb T:=\{v\in\mathbb R^K:\mathbf 1^\top v=0\}$ for the tangent space of the simplex, $\Pi_{\mathbb T}$ for the orthogonal projector onto it, and, for a symmetric matrix $H$,
\begin{equation*}
\lambda^\perp_{\min}(H):=\min\{v^\top Hv:\ v\in\mathbb T,\ \lVert v\rVert_2=1\},\qquad H_t:=\sum_{z\in V_t}N_t(z)\,\theta(z,\cdot)\theta(z,\cdot)^\top .
\end{equation*}

\emph{Strategies.} In period $t+1$ a strategy deploys the policy $\pi[\hat\nu_t]$. \emph{Adaptation} $A$ does this in each period. The \emph{experiment} $E$ coincides with $A$ except at one review period $t_0$, where a cohort of $c$ units, chosen at random, receives actions drawn uniformly from a declared set $\mathcal A_d$. We write $Z$ for the cells thereby covered and $V:=V_{t_0}$ for the cells with positive record at $t_0$, with $Z\cap V=\emptyset$; the cohort's observations at period $t_0$ receive weight $w$. We set $H_Z:=\sum_{z\in Z}N_Z(z)\theta(z,\cdot)\theta(z,\cdot)^\top$ with $N_Z(z):=n\,d_{t_0}(z)$ the unweighted test mass. For $t\ge t_0$ the term $(1-\omega)^{t-t_0}wH_Z$ then enters the information matrix $H^E_t$ of $E$, whose routine part we write $H_V(t)$. We write $A\circ d$ for the assignment sequence that equals $A$'s in each period except that the test replaces the cohort's assignment at $t_0$.

\emph{Values.} The zone's total expected reward under $\sigma$ is $\mathcal V^\sigma:=n\sum_{t=1}^{T}d^\sigma_tr_t(\pi^\sigma_t)$, and we define
\begin{equation*}
\mathrm{price}(t_0):=\mathcal V^{A}-\mathcal V^{A\circ d},\qquad \mathrm{Gain}:=\mathcal V^{E}-\mathcal V^{A\circ d}.
\end{equation*}

\emph{Minorization condition.} For a set of policies $\Pi_0$ and a set $\mathcal N\subseteq\Delta_K$ of regime distributions, there exist $\epsilon\in(0,1]$ and $\varrho\in\mathbb N$ such that for each $\pi\in\Pi_0$ and each sequence $\mu_1,\dots,\mu_\varrho\in\mathcal N$ there is a probability vector $\phi$ with $P_{\pi,\mu_1}\cdots P_{\pi,\mu_\varrho}(x,\cdot)\ge\epsilon\,\phi(\cdot)$ for each state $x$. For $\mu\in\mathcal N$ the frozen chain $P_{\pi,\mu}$ then has a unique stationary distribution; we write $g_\mu(\pi)$ for its average reward and $h_\mu(\pi):=\sum_{s\ge0}\big(P^s_{\pi,\mu}r_{\pi,\mu}-g_\mu(\pi)\mathbf 1\big)$ for its bias (Lemma~\ref{lem:tv-F}). We assume further that $\mu\mapsto h_\mu(\pi)$ is Lipschitz on $\mathcal N$ in $\lVert\cdot\rVert_1$ with a common finite constant $L_h$, and we write $\bar h:=\sup_{\pi\in\Pi_0,\mu\in\mathcal N}\lVert h_\mu(\pi)\rVert_\infty$.

\emph{Switch condition.} There are policies $\pi_0\ne\pi_1$ and periods $t_0<t^E$, $t_0\le t^A$ such that $E$ deploys $\pi_0$ on $(t_0,t^E]$ and $\pi_1$ on $(t^E,T]$, and $A$ follows $\pi_0$ on $[t_0,t^A]$ and $\pi_1$ on $(t^A,T]$.

\begin{restated}{Theorem}{thm:testvalue}{Value of a designed test over structural adaptation, formal statement}
Suppose the standing assumptions are in force. Then the following statements hold.

(i) \emph{Decomposition.} $\mathcal V^{E}-\mathcal V^{A}=\mathrm{Gain}-\mathrm{price}(t_0)$.

(ii) \emph{No value when the record settles the decision.} Suppose $\nu_t=\nu$ for each $t$ and both strategies use the ridge limit, and let $\mathcal I_V(\nu):=\{\mu\in\Delta_K:\Theta_V\mu=\Theta_V\nu\}$ be the regime distributions giving the same visited-cell means as that distribution. If $\pi[\mu]=\pi[\nu]$ for each $\mu\in\mathcal I_V(\nu)$, then for each $c\ge1$, $w\ge1$, and $\mathcal A_d$ the strategies $E$ and $A\circ d$ have identical assignments in each period, $\mathrm{Gain}=0$, and $\mathcal V^E-\mathcal V^A=-\mathrm{price}(t_0)$.

(iii) \emph{Identification failure.} Suppose $\nu_t=\nu$ for each $t$ and $\omega=0$; let $\tilde\pi\ne\pi^*:=\pi[\nu]$ be a policy with $\operatorname{supp}\tilde\pi\subseteq V$ that adaptation uses in each period $t\ge t_0$; suppose $\operatorname{rank}G_{V\cup Z}=K$, where $G_W:=[\Theta_W;\mathbf 1^\top]$; and write $H^E_t=H_V(t)+wH_Z$.
\begin{enumerate}
\item In the ridge limit, $\hat\nu^E_t=\nu$ for each $t\ge t_0$, and therefore $E$ deploys $\pi^*$ in each period $t>t_0$, for each $c\ge1$ and $w\ge1$.
\item With ridge weight $n_0>0$, for each $t\ge t_0$,
$\lVert\hat\nu^E_t-\nu\rVert_2\le n_0\lVert\nu^{\mathrm{sim}}-\nu\rVert_2\big/\big(\lambda^\perp_{\min}(H^E_t)+n_0\big)$ and $\lambda^\perp_{\min}(H^E_t)\ge\lambda^\perp_{\min}(H_V(t_0)+wH_Z)>0$. Therefore, with $r(\nu):=\operatorname{dist}_2(\nu,\Delta_K\setminus D(\pi^*))$ assumed positive, $E$ selects $\pi^*$ in period $t+1$ whenever $\lambda^\perp_{\min}(H^E_t)>n_0\big(\lVert\nu^{\mathrm{sim}}-\nu\rVert_2/r(\nu)-1\big)$.
\item Suppose the minorization condition holds for $\Pi_0=\{\pi^*,\tilde\pi\}$ and $\mathcal N=\{\nu\}$, and that $E$ deploys $\pi^*$ in each period $t>t_0$. Then
$\mathrm{Gain}\ge n(T-t_0)\big(g_\nu(\pi^*)-g_\nu(\tilde\pi)\big)-2n\big(\lVert h_\nu(\pi^*)\rVert_\infty+\lVert h_\nu(\tilde\pi)\rVert_\infty\big)$.
\end{enumerate}

(iv) \emph{Estimation lag under drift.} Suppose the drift condition holds and $\operatorname{rank}G_{V_t}=K$ at the periods considered, and write $\mathcal Q_t:=\sum_{z\in V_t}N_t(z)\bar a_t(z)^2$ and $\bar a_t:=\max_{z\in V_t}\bar a_t(z)$.
\begin{enumerate}
\item For each $t$ and $z\in V_t$, $\bar m_t(z)=\theta(z,\cdot)^\top\bar\nu_t(z)$ with $\bar\nu_t(z):=\sum_{r\le t}\omega_{t,r}(z)\nu_r\in\Delta_K$ and $\lVert\bar\nu_t(z)-\nu_t\rVert_1\le\lambda\bar a_t(z)$, and
\begin{equation*}
\lVert\hat\nu^A_t-\nu_t\rVert_2\le\frac{B_\theta\lambda\sqrt{\lambda_{\max}(H_t)\,\mathcal Q_t}+n_0\lVert\nu^{\mathrm{sim}}-\nu_t\rVert_2}{\lambda^\perp_{\min}(H_t)+n_0}.
\end{equation*}
Suppose $q_r(z)=q(z)$ for each $r\le t$ and $z\in V_t$ (stationary occupancy). Then $\bar\nu_t(z)=\bar\nu_t$ does not depend on $z$; at $n_0=0$ the fit equals $\bar\nu_t$ exactly; and $\bar a_t=(t-1)/2$ if $\omega=0$ with the record starting at period $1$, while $\bar a_t\le(1-\omega)/\omega$ if $\omega\in(0,1)$.
\item For each $t\ge t_0$ at which $Z$ receives no routine observation on $(t_0,t]$, with $H^E_t=H_V(t)+(1-\omega)^{t-t_0}wH_Z$, $\mathcal Q^V_t$ the routine cells' term, and $N^{\mathrm{tot}}_Z:=\sum_{z\in Z}N_Z(z)$,
\begin{equation*}
\lVert\hat\nu^E_t-\nu_t\rVert_2\le\frac{B_\theta\lambda\sqrt{\lambda_{\max}(H^E_t)\big(\mathcal Q^V_t+(1-\omega)^{t-t_0}wN^{\mathrm{tot}}_Z(t-t_0)^2\big)}+n_0\lVert\nu^{\mathrm{sim}}-\nu_t\rVert_2}{\lambda^\perp_{\min}(H^E_t)+n_0}.
\end{equation*}
\item Suppose the switch condition with $t^E\le t^A$, and the minorization condition for $\Pi_0=\{\pi_0,\pi_1\}$ and $\mathcal N\supseteq\{\nu_t:t_0\le t\le T\}$. Then
$\mathrm{Gain}\ge n(\delta-\kappa_1\lambda)(t^A-t^E)-n\kappa_0$ with $\delta:=\min_{t\in(t^E,t^A]}\big(g_{\nu_t}(\pi_1)-g_{\nu_t}(\pi_0)\big)$, $\kappa_0:=4\bar h+2B_\theta\varrho/\epsilon\le10B_\theta\varrho/\epsilon$, and $\kappa_1:=2L_h$.
\item If instead the switch condition holds with $t^A<t^E$, then $\mathrm{Gain}\le-n(\delta'-\kappa_1\lambda)(t^E-t^A)+n\kappa_0$ with $\delta':=\min_{t\in(t^A,t^E]}\big(g_{\nu_t}(\pi_1)-g_{\nu_t}(\pi_0)\big)$.
\end{enumerate}
\end{restated}

\begin{corollary}[Switch times under constant-velocity drift]\label{cor:switch}
Assume the standing assumptions, the drift condition, $n_0=0$, $\omega=0$, stationary occupancy on $V$ for the routine record from period $1$, $\operatorname{rank}G_V=K$, and $\nu_r=\nu_1+(r-1)\lambda e$ with $\lVert e\rVert_1=1$. Suppose that, on a neighbourhood of the path $\{\nu_r\}$ and of the fits, the boundary between $D(\pi_0)$ and $D(\pi_1)$ is the hyperplane $\{\mu:\langle a,\mu\rangle=b\}$ with $\langle a,e\rangle>0$ and $D(\pi_1)$ on the side $\langle a,\mu\rangle>b$, crossed by $\nu_t$ at $t^*$. Then the following hold.
\begin{enumerate}
\item Adaptation's fit satisfies $\bar\nu_t=\nu_t-\lambda\frac{t-1}{2}e$; it reaches the hyperplane at $t=2t^*-1$ and lies strictly on the $\pi_1$ side for each $t>2t^*-1$.
\item The experiment's fit lies in $D(\pi_1)$ at each $t\in(t_0,t^E]$ at which the right-hand side of clause (iv.2) is smaller than $\lambda\langle a,e\rangle(t-t^*)/\lVert a\rVert_2$.
\item Let $M_t$ be the restriction of $\Pi_{\mathbb T}(H_V(t)+wH_Z)\Pi_{\mathbb T}$ to $\mathbb T$ and suppose the unconstrained tangent-space minimizer of $J_t$ is feasible. Then the experiment's fit lies in $D(\pi_0)$ at $t$ if and only if $(t-t^*)\langle a,e\rangle<(\Pi_{\mathbb T}a)^\top M_t^{-1}\Pi_{\mathbb T}\big[\tfrac{t-1}{2}H_V(t)+(t-t_0)wH_Z\big]e$. If $\operatorname{rank}G_Z=K$, then as $w\to\infty$ the right-hand side tends to $(t-t_0)\langle a,e\rangle$, and the condition to $t_0<t^*$.
\end{enumerate}
\end{corollary}
\subsubsection{Lemmas}\label{app:testvalue-lemmas}

Throughout, $\Theta_W$ is the $|W|\times K$ matrix with rows $\theta(z,\cdot)$, $z\in W$, and $N$ the diagonal matrix of the masses $N_t(z)$, so that $H_t=\Theta_{V_t}^\top N\Theta_{V_t}$. We first establish the age-weighted bias and mean-age identities for the structural fit, then its weighted least-squares error, rank, and exactness, the average-reward telescoping bounds, and the sign of the price, assembling the ingredients of Theorem~\ref{thm:testvalue}.

\begin{lemma}[A cell mean is an age-weighted average of the regime distribution]\label{lem:tv-A}
Let the standing assumptions hold and fix $z\in V_t$. Then $\bar m_t(z)=\theta(z,\cdot)^\top\bar\nu_t(z)$, where $\bar\nu_t(z):=\sum_{r\le t}\omega_{t,r}(z)\nu_r\in\Delta_K$. Suppose in addition that the drift condition holds. Then $\lVert\bar\nu_t(z)-\nu_t\rVert_1\le\lambda\bar a_t(z)$ and $|\bar m_t(z)-m_t(z)|\le B_\theta\lambda\bar a_t(z)$.
\end{lemma}
\begin{proof}
We first compute the cell mean. Since $m_r(z)=\theta(z,\cdot)^\top\nu_r$, we have $R_t(z)=\theta(z,\cdot)^\top\sum_{r\le t}(1-\omega)^{t-r}q_r(z)\nu_r$; dividing by $N_t(z)>0$ gives the first claim. The weights $\omega_{t,r}(z)$ are nonnegative and sum to one, which places $\bar\nu_t(z)\in\Delta_K$. Under the drift condition, $\lVert\nu_r-\nu_t\rVert_1\le\sum_{s=r}^{t-1}\lVert\nu_{s+1}-\nu_s\rVert_1\le\lambda(t-r)$. We find that $\lVert\bar\nu_t(z)-\nu_t\rVert_1=\lVert\sum_r\omega_{t,r}(z)(\nu_r-\nu_t)\rVert_1\le\lambda\sum_r\omega_{t,r}(z)(t-r)=\lambda\bar a_t(z)$. The second bound follows from $|\bar m_t(z)-m_t(z)|=|\theta(z,\cdot)^\top(\bar\nu_t(z)-\nu_t)|\le\lVert\theta(z,\cdot)\rVert_\infty\lVert\bar\nu_t(z)-\nu_t\rVert_1$.
\end{proof}

\begin{lemma}[Mean ages under stationary occupancy]\label{lem:tv-B}
Suppose that $q_r(z)=q(z)>0$ for each $r\le t$. Then $\omega_{t,r}(z)=(1-\omega)^{t-r}/\sum_{r'\le t}(1-\omega)^{t-r'}$ does not depend on $z$, and $\bar a_t(z)=\bar a_t$, where $\bar a_t=(t-1)/2$ if $\omega=0$ and the record starts at period $1$, and $\bar a_t\le(1-\omega)/\omega$ if $\omega\in(0,1)$.
\end{lemma}
\begin{proof}
The first claim is immediate from the definition. At $\omega=0$ we compute $\bar a_t=\frac1t\sum_{k=0}^{t-1}k=\frac{t-1}{2}$. For $\omega\in(0,1)$ put $\rho:=1-\omega$; then $\bar a_t=\sum_{k=0}^{t-1}k\rho^k/\sum_{k=0}^{t-1}\rho^k$ is the mean of the geometric distribution $p_k\propto\rho^k$ on $\{0,1,\dots\}$ conditioned on $\{k\le t-1\}$. Recall that the unconditional mean $\rho/(1-\rho)$ is a mixture of the conditional means on $\{k\le t-1\}$ and on $\{k\ge t\}$, and the second exceeds $t-1$, which in turn exceeds the first. We conclude that $\bar a_t\le\rho/(1-\rho)=(1-\omega)/\omega$.
\end{proof}

\begin{lemma}[Constrained weighted least squares on the simplex]\label{lem:tv-C}
Let $\Theta=\Theta_W$, let $N\succ0$ be diagonal, and set $H:=\Theta^\top N\Theta$. Fix $n_0\ge0$ with $\lambda^\perp_{\min}(H)+n_0>0$, whose first term is the restricted smallest eigenvalue, together with $\nu\in\Delta_K$ and $e\in\mathbb R^{|W|}$, and let
\begin{equation*}
J(\mu):=\sum_{z\in W}N(z)\big(\theta(z,\cdot)^\top\nu+e(z)-\theta(z,\cdot)^\top\mu\big)^2+n_0\lVert\mu-\nu^{\mathrm{sim}}\rVert_2^2,\qquad \hat\mu:=\arg\min_{\mu\in\Delta_K}J(\mu).
\end{equation*}
Then $\hat\mu$ is unique, and $\lVert\hat\mu-\nu\rVert_2\le\big(\lVert\Theta^\top Ne\rVert_2+n_0\lVert\nu^{\mathrm{sim}}-\nu\rVert_2\big)\big/\big(\lambda^\perp_{\min}(H)+n_0\big)$.
\end{lemma}
\begin{proof}
Write $\mu=\nu+v$; since $\mu,\nu\in\Delta_K$, we have $v\in\mathbb T$, the tangent space of the simplex, and $F:=\Delta_K-\nu\subset\mathbb T$ is convex, compact, and contains $0$. Put $s:=\nu^{\mathrm{sim}}-\nu\in\mathbb T$, yielding $J(\nu+v)=v^\top(H+n_0I)v-2v^\top(\Theta^\top Ne+n_0s)+\mathrm{const}$. Let $M$ be the restriction of $\Pi_{\mathbb T}(H+n_0I)\Pi_{\mathbb T}$ to $\mathbb T$, with the orthogonal projection onto the tangent space applied on each side, with smallest eigenvalue $\lambda_M:=\lambda^\perp_{\min}(H)+n_0>0$; then $J$ is strictly convex on $\mathbb T$ and $\hat\mu$ is unique. Setting $\beta:=\Pi_{\mathbb T}(\Theta^\top Ne+n_0s)$ and $\hat v:=M^{-1}\beta$, completing the square gives $J(\nu+v)=\lVert v-\hat v\rVert_M^2+\mathrm{const}$ on $\mathbb T$, where $\lVert u\rVert_M^2:=u^\top Mu$. The constrained minimizer $\hat v_c$ is thus the $\lVert\cdot\rVert_M$-projection of $\hat v$ onto $F$. Since projection onto a closed convex set in a Hilbert norm is nonexpansive and $0\in F$ is its own projection, $\lVert\hat v_c\rVert_M\le\lVert\hat v\rVert_M$. Finally $\lVert\hat v\rVert_M^2=\beta^\top M^{-1}\beta\le\lVert\beta\rVert_2^2/\lambda_M$ and $\lVert\hat v_c\rVert_M^2\ge\lambda_M\lVert\hat v_c\rVert_2^2$, and since $\Pi_{\mathbb T}$ is a contraction, $\lVert\hat v_c\rVert_2\le\lVert\beta\rVert_2/\lambda_M\le(\lVert\Theta^\top Ne\rVert_2+n_0\lVert s\rVert_2)/\lambda_M$.
\end{proof}

\begin{lemma}[The staleness term, and exactness under consistent targets]\label{lem:tv-D}
Work in the setting of Lemma~\ref{lem:tv-C}, and let $e(z)=\theta(z,\cdot)^\top(\bar\nu(z)-\nu)$ with $\lVert\bar\nu(z)-\nu\rVert_1\le\lambda a(z)$ for numbers $a(z)\ge0$. Then
\begin{equation*}
\begin{gathered}
\lVert\Theta^\top Ne\rVert_2\le\sqrt{\lambda_{\max}(H)}\sqrt{\textstyle\sum_{z\in W}N(z)e(z)^2}\le B_\theta\lambda\sqrt{\lambda_{\max}(H)\,\mathcal Q},\\
\mathcal Q:=\textstyle\sum_{z\in W}N(z)a(z)^2 .
\end{gathered}
\end{equation*}
Suppose in addition that $n_0=0$, that $\lambda^\perp_{\min}(H)>0$, and that $\bar\nu(z)=\bar\nu\in\Delta_K$ does not depend on $z$. Then $\hat\mu=\bar\nu$.
\end{lemma}
\begin{proof}
For the first inequality, write $\Theta^\top Ne=\Theta^\top N^{1/2}(N^{1/2}e)$ and use $\lVert\Theta^\top N^{1/2}\rVert^2_{\mathrm{op}}=\lambda_{\max}(H)$; the second follows from $|e(z)|\le\lVert\theta(z,\cdot)\rVert_\infty\lVert\bar\nu(z)-\nu\rVert_1\le B_\theta\lambda a(z)$. For the last claim, under its conditions $J(\mu)=(\mu-\bar\nu)^\top H(\mu-\bar\nu)\ge0$, with equality if and only if $\Theta_W(\mu-\bar\nu)=0$. Since $\mu-\bar\nu\in\mathbb T$ and $\lambda^\perp_{\min}(H)>0$, equality forces $\mu=\bar\nu$, and $\bar\nu\in\Delta_K$ is feasible.
\end{proof}

\begin{lemma}[Rank and the restricted eigenvalue]\label{lem:tv-E}
Let $N(z)>0$ for $z\in W$ and set $H_W=\Theta_W^\top N\Theta_W$. Then $\lambda^\perp_{\min}(H_W)>0$ if and only if $\operatorname{rank}G_W=K$, and in that case $\{\mu\in\Delta_K:\Theta_W\mu=\Theta_W\nu\}=\{\nu\}$ for each $\nu\in\Delta_K$.
\end{lemma}
\begin{proof}
For $v\in\mathbb T$, the quadratic form $v^\top H_Wv=\sum_{z\in W}N(z)(\theta(z,\cdot)^\top v)^2$ vanishes if and only if $\Theta_Wv=0$. This, together with $\mathbf 1^\top v=0$, says $G_Wv=0$. It follows that $\lambda^\perp_{\min}(H_W)=0$ if and only if $\ker G_W\ne\{0\}$, and this is the equivalence. For the second claim, if $\mu,\nu\in\Delta_K$ and $\Theta_W\mu=\Theta_W\nu$, then $\mu-\nu\in\ker G_W$, which is trivial under the rank condition.
\end{proof}

\begin{lemma}[Contraction, bias, and telescoping under uniform minorization]\label{lem:tv-F}
Assume the minorization condition for $\Pi_0$ and $\mathcal N$, and let $\pi\in\Pi_0$.
\begin{enumerate}
\item For each sequence $\mu_1,\mu_2,\dots$ in $\mathcal N$ and all probability vectors $d,d'$, $\lVert(d-d')P_{\pi,\mu_1}\cdots P_{\pi,\mu_s}\rVert_1\le2(1-\epsilon)^{\lfloor s/\varrho\rfloor}$.
\item For $\mu\in\mathcal N$ the frozen chain $P_{\pi,\mu}$ has a unique stationary distribution $\varsigma_\mu$; with $g_\mu(\pi):=\varsigma_\mu r_{\pi,\mu}$ the bias series converges absolutely, $\lVert h_\mu(\pi)\rVert_\infty\le2B_\theta\varrho/\epsilon$, and $h_\mu(\pi)=r_{\pi,\mu}-g_\mu(\pi)\mathbf 1+P_{\pi,\mu}h_\mu(\pi)$.
\item For each $W\ge1$ and probability vector $d$, $\big|\sum_{s=0}^{W-1}dP^s_{\pi,\mu}r_{\pi,\mu}-Wg_\mu(\pi)\big|\le2\lVert h_\mu(\pi)\rVert_\infty$.
\item For each sequence $\mu_1,\mu_2,\dots$ in $\mathcal N$, all probability vectors $d,d'$, and each $W\ge1$,
\begin{equation*}
\Big|\sum_{s=1}^{W}(d-d')P_{\pi,\mu_1}\cdots P_{\pi,\mu_{s-1}}r_{\pi,\mu_s}\Big|\le2B_\theta\varrho/\epsilon .
\end{equation*}
\end{enumerate}
\end{lemma}
\begin{proof}
For (1), recall that a stochastic matrix $Q$ with $Q(x,\cdot)\ge\epsilon\phi$ for all $x$ has Dobrushin coefficient at most $1-\epsilon$. We know that $\lVert(d-d')Q\rVert_1\le(1-\epsilon)\lVert d-d'\rVert_1$ for zero-sum $d-d'$ \citep{levin2009markov}. By the minorization condition each block of $\varrho$ consecutive kernels is such a $Q$; there are $\lfloor s/\varrho\rfloor$ complete blocks, the remaining factors are $\lVert\cdot\rVert_1$-contractions, and $\lVert d-d'\rVert_1\le2$. We next prove (2). Applying (1) with the constant sequence $\mu$, all rows of $P^s_{\pi,\mu}$ converge to one stationary and unique probability vector $\varsigma_\mu$. Since $\varsigma_\mu=\varsigma_\mu P^s_{\pi,\mu}$, we obtain $|P^s_{\pi,\mu}r_{\pi,\mu}(x)-g_\mu(\pi)|=|(\delta_x-\varsigma_\mu)P^s_{\pi,\mu}r_{\pi,\mu}|\le2B_\theta(1-\epsilon)^{\lfloor s/\varrho\rfloor}$, using $\lVert r_{\pi,\mu}\rVert_\infty\le B_\theta$; summing over $s\ge0$ bounds the series by $2B_\theta\varrho/\epsilon$. Finally $P_{\pi,\mu}h_\mu(\pi)=\sum_{s\ge0}(P^{s+1}_{\pi,\mu}r_{\pi,\mu}-g_\mu(\pi)\mathbf 1)=h_\mu(\pi)-(r_{\pi,\mu}-g_\mu(\pi)\mathbf 1)$. For (3), the Poisson equation of part (2) yields $\sum_{s=0}^{W-1}P^s(r-g\mathbf 1)=\sum_{s=0}^{W-1}(P^sh-P^{s+1}h)=h-P^Wh$; multiplying by $d$ and bounding $|dh|,|dP^Wh|\le\lVert h\rVert_\infty$ gives the claim. For (4), each term is bounded by $\lVert(d-d')P_{\pi,\mu_1}\cdots P_{\pi,\mu_{s-1}}\rVert_1\lVert r_{\pi,\mu_s}\rVert_\infty\le2B_\theta(1-\epsilon)^{\lfloor(s-1)/\varrho\rfloor}$ by (1); summing over $s\ge1$ gives the bound.
\end{proof}

\begin{lemma}[A stationary policy on a drifting window]\label{lem:tv-G}
Assume the minorization condition for $\Pi_0\ni\pi$ and $\mathcal N$, and assume the drift condition. Let a unit follow $\pi$ over the periods $s+1,\dots,s+W$ with kernels $P_t:=P_{\pi,\nu_t}$ and rewards $r_t:=r_{\pi,\nu_t}$, where $\nu_t\in\mathcal N$ for $s+1\le t\le s+W+1$, starting from $d_{s+1}$ and with $d_{t+1}=d_tP_t$. Write $h_t:=h_{\nu_t}(\pi)$ and $g_t:=g_{\nu_t}(\pi)$. Then $\big|\sum_{t=s+1}^{s+W}d_tr_t-\sum_{t=s+1}^{s+W}g_t\big|\le2\bar h+WL_h\lambda$.
\end{lemma}
\begin{proof}
By Lemma~\ref{lem:tv-F}(2) at the frozen $\nu_t$ we have $r_t-g_t\mathbf 1=h_t-P_th_t$, from which $d_tr_t-g_t=d_th_t-d_{t+1}h_t=(d_th_t-d_{t+1}h_{t+1})+d_{t+1}(h_{t+1}-h_t)$. Summing over the window, the first terms telescope to $d_{s+1}h_{s+1}-d_{s+W+1}h_{s+W+1}$, which is bounded by $2\bar h$, while each remaining term is bounded by $\lVert h_{t+1}-h_t\rVert_\infty\le L_h\lVert\nu_{t+1}-\nu_t\rVert_1\le L_h\lambda$.
\end{proof}

\begin{lemma}[Sign of the price under finite-horizon optimality]\label{lem:tv-H}
Suppose that $A$'s assignment sequence over the periods $t_0,\dots,T$ maximizes $n\sum_{t=t_0}^{T}d_tr_t(\cdot)$ over all assignment sequences from the state distribution $d^A_{t_0}$ under the true tables. Then $\mathrm{price}(t_0)\ge0$.
\end{lemma}
\begin{proof}
Note that $A$ and $A\circ d$ agree before $t_0$, and thus share $d_{t_0}$. Since $A\circ d$ is one admissible assignment sequence from $t_0$ on, its total over $t_0,\dots,T$ is at most $A$'s.
\end{proof}

\begin{lemma}[Constant regime distribution: exact means, nested solution sets, and the ridge limit]\label{lem:tv-I}
Assume the standing assumptions and suppose that $\nu_t=\nu$ for each $t$.
\begin{enumerate}
\item For each $t$ and each $z\in V_t$, $\bar m_t(z)=\theta(z,\cdot)^\top\nu$, and the least-squares term of $J_t$ equals $(\mu-\nu)^\top H_t(\mu-\nu)$, whose zero set on $\Delta_K$ is $\mathcal I_{V_t}(\nu):=\{\mu\in\Delta_K:\Theta_{V_t}\mu=\Theta_{V_t}\nu\}$.
\item $V_t\subseteq V_{t+1}$ for each $t$, and hence $\mathcal I_{V_{t+1}}(\nu)\subseteq\mathcal I_{V_t}(\nu)$.
\item As $n_0\downarrow0$ the minimizer $\hat\nu_t(n_0)$ converges to the Euclidean projection of $\nu^{\mathrm{sim}}$ onto the convex compact set $\mathcal I_{V_t}(\nu)$; in particular, the ridge limit lies in $\mathcal I_{V_t}(\nu)$.
\end{enumerate}
\end{lemma}
\begin{proof}
We prove the three parts in order. For (1), applying Lemma~\ref{lem:tv-A} with $\nu_r\equiv\nu$, $\bar\nu_t(z)=\nu$, from which the stated identity and zero set follow. For (2), a cell with $q_r(z)>0$ at some $r\le t$ has $N_{t'}(z)\ge(1-\omega)^{t'-r}q_r(z)>0$ for each $t'\ge r$; the occupied cells nest across periods and the solution sets nest in turn. For (3), put $f(\mu):=(\mu-\nu)^\top H_t(\mu-\nu)$ and $g(\mu):=\lVert\mu-\nu^{\mathrm{sim}}\rVert_2^2$, and note that $J_t=f+n_0g$; let $\mu^\star$ minimize the strictly convex $g$ over $\arg\min_{\Delta_K}f=\mathcal I_{V_t}(\nu)$, on which $f=0$. For each $n_0>0$ we have $f(\hat\nu_t(n_0))+n_0g(\hat\nu_t(n_0))\le n_0g(\mu^\star)$. It follows that $f(\hat\nu_t(n_0))\le n_0g(\mu^\star)\to0$ and $g(\hat\nu_t(n_0))\le g(\mu^\star)$. Any limit point $\bar\mu$ satisfies $f(\bar\mu)=0$, placing $\bar\mu\in\mathcal I_{V_t}(\nu)$, and $g(\bar\mu)\le g(\mu^\star)$, which forces $\bar\mu=\mu^\star$ by uniqueness. We conclude that the whole sequence converges to $\mu^\star$.
\end{proof}
\subsubsection{Proof of Theorem~\ref{thm:testvalue}}\label{app:testvalue-proof}

We work in the population limit with a randomly chosen cohort. We prove the four clauses in turn and close with the corollary: clauses (ii) and (iii.1) are ridge-limit identification arguments, clause (iii.2) a perturbation bound, and clauses (iii.3) and (iv.3)--(iv.4) value comparisons that lean on the stationary-policy lemmas by name. In this limit, $E$'s routine record at $t_0$ has the same support $V$ as $A$'s, with the period-$t_0$ masses scaled by $(n-c)/n$, the fraction of the $n$ units left on the routine assignment. For $t\ge t_0$, $E$'s record covers $V\cup Z$, with $H^E_t=H_V(t)+(1-\omega)^{t-t_0}wH_Z$.

\begin{proof}[Proof of clause (i)]
Splitting the value difference through the intermediate strategy, we obtain $\mathcal V^E-\mathcal V^A=(\mathcal V^E-\mathcal V^{A\circ d})+(\mathcal V^{A\circ d}-\mathcal V^A)=\mathrm{Gain}-\mathrm{price}(t_0)$. The price has no sign in general; it is nonnegative when $A$'s assignments from $t_0$ on maximize the total reward under the true tables (Lemma~\ref{lem:tv-H}), a condition the theorem does not assume.
\end{proof}

\begin{proof}[Proof of clause (ii)]
By Lemma~\ref{lem:tv-I}(3) the ridge-limit fits lie in $\mathcal I_{V^A_t}(\nu)$ and $\mathcal I_{V^E_t}(\nu)$ for each $t$. For $t\ge t_0$ we obtain $V^A_t\supseteq V$ and $V^E_t\supseteq V$ by Lemma~\ref{lem:tv-I}(2), so both fits lie in $\mathcal I_V(\nu)$, where $\pi[\cdot]\equiv\pi[\nu]$ by the condition of clause (ii). It follows that $A$ and $E$ follow $\pi[\nu]$ in each period $t>t_0$. Since $A\circ d$ copies $A$'s assignments after $t_0$ by construction, the strategies $E$ and $A\circ d$ agree through $t_0$ and after $t_0$ both follow $\pi[\nu]$, so their assignments coincide at each period. Consequently $\mathcal V^E=\mathcal V^{A\circ d}$ and $\mathrm{Gain}=0$, and by clause (i) we obtain $\mathcal V^E-\mathcal V^A=-\mathrm{price}(t_0)$.
\end{proof}

\begin{proof}[Proof of clause (iii)]
We first prove item 1. For $t\ge t_0$, $E$'s record covers $V\cup Z$ with positive masses. By Lemma~\ref{lem:tv-I}(2),(3) the ridge-limit fit lies in $\mathcal I_{V^E_t}(\nu)\subseteq\mathcal I_{V\cup Z}(\nu)$, and by Lemma~\ref{lem:tv-E} with $\operatorname{rank}G_{V\cup Z}=K$ we obtain $\mathcal I_{V\cup Z}(\nu)=\{\nu\}$. We conclude that $\hat\nu^E_t=\nu$ for each $t\ge t_0$, and $E$ selects $\pi[\nu]=\pi^*$ in each period $t>t_0$.

For item 2, recall from Lemma~\ref{lem:tv-I}(1) that the targets are exact. Applying Lemma~\ref{lem:tv-C} with $e\equiv0$ to $E$'s record at $t\ge t_0$, we obtain $\lVert\hat\nu^E_t-\nu\rVert_2\le n_0\lVert\nu^{\mathrm{sim}}-\nu\rVert_2/(\lambda^\perp_{\min}(H^E_t)+n_0)$. With $\omega=0$, $H^E_t=H_V(t)+wH_Z\succeq H_V(t_0)+wH_Z$, yielding $\lambda^\perp_{\min}(H^E_t)\ge\lambda^\perp_{\min}(H_V(t_0)+wH_Z)$; this is positive by Lemma~\ref{lem:tv-E}, since the masses on $V\cup Z$ are positive and $\operatorname{rank}G_{V\cup Z}=K$. If $\lambda^\perp_{\min}(H^E_t)>n_0(\lVert\nu^{\mathrm{sim}}-\nu\rVert_2/r(\nu)-1)$, then, on rearranging, the bound is smaller than $r(\nu)$, and $\hat\nu^E_t$ lies within distance less than $r(\nu)$ of $\nu$, and $r(\nu)=\operatorname{dist}_2(\nu,\Delta_K\setminus D(\pi^*))$, hence in $D(\pi^*)$, and $E$ selects $\pi^*$ in period $t+1$.

We now prove item 3. The strategies $E$ and $A\circ d$ have identical assignments through $t_0$, and thus the same state distribution $d_{t_0+1}$. From $t_0+1$ on, $E$ runs $\pi^*$ and $A\circ d$ runs $\tilde\pi$, both stationary, under the frozen $\nu$. Writing $W:=T-t_0$, Lemma~\ref{lem:tv-F}(3) establishes $\big|\sum_{t=t_0+1}^{T}d^E_tr_{\pi^*,\nu}-Wg_\nu(\pi^*)\big|\le2\lVert h_\nu(\pi^*)\rVert_\infty$ and the same bound for $\tilde\pi$ along $A\circ d$. Since $\mathrm{Gain}=\mathcal V^E-\mathcal V^{A\circ d}=n\sum_{t=t_0+1}^{T}\big(d^E_tr_{\pi^*,\nu}-d^{A\circ d}_tr_{\tilde\pi,\nu}\big)$, subtracting and multiplying by $n$ gives the claim.
\end{proof}

\begin{proof}[Proof of clause (iv)]
We first prove item 1. Fix $t$ and put $e_t(z):=\bar m_t(z)-m_t(z)=\theta(z,\cdot)^\top(\bar\nu_t(z)-\nu_t)$ for $z\in V_t$ (Lemma~\ref{lem:tv-A}), with $\lVert\bar\nu_t(z)-\nu_t\rVert_1\le\lambda\bar a_t(z)$. The objective $J_t$ is that of Lemma~\ref{lem:tv-C} with $\nu:=\nu_t$, $W:=V_t$, and this $e$, and Lemmas~\ref{lem:tv-C} and~\ref{lem:tv-D} yield the bound on $\lVert\hat\nu^A_t-\nu_t\rVert_2$ stated in clause~(iv.1), with $\lambda^\perp_{\min}(H_t)>0$ by Lemma~\ref{lem:tv-E} and $\operatorname{rank}G_{V_t}=K$. Under stationary occupancy Lemma~\ref{lem:tv-B} yields $\bar\nu_t(z)=\bar\nu_t$, Lemma~\ref{lem:tv-D} identifies $\hat\nu^A_t=\bar\nu_t$, and $\bar a_t$ is as in Lemma~\ref{lem:tv-B}.

For item 2 we apply the same argument to $E$'s record at $t\ge t_0$. Lemma~\ref{lem:tv-A} holds cell by cell whatever the weights, giving $\lVert\hat\nu^E_t-\nu_t\rVert_2\le\big(B_\theta\lambda\sqrt{\lambda_{\max}(H^E_t)\mathcal Q^E_t}+n_0\lVert\nu^{\mathrm{sim}}-\nu_t\rVert_2\big)/(\lambda^\perp_{\min}(H^E_t)+n_0)$ with $\mathcal Q^E_t=\sum_{z\in V^E_t}N_t(z)\bar a_t(z)^2$. Suppose no routine observation lies in $Z$ on $(t_0,t]$. Then each tested cell carries mass $(1-\omega)^{t-t_0}wN_Z(z)$ at the single age $t-t_0$, and we obtain $\mathcal Q^E_t=\mathcal Q^V_t+(1-\omega)^{t-t_0}wN^{\mathrm{tot}}_Z(t-t_0)^2$ and $H^E_t=H_V(t)+(1-\omega)^{t-t_0}wH_Z$, the form stated in clause~(iv.2).

We now prove item 3. Under the switch condition with $t^E\le t^A$, the strategies $E$ and $A\circ d$ run the same test at $t_0$ and both follow the pre-switch policy $\pi_0$ on $(t_0,t^E]$; they therefore share $d_{t^E+1}$. On $(t^E,t^A]$, of length $W:=t^A-t^E$, $E$ switches to the post-switch policy $\pi_1$ and $A\circ d$ holds fixed $\pi_0$; on $(t^A,T]$ both deploy $\pi_1$ from $d^E_{t^A+1}\ne d^{A\circ d}_{t^A+1}$. Consequently
\begin{equation*}
\mathrm{Gain}=n\Big[\sum_{t\in(t^E,t^A]}\big(d^E_tr_{\pi_1,\nu_t}-d^{A\circ d}_tr_{\pi_0,\nu_t}\big)+\sum_{t\in(t^A,T]}\big(d^E_t-d^{A\circ d}_t\big)r_{\pi_1,\nu_t}\Big].
\end{equation*}
Extend $\nu_{T+1}:=\nu_T$. By Lemma~\ref{lem:tv-G} along $E$'s trajectory for $\pi_1$ and along $A\circ d$'s for $\pi_0$, the first sum is at least, with $\delta:=\min_{t\in(t^E,t^A]}\big(g_{\nu_t}(\pi_1)-g_{\nu_t}(\pi_0)\big)$, $\sum_{t\in(t^E,t^A]}\big(g_{\nu_t}(\pi_1)-g_{\nu_t}(\pi_0)\big)-2(2\bar h+WL_h\lambda)\ge W\delta-4\bar h-2WL_h\lambda$. By Lemma~\ref{lem:tv-F}(4) with $d=d^E_{t^A+1}$, $d'=d^{A\circ d}_{t^A+1}$, and the kernels $P_{\pi_1,\nu_t}$, the second sum is at least $-2B_\theta\varrho/\epsilon$. Therefore $\mathrm{Gain}\ge n[(\delta-2L_h\lambda)W-4\bar h-2B_\theta\varrho/\epsilon]$, which is the claim with $\kappa_1=2L_h$ and $\kappa_0=4\bar h+2B_\theta\varrho/\epsilon\le10B_\theta\varrho/\epsilon$ by Lemma~\ref{lem:tv-F}(2).

It remains to prove item 4. Under the switch condition with $t^A<t^E$ the strategies agree through $t^A$; on $(t^A,t^E]$, $A\circ d$ deploys $\pi_1$ and $E$ deploys $\pi_0$; after $t^E$ both deploy $\pi_1$. Repeating the computation of item~3 with the roles of the two strategies exchanged, with $\delta':=\min_{t\in(t^A,t^E]}\big(g_{\nu_t}(\pi_1)-g_{\nu_t}(\pi_0)\big)$, yields $-\mathrm{Gain}\ge n[(\delta'-2L_h\lambda)(t^E-t^A)-4\bar h-2B_\theta\varrho/\epsilon]$.
\end{proof}

\begin{proof}[Proof of Corollary~\ref{cor:switch}]
We first prove item 1. By Lemma~\ref{lem:tv-B} and clause (iv.1), $\hat\nu^A_t=\bar\nu_t=\frac1t\sum_{r=1}^{t}\nu_r=\nu_1+\lambda e\frac1t\sum_{r=1}^{t}(r-1)=\nu_t-\lambda\frac{t-1}{2}e$. With $t^*$ defined by $\langle a,\nu_1\rangle+\lambda\langle a,e\rangle(t^*-1)=b$, we obtain $\langle a,\bar\nu_t\rangle-b=\lambda\langle a,e\rangle\big[\frac{t-1}{2}-(t^*-1)\big]$, which is zero at $t=2t^*-1$ and positive for $t>2t^*-1$.

For item 2, note that for $t>t^*$ the Euclidean distance from $\nu_t$ to the hyperplane is $\lambda\langle a,e\rangle(t-t^*)/\lVert a\rVert_2$ and $\nu_t$ lies on the $\pi_1$ side; any point of $\Delta_K$ within a smaller distance of $\nu_t$ lies on the same side; since clause~(iv.2) bounds the distance of $\hat\nu^E_t$ from $\nu_t$, it too lies on that side.

We now prove item 3. Under these conditions the routine targets are $\theta(z,\cdot)^\top\bar\nu_t$ with $u_1:=\bar\nu_t-\nu_t=-\lambda\frac{t-1}{2}e$ by item 1, and the tested targets are $\theta(z,\cdot)^\top\nu_{t_0}$ with $u_2:=\nu_{t_0}-\nu_t=-\lambda(t-t_0)e$; both lie in $\mathbb T$ since $\mathbf 1^\top e=0$. Writing $\mu=\nu_t+v$, we obtain $J_t(\mu)=(u_1-v)^\top H_V(t)(u_1-v)+(u_2-v)^\top wH_Z(u_2-v)$, whose stationarity condition on $\mathbb T$, with $M_t$ the restriction of $\Pi_{\mathbb T}(H_V(t)+wH_Z)\Pi_{\mathbb T}$ to $\mathbb T$, is $M_t\hat v=-\lambda\Pi_{\mathbb T}\big[\frac{t-1}{2}H_V(t)+(t-t_0)wH_Z\big]e$. If $\nu_t+\hat v\in\Delta_K$, then the constrained minimizer is $\hat v$ (Lemma~\ref{lem:tv-C}: the projection of a feasible point is itself), and $\langle a,\hat\nu^E_t\rangle-b=\lambda\langle a,e\rangle(t-t^*)-\lambda(\Pi_{\mathbb T}a)^\top M_t^{-1}\Pi_{\mathbb T}\big[\frac{t-1}{2}H_V(t)+(t-t_0)wH_Z\big]e$. This is negative exactly under the stated condition. If $\operatorname{rank}G_Z=K$, then as $w\to\infty$ we find $M_t^{-1}\Pi_{\mathbb T}(wH_Z)\Pi_{\mathbb T}\to I_{\mathbb T}$ and $M_t^{-1}\Pi_{\mathbb T}H_V(t)\Pi_{\mathbb T}\to0$, and it follows that the right-hand side tends to $(t-t_0)\langle a,e\rangle$.
\end{proof}

\end{document}